\documentclass{article}
\PassOptionsToPackage{numbers, sort, compress}{natbib}

\usepackage[main,final]{neurips_2026}

\usepackage[T1]{fontenc}
\usepackage{amsmath}
\usepackage{amssymb}
\usepackage{amsthm}
\usepackage{booktabs}
\usepackage{microtype}
\usepackage[table]{xcolor}
\usepackage{graphicx}
\usepackage{rotating}
\usepackage{multirow}
\usepackage{enumitem}
\usepackage{wrapfig}
\usepackage{dsfont}
\usepackage{subcaption}
\usepackage{adjustbox}
\usepackage{algorithm}
\usepackage{algpseudocode}

\newtheorem{theorem}{Theorem}
\newtheorem{proposition}[theorem]{Proposition}
\newtheorem{remark}[theorem]{Remark}
\newtheorem{lemma}[theorem]{Lemma}
\newtheorem{definition}[theorem]{Definition}

\newtheorem{convention}[theorem]{Convention}
\newtheorem{corollary}[theorem]{Corollary}

\usepackage{mathtools}
\usepackage{xspace}
\usepackage{tikz}
\usetikzlibrary{arrows.meta,positioning,calc,backgrounds}
\usepackage{hyperref}
\hypersetup{
    colorlinks,
    linkcolor={blue!80!black},
    citecolor={blue!80!black},
    urlcolor={blue!80!black}
}

\makeatletter
\g@addto@macro\normalsize{%
    \setlength{\abovedisplayskip}{2.5pt}
    \setlength{\belowdisplayskip}{2.5pt}
    \setlength{\abovedisplayshortskip}{2.5pt}
    \setlength{\belowdisplayshortskip}{2.5pt}
}
\makeatother

\newcommand*{\R}{\mathbb{R}}

\newcommand*{\abs}[1]{\left \lvert #1 \right \rvert }
\newcommand*{\norm}[1]{\left \lVert #1 \right \rVert}
\newcommand{\sigTPP}{\textsc{SigTPP}}
\newcommand{\sigREG}{\textsc{SigREG}}
\newcommand{\determodel}{\textsc{Deter}}
\newcommand{\gammamodel}{\textsc{Gamma}}
\newcommand{\histintloss}{L_{\lambda}}
\newcommand{\histITloss}{L_{\log(\tau)}}
\newcommand{\ACD}{\mathrm{ACD}}
\newcommand{\PCD}{\mathrm{PCD}}
\newcommand{\sigW}{\mathrm{Sig\text{-}W}_1}
\newcommand{\WAS}{\mathrm{W}_1}
\newcommand{\MAEmed}{MAE\textsubscript{med}}
\newcommand{\MSEmean}{MSE\textsubscript{mean}}
\newcommand{\E}{\mathbb{E}}
\newcommand{\mes}{\mathop{}\!\mathrm{d}}
\newcommand\given[1][]{\,#1\vert\,} % for conditional probability
\newcommand*{\lowerbetter}{Lower is better~$(\downarrow)$.}
\newcommand*{\metricdefs}{Metric definitions in Appendix~\ref{sec::metrics}.}
\newcommand*{\resultlegend}{\lowerbetter{} Best in \textbf{bold}, second best \underline{underlined}. \metricdefs}
\newcommand*{\resultcaption}[1]{%
  \caption{#1 Values report standard error in parentheses over $100$ bootstrap replicates. \resultlegend}%
}
\newcommand*{\rankcaption}[1]{%
  \caption{#1 Values are means, with standard deviations in parentheses. \resultlegend}%
}
\newcommand*{\ratiocaption}[2]{%
  \caption{Geometric means of per-dataset score ratios to \determodel~over #1. The vertical axis is logarithmic. The dashed line marks \determodel~$=1$. \lowerbetter{} #2 Metric definitions are provided in Appendix~\ref{sec::metrics}.}%
}
\newcommand{\Tmax}{T_{\max}}
\newcommand{\dtppsym}{d_{\mathcal{N}}}
\newcommand{\dtpp}[2]{\dtppsym\!\left(#1,#2\right)}
\DeclareMathOperator{\11charac}{\mathds{1}}
\newcommand{\clag}{c\`{a}dl\`{a}g\xspace}
\newcommand{\holder}{H\"{o}lder\xspace}
\newcommand{\DatasetDiagnosticsFigure}[3]{%
  \begin{figure}[!htbp]
      \centering
      \begin{subfigure}[b]{0.49\textwidth}
          \vspace{0pt}
          \centering
          \includegraphics[width=0.90\textwidth]{images/datasets/#3/#3_sample_paths.pdf}
          \par\vspace{-4pt}
          \makebox[\textwidth][c]{\small(a)}
      \end{subfigure}
      \hfill
      \begin{subfigure}[b]{0.49\textwidth}
          \vspace{0pt}
          \centering
          \includegraphics[width=0.84\textwidth]{images/datasets/#3/#3_intensity_pdf.pdf}
          \par\vspace{-4pt}
          \makebox[\textwidth][c]{\small(b)}
      \end{subfigure}
      \par\vspace{4pt}
      \begin{subfigure}[b]{0.49\textwidth}
          \vspace{0pt}
          \centering
          \includegraphics[width=0.90\textwidth]{images/datasets/#3/#3_acf.pdf}
          \par\vspace{-4pt}
          \makebox[\textwidth][c]{\small(c)}
      \end{subfigure}
      \hfill
      \begin{subfigure}[b]{0.49\textwidth}
          \vspace{0pt}
          \centering
          \includegraphics[width=0.78\textwidth]{images/datasets/#3/#3_correlation.pdf}
          \par\vspace{-4pt}
          \makebox[\textwidth][c]{\small(d)}
      \end{subfigure}

      \caption{Diagnostics for the #1 dataset (#2). (a) Representative sample paths from the dataset. (b) Empirical temporal summaries. Top: intensity function. Bottom: interarrival times PDF. (c) Empirical autocorrelation of interarrival times. (d) Correlation structure between interarrival times.}
      \label{fig::#3_diagnostics}
  \end{figure}
}

\title{From Jumps to Signatures: A Generative Method for Temporal Point Processes}

\author{%
    Niels Cariou-Kotlarek$^{\spadesuit,\bigstar}$
    {\,\normalfont and\,}
    Vasileios Lampos$^{\spadesuit}$\\\\[-0.5em]
    $^{\spadesuit}$~Centre for Artificial Intelligence, Computer Science, UCL\\
    $^{\bigstar}$~Department of Mathematics, UCL\\[0.2em]
    \{\texttt{niels.kotlarek.23}, \texttt{v.lampos}\}\texttt{@ucl.ac.uk}
}

\begin{document}

\maketitle
\begin{abstract}
Rough path signatures are a universal feature map for continuous paths and, via the expected signature, characterise path distributions. These guarantees do not directly extend to \clag paths of Temporal Point Processes (TPPs), limiting the use of signature methods for event sequences. 
Furthermore, neural TPP models, including recent generative approaches, optimise per-event objectives with no global sequence-level loss, while evaluation of variable-length event sequences lacks distributional discrepancy measures.
This paper proposes a common pathwise framework for addressing these limitations.
We introduce the interarrival embedding, a stable, injective lift from jump paths to continuous paths of bounded variation, extending signature methods to discrete event sequences. 
Our theoretical contributions give rise to \sigTPP, the first signature-based generative model for TPPs, trained using a path-level loss on complete trajectories.
We further analyse the space of counting paths and derive $3$ distributional discrepancies, providing mathematically justified tools for evaluating generative TPP models.
Across $12$ tasks, \sigTPP~achieves the best mean rank based on $8$ complementary metrics, outperforms or is within one standard error of the strongest baseline (which varies based on performance) in $57\%$ of the dataset-metric pairs, and according to a relative score, improves against each baseline by at least $24\%$ on average.

\textbf{Source code and datasets:} \href{https://github.com/Code-Cornelius/sigtpp}{\texttt{github.com/Code-Cornelius/sigtpp}}
\end{abstract}

\section{Introduction}
\label{sect::intro}

Temporal Point Processes (TPPs) model discrete events occurring in continuous time~\cite{hawkes1971,daley2003introduction}. 
They underpin applications across finance~\cite{bacry2015hawkes}, seismology~\cite{dascher_cousineau_using_2023}, information diffusion~\cite{hua_personalized_2022}, and health~\cite{vassoy_time_2019, gajardo_point_2023}.
Neural TPPs (NTPPs) parametrise the conditional or interarrival density with expressive neural networks~\cite{original_2016, mei2016neural, Omi_paper, zhang2020self}, including continuous-time state-space architectures~\cite{s2p2}, and are often trained via log-likelihood maximisation.
Yet, likelihood-based training introduces fundamental constraints~\cite{shchur2021neural}. It either restricts models to tractable likelihood families~\cite{shchur2020fast} or relies on expensive numerical approximations~\cite{mei2016neural, Omi_paper}, limiting both expressiveness and scalability.
This has led to a growing interest in alternative generative paradigms, such as diffusion-based methods~\cite{yuan2023spatio, psdiff}, normalising flows~\cite{shchur2020fast}, adversarial training~\cite{xiao2017wasserstein}, and reinforcement learning~\cite{li2018learning}, with frameworks~\cite{lin2022exploring} unifying several of these perspectives. 
At the same time, the signature transform from rough path theory~\cite{lyons2007differential, friz2010multidimensional} has become central to time series modelling~\cite{chevyrev2016primer, kidger2019deep, cuchiero_universal_2025, perez_arribas_signature-based_2018}. Acting as a universal feature map, it encodes pathwise information into a tensor series up to tree-like equivalence~\cite{hambly2010uniqueness, lyons2007differential}, while the expected signature becomes a moment generating function on path space~\cite{LyonsChevyrev2013, chevyrev2018signature}, supporting generative models such as the Signature-Wasserstein GAN~\cite{NiEtAl2020, LiaoEtAl2024}.

Despite this activity, the current landscape has significant structural gaps.
The prevailing training objectives are local as they decompose into a sum of per-event losses, each comparing one event to the ground truth conditional on preceding events.
This is also true for recent non-autoregressive models as they fit reverse denoising transitions~\cite{addthin}, transport fields~\cite{eventflow}, or edit rates at sampled intermediate states~\cite{editflow}, rather than directly minimising a single global discrepancy between generated and empirical trajectories.
Consequently, global structural properties, such as autocorrelation, long-range inter-event dependencies, and joint timing distributions, receive no direct loss signal.

Furthermore, the evaluation methods for generative TPP models are underdeveloped~\cite{mukherjee_neural_2025, shchur2021neural}. 
Negative log-likelihood (NLL) measures local conditional fit and is insensitive to global sequence structure. Low NLL does not preclude mode collapse, misspecified tails, or distorted inter-event dependencies~\cite{lin2021empirical, lin2022exploring, mukherjee_neural_2025}. One-step-ahead pointwise metrics (e.g. mean absolute error) share this limitation and additionally favour underdispersed predictions~\cite{okawa_deep_2019, okawa_context-aware_2022, imitationlearning, zhou_automatic_2023, zhou2022neural}.
A practical shift towards sequence-level distributional losses, such as the maximum mean discrepancy (MMD) with a Gaussian kernel~\cite{shchur2020fast, psdiff} and the Wasserstein-1 distance~\cite{addthin, editflow}, has emerged, yet whether these losses define meaningful distances on the space of counting paths has not been theoretically established.
Signatures would be a suitable candidate to address the aforementioned shortcomings, as they can naturally handle variable-length sequences. However, the determinacy and universal nonlinearity that underpin signature-based generative modelling~\cite{LyonsChevyrev2013, hambly2010uniqueness, chevyrev2018signature} are formulated for continuous paths of bounded variation. TPP sample paths are \clag\ step functions, on which these guarantees fail. 

This paper attempts to address these limitations with the following contributions:
\begin{enumerate}[leftmargin=*, topsep=-2pt, itemsep=-1.6pt]
\item We introduce the interarrival embedding, a stable, injective lift to continuous paths of bounded variation, enabling signature methods on jump processes (Section~\ref{subsect::phi}).
\item Building on this embedding approach, we propose \sigTPP, a signature-based generative model for TPPs that replaces per-event objectives with a global signature loss on counting paths (Section~\ref{subsect::sigtpp}).
\item We provide 3 theoretically grounded discrepancies on counting-path distributions, namely the energy ($\mathcal{E}$), Wasserstein-1 ($\WAS$), and signature-Wasserstein-1 distances ($\sigW$), for evaluating generative TPP models (Section~\ref{subsect::evaluation}). We also showcase the utility of other diagnostic metrics for NTPP evaluation, including moment and histogram-based summaries.
\item We demonstrate empirically that \sigTPP~is a strong generative model for TPPs. It achieves the best relative score on $5$ out of the $8$ metrics. \sigTPP~achieves the best performance or is within one standard error of the strongest baseline in $57\%$ of all metric-dataset pairs (Section~\ref{sec::experiments}).
\end{enumerate}

\section{Related Work}
\label{sect::related_work}
Likelihood-based NTPPs~\cite{original_2016, zhang2020self, mei2016neural, s2p2} compute the optimisation target as the product of the conditional densities and the survival probability of the final interval~\cite{daley2003introduction}. Exact evaluation requires tractable densities~\cite{Omi_paper, shchur2020intensity, shchur2020fast}. Otherwise, models integrate the intensity numerically~\cite{mei2016neural}. Beyond likelihood, many autoregressive models train on the next interarrival time~\cite{xiao2017wasserstein, li2018learning, lin2022exploring}. Our proposed method (\sigTPP) is also autoregressive but compares the expected signatures of complete sequences. This provides a sample-based training objective without likelihood evaluation or a learned critic. On the other hand, non-autoregressive models generate whole sequences through intermediate states~\cite{addthin, psdiff, editflow}.~\sigTPP~retains autoregressive sampling and trains directly on complete generated paths, without fitting an auxiliary denoising, editing, or transport process. In parallel, signature objectives have been used to train generators of continuous paths. Signature-Wasserstein models match the expected signatures~\cite{NiEtAl2020, LiaoEtAl2024}, and signature-kernel scores train neural stochastic differential equations~\cite{issa2023nonadversarial}. Applying signature objectives to event sequences requires a suitable representation. The Marcus signature extends universality to continuous functions of \clag paths by interpolating jumps of the counting path~\cite{cuchiero_universal_2025}, while our proposed interarrival embedding records waiting times as the values of a continuous path. With time augmentation, it also combines event times and interarrival durations, the temporal quantities used in neural TPPs~\cite{original_2016, lin2022exploring}.

\section{Preliminaries}
\label{sect::preliminaries}
\textbf{Temporal Point Processes{\normalfont~\cite{daley2003introduction}}.} Fix a finite horizon $\Tmax\in\R_{>0}$.
A realisation of a TPP on $[0,\Tmax)$ is a finite sequence of strictly increasing event times, $0<t_1<\cdots<t_m<\Tmax$.
To define a metric on event sequences and study its topological properties, we work with the equivalent counting-path representation.
Let $\mathcal{N}$ denote the set of nondecreasing, \clag functions $\eta \colon [0,\Tmax] \to \mathbb{N}$ with $\eta_0=0$, $\eta_{\Tmax}<\infty$, unit jumps only, and no jump at $\Tmax$. $\eta$ and $(t_k)_{k=1}^m$ are related by $\eta_t=\#\{k:t_k \leq t\}$.
A TPP is an $\mathcal{N}$-valued random variable $N$ on a probability space $(\Omega,\mathcal{F},\mathbb{P})$. The conditional intensity $\lambda(t \mid \mathcal{H}_t) := \lim_{\Delta t \downarrow 0} \Delta t^{-1}\,\E\bigl[N_{t+\Delta t}-N_t \given \mathcal{H}_t\bigr]$, where $\mathcal{H}_t := \sigma(N_s : s < t)$, characterises the instantaneous rate of events given history.

We adopt the following boundary convention throughout, since the embedding and its stability analysis use the same augmented event grid.
\begin{convention}[Boundary Convention]
\label{conv::boundary}
For $\eta\in\mathcal{N}$ with event times $0<t_1<\cdots<t_m<\Tmax$, we augment the sequence by setting $t_0:=0$ and $t_{m+1}:=\Tmax$, and define the corresponding interarrival times $\tau_0:= 0$, $\tau_k:=t_k-t_{k-1}$ where $k=1,\dots,m$, and $\tau_{m+1}:=\Tmax-t_m$.
\end{convention}

\textbf{Rough Path Signatures{\normalfont~\cite{lyons2007differential, friz2010multidimensional}}.}
The signature transform provides an injective representation, up to tree-like equivalence, of continuous paths of finite variation into the tensor algebra, $T((\mathbb{R}^d))$.
We restrict attention to continuous paths of bounded variation (as opposed to finite $p$-variation).
\begin{definition}[Signature]
\label{def::sig}
Let $X : [0,\Tmax] \to \R^d$ be a continuous path of bounded variation. The signature of $X$, denoted by $S(X)$, is defined as the infinite sequence of iterated integrals
\begin{equation}
    S(X) = (1, S(X)^1, S(X)^2, \dots), \quad \text{where} \quad
    S(X)^k = \int_{0<u_1<\cdots<u_k<\Tmax} \mes X_{u_1} \otimes \cdots \otimes \mes X_{u_k} \, .
\end{equation}
$S^M(X)$ denotes the truncated signature of $X$ of degree $M$, i.e. $S^M(X) = (1, S(X)^1,\dots,S(X)^M).$
\end{definition}
\begin{theorem}[Determinacy~\cite{LyonsChevyrev2013}]
\label{thrm::determinacy_sig}
    Let $X$ and $Y$ be continuous random paths of bounded variation. If $\mathbb{E}[S(X)] = \mathbb{E}[S(Y)]$ and $\mathbb{E}[S(X)]$ has infinite radius of convergence, then $S(X)$ and $S(Y)$ have the same distribution.
\end{theorem}

\section{Signature-based Generative TPPs (\sigTPP)}
\label{sect::methods}

The signature transform is defined on continuous paths of bounded variation, whereas counting paths in $\mathcal{N}$ are \clag\ step functions. Our approach to addressing this rests on two components. First, we introduce a deterministic embedding $\Phi$ into the space $ \mathcal{C}:=C([0, \Tmax], \R)$ of continuous functions on $[0,\Tmax]$. $\Phi$ maps a discrete event sequence to a continuous, piecewise-linear path on which the signature transform applies (Section~\ref{subsect::phi}).
Then, we train the proposed \sigTPP~on the Euclidean distance between the truncated expected signatures of the embedded model-generated and observed counting paths (Section~\ref{subsect::sigtpp}). We also provide theoretical guarantees ensuring the signature loss reflects genuine discrepancies between counting-path distributions (Section~\ref{subsect::theory}), and we propose principled evaluation metrics for generative TPPs (Section~\ref{subsect::evaluation}).

\subsection{The Interarrival Embedding \texorpdfstring{$\Phi$}{Phi}}
\label{subsect::phi}
\begin{definition}[Interarrival Embedding]\label{def::phi_embedding}
For $\eta \in \mathcal N$ with $m$ events, let $0 =: t_0 < t_1 < \cdots < t_m < t_{m+1} := \Tmax$ be the augmented grid and $\tau_k := t_k - t_{k-1}$ the interarrival times. 
The interarrival embedding is the continuous, bounded-variation, piecewise-linear function $x := \Phi(\eta) \in \mathcal C$ satisfying
\begin{equation}
\label{eq::phi_embedding}
\begin{aligned}
    x(t_k) &= \tau_k 
    \quad \text{ for } k=0,\dots,m{+}1 \, , \text{ and} \\
    \Phi(\eta)_t &= \tau_k 
    + \frac{t - t_k}{t_{k+1} - t_k}
    \bigl(\tau_{k+1} - \tau_k\bigr) \, ,
    \quad \text{where } t\in[t_k,t_{k+1}],\ k=0,\dots,m  \, .
\end{aligned}
\end{equation}
\end{definition}
\begin{wrapfigure}{r}{0.36\linewidth}
\vspace{-0.15in}
\centering
\includegraphics[width=\linewidth]{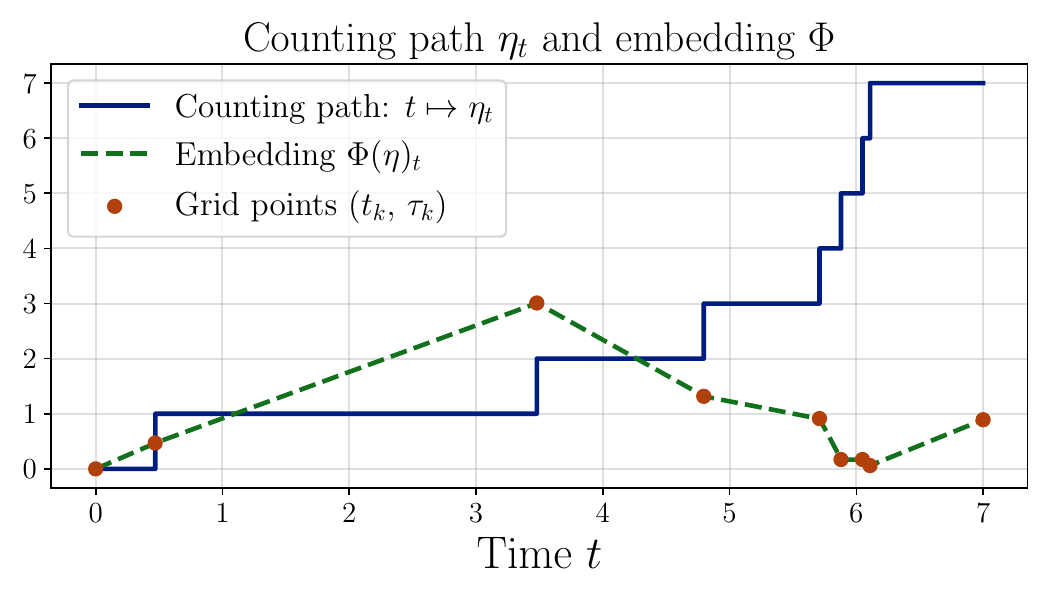}
\caption{\small Illustration of the interarrival embedding $\Phi$ (Definition~\ref{def::phi_embedding}) for a sampled Poisson process ($\lambda=1$, $\Tmax = 7$). Orange dots show the duration of the most recent interarrival interval.}
\label{fig::embedding_illus}
\end{wrapfigure}
Equivalently, $\Phi(\eta)$ is the unique continuous function that is affine on every interval and satisfies $\Phi(\eta)_{t_k} \;=\; \tau_k$, $k\le m+1$. Figure~\ref{fig::embedding_illus} illustrates this on a Poisson sample: the counting process (blue, \clag) records cumulative event counts, the embedding points (orange) record the duration of the most recent interarrival interval, and $\Phi(\eta)$ (green dashed line) is the piecewise-linear interpolant of this staircase on the augmented event grid. The embedded path is continuous and of bounded variation, satisfying the assumptions required by the signature transform (Definition~\ref{def::sig}). This piecewise-linear lift is the standard parameter-free construction for signatures~\cite{chevyrev2016primer}. It produces a continuous path of bounded variation, as required by the Hambly-Lyons uniqueness theorem and the Chevyrev-Lyons expected-signature determinacy result~\cite{hambly2010uniqueness, chevyrev2018signature}. We also prove that its affine structure yields the global Lipschitz bound (Theorem~\ref{thrm::stability}). In practice, we apply the signature transform to the time-augmented path $t \mapsto \widetilde{\Phi}(\eta)_t := \bigl(t,\Phi(\eta)_t\bigr)$ in $\R^2$. This augmentation is necessary because, for a scalar path $X$, every signature level satisfies $S(X)^k=(X_{\Tmax}-X_0)^k/k!$. The signature of $\Phi(\eta)$ alone therefore depends only on its total increment and cannot distinguish event sequences. By contrast, the strictly increasing time coordinate makes $\widetilde{\Phi}(\eta)$ unique in its class and hence identifiable from its signature~\cite{hambly2010uniqueness}. The resulting two-channel path combines absolute event times and interarrival durations, the standard temporal inputs to neural TPPs~\cite{original_2016, lin2022exploring}. These quantities already determine the counting path, and injectivity of $\Phi$ ensures that no event information is lost. The stability of the map $\Phi$ and a constructive inverse $\Psi := \Phi^{-1}$ are established in Theorem~\ref{thrm::stability}.

\subsection{Model Optimisation}
\label{subsect::sigtpp}

Our goal is to train a parametric generator $\mu_\theta$ whose output distribution faithfully reproduces the unknown data distribution $\nu$, without specifying a parametric intensity model.
A natural loss between distributions $\mu_\theta$ and $\nu$ is the Wasserstein-1 distance used in Wasserstein generative adversarial network (GAN) models~\cite{xiao2017wasserstein}, expressible via Kantorovich-Rubinstein duality as
\begin{equation}
    \WAS(\mu_\theta, \nu) = \sup_{\norm{f}_{\mathrm{Lip}} \leq 1} \E_{x \sim \mu_\theta}\bigl[f(x)\bigr] - \E_{y \sim \nu}\bigl[f(y)\bigr] \,.
\end{equation}
Approximating the supremum over 1-Lipschitz functions often requires training a neural network discriminator with gradient penalties or weight clipping. This adversarial procedure can lead to training instability and convergence issues~\cite{arjovsky2017}, especially when modelling TPPs~\cite{xiao2017wasserstein}. Recent work on generative modelling of continuous paths proposes replacing the adversarial discriminator with a signature-based metric, denoted by $\sigW$~\cite{NiEtAl2020, LiaoEtAl2024}. By leveraging the universal nonlinearity signature property, any 1-Lipschitz test function can be approximated by a linear functional on the signature. Consequently, the optimisation over the infinite-dimensional function space reduces to a linear operation on the (truncated for numerical purposes) tensor algebra $T^M(\mathbb{R}^d)$.
Given that TPP sample paths are \clag step functions, we apply the signature loss to their time-augmented embeddings $\widetilde{\Phi}(\eta)$.
Hence, the loss for optimising the proposed model, \sigTPP, is given by:
\begin{equation}
    \label{eq::sig_metric}
    \sigW^{(M)}(\mu_\theta,\nu)
    :=
    \norm{\E_{N_\theta \sim \mu_\theta}\bigl[S^M(\widetilde{\Phi}(N_\theta))\bigr]
        - \E_{N \sim \nu}\bigl[S^M(\widetilde{\Phi}(N))\bigr]}_{2} \, .
\end{equation}
Here $\norm{\cdot}_{2}$ denotes the $\ell^2$ norm on the truncated tensor algebra $T^M(\mathbb{R}^2)$.
The training procedure is summarised in Figure~\ref{fig::sigtpp_diagram}. The expected signature under the data measure $\nu$ is precomputed once on the training set. Each gradient step requires only computing signatures under $\mu_\theta$.
\begin{figure}[t]
    \centering
    \resizebox{\linewidth}{!}{
    \begin{tikzpicture}[
        x=1cm, y=1cm,
        >=Stealth,
        every node/.style={font=\small},
        plotbox/.style={draw=none,
        inner sep=0pt},
        methodbox/.style={draw=black!40, fill=black!2, rounded corners=3pt,
        minimum height=0.48cm, inner sep=3.5pt,
        line width=0.35pt, font=\scriptsize},
        arrmain/.style={-{Stealth[length=3pt,width=2.5pt]},
        line width=0.5pt, black!40},
        backprop/.style={-{Stealth[length=3pt,width=2.5pt]},
        densely dashed, line width=0.5pt, red!50!black},
    ]

\definecolor{outlineblue}{RGB}{20,60,140}
\definecolor{fillblue}{RGB}{215,225,245}
\definecolor{midblue}{RGB}{120,145,190}

% ===================== MAIN ROW ==================================

\node[methodbox] (mutheta) at (0,0) {$N \sim \mu_\theta$};

\node[plotbox, right=0.35cm of mutheta] (pp)
    {\includegraphics[width=3.5cm, trim=10pt 10pt 10pt 10pt, clip]{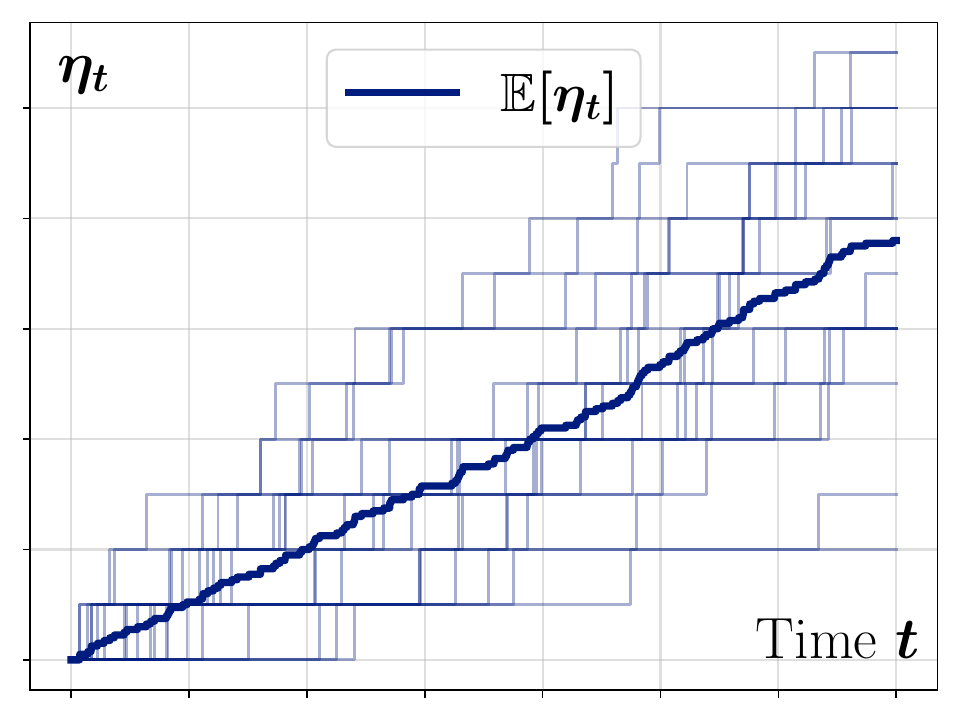}};
\draw[arrmain] (mutheta.east) -- (pp.west);

\node[methodbox, right=0.35cm of pp] (phi) {$\widetilde{\Phi}$};
\draw[arrmain] (pp.east) -- (phi.west);

\node[plotbox, right=0.35cm of phi] (emb)
    {\includegraphics[width=3.5cm, trim=10pt 10pt 10pt 10pt, clip]{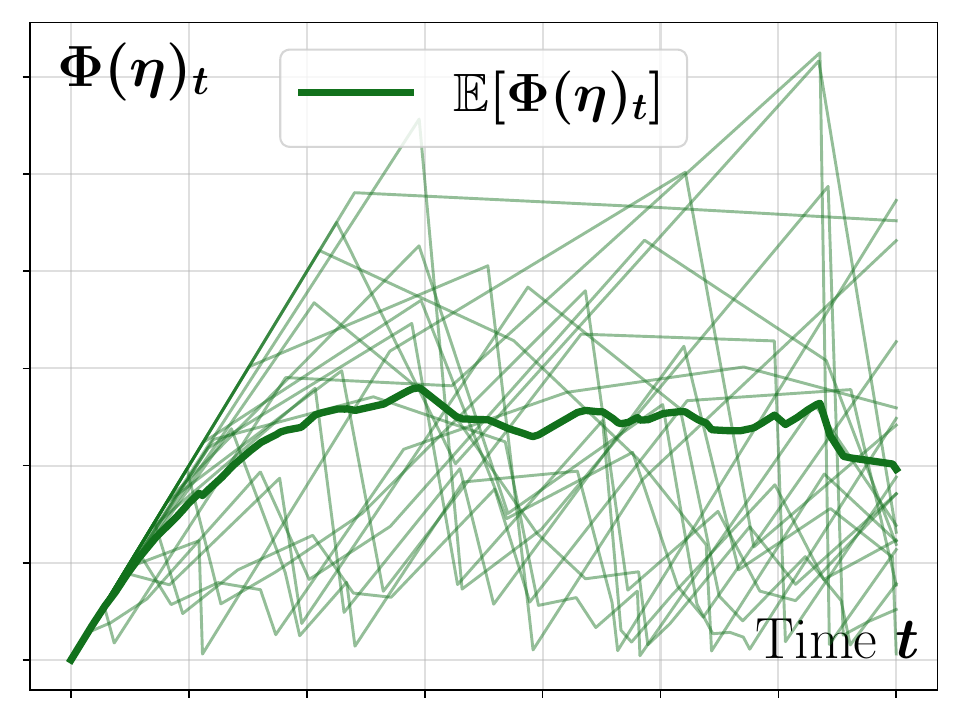}};
\draw[arrmain] (phi.east) -- (emb.west);

\node[methodbox, right=0.35cm of emb] (sig) {$\mathrm{S}^M$};
\draw[arrmain] (emb.east) -- (sig.west);

% Capsules
\begin{scope}[shift={($(sig.east)+(0.35,0)$)}, scale=1.45]
    \fill[fillblue!50, rounded corners=2.5pt]
    (0.18,-0.53) rectangle (0.66,0.87);
    \draw[outlineblue!25, dotted, line width=0.3pt, rounded corners=2.5pt]
    (0.18,-0.53) rectangle (0.66,0.87);
    \fill[outlineblue!40] (0.42, 0.65) circle (0.09);
    \fill[midblue!50]     (0.42, 0.44) circle (0.09);
    \fill[midblue!50]     (0.42,-0.32) circle (0.09);

    \fill[fillblue!75, rounded corners=2.5pt]
    (0.09,-0.60) rectangle (0.57,0.80);
    \draw[outlineblue!35, dotted, line width=0.3pt, rounded corners=2.5pt]
    (0.09,-0.60) rectangle (0.57,0.80);
    \fill[outlineblue!55] (0.33, 0.58) circle (0.09);
    \fill[midblue!65]     (0.33, 0.37) circle (0.09);
    \fill[midblue!65]     (0.33,-0.39) circle (0.09);

    \fill[fillblue, rounded corners=2.5pt]
    (0,-0.68) rectangle (0.48,0.72);
    \draw[outlineblue, dotted, line width=0.4pt, rounded corners=2.5pt]
    (0,-0.68) rectangle (0.48,0.72);
    \fill[outlineblue] (0.24, 0.50) circle (0.09);
    \fill[midblue]     (0.24, 0.29) circle (0.09);
    \foreach \yy in {0.04,-0.04,-0.12} {
        \fill[black] (0.24,\yy) circle (0.014);
    }
    \fill[midblue] (0.24,-0.47) circle (0.09);
    \node[right, font=\tiny, text=black] (slabel-one) at (0.65, 0.50) {$S_1$};
    \node[right, font=\tiny, text=black] (slabel-two) at (0.65, 0.29) {$S_2$};
    \node[right, font=\tiny, text=black] (slabel-M)   at (0.65,-0.47) {$S_M$};

    \coordinate (cap-left)  at (0, 0);
    \coordinate (cap-right) at (0.78, 0);
    \coordinate (cap-bot)   at (0.28, -0.68);
    \coordinate (cap-top)   at (0.33,  0.87);
\end{scope}

\draw[arrmain] (sig.east) -- (cap-left);

% ==================== E -> straight arrow -> loss label =========

\node[methodbox] (Ebox) at ($(cap-right)+(0.85,0)$) {$\E$};
\draw[arrmain] (cap-right) -- (Ebox.west);

\node[font=\footnotesize, text=black, anchor=west] (lossname)
    at ($(Ebox.east)+(0.35,0)$) {$\sigW^{(M)}(\mu_\theta,\nu)$};
\draw[arrmain] (Ebox.east) -- (lossname.west);

% ===================== CAPSULE LABEL (below) =====================

\node[below=3pt of cap-bot, font=\scriptsize, text=black!55]
(caplabel) {$T^{M}\!\bigl(\mathbb{R}^2\bigr)$};

% ===================== BACKPROPAGATION ===========================

\coordinate (bpbot) at ($(mutheta.south)+(0,-1.4)$);
\coordinate (bpbot-right) at (lossname.south |- bpbot);
\coordinate (bpbot-left)  at (mutheta.south |- bpbot);

\draw[backprop, rounded corners=6pt]
(lossname.south)
-- (bpbot-right)
-- (bpbot-left)
-- (mutheta.south);

\node[font=\footnotesize, text=red!42!black, fill=white, inner sep=1.5pt]
at ($(bpbot-right)!0.5!(bpbot-left)$)
    {$\nabla_\theta$\; backpropagation through $\mu_\theta$};

    \end{tikzpicture}
    }
    \caption{\sigTPP's optimisation. A counting path sampled from $\mu_\theta$ is embedded with $\widetilde{\Phi}$ to a continuous path, mapped to its truncated signature in $T^M(\mathbb{R}^2)$, and summarised by its expectation under the model. The loss ($\sigW$) compares this against the expected signature of the embedded targets' paths, opposed to the usual per-event loss. Loss gradient flows back through $\mu_\theta$.}
    \label{fig::sigtpp_diagram}
\end{figure}

\sigTPP~is a $1$-step autoregressive generator that follows the standard NTPP design~\cite{lin2022exploring}. A single-layer LSTM encodes the history. Each log-interarrival time $\log\tau_i$ is first mapped to a learnable sinusoidal positional encoding, and processed recurrently to produce the history state. The decoder is a shallow perceptron that predicts $\log \tau_{i+1}$ from the recurrent state $h_i$, the current cumulative time, and an exogenous noise variable $\varepsilon \sim \log(\mathrm{Exp}(1))$. The core temporal training procedure is given in Algorithm~\ref{alg:sigtpp-training} of Appendix~\ref{app::algorithms}; it repeatedly draws free-running trajectories from the generator via the sampling routine of Algorithm~\ref{alg:sigtpp-sampling-one-indexed}.  Appendix~\ref{subsect::trainingdetails} gives further implementation and training details.

\subsection{Theoretical Guarantees}
\label{subsect::theory}
To analyse the stability of the embedding $\Phi$ and to quantify discrepancies between counting-path distributions, we propose to work with the following metric.
\begin{definition}[Counting Path Metric]
\label{def::tppdistance}
We define the pathwise TPP metric between two counting processes $\eta,\xi\in\mathcal{N}$ on $[0,\Tmax)$ by $\dtpp{\eta}{\xi} := \int_{0}^{\Tmax} \abs{\eta_t-\xi_t}\,\mes t$.
\end{definition}
Given that counting paths are \clag, distinct event sets $\eta\neq\xi$ yield paths that differ on an interval of positive Lebesgue measure. Hence, $\dtppsym(\eta,\xi)=0$ implies $\eta=\xi$, and $\dtppsym$ is a genuine metric on~$\mathcal{N}$ (see also Proposition~\ref{prop::negtypeinheritance}).
Appendix~\ref{app::proof_tpp_metric_w1} identifies $\dtppsym$ with Wasserstein-1 between event-time point measures padded at $\Tmax$.  This result puts the optimal transport perspective of~\citet{xiao2017wasserstein} on a rigorous footing, yielding a genuine metric on~$\mathcal{N}$.
The space $\mathcal{N}$ is not a linear subspace of $L^1([0,\Tmax])$. The difference $\eta_t-\xi_t$ is signed, scalar multiples of counting paths are not counting paths, and the sum or difference of two counting paths need not have unit jumps. The metric structure is the natural level at which to study stability.
Consequently, standard tools from linear and functional analysis are not directly available.
Most notably, there is no differentiable structure on which to apply an inverse function theorem, no open mapping theorem to promote a continuous bijection to a bi-Lipschitz one, and no compact embedding to fall back on since event counts are unbounded.
We therefore work throughout with $L^1$-type metrics on both $\mathcal N$ and $\mathcal C$, which is also natural for signature-based losses that are computed through integrals.

Hence, the embedding chain from counting paths to the tensor algebra is given by:
\begin{equation}
    (\mathcal{N}, \dtppsym)
\xrightarrow{\;\Phi\;}
(\mathcal{C}, d_1)
\xrightarrow{\;\mathrm{t \mapsto (t,\Phi(\eta)_t)}\;}
C([0,\Tmax],\R^2)
\xrightarrow{\;\mathrm{Signature}\;}
T((\R^2)) \, ,
\end{equation}
where $\mathcal N$ is equipped with a metric $\dtppsym$ and $(\mathcal{C}, d_1)$ denotes the space of continuous paths $[0,\Tmax] \to \R$ equipped with the $L^1$ metric $d_1(x, y) := \int_0^{\Tmax} \abs{x(t)-y(t)}\, \mes t$. Theorem~\ref{thrm::stability} establishes that $\Phi$ is Lipschitz and injective, with a controlled inverse under mild assumptions.
The Lipschitz bound ensures that the signature loss (Eq.~\ref{eq::sig_metric}) inherits regularity from the data distribution. Closeness in counting-path space implies closeness of expected signatures. Injectivity guarantees that no event information is lost in the embedding, which is what allows $\sigW^{(M)}$ to separate distinct laws on $\mathcal N$. Finally, the \holder bound gives the converse direction under a mild separation assumption that is satisfied by any temporal data recorded at finite precision.

\begin{theorem}[Stability of the Embedding, full proofs in Appendix~\ref{app::proofs}]
    \label{thrm::stability}
    The embedding\, $\Phi\colon(\mathcal N,\dtppsym)\to(\mathcal C,d_1)$ satisfies the following properties:
    \begin{enumerate}[leftmargin=*, topsep=-2pt, itemsep=-4pt]
        \item \textbf{Lipschitz continuity.}\;
        For all $\eta,\xi\in\mathcal N$,
            $d_1\bigl(\Phi(\eta),\Phi(\xi)\bigr) \le 5\cdot\Tmax \cdot\dtpp{\eta}{\xi}$.
        Hence, it is continuous and measurable \textup{(Appendix~\ref{app::proofs::lip}, Theorem~\ref{thrm::Phi_lipschitz})}. 
        The constant $5$ is not tight but sufficient.
        \item \textbf{Bijection.}\;
        $\Phi$ is injective on $\mathcal N$, hence a bijection $\Phi\colon\mathcal N \xrightarrow{\sim} \Phi(\mathcal N)$. The inverse $\Psi := \Phi^{-1}$ is recovered constructively by a level-set recursion
        \textup{(Appendix~\ref{app::proofs::bij}, Theorem~\ref{thrm::bijection})}.
    \end{enumerate}
    Moreover, on counting paths whose interarrival times $\tau_1,\dots,\tau_{m+1}$ (Convention~\ref{conv::boundary})  are bounded below by any fixed $\delta>0$, the inverse $\Psi$ is \holder continuous of exponent $1/2$ \textup{(see also Appendix~\ref{app::proofs::cont}, Theorem~\ref{thrm::holder_main})}.
\end{theorem}

\subsection{Evaluation Metrics for Counting-Path Distributions}
\label{subsect::evaluation}

Evaluating generative TPPs requires discrepancy measures between distributions on $\mathcal{N}$. Current practice borrows tools from continuous data settings, e.g.\ MMD with Gaussian kernels~\cite{shchur2020fast, psdiff} or $\WAS$ on event times~\cite{addthin, editflow}, without verifying that these are valid distances on counting-path space. We show that the metric $\dtppsym$ gives rise to mathematically valid distributional discrepancies on $\mathcal{N}$.

Throughout this section, uppercase $P, Q \in \mathcal{P}_1(\mathcal{N})$ denote probability distributions over counting paths $\mathcal{N}$ with finite first $\dtppsym$-moment. The energy ($\mathcal{E}$) and Wasserstein-1 ($\WAS$) distances use $\dtppsym$ as ground cost. The signature discrepancy $\sigW^{(M)}$ compares expected features of the embedded paths.

\textbf{Energy distance ($\mathcal{E}$).} The energy distance ($\mathcal{E}$) is a standard metric induced by an underlying distance function~\cite{gneiting_strictly_2007, szekely_energy_2013} (Appendix~\ref{sec::metrics}).
Nonnegativity of $\mathcal{E}$ requires that $(\mathcal{N}, \dtppsym)$ be of negative type~\cite{szekely_energy_2013}. The map $\eta \mapsto (t \mapsto \eta_t)$ is an isometry from $(\mathcal{N}, \dtppsym)$ into $L^1([0, \Tmax])$, which is of negative type, and the property is inherited by subspaces (Appendix~\ref{subsec::distances_proba_with_dtpp}). Hence $D := \sqrt{\mathcal{E}}$ is a well-defined discrepancy. We further show in Appendix~\ref{app::counter_ex_strong} that $\dtppsym$ is not of strong negative type, so $D$ is a pseudometric rather than a metric. For the purpose of ranking models, this is sufficient.

\textbf{Wasserstein-1 distance ($\WAS$).} Since $(\mathcal{N}, \dtppsym)$ is a metric space, $\WAS$~\cite{ramdas2017wasserstein, villani_optimal_2009} is a well-defined metric on $\mathcal{P}_1(\mathcal{N})$. We recall its definition in Appendix~\ref{sec::metrics}.

\textbf{Signature discrepancy ($\sigW$).} The loss $\sigW^{(M)}$ (Eq.~\ref{eq::sig_metric}) can also be used as an evaluation metric in the continuous setting and we adapt it to TPPs with our embedding.
Since $\Phi$ is injective and measurable (Theorem~\ref{thrm::stability}), its time augmentation $\widetilde\Phi$ is also injective and measurable. Thus, distinct laws $P \neq Q$ on $\mathcal{N}$ induce distinct pushforwards $\widetilde\Phi_\sharp P \neq \widetilde\Phi_\sharp Q$. 
If these pushforward laws satisfy the moment condition of Theorem~\ref{thrm::determinacy_sig}, then the expected signature determines the law, and hence $\E_{\eta\sim P}[S(\widetilde\Phi(\eta))] \neq \E_{\xi\sim Q}[S(\widetilde\Phi(\xi))]$. Consequently, for any such pair $P \neq Q$, there exists a finite signature degree $M$ such that $\sigW^{(M)}(P, Q)>0.$
For each fixed $M$ with finite expected signature coordinates, $\sigW^{(M)}$, being an $\ell^2$ distance between expected truncated signatures, is symmetric and satisfies the triangle inequality, and is therefore a pseudometric.

\textbf{Pointwise Losses.} As mentioned in the introduction, pointwise losses fail in a quantifiable, easily reproducible way on the very benchmarks the community uses to evaluate generative TPPs. For an interarrival distribution $P$ with finite first moment, mean absolute error (MAE) is minimised by $\hat\tau = \mathrm{median}(P)$. With a finite second moment, mean squared error (MSE) is minimised by $\hat\tau =\E[P]$~\cite{gneiting_strictly_2007}. Both optima are Dirac measures. Neither can, in principle, distinguish a calibrated generative model from a deterministic regressor. See Tables~\ref{table::pointwise_results_combined} and~\ref{table::pointwise_results_additional} for an empirical confirmation, where \determodel~has the lowest sample-based MAE on all $12$ datasets despite poor pathwise scores.

\section{Experimental Results}
\label{sec::experiments}
We evaluate \sigTPP~on controlled synthetic regimes and real-world tasks.
We assess performance across multiple axes, from path-level distributional fidelity to moment structure, and examine the reliability of pointwise metrics as evaluation criteria. We further study the dependence of \sigTPP~on the signature truncation degree.

\subsection{Datasets, Baselines, and Evaluation Metrics}
\label{subsect::data-baselines}
We use $12$ datasets in total, $6$ synthetic and $6$ real-world. Table~\ref{table::datasets_stats} and Figures~\ref{fig::hp_three_marks_diagnostics}-\ref{fig::yelp_mississauga_diagnostics} summarise their basic characteristics, and Appendix~\ref{sect::datasets} describes each dataset in more detail.
The synthetic corpus covers the canonical regimes of TPP modelling: Poisson process (PS), inhomogeneous Poisson (IP), univariate (H1) and 3-dimensional (H3) Hawkes processes, self-correcting (SC), and stationary renewal (SR) processes.
The real-world corpus encompasses datasets from prior work, namely Earthquake (EQ), Stack Overflow (SO), Taobao (TB), and Taxi (TX) from EasyTPP~\cite{xue2024easytpp}, and Yelp Mississauga (YLP) and Yelp Airport (YLA) from the benchmark proposed by~\citet{addthin}.

We compare \sigTPP~against $5$ baselines spanning the main paradigms for generative TPP modelling.
We focus on autoregressive baselines that have a neural architecture similar to \sigTPP\ (similar number of parameters) to make comparisons fair.
These include the denoising diffusion probabilistic model (DDPM) and variational autoencoder (VAE)~\cite{lin2022exploring}, the Wasserstein GAN (WGAN)~\cite{xiao2017wasserstein}, and a maximum-likelihood Gamma fit~\cite{coletta_learning_2022, cont_limit_2023} as a parametric reference (\gammamodel). To isolate the contribution of distributional modelling, we also include \determodel~\cite{lin2022exploring}, a deterministic regressor whose predictive law is a Dirac measure and which by construction cannot capture distributional structure.

We evaluate the goodness of fit along the following axes: pathwise discrepancies via $\mathcal{E}$ (energy distance) and $\WAS$ under our proposed metric $\dtppsym$, signature-based discrepancy via $\sigW$, pointwise probabilistic accuracy via the continuous ranked probability score (CRPS), marginal and temporal structure via histograms of interarrival times ($\histITloss$) and intensity ($\histintloss$), pairwise correlation discrepancy (PCD), and autocorrelation discrepancy (ACD) across sequence components. Lower is better for all reported metrics. Metric definitions and a more expansive justification for their appropriateness are deferred to Appendix~\ref{sec::metrics}. The main benchmark tables report mean test performance, with standard errors in parentheses, computed as the sample standard deviation over $100$ bootstrap replicate estimates.

For each dataset, hyperparameters are selected on the validation split using a rank-based aggregate criterion over the validation diagnostics. The criterion sums within-dataset ranks across the ranked validation metrics. During training, we select checkpoints based on the validation $\histITloss$. The aggregate rank criterion is used only for model selection, and not as a performance measure.

Since evaluation metrics can differ substantially in scale across datasets, we report the relative score computed as a 2-step process (see also~\cite{ansari2024chronos, das2024timesfm, gruver2023llmtime}). First, we compute, for each model, a relative score defined as the ratio between the model's score and that of a baseline model, here \determodel. We then aggregate these relative scores across tasks using the geometric mean~\cite{fleming1986how}, which is the appropriate summary statistic for multiplicative relative comparisons, and makes the model ordering independent of the choice of the baseline (see Appendix~\ref{subsec::relative_score}).

\subsection{Recovering Known Synthetic Processes}
\label{subsect::synth}

\begin{table}[t]
\renewcommand{\arraystretch}{0.9}%
\centering
\resultcaption{Performance of \sigTPP~and baselines on synthetic point processes. Results for SC and SR are in Table~\ref{table::additional_datasets}.}
\label{table::synth_results_dist}
\smallskip
\setlength{\tabcolsep}{4pt} 
\setlength{\aboverulesep}{1pt}
\setlength{\belowrulesep}{1pt}
\resizebox{\textwidth}{!}{%
\begin{tabular}[t]{c l*{4}{r}@{\hspace{6pt}}c*{4}{r}}
  \toprule
  & Model & \multicolumn{1}{c}{PS} & \multicolumn{1}{c}{IP} & \multicolumn{1}{c}{H1} & \multicolumn{1}{c}{H3} & & \multicolumn{1}{c}{PS} & \multicolumn{1}{c}{IP} & \multicolumn{1}{c}{H1} & \multicolumn{1}{c}{H3} \\
  \cmidrule(lr){2-6}\cmidrule(lr){8-11}
  \multirow{6}{*}{\rotatebox[origin=c]{90}{\shortstack[c]{$\mathcal{E}$ {\scriptsize $(\times 10^{\scalebox{1}[1]{-}3})$}}}}
  & \sigTPP & \textbf{0.53} \footnotesize{(0.93)} & \textbf{0.42} \footnotesize{(0.32)} & \textbf{1.16} \footnotesize{(0.74)} & \textbf{1.32} \footnotesize{(1.60)} & \multirow{6}{*}{\rotatebox[origin=c]{90}{\shortstack[c]{$\WAS$ {\scriptsize $(\times 10^{\scalebox{1}[1]{-}2})$}}}} & \textbf{5.26} \footnotesize{(0.20)} & \textbf{6.55} \footnotesize{(0.14)} & \textbf{3.75} \footnotesize{(0.31)} & \textbf{10.1} \footnotesize{(0.42)} \\
  & \cellcolor{gray!15}VAE & \cellcolor{gray!15}8.64 \footnotesize{(3.47)} & \cellcolor{gray!15}\underline{2.24} \footnotesize{(0.64)} & \cellcolor{gray!15}16.8 \footnotesize{(3.85)} & \cellcolor{gray!15}20.0 \footnotesize{(6.11)} &  & \cellcolor{gray!15}7.12 \footnotesize{(0.64)} & \cellcolor{gray!15}\underline{6.76} \footnotesize{(0.18)} & \cellcolor{gray!15}6.84 \footnotesize{(0.57)} & \cellcolor{gray!15}13.8 \footnotesize{(1.11)} \\
  & DDPM & \underline{2.72} \footnotesize{(1.11)} & 2.56 \footnotesize{(0.63)} & \underline{10.7} \footnotesize{(2.73)} & \underline{1.61} \footnotesize{(1.35)} &  & \underline{6.15} \footnotesize{(0.24)} & 7.27 \footnotesize{(0.15)} & \underline{5.93} \footnotesize{(0.49)} & \underline{10.5} \footnotesize{(0.50)} \\
  & \cellcolor{gray!15}WGAN & \cellcolor{gray!15}4.99 \footnotesize{(1.83)} & \cellcolor{gray!15}3.20 \footnotesize{(1.03)} & \cellcolor{gray!15}16.9 \footnotesize{(3.93)} & \cellcolor{gray!15}11.4 \footnotesize{(4.00)} &  & \cellcolor{gray!15}6.43 \footnotesize{(0.40)} & \cellcolor{gray!15}7.68 \footnotesize{(0.29)} & \cellcolor{gray!15}6.97 \footnotesize{(0.61)} & \cellcolor{gray!15}12.8 \footnotesize{(0.83)} \\
  & \determodel & 330 \footnotesize{(12.8)} & 108 \footnotesize{(4.85)} & 510 \footnotesize{(15.7)} & 640 \footnotesize{(25.3)} &  & 28.3 \footnotesize{(0.65)} & 18.2 \footnotesize{(0.34)} & 38.1 \footnotesize{(0.62)} & 54.6 \footnotesize{(1.10)} \\
  & \cellcolor{gray!15}\gammamodel & \cellcolor{gray!15}4.38 \footnotesize{(2.58)} & \cellcolor{gray!15}20.3 \footnotesize{(3.22)} & \cellcolor{gray!15}17.3 \footnotesize{(3.87)} & \cellcolor{gray!15}25.5 \footnotesize{(5.92)} &  & \cellcolor{gray!15}6.33 \footnotesize{(0.54)} & \cellcolor{gray!15}9.50 \footnotesize{(0.49)} & \cellcolor{gray!15}7.05 \footnotesize{(0.59)} & \cellcolor{gray!15}15.1 \footnotesize{(1.07)} \\
  \cmidrule(lr){2-6}\cmidrule(lr){8-11}
  \multirow{6}{*}{\rotatebox[origin=c]{90}{\shortstack[c]{$\sigW$ {\scriptsize $(\times 10^{\scalebox{1}[1]{-}5})$}}}}
  & \sigTPP & \textbf{1.42} \footnotesize{(1.52)} & 1.20 \footnotesize{(1.05)} & \textbf{44.0} \footnotesize{(29.2)} & 1.02 \footnotesize{(0.75)} & \multirow{6}{*}{\rotatebox[origin=c]{90}{\shortstack[c]{CRPS {\scriptsize $(\times 10^{\scalebox{1}[1]{-}1})$}}}} & 4.64 \footnotesize{(0.08)} & \underline{3.04} \footnotesize{(0.03)} & 10.2 \footnotesize{(0.27)} & 2.40 \footnotesize{(0.03)} \\
  & \cellcolor{gray!15}VAE & \cellcolor{gray!15}13.4 \footnotesize{(6.42)} & \cellcolor{gray!15}\textbf{0.83} \footnotesize{(0.89)} & \cellcolor{gray!15}155 \footnotesize{(28.1)} & \cellcolor{gray!15}\textbf{0.25} \footnotesize{(0.26)} &  & \cellcolor{gray!15}\textbf{4.56} \footnotesize{(0.08)} & \cellcolor{gray!15}3.07 \footnotesize{(0.03)} & \cellcolor{gray!15}\textbf{9.53} \footnotesize{(0.26)} & \cellcolor{gray!15}\underline{2.33} \footnotesize{(0.03)} \\
  & DDPM & 129 \footnotesize{(24.3)} & 64.0 \footnotesize{(9.18)} & 302 \footnotesize{(65.9)} & 10.9 \footnotesize{(2.70)} &  & 4.81 \footnotesize{(0.09)} & 3.27 \footnotesize{(0.04)} & 10.1 \footnotesize{(0.30)} & \textbf{2.23} \footnotesize{(0.03)} \\
  & \cellcolor{gray!15}WGAN & \cellcolor{gray!15}17.4 \footnotesize{(6.92)} & \cellcolor{gray!15}\underline{1.01} \footnotesize{(0.70)} & \cellcolor{gray!15}\underline{141} \footnotesize{(30.9)} & \cellcolor{gray!15}0.80 \footnotesize{(0.70)} &  & \cellcolor{gray!15}\underline{4.58} \footnotesize{(0.08)} & \cellcolor{gray!15}\textbf{3.03} \footnotesize{(0.03)} & \cellcolor{gray!15}\underline{9.59} \footnotesize{(0.28)} & \cellcolor{gray!15}2.36 \footnotesize{(0.03)} \\
  & \determodel & 62.7 \footnotesize{(12.7)} & 8.01 \footnotesize{(1.76)} & 862 \footnotesize{(112)} & 1.09 \footnotesize{(0.39)} &  & 6.29 \footnotesize{(0.11)} & 4.34 \footnotesize{(0.05)} & 12.7 \footnotesize{(0.37)} & 3.23 \footnotesize{(0.04)} \\
  & \cellcolor{gray!15}\gammamodel & \cellcolor{gray!15}\underline{5.76} \footnotesize{(4.25)} & \cellcolor{gray!15}23.8 \footnotesize{(4.84)} & \cellcolor{gray!15}284 \footnotesize{(51.3)} & \cellcolor{gray!15}\underline{0.63} \footnotesize{(0.29)} &  & \cellcolor{gray!15}4.60 \footnotesize{(0.08)} & \cellcolor{gray!15}3.21 \footnotesize{(0.03)} & \cellcolor{gray!15}9.95 \footnotesize{(0.26)} & \cellcolor{gray!15}2.39 \footnotesize{(0.03)} \\
  \cmidrule(lr){2-6}\cmidrule(lr){8-11}
  \multirow{6}{*}{\rotatebox[origin=c]{90}{\shortstack[c]{$\histintloss$ {\scriptsize $(\times 10^{\scalebox{1}[1]{-}2})$}}}}
  & \sigTPP & \underline{3.35} \footnotesize{(0.46)} & \textbf{3.09} \footnotesize{(0.35)} & \textbf{6.09} \footnotesize{(1.04)} & \underline{2.68} \footnotesize{(0.40)} & \multirow{6}{*}{\rotatebox[origin=c]{90}{\shortstack[c]{$\histITloss$ {\scriptsize $(\times 10^{\scalebox{1}[1]{-}2})$}}}} & \underline{2.95} \footnotesize{(0.60)} & \textbf{1.97} \footnotesize{(0.37)} & \textbf{3.84} \footnotesize{(0.84)} & 7.64 \footnotesize{(0.56)} \\
  & \cellcolor{gray!15}VAE & \cellcolor{gray!15}4.46 \footnotesize{(0.57)} & \cellcolor{gray!15}4.90 \footnotesize{(0.43)} & \cellcolor{gray!15}13.1 \footnotesize{(1.22)} & \cellcolor{gray!15}\textbf{2.25} \footnotesize{(0.40)} &  & \cellcolor{gray!15}4.16 \footnotesize{(0.58)} & \cellcolor{gray!15}5.47 \footnotesize{(0.46)} & \cellcolor{gray!15}\underline{4.07} \footnotesize{(0.67)} & \cellcolor{gray!15}4.01 \footnotesize{(0.48)} \\
  & DDPM & 7.93 \footnotesize{(0.64)} & 7.18 \footnotesize{(0.41)} & 15.2 \footnotesize{(1.05)} & 5.10 \footnotesize{(0.50)} &  & 6.08 \footnotesize{(0.72)} & 4.10 \footnotesize{(0.45)} & 5.45 \footnotesize{(0.89)} & \underline{2.32} \footnotesize{(0.47)} \\
  & \cellcolor{gray!15}WGAN & \cellcolor{gray!15}7.31 \footnotesize{(0.67)} & \cellcolor{gray!15}\underline{4.47} \footnotesize{(0.37)} & \cellcolor{gray!15}13.3 \footnotesize{(1.31)} & \cellcolor{gray!15}3.37 \footnotesize{(0.49)} &  & \cellcolor{gray!15}3.74 \footnotesize{(0.55)} & \cellcolor{gray!15}4.70 \footnotesize{(0.46)} & \cellcolor{gray!15}5.33 \footnotesize{(0.74)} & \cellcolor{gray!15}4.70 \footnotesize{(0.56)} \\
  & \determodel & 12.6 \footnotesize{(0.71)} & 12.8 \footnotesize{(0.43)} & 20.2 \footnotesize{(1.25)} & 4.17 \footnotesize{(0.52)} &  & 60.7 \footnotesize{(5.26)} & 48.4 \footnotesize{(1.46)} & 53.5 \footnotesize{(2.86)} & 66.6 \footnotesize{(6.25)} \\
  & \cellcolor{gray!15}\gammamodel & \cellcolor{gray!15}\textbf{3.17} \footnotesize{(0.45)} & \cellcolor{gray!15}11.0 \footnotesize{(0.47)} & \cellcolor{gray!15}\underline{8.11} \footnotesize{(1.08)} & \cellcolor{gray!15}3.41 \footnotesize{(0.43)} &  & \cellcolor{gray!15}\textbf{2.86} \footnotesize{(0.64)} & \cellcolor{gray!15}\underline{3.21} \footnotesize{(0.48)} & \cellcolor{gray!15}6.74 \footnotesize{(1.00)} & \cellcolor{gray!15}\textbf{1.71} \footnotesize{(0.59)} \\
  \cmidrule(lr){2-6}\cmidrule(lr){8-11}
  \multirow{6}{*}{\rotatebox[origin=c]{90}{\shortstack[c]{ACD {\scriptsize $(\times 10^{\scalebox{1}[1]{-}2})$}}}}
  & \sigTPP & \underline{1.35} \footnotesize{(0.46)} & \textbf{0.70} \footnotesize{(0.20)} & \underline{1.70} \footnotesize{(0.45)} & \textbf{1.54} \footnotesize{(0.39)} & \multirow{6}{*}{\rotatebox[origin=c]{90}{\shortstack[c]{PCD {\scriptsize $(\times 10^{\scalebox{1}[1]{-}2})$}}}} & \textbf{7.36} \footnotesize{(0.53)} & \textbf{6.01} \footnotesize{(0.42)} & 9.07 \footnotesize{(0.85)} & 8.56 \footnotesize{(0.40)} \\
  & \cellcolor{gray!15}VAE & \cellcolor{gray!15}1.53 \footnotesize{(0.42)} & \cellcolor{gray!15}\underline{0.71} \footnotesize{(0.20)} & \cellcolor{gray!15}2.73 \footnotesize{(0.62)} & \cellcolor{gray!15}2.44 \footnotesize{(0.45)} &  & \cellcolor{gray!15}7.43 \footnotesize{(0.53)} & \cellcolor{gray!15}6.40 \footnotesize{(0.43)} & \cellcolor{gray!15}\underline{8.82} \footnotesize{(0.86)} & \cellcolor{gray!15}\underline{8.39} \footnotesize{(0.40)} \\
  & DDPM & 1.86 \footnotesize{(0.52)} & 1.64 \footnotesize{(0.25)} & 2.54 \footnotesize{(0.47)} & \underline{1.76} \footnotesize{(0.44)} &  & 7.54 \footnotesize{(0.51)} & 6.27 \footnotesize{(0.43)} & \textbf{8.25} \footnotesize{(0.97)} & \textbf{8.33} \footnotesize{(0.42)} \\
  & \cellcolor{gray!15}WGAN & \cellcolor{gray!15}2.39 \footnotesize{(0.49)} & \cellcolor{gray!15}1.94 \footnotesize{(0.28)} & \cellcolor{gray!15}\textbf{1.60} \footnotesize{(0.39)} & \cellcolor{gray!15}5.03 \footnotesize{(0.45)} &  & \cellcolor{gray!15}7.49 \footnotesize{(0.53)} & \cellcolor{gray!15}\underline{6.15} \footnotesize{(0.43)} & \cellcolor{gray!15}8.87 \footnotesize{(0.84)} & \cellcolor{gray!15}8.82 \footnotesize{(0.41)} \\
  & \determodel & 49.3 \footnotesize{(0.53)} & 35.3 \footnotesize{(0.28)} & 52.5 \footnotesize{(0.71)} & 63.4 \footnotesize{(0.45)} &  & 96.5 \footnotesize{(0.43)} & 82.7 \footnotesize{(0.28)} & 59.7 \footnotesize{(0.83)} & 97.9 \footnotesize{(0.22)} \\
  & \cellcolor{gray!15}\gammamodel & \cellcolor{gray!15}\textbf{1.32} \footnotesize{(0.44)} & \cellcolor{gray!15}1.28 \footnotesize{(0.26)} & \cellcolor{gray!15}4.08 \footnotesize{(0.71)} & \cellcolor{gray!15}6.12 \footnotesize{(0.45)} &  & \cellcolor{gray!15}\underline{7.39} \footnotesize{(0.53)} & \cellcolor{gray!15}6.39 \footnotesize{(0.42)} & \cellcolor{gray!15}9.03 \footnotesize{(0.85)} & \cellcolor{gray!15}8.74 \footnotesize{(0.41)} \\
  \bottomrule
\end{tabular}%
}%
\end{table}

We first use $6$ synthetic point processes with known data-generating mechanisms to evaluate \sigTPP. 
Table~\ref{table::synth_results_dist} presents the discrepancy metrics $\mathcal{E}$ and $\WAS$ under $\dtppsym$, $\sigW$, CRPS, $\histITloss$, $\histintloss$, ACD, and PCD. The results for SC and SR appear in Table~\ref{table::additional_datasets}.
Our first observation is that \sigTPP\ improves distributional fidelity (Section~\ref{subsect::evaluation}) relative to the strongest baselines.
\sigTPP~obtains the lowest $\mathcal{E}$ and $\WAS$ on $5$ out of the $6$ datasets, and on $3$ out of $6$ w.r.t. $\sigW$. 
The relative score (Appendix~\ref{subsec::relative_score}) of our model is the lowest under $\mathcal{E}$, $\WAS$, and $\sigW$. Compared to the second-best model for each metric (DDPM for $\mathcal{E}$ and $\WAS$, and VAE for $\sigW$), \sigTPP~improves the performance by $68\%$, $13\%$, and $34\%$, respectively (Figure~\ref{fig::normalised_score_results_synth} and Table~\ref{table::geo_improvement_vs_deter_grouped}).
These gains are consistent with \sigTPP's training objective, which targets the joint temporal law rather than local event-level behaviour.

However, it is also striking that while \sigTPP~improves the distributional fit, it occasionally underperforms according to CRPS ($4\%$ and $1\%$ error increase compared to VAE and WGAN, but $7\%$ decrease compared to DDPM). This suggests that optimising a path-level distributional objective favours global distributional alignment, but does not necessarily improve conditional calibration relative to methods trained directly with conditional losses.
Nevertheless, a number of summary-statistic discrepancies indicate that \sigTPP~remains competitive.
$\histITloss, \histintloss, \ACD$, and $\PCD$ exhibit overlapping standard errors for many methods. \sigTPP\ is either the best method or within one standard error of the best baseline in $20$ out of $24$ metric-dataset pairs: in $2(+2)/6$ datasets for $\histITloss$, $4(+1)/6$ for $\histintloss$, $3(+2)/6$ for ACD, and $3(+3)/6$ for PCD, where the number in parentheses denotes outcomes within one standard error of the best. When ranking methods separately across each dataset and metric, \sigTPP~has the best average rank of $2.00$, followed by VAE with an average rank of $2.88$ (Table~\ref{table::avg_rank_by_metric_grouped}).

Figures~\ref{fig::qq_epdf_qualitative}(a)~and~\ref{fig::qq_epdf_qualitative}(b) illustrate two of these statistics. On SR, whose lognormal interarrival times are heavy-tailed, the QQ plot shows that \sigTPP~and DDPM stay closest to the diagonal in the upper tail, whereas VAE, WGAN, and \gammamodel~fall below it and generate too few long gaps. DDPM has the lowest $\histITloss$ on SR, with \sigTPP~within one standard error (Table~\ref{table::additional_datasets}). On H1, the intensity of \sigTPP~stays close to the reference, whereas the intensities of VAE and DDPM rise towards the end of the window. \sigTPP~has the lowest $\histintloss$ on H1 (Table~\ref{table::synth_results_dist}). Figures~\ref{fig::sc_qualitative_summary} and~\ref{fig::sr_qualitative_summary} give the full qualitative comparison on SC and SR, including the correlation and autocorrelation structures.

Overall, across all synthetic metric-dataset pairs, \sigTPP~is best in $25/48$ cases and within one standard error of the best in another $11/48$ cases, making it competitive in about $75\%$ of them. This is further supported by the relative score, with \sigTPP~achieving the best performance on $6$ out of $8$ metrics.
In pairwise comparisons against each baseline, averaged over datasets and metrics, \sigTPP\ outperforms WGAN by $50\%$, VAE by $39\%$, and DDPM by $46\%$.

\begin{figure}[!b]
  \centering
  \includegraphics[width=\textwidth]{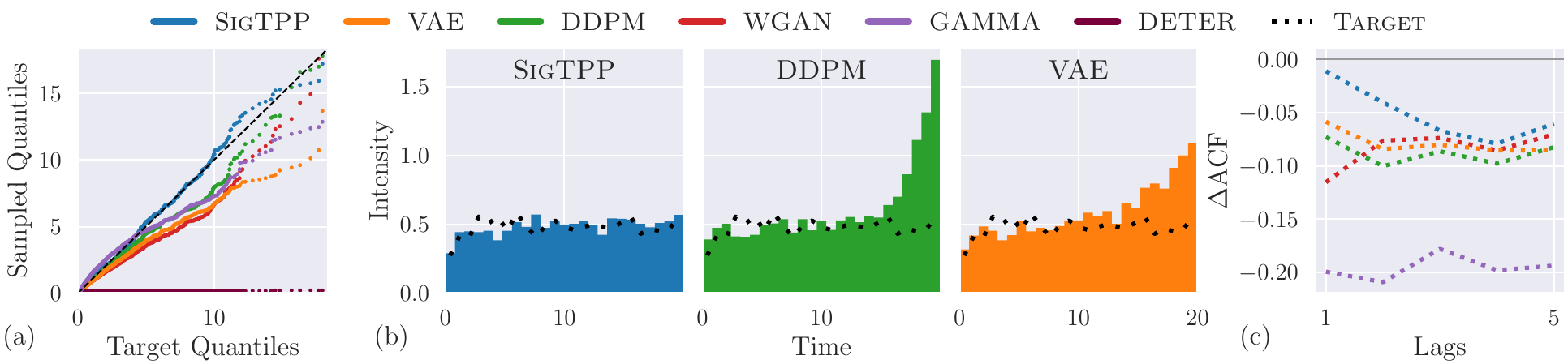}
\caption{Qualitative fit on $3$ datasets, one per summary statistic: (a) $\histITloss$, (b) $\histintloss$, and (c) ACD. Colours mark models as in the legend; black marks the reference data.
(a) Quantile-quantile (QQ) plot of interarrival times for the $6$ models on SR. 
(b) Empirical intensity functions $\hat{\lambda}(t)$ on H1 for \sigTPP~and the two baselines with the best overall mean rank, DDPM, and VAE. The dotted line is the reference intensity.
(c) Signed autocorrelation error of interarrival times on EQ for lags $1$ to $5$. \determodel~is omitted because its much larger error compresses the scale.}
\label{fig::qq_epdf_qualitative}
\end{figure}

\subsection{Real-World Tasks}
\label{subsect::real}
\begin{table}[t]
\renewcommand{\arraystretch}{1.05}
\setlength{\tabcolsep}{4pt}
  \centering
  \resultcaption{%
    Performance of \sigTPP~and baselines on real-world tasks. Results for YLA are in Table~\ref{table::additional_datasets}.}
  \label{table::real_results_dist}
  \vspace{0.4em}
  \small
\setlength{\tabcolsep}{0.005\linewidth}
\setlength{\aboverulesep}{0.5pt}
\setlength{\belowrulesep}{0.5pt}
  \begin{adjustbox}{width=\linewidth}
    \begin{tabular}[t]{c l*{5}{r}@{\hspace{6pt}}c*{5}{r}}
      \toprule
      & Model & \multicolumn{1}{c}{EQ} & \multicolumn{1}{c}{SO} & \multicolumn{1}{c}{TB} & \multicolumn{1}{c}{TX} & \multicolumn{1}{c}{YLP} & & \multicolumn{1}{c}{EQ} & \multicolumn{1}{c}{SO} & \multicolumn{1}{c}{TB} & \multicolumn{1}{c}{TX} & \multicolumn{1}{c}{YLP} \\
      \cmidrule(lr){2-7}\cmidrule(lr){9-13}
      \multirow{6}{*}{\rotatebox[origin=c]{90}{\shortstack[c]{$\mathcal{E}$ {\scriptsize $(\times 10^{\scalebox{1}[1]{-}3})$}}}}
    & \sigTPP & \underline{32.4} \footnotesize{(0.65)} & \underline{12.0} \footnotesize{(1.71)} & 512 \footnotesize{(25.2)} & \textbf{76.3} \footnotesize{(2.95)} & \underline{38.4} \footnotesize{(14.9)}& \multirow{6}{*}{\rotatebox[origin=c]{90}{\shortstack[c]{$\WAS$ {\scriptsize $(\times 10^{\scalebox{1}[1]{-}2})$}}}} & \textbf{2.91} \footnotesize{(0.04)} & \underline{7.44} \footnotesize{(0.42)} & 58.2 \footnotesize{(2.56)} & \underline{8.02} \footnotesize{(0.22)} & \textbf{15.0} \footnotesize{(1.86)} \\
  & \cellcolor{gray!15}VAE & \cellcolor{gray!15}34.5 \footnotesize{(0.64)} & \cellcolor{gray!15}\textbf{9.64} \footnotesize{(2.55)} & \cellcolor{gray!15}545 \footnotesize{(24.7)} & \cellcolor{gray!15}77.7 \footnotesize{(2.73)} & \cellcolor{gray!15}\textbf{33.6} \footnotesize{(17.6)}&  & \cellcolor{gray!15}\underline{2.91} \footnotesize{(0.05)} & \cellcolor{gray!15}\textbf{7.08} \footnotesize{(0.76)} & \cellcolor{gray!15}59.8 \footnotesize{(2.40)} & \cellcolor{gray!15}8.13 \footnotesize{(0.23)} & \cellcolor{gray!15}\underline{16.5} \footnotesize{(2.90)} \\
  & DDPM & \textbf{26.0} \footnotesize{(0.66)} & 24.5 \footnotesize{(6.37)} & \textbf{455} \footnotesize{(21.5)} & \underline{76.3} \footnotesize{(2.77)} & 57.5 \footnotesize{(20.5)}&  & 3.02 \footnotesize{(0.05)} & 10.9 \footnotesize{(1.22)} & \textbf{52.8} \footnotesize{(2.31)} & \textbf{7.74} \footnotesize{(0.18)} & 20.0 \footnotesize{(2.06)} \\
  & \cellcolor{gray!15}WGAN & \cellcolor{gray!15}32.6 \footnotesize{(0.71)} & \cellcolor{gray!15}26.5 \footnotesize{(3.34)} & \cellcolor{gray!15}\underline{502} \footnotesize{(22.4)} & \cellcolor{gray!15}81.9 \footnotesize{(2.80)} & \cellcolor{gray!15}49.2 \footnotesize{(24.3)} &  & \cellcolor{gray!15}2.99 \footnotesize{(0.04)} & \cellcolor{gray!15}10.3 \footnotesize{(0.64)} & \cellcolor{gray!15}\underline{56.1} \footnotesize{(2.07)} & \cellcolor{gray!15}8.38 \footnotesize{(0.27)} & \cellcolor{gray!15}18.8 \footnotesize{(2.78)} \\
  & \determodel & 47.2 \footnotesize{(0.92)} & 478 \footnotesize{(22.6)} & 958 \footnotesize{(32.1)} & 161 \footnotesize{(5.46)} & 2402 \footnotesize{(90)} &  & 3.74 \footnotesize{(0.07)} & 36.2 \footnotesize{(0.80)} & 86.0 \footnotesize{(2.48)} & 14.1 \footnotesize{(0.40)} & 143 \footnotesize{(3.55)} \\
  & \cellcolor{gray!15}\gammamodel & \cellcolor{gray!15}35.8 \footnotesize{(0.66)} & \cellcolor{gray!15}31.4 \footnotesize{(3.78)} & \cellcolor{gray!15}539 \footnotesize{(23.5)} & \cellcolor{gray!15}85.6 \footnotesize{(2.99)} & \cellcolor{gray!15}250 \footnotesize{(59.0)} &  & \cellcolor{gray!15}3.11 \footnotesize{(0.05)} & \cellcolor{gray!15}11.0 \footnotesize{(0.66)} & \cellcolor{gray!15}60.3 \footnotesize{(2.15)} & \cellcolor{gray!15}8.79 \footnotesize{(0.28)} & \cellcolor{gray!15}35.2 \footnotesize{(3.68)} \\
  \cmidrule(lr){2-7}\cmidrule(lr){9-13}
  \multirow{6}{*}{\rotatebox[origin=c]{90}{\shortstack[c]{$\sigW$ {\scriptsize $(\times 10^{\scalebox{1}[1]{-}5})$}}}}
  & \sigTPP & 409 \footnotesize{(36.3)} & \textbf{0.09} \footnotesize{(0.10)} & \underline{2.27} \footnotesize{(0.94)} & \underline{4.74} \footnotesize{(1.22)} & \underline{0.77} \footnotesize{(0.86)} & \multirow{6}{*}{\rotatebox[origin=c]{90}{\shortstack[c]{CRPS {\scriptsize $(\times 10^{\scalebox{1}[1]{-}1})$}}}} & 5.61 \footnotesize{(0.11)} & 5.43 \footnotesize{(0.10)} & 0.37 \footnotesize{(0.01)} & \underline{1.13} \footnotesize{(0.02)} & 3.22 \footnotesize{(0.13)} \\
  & \cellcolor{gray!15}VAE & \cellcolor{gray!15}\underline{397} \footnotesize{(38.1)} & \cellcolor{gray!15}3.70 \footnotesize{(0.52)} & \cellcolor{gray!15}\textbf{0.45} \footnotesize{(0.30)} & \cellcolor{gray!15}9.63 \footnotesize{(1.61)} & \cellcolor{gray!15}1.08 \footnotesize{(1.10)} &  & \cellcolor{gray!15}\underline{5.13} \footnotesize{(0.12)} & \cellcolor{gray!15}\underline{4.64} \footnotesize{(0.09)} & \cellcolor{gray!15}\underline{0.36} \footnotesize{(0.01)} & \cellcolor{gray!15}\textbf{1.06} \footnotesize{(0.02)} & \cellcolor{gray!15}\textbf{2.74} \footnotesize{(0.12)} \\
  & DDPM & 1144 \footnotesize{(56)} & 2.58 \footnotesize{(0.44)} & 9.78 \footnotesize{(1.99)} & 5.52 \footnotesize{(1.21)} & 1.70 \footnotesize{(1.08)} &  & \textbf{4.57} \footnotesize{(0.11)} & \textbf{4.63} \footnotesize{(0.10)} & \textbf{0.34} \footnotesize{(0.01)} & 1.20 \footnotesize{(0.02)} & \underline{2.76} \footnotesize{(0.11)} \\
  & \cellcolor{gray!15}WGAN & \cellcolor{gray!15}603 \footnotesize{(45.9)} & \cellcolor{gray!15}3.95 \footnotesize{(0.64)} & \cellcolor{gray!15}7.81 \footnotesize{(1.44)} & \cellcolor{gray!15}6.80 \footnotesize{(1.51)} & \cellcolor{gray!15}\textbf{0.42} \footnotesize{(0.47)} &  & \cellcolor{gray!15}5.85 \footnotesize{(0.10)} & \cellcolor{gray!15}4.93 \footnotesize{(0.09)} & \cellcolor{gray!15}0.37 \footnotesize{(0.01)} & \cellcolor{gray!15}1.15 \footnotesize{(0.02)} & \cellcolor{gray!15}2.91 \footnotesize{(0.11)} \\
  & \determodel & \textbf{146} \footnotesize{(14.5)} & \underline{0.19} \footnotesize{(0.16)} & 49.7 \footnotesize{(4.45)} & 24.5 \footnotesize{(2.82)} & 7.78 \footnotesize{(0.61)} &  & 7.05 \footnotesize{(0.16)} & 6.51 \footnotesize{(0.12)} & 0.42 \footnotesize{(0.01)} & 1.49 \footnotesize{(0.02)} & 3.66 \footnotesize{(0.15)} \\
  & \cellcolor{gray!15}\gammamodel & \cellcolor{gray!15}572 \footnotesize{(38.1)} & \cellcolor{gray!15}6.66 \footnotesize{(0.81)} & \cellcolor{gray!15}12.3 \footnotesize{(3.47)} & \cellcolor{gray!15}\textbf{3.07} \footnotesize{(1.16)} & \cellcolor{gray!15}10.6 \footnotesize{(1.40)} &  & \cellcolor{gray!15}6.16 \footnotesize{(0.09)} & \cellcolor{gray!15}4.91 \footnotesize{(0.08)} & \cellcolor{gray!15}0.37 \footnotesize{(0.01)} & \cellcolor{gray!15}1.15 \footnotesize{(0.02)} & \cellcolor{gray!15}3.19 \footnotesize{(0.11)} \\
  \cmidrule(lr){2-7}\cmidrule(lr){9-13}
  \multirow{6}{*}{\rotatebox[origin=c]{90}{\shortstack[c]{$\histintloss$ {\scriptsize $(\times 10^{\scalebox{1}[1]{-}2})$}}}}
  & \sigTPP & 19.3 \footnotesize{(0.99)} & 9.00 \footnotesize{(0.68)} & \textbf{3.77} \footnotesize{(0.90)} & 6.31 \footnotesize{(1.05)} & 13.3 \footnotesize{(1.14)} & \multirow{6}{*}{\rotatebox[origin=c]{90}{\shortstack[c]{$\histITloss$ {\scriptsize $(\times 10^{\scalebox{1}[1]{-}2})$}}}} & 12.9 \footnotesize{(0.84)} & 10.5 \footnotesize{(0.41)} & \underline{7.33} \footnotesize{(0.39)} & \underline{4.35} \footnotesize{(0.44)} & \underline{4.73} \footnotesize{(0.99)} \\
  & \cellcolor{gray!15}VAE & \cellcolor{gray!15}\textbf{17.9} \footnotesize{(1.12)} & \cellcolor{gray!15}8.94 \footnotesize{(0.69)} & \cellcolor{gray!15}9.57 \footnotesize{(1.06)} & \cellcolor{gray!15}6.99 \footnotesize{(1.04)} & \cellcolor{gray!15}\textbf{10.8} \footnotesize{(1.14)}&  & \cellcolor{gray!15}14.8 \footnotesize{(0.97)} & \cellcolor{gray!15}10.6 \footnotesize{(0.56)} & \cellcolor{gray!15}\textbf{6.29} \footnotesize{(0.53)} & \cellcolor{gray!15}4.36 \footnotesize{(0.40)} & \cellcolor{gray!15}\textbf{4.02} \footnotesize{(0.68)} \\
  & DDPM & \underline{19.0} \footnotesize{(0.62)} & 9.95 \footnotesize{(0.63)} & \underline{7.01} \footnotesize{(1.08)} & \textbf{4.57} \footnotesize{(1.04)} & 16.9 \footnotesize{(1.11)}&  & \textbf{9.33} \footnotesize{(0.82)} & \underline{5.22} \footnotesize{(0.63)} & 10.6 \footnotesize{(0.72)} & \textbf{2.02} \footnotesize{(0.40)} & 6.47 \footnotesize{(0.80)} \\
  & \cellcolor{gray!15}WGAN & \cellcolor{gray!15}24.1 \footnotesize{(1.13)} & \cellcolor{gray!15}\textbf{6.25} \footnotesize{(0.77)} & \cellcolor{gray!15}14.0 \footnotesize{(1.24)} & \cellcolor{gray!15}\underline{5.06} \footnotesize{(0.74)} & \cellcolor{gray!15}\underline{13.3} \footnotesize{(1.03)}&  & \cellcolor{gray!15}\underline{9.58} \footnotesize{(0.58)} & \cellcolor{gray!15}\textbf{2.72} \footnotesize{(0.62)} & \cellcolor{gray!15}7.86 \footnotesize{(0.70)} & \cellcolor{gray!15}6.72 \footnotesize{(0.55)} & \cellcolor{gray!15}5.69 \footnotesize{(0.93)} \\
  & \determodel & 20.3 \footnotesize{(2.84)} & \underline{7.38} \footnotesize{(0.37)} & 41.6 \footnotesize{(3.32)} & 18.0 \footnotesize{(1.33)} & 25.3 \footnotesize{(0.93)} &  & 48.5 \footnotesize{(3.15)} & 54.9 \footnotesize{(7.16)} & 81.0 \footnotesize{(2.02)} & 56.9 \footnotesize{(6.67)} & 45.5 \footnotesize{(1.62)} \\
  & \cellcolor{gray!15}\gammamodel & \cellcolor{gray!15}27.0 \footnotesize{(1.12)} & \cellcolor{gray!15}11.5 \footnotesize{(0.60)} & \cellcolor{gray!15}16.2 \footnotesize{(1.08)} & \cellcolor{gray!15}7.05 \footnotesize{(0.73)} & \cellcolor{gray!15}31.7 \footnotesize{(1.20)} &  & \cellcolor{gray!15}9.93 \footnotesize{(0.83)} & \cellcolor{gray!15}5.70 \footnotesize{(0.75)} & \cellcolor{gray!15}15.1 \footnotesize{(1.26)} & \cellcolor{gray!15}4.84 \footnotesize{(0.41)} & \cellcolor{gray!15}13.2 \footnotesize{(0.85)} \\
  \cmidrule(lr){2-7}\cmidrule(lr){9-13}
  \multirow{6}{*}{\rotatebox[origin=c]{90}{\shortstack[c]{ACD {\scriptsize $(\times 10^{\scalebox{1}[1]{-}2})$}}}}
  & \sigTPP & \textbf{3.95} \footnotesize{(0.68)} & 8.93 \footnotesize{(0.51)} & \textbf{0.94} \footnotesize{(0.20)} & 2.97 \footnotesize{(0.43)} & \textbf{2.11} \footnotesize{(0.40)}& \multirow{6}{*}{\rotatebox[origin=c]{90}{\shortstack[c]{PCD {\scriptsize $(\times 10^{\scalebox{1}[1]{-}2})$}}}} & \textbf{7.51} \footnotesize{(0.53)} & 12.4 \footnotesize{(0.33)} & 5.94 \footnotesize{(0.26)} & 6.11 \footnotesize{(0.23)} & 12.9 \footnotesize{(0.69)} \\
  & \cellcolor{gray!15}VAE & \cellcolor{gray!15}7.17 \footnotesize{(0.76)} & \cellcolor{gray!15}\textbf{2.79} \footnotesize{(0.51)} & \cellcolor{gray!15}2.48 \footnotesize{(0.40)} & \cellcolor{gray!15}\underline{2.26} \footnotesize{(0.42)} & \cellcolor{gray!15}\underline{2.61} \footnotesize{(0.53)}&  & \cellcolor{gray!15}9.92 \footnotesize{(0.74)} & \cellcolor{gray!15}\textbf{10.2} \footnotesize{(0.28)} & \cellcolor{gray!15}\textbf{5.83} \footnotesize{(0.27)} & \cellcolor{gray!15}\underline{5.85} \footnotesize{(0.22)} & \cellcolor{gray!15}\textbf{12.4} \footnotesize{(0.71)} \\
  & DDPM & 6.95 \footnotesize{(0.76)} & \underline{5.17} \footnotesize{(0.51)} & \underline{1.41} \footnotesize{(0.35)} & \textbf{1.87} \footnotesize{(0.45)} & 3.35 \footnotesize{(0.60)}&  & 8.44 \footnotesize{(0.69)} & 11.1 \footnotesize{(0.32)} & \underline{5.84} \footnotesize{(0.25)} & \textbf{5.61} \footnotesize{(0.22)} & \underline{12.8} \footnotesize{(0.68)} \\
  & \cellcolor{gray!15}WGAN & \cellcolor{gray!15}\underline{6.95} \footnotesize{(0.76)} & \cellcolor{gray!15}6.57 \footnotesize{(0.51)} & \cellcolor{gray!15}3.20 \footnotesize{(0.41)} & \cellcolor{gray!15}3.48 \footnotesize{(0.47)} & \cellcolor{gray!15}5.51 \footnotesize{(0.46)} &  & \cellcolor{gray!15}\underline{7.55} \footnotesize{(0.54)} & \cellcolor{gray!15}\underline{11.0} \footnotesize{(0.31)} & \cellcolor{gray!15}5.92 \footnotesize{(0.28)} & \cellcolor{gray!15}6.25 \footnotesize{(0.22)} & \cellcolor{gray!15}13.1 \footnotesize{(0.69)} \\
  & \determodel & 61.9 \footnotesize{(0.76)} & 72.7 \footnotesize{(0.51)} & 74.2 \footnotesize{(0.41)} & 60.7 \footnotesize{(0.50)} & 64.4 \footnotesize{(0.62)} &  & 73.4 \footnotesize{(0.81)} & 79.3 \footnotesize{(0.27)} & 94.8 \footnotesize{(0.34)} & 43.4 \footnotesize{(0.31)} & 72.4 \footnotesize{(0.80)} \\
  & \cellcolor{gray!15}\gammamodel & \cellcolor{gray!15}16.8 \footnotesize{(0.76)} & \cellcolor{gray!15}8.03 \footnotesize{(0.51)} & \cellcolor{gray!15}2.92 \footnotesize{(0.41)} & \cellcolor{gray!15}3.51 \footnotesize{(0.47)} & \cellcolor{gray!15}4.61 \footnotesize{(0.61)} &  & \cellcolor{gray!15}18.8 \footnotesize{(0.81)} & \cellcolor{gray!15}12.0 \footnotesize{(0.35)} & \cellcolor{gray!15}5.86 \footnotesize{(0.27)} & \cellcolor{gray!15}6.24 \footnotesize{(0.22)} & \cellcolor{gray!15}13.3 \footnotesize{(0.69)} \\
      \bottomrule
    \end{tabular}
  \end{adjustbox}
\end{table}

We now focus our empirical analysis on $6$ real-world tasks under the same discrepancy metrics.
The results are enumerated in Tables~\ref{table::real_results_dist} and~\ref{table::additional_datasets} and, again, indicate strong performance for \sigTPP~(see Table~\ref{table::geo_improvement_vs_deter_grouped}).
It achieves the lowest $\WAS$  or lies within the best model's standard error on $4/6$ real-world datasets. 
This is confirmed by the relative scores, for which it obtains the lowest $\mathcal{E}$, $\WAS$, and $\sigW$ among all models. The second-best model on real-world tasks is VAE, which is outperformed on geometric average by $7\%$ across all metrics and datasets.
The improvements in $\mathcal{E}$ and $\WAS$ support pathwise distributional agreement under criteria distinct from the signature-based training objective.
The contrast with the synthetic setting is informative. 
Synthetic data is generated by a known mechanism, so distributional recovery is assessed under correct specification. 
The real-world datasets instead combine heterogeneous trajectories, variable observation windows, and latent contextual effects, so the empirical path distribution is only an approximation to a single data-generating law. 
This makes path-level comparison intrinsically harder and likely explains why the gains of \sigTPP~are smaller and less uniform than in the synthetic experiments. 

The remaining summary-statistic discrepancies provide a more nuanced comparison. 
As on the synthetic datasets, $\histITloss$, $\histintloss$, ACD, and PCD exhibit larger sampling variability, and several methods have overlapping standard errors. 
Across these metrics, \sigTPP~performs best in $5$/$24$ metric-dataset pairs and lies within one standard error of the best method in an additional $4$/$24$ cases. Thus, although CRPS favours the conditional baselines on aggregate, with \sigTPP~trailing VAE, WGAN, and DDPM by $9\%$, $2\%$, and $10\%$, respectively, it remains competitive on real-world data.
In a separate ablation (Appendix~\ref{subsect::conditional_supervision}), an auxiliary next-event regression loss is added to \sigTPP~and narrows the CRPS gap to VAE, the strongest baseline on CRPS (from trailing by $20\%$ to $3\%$). 
This regression loss variant improves \sigTPP's $\mathcal{E}$, $\WAS$, and $\sigW$ by $26\%$, $12\%$, and $30\%$, respectively.
Despite its weaker CRPS results and mixed performance on summary statistics, \sigTPP~achieves the lowest geometric relative score overall, and specifically on $5$ out of $8$ metrics (Figure~\ref{fig::normalised_score_results_all}).
In pairwise comparisons, \sigTPP~improves over WGAN by $20\%$, VAE by $7\%$, and DDPM by $17\%$. 

Figure~\ref{fig::qq_epdf_qualitative}(c) shows the short range dependence that ACD measures on EQ. All plotted models underestimate the autocorrelation of interarrival times, but \sigTPP~stays closest to the reference at every lag, by a wide margin at lags $1$ and $2$. ACD averages the absolute error over these lags, and \sigTPP~has the lowest ACD on EQ (Table~\ref{table::real_results_dist}). Appendix Figures~\ref{fig::eq_qualitative_summary}-\ref{fig::yelp_qualitative_summary} provide dataset-level qualitative comparisons of the $6$ models and further illustrate these trends.

Average ranks provide additional insights (Table~\ref{table::avg_rank_by_metric_grouped}). \sigTPP~achieves the second-best average rank of $2.54$, behind VAE ($2.38$). Relative scoring nevertheless favours \sigTPP~(see Figure~\ref{fig::normalised_score_results_real}). Ranks record only the ordering of models, whereas geometric score ratios also reflect the size of performance differences. Overall, the real-world geometric scores mirror the synthetic results, but with smaller margins and greater dataset-level variability. The persistence of improvements over pairwise comparisons, and average ranks in $\mathcal{E}$, $\WAS$, and $\sigW$ suggests that the signature-based objective continues to provide useful distributional alignment under misspecification.

More generally, across all $12$ datasets and $8$ metrics, \sigTPP~has the lowest overall mean rank among the $6$ models ($2.27$), followed by VAE ($2.62$). It has the lowest geometric relative score on $5$ of the $8$ metrics, the 2\textsuperscript{nd}-lowest on $\histITloss$ and PCD (behind DDPM), and the 4\textsuperscript{th}-lowest on CRPS (Figure~\ref{fig::normalised_score_results_all}). We also compare \sigTPP~with $5$ methods from EasyTPP~\cite{xue2024easytpp} on a subset of $6$ datasets (Appendix~\ref{sec::additional_results}). Consistent with the main benchmark, \sigTPP~reduces scores by at least $28\%$ against every EasyTPP baseline on geometric average across the available dataset-metric pairs and achieves the best overall average rank of $2.40$, followed by NHP with rank 2.94 (Table~\ref{table::easytpp_avg_rank_by_metric}).

Finally, Tables~\ref{table::pointwise_results_combined} and~\ref{table::pointwise_results_additional} illustrate why pointwise errors are not suitable primary metrics for generative TPP evaluation. The deterministic baseline \determodel~outputs a single prediction for each event time and is therefore directly optimised for pointwise accuracy rather than distributional fidelity. Although it performs poorly on most path-level and distributional metrics, in some cases by several orders of magnitude, it achieves rank $1.00$ on MAE, outperforming every distributional model on every dataset. For \MSEmean, which evaluates the predictive mean, its average ranks are $4.33$ on synthetic tasks and $5.17$ on real-world tasks. The median and mean summarise only part of the predictive law, so accurate point predictions do not establish a good fit. Hence, we do not use them as metrics for generative TPP evaluation, and instead recommend path-level distributional distances.
For forecasting-focused tasks, longer-horizon predictive metrics may also be appropriate.

\begin{figure}[!t]
\centering
\includegraphics[width=\textwidth]{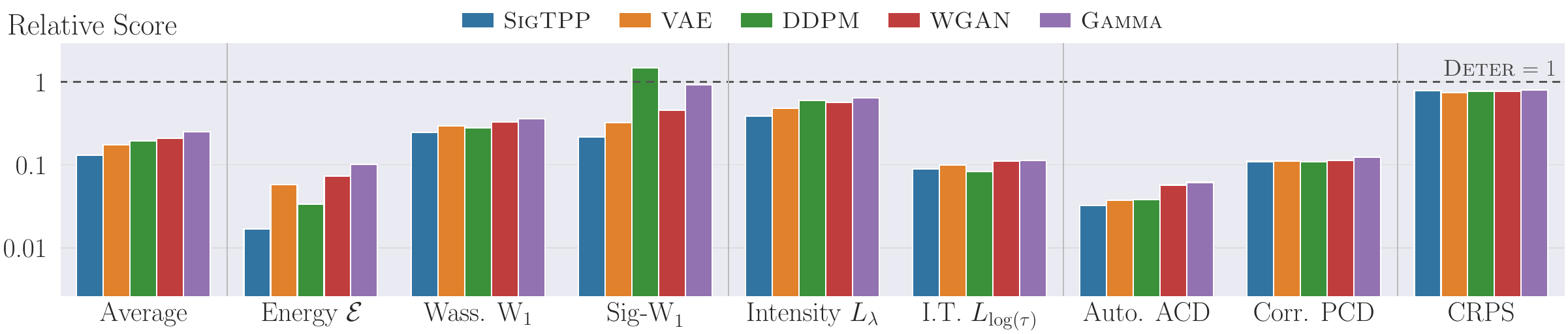}
\ratiocaption{all $12$ datasets}{The exact values are enumerated in Table~\ref{table::geo_improvement_vs_deter_grouped}.}
\label{fig::normalised_score_results_all}
\end{figure}

\subsection{Ablation of the Signature Degree}
\label{subsect::signature_degree}

In Table~\ref{table::ablation_M}, we isolate the effect of the signature degree $M$ on H3, TX, and SO, $3$ datasets with different statistical properties and sizes (see Table~\ref{table::datasets_stats}). 
Although the number of signature coordinates grows as $\mathcal{O}(d^M)$, gains depend on the dataset and metric.
On H3, increasing from $M=2$ to $M=3$ improves all $4$ metrics shown. $M=4$ gives the lowest $\mathcal{E}$ and $\WAS$, while $M=5$ gives the lowest $\sigW$ and CRPS. The identical scores at $M=6$ and $M=8$ are consistent with the diagnostic's first inactive degree of $6$ on H3 (Table~\ref{table::sig_degree_datasets}).
On TX, most of the improvement occurs from $M=2$ to $M=3$; larger values produce only small changes in $\WAS$ and CRPS. 
On SO, $\mathcal{E}$ and $\WAS$ are lowest at $M=3$ and worsen at $M=4$ and $M=5$, while CRPS is highest at $M=4$. Scores at larger $M$ are non-monotone.

No tested degree is best across all $4$ metrics on all $3$ datasets. To understand that better, note that increasing $M$ changes the set of coordinates used for training rather than simply refining the same discrepancy. 
Low-order coordinates often dominate the signature scale, while higher-order iterated integrals decay factorially for bounded variation paths, despite the exponential growth in their number~\citep{lyons2007differential}. 
Consequently, a large $M$ can introduce weakly varying or poorly conditioned features, increasing estimation noise without necessarily improving the learned path distribution. 
This is consistent with the dead-degree statistics in Table~\ref{table::sig_degree_datasets}, showing the first inactive degree at $6$ on H3. Overall, the ablation supports $M=3$ for the evaluation and suggests that using a degree higher than $8$ does not yield further improvement.

\begin{table}[h]
\centering
\footnotesize
\renewcommand{\arraystretch}{1.05}
\setlength{\tabcolsep}{1.6pt}
\setlength{\aboverulesep}{0.5pt}
\setlength{\belowrulesep}{0.5pt}
\caption{Ablation of the training signature degree $M$ for \sigTPP~on H3, TX, and SO. Values are mean test scores over $100$ bootstrap replicates. \resultlegend}
\label{table::ablation_M}
\vspace{0.4em}
\begin{adjustbox}{max width=\linewidth}
\begin{tabular}{l *{6}{>{\columncolor{gray!10}}r} @{\hspace{0.8em}} *{6}{r} @{\hspace{0.8em}} *{6}{>{\columncolor{gray!10}}r} @{\hspace{0.8em}} *{6}{r}}
\toprule
& \multicolumn{6}{c}{$\mathcal{E}$ ($\times 10^{-3}$) \ $\downarrow$}
& \multicolumn{6}{c}{$\WAS$ ($\times 10^{-2}$) \ $\downarrow$}
& \multicolumn{6}{c}{$\sigW$ ($\times 10^{-5}$)  \ $\downarrow$}
& \multicolumn{6}{c}{CRPS ($\times 10^{-1}$) \ $\downarrow$} \\
\cmidrule(lr){2-7}\cmidrule(lr){8-13}\cmidrule(lr){14-19}\cmidrule(lr){20-25}
\multicolumn{1}{c}{$M$} & \multicolumn{1}{c}{2} & \multicolumn{1}{c}{3} & \multicolumn{1}{c}{4} & \multicolumn{1}{c}{5} & \multicolumn{1}{c}{6} & \multicolumn{1}{c}{8} & \multicolumn{1}{c}{2} & \multicolumn{1}{c}{3} & \multicolumn{1}{c}{4} & \multicolumn{1}{c}{5} & \multicolumn{1}{c}{6} & \multicolumn{1}{c}{8} & \multicolumn{1}{c}{2} & \multicolumn{1}{c}{3} & \multicolumn{1}{c}{4} & \multicolumn{1}{c}{5} & \multicolumn{1}{c}{6} & \multicolumn{1}{c}{8} & \multicolumn{1}{c}{2} & \multicolumn{1}{c}{3} & \multicolumn{1}{c}{4} & \multicolumn{1}{c}{5} & \multicolumn{1}{c}{6} & \multicolumn{1}{c}{8} \\
\cmidrule(lr){2-7}\cmidrule(lr){8-13}\cmidrule(lr){14-19}\cmidrule(lr){20-25}
H3 & 192.1 &  5.21 &  \textbf{1.88} &  5.48 &  \underline{2.58} &  \underline{2.58} & 31.09 & 11.644 & \textbf{10.30} & 11.76 & \underline{10.79} & \underline{10.79} & 25.50 & \underline{0.685} & 1.344 & \textbf{0.656} & 1.845 & 1.845 & 3.369 & 2.437 & \underline{2.384} & \textbf{2.379} & 2.411 & 2.411 \\
\midrule
TX & 86.95 & \textbf{79.31} & \underline{79.46} & 79.91 & 80.39 & 79.68 & 9.491 & 8.343 & 8.270 & 8.533 & \textbf{8.163} & \underline{8.172} & 1.363 & 0.796 & 0.664 & \textbf{0.341} & \underline{0.456} & 1.627 & 1.257 & 1.250 & 1.168 & 1.166 & \underline{1.152} & \textbf{1.147} \\
\midrule
SO & 19.93 &  \textbf{8.73} & \underline{10.68} & 24.81 & 19.54 & 11.78 & 8.619 & \textbf{6.321} & 8.225 & 10.02 & 9.432 & \underline{7.050} & 0.514 & 0.980 & 0.456 & \underline{0.144} & 0.189 & \textbf{0.136} & 5.446 & \textbf{5.259} & 10.55 & 5.409 & \underline{5.375} & 6.165 \\
\bottomrule
\end{tabular}
\end{adjustbox}
\end{table}

\section{Conclusion}
This paper introduces a common pathwise framework for generative modelling and evaluation of TPPs.
The interarrival embedding $\Phi$ provides a stable, injective lift from \clag counting paths to continuous bounded-variation paths (Theorem~\ref{thrm::stability}), making signature-based tools applicable to discrete event sequences.
Building on $\Phi$, \sigTPP~optimises a global signature Wasserstein loss on whole paths, rather than a sum of local conditional losses.
The same counting-path metric $\dtppsym$ supports the energy distance ($\mathcal{E}$), $\WAS$, as well as $\sigW$ through the embedding.
It also clarifies a practical failure mode of pointwise evaluation: MAE can reward deterministic regressors even when they fail to capture distributional variation.
Across $12$ datasets, \sigTPP~achieves the best average rank on $6$ of $8$ evaluation criteria, best average relative score on $5$ of them, and outperforms every other baseline by at least $24\%$ on average across all metrics. The gains are largest on pathwise distributional criteria, while diffusion and variational baselines remain competitive on marginal calibration. The limitations of our work are discussed in Appendix~\ref{sect::limitations}.

\newpage
\section*{Acknowledgements}
NCK is supported by EPSRC grants EP/T517793/1 and EP/W524335/1. VL would like to thank all levels of support from grant EP/X031276/1 (EPSRC) and the SOFAIR Lab (UKRI).

{
\small
\bibliographystyle{plainnat}
\bibliography{biblio}
}

% \newpage
% \input{checklist.tex}

\newpage
\appendix
\section*{Appendix}

\setcounter{equation}{0}
\setcounter{figure}{0}
\setcounter{table}{0}
\setcounter{theorem}{0}
\renewcommand{\theequation}{S\arabic{equation}}
\renewcommand{\thefigure}{S\arabic{figure}}
\renewcommand{\thetable}{S\arabic{table}}
\renewcommand{\thetheorem}{S\arabic{theorem}}
\renewcommand{\theHequation}{S\arabic{equation}}
\renewcommand{\theHfigure}{S\arabic{figure}}
\renewcommand{\theHtable}{S\arabic{table}}
\renewcommand{\theHtheorem}{S\arabic{theorem}}

\section{A Pathwise Metric on Counting Paths}
\subsection{Distances induced by the Pathwise Metric on Counting Paths}
\label{subsec::distances_proba_with_dtpp}
\begin{definition}[Metric of Negative Type]
    \label{def::negative_type}
    A metric space $(X, d)$ is of negative type if, for every $n \in \mathbb{N}$, all points $x_1, \dots, x_n \in X$, and all real coefficients $a_1, \dots, a_n\in\mathbb{R}$ satisfying $\sum_{i=1}^n a_i = 0$, the following inequality holds:
    \begin{equation}
        \sum_{i,j=1}^n a_i a_j \, d(x_i, x_j) \leq 0 \, .
    \end{equation}
\end{definition}

Geometrically, a metric $d$ is of negative type if and only if the metric space $(X, \sqrt{d})$ can be embedded isometrically into a Hilbert space~\cite{schoenberg_metric_1938}, i.e. if there exists a map $\iota: X \to \mathcal{H}$ into a Hilbert space $\mathcal{H}$ such that
\begin{equation}
    d(x, y) = \norm{\iota(x) - \iota(y)}_{\mathcal{H}}^2 \, .
\end{equation}

\begin{proposition}[$L^1$ is of Negative Type~\cite{bretagnolle_lois_1967, deza_geometry_1997, lyons2013distance}]
    \label{prop::l1negtype}
    Let $(\Omega, \mathcal{F}, \mu)$ be a measure space.
    Then the metric space $(L^1(\mu), \norm{\cdot}_{L^1})$ is of negative type.
    Equivalently, for all $f_1,\dots,f_n \in L^1(\mu)$ and all $a_1,\dots,a_n \in \R$ such that $\sum_i a_i = 0$, we have
    \begin{equation}
        \sum_{i,j=1}^n a_i a_j\, \norm{f_i - f_j}_{L^1} \le 0 \, .
    \end{equation}
\end{proposition}

\begin{proposition}[Subspace Inheritance~\cite{bretagnolle_lois_1967, deza_geometry_1997}]
    \label{prop::negtypeinheritance}
    Let $(X,d)$ be a metric space of negative type, and let $Y \subset X$.
    Then the restricted metric space $(Y, d|_{Y\times Y})$ is also of negative type.
\end{proposition}

\begin{theorem}[Negative Type Property]
\label{thrm:neg_type}
The metric space $(\mathcal{N}, \dtppsym)$ is of negative type.
\end{theorem}
\begin{proof}
    Identifying each $\eta \in \mathcal N$ with its counting path $t \mapsto \eta_t$ gives an isometric embedding of $(\mathcal N,\dtppsym)$ into $L^1([0,\Tmax])$, since $\dtpp{\eta}{\xi}=\norm{\eta-\xi}_{L^1}$. Proposition~\ref{prop::l1negtype} states that $L^1$ is of negative type, and Proposition~\ref{prop::negtypeinheritance} shows that negative type is inherited by subspaces. Therefore, $(\mathcal N,\dtppsym)$ is of negative type.
\end{proof}

\subsection{Direct Integration}
\label{app::proof_tpp_metric_int}

We begin with the following exact discrete formulation of $\dtppsym$ for efficient computation.
\begin{proposition}[Discrete Computation]
\label{prop::discrete_dtpp}
    Let $\eta, \xi \in \mathcal N$ with event times $(t_k)_{k=1}^m$ and $(s_k)_{k=1}^n$. Assuming without loss of generality that $m \leq n$, the distance is given by
    \begin{equation}
        \label{eq::discrete_dtpp}
        \dtpp{\eta}{\xi} = \sum_{k=1}^{m} \abs{t_k - s_k} + (n - m)\Tmax - \sum_{k=m+1}^{n} s_{k} \, .
    \end{equation}
\end{proposition}

\begin{proof}
We prove the discrete expression from the integral representation, as written in Eq.~\ref{eq::discrete_dtpp}. Assuming $m \leq n$, we prove that
\begin{equation}
    \dtpp{\eta}{\xi}
    = \sum_{i=1}^{m} \abs{t_i - s_i}
    + (n - m) \Tmax
    - \sum_{i=m+1}^{n} s_{i} \, .
\end{equation}
Let
\begin{equation}
    f(t) := \eta_t - \xi_t, \quad I(\Tmax) := \int_0^{\Tmax} \abs{\eta_t-\xi_t}\,\mes t
    = \int_0^{\Tmax} \abs{f(t)}\,\mes t \,.
\end{equation}
$f$ records the imbalance between the number of events of $\eta$ and $\xi$ that
have occurred up to time $t$: it is constant between event times and changes by
$+1$ whenever some $t_i$ is crossed and by $-1$ whenever some $s_j$ is crossed.
Pairing events by rank and setting $t_i:=+\infty$ for $i=m{+}1,\dots,n$ (the
process $\eta$ has no further events), we may write
\begin{equation}
f(t)=\eta_t-\xi_t
     =\sum_{i=1}^{n}\bigl(\11charac_{\{t\ge t_i\}}-\11charac_{\{t\ge s_i\}}\bigr) \, .
\end{equation}
The $i$-th summand equals $+1$ for $t\in[t_i,s_i)$ and $-1$ for $t\in[s_i,t_i)$. In either case its magnitude is the indicator of $[t_i\wedge s_i,\,t_i\vee s_i)$.
Since both $(t_i)$ and $(s_i)$ are nondecreasing, two summands of opposite sign
cannot be nonzero at the same $t$: if term $k$ contributes $+1$ and term $l$
contributes $-1$ with $k<l$, then $s_l\le t<s_k$ contradicts $s_k\le s_l$ (and
symmetrically for $k>l$). Hence at every $t$ the nonzero summands share a common
sign, so the absolute value passes through the sum,
\begin{equation}
\abs{f(t)}=\sum_{i=1}^{n}\bigl|\11charac_{\{t\ge t_i\}}-\11charac_{\{t\ge s_i\}}\bigr|
      =\sum_{i=1}^{n}\11charac_{[t_i\wedge s_i,\,t_i\vee s_i)}(t) \, .
\end{equation}
Splitting off the unpaired indices $i>m$, for which $t_i=+\infty$ and hence
$[t_i\wedge s_i,\,t_i\vee s_i)\cap[0,\Tmax]=[s_i,\Tmax]$, gives, for all $t\in[0,\Tmax]$,
\begin{equation}
    \label{eq:indicator_representation}
    \abs{f(t)}
    =
    \sum_{i=1}^{m} \11charac_{[t_i \wedge s_i,\, t_i \vee s_i)}(t)
    +
    \sum_{j=m+1}^{n} \11charac_{[s_j,\, \Tmax]}(t) \, .
\end{equation}
The first sum counts rank-paired times $(t_i, s_i)$, for which exactly one of the
two events has occurred by time $t$, while the second sum accounts for the
unpaired events $s_j$ with $j>m$.
Integrating both sides of Eq.~\ref{eq:indicator_representation}, we obtain
\begin{align}
    I(\Tmax)
    &= \sum_{i=1}^{m} (t_i \vee s_i - t_i \wedge s_i)
    + \sum_{j=m+1}^{n} (\Tmax - s_j) \notag \\
    &= \sum_{i=1}^{m} |t_i - s_i|
    + (n-m)\Tmax - \sum_{j=m+1}^{n} s_j = \dtpp{\eta}{\xi} \, ,
\end{align}
which establishes Eq.~\ref{eq::discrete_dtpp} for $m \leq n$.
The case $m > n$ follows by symmetry.
\end{proof}

\subsection{Optimal Transport Interpretation}
\label{app::proof_tpp_metric_w1}
Let $\eta, \xi \in \mathcal N$ with event times $(t_k)_{k=1}^m$ and $(s_k)_{k=1}^n$ respectively. For the optimal transport argument, associate to each the point measure $\mu := \sum_{k=1}^{m} \delta_{t_k}$ and $\nu := \sum_{k=1}^{n} \delta_{s_k}$. We write $\delta_{\Tmax}$ for the Dirac measure at $\Tmax$, and use $(\cdot)^+ = \max( \cdot , 0 )$.

We establish that the Wasserstein distance~\cite{ramdas2017wasserstein} for the deterministic measures with anchor points at $\Tmax$ is indeed the $L^1$ distance between the respective counting processes:
\begin{equation}
    \WAS(\tilde\mu,\tilde\nu)
    = \int_0^{\Tmax} \abs{\eta_t-\xi_t}\,\mes t \,,
\end{equation}
where, assuming that $\eta_{\Tmax} = m, \xi_{\Tmax} = n$,
\begin{equation}
    \tilde\mu := \mu + (n-m)^+ \delta_{\Tmax} \, ,
    \qquad
    \tilde\nu := \nu + (m-n)^+ \delta_{\Tmax} \, .
\end{equation}
This establishes the correspondence between our definition and the optimal transport interpretation.
\begin{proof}
By construction, both padded measures have equal total mass. Since their supports are contained in $[0,\Tmax]$, both $\tilde\mu$ and $\tilde\nu$ have finite first moments.
The one-dimensional transport identity extends to finite positive measures of equal mass and finite first moment~\cite{villani_optimal_2009}:
\begin{equation}
    \WAS(\tilde\mu,\tilde\nu)
    \triangleq \inf_{\gamma \in \Gamma (\tilde\mu,\tilde\nu)}
    \int_{\R \times \R} \abs{x - y} \gamma (\mes x,\mes y)
    = \int_{\R} \abs{F_{\tilde\mu}(t) - F_{\tilde\nu}(t)}\,\mes t \, .
\end{equation}
For $t < \Tmax$, no padded atoms have appeared. Hence, $F_{\tilde\mu}(t)= \tilde\mu((-\infty,t]) = \mu((-\infty,t])= \eta_t,$ and analogously, $F_{\tilde\nu}(t) = \xi_t$. Thus,
\begin{equation}
    \abs{F_{\tilde\mu}(t)-F_{\tilde\nu}(t)} =
    \begin{cases}
        \abs{\eta_t-\xi_t}, & t \in [0,\Tmax ) \, , \\
        0, & t \notin [0,\Tmax) \, .
    \end{cases}
\end{equation}
Therefore, we get the equality
\begin{equation}
    \WAS(\tilde\mu,\tilde\nu) = \int_{\R} \abs{F_{\tilde\mu}(t)-F_{\tilde\nu}(t)}\,\mes t = \int_0^{\Tmax} \abs{\eta_t-\xi_t} \,\mes t \, .
\end{equation}
\end{proof}

\subsection{Counterexample for Strong Negative Type}
\label{app::counter_ex_strong}

The TPP distance is of negative type but not of strong negative type. Consider four realised TPPs with two events each:
\begin{equation}
    \xi_1=(0,2),\qquad \xi_2=(1,3),\qquad \xi_3=(0,3),\qquad \text{and} \qquad \xi_4=(1,2) \, .
\end{equation}
Then, with coefficients $a_1=a_2=1$ and $a_3=a_4=-1$ (such that $\sum_i a_i=0$) one has
\begin{align}
    \sum_{i,j=1}^4 a_i a_j\, \dtpp{\xi_i}{\xi_j}
    &= 2\,\dtpp{\xi_1}{\xi_2}
    + 2\,\dtpp{\xi_3}{\xi_4} \notag\\
    &\quad
    - 2 \cdot \big(
    \dtpp{\xi_1}{\xi_3}
    + \dtpp{\xi_1}{\xi_4}
    + \dtpp{\xi_2}{\xi_3}
    + \dtpp{\xi_2}{\xi_4}
    \big) \notag\\
    &= 2\cdot 2 + 2\cdot 2 - 2(1+1+1+1) \notag\\
    &= 0 \, ,
\end{align}
with the $\xi_i$ not all equal, which violates strong negative type. In general, $L^1$ spaces are of negative type but not of strong negative type~\cite{lyons2013distance}.

\section{Proofs of Theorem~\ref{thrm::stability}}
\label{app::proofs}

Let $\mathcal{N}_\delta$ denote the counting paths whose interarrival times, including the terminal interval, are all at least $\delta>0$. Following Theorem~\ref{thrm::stability}, we prove the following statements.
\begin{enumerate}
    \item The map $\Phi$ is Lipschitz (Theorem~\ref{thrm::Phi_lipschitz}).
    \item The map $\Phi$ is injective on $\mathcal N$, hence a bijection between $\mathcal N$ and its image $\Phi(\mathcal N)$ (Theorem~\ref{thrm::bijection}).
    \item The inverse map $\Psi = \Phi|_{\mathcal N}^{-1}$ is continuous on $\Phi(\mathcal N_\delta)$ for each $\delta>0$. More precisely, it satisfies a $\tfrac12$-\holder bound on the separated subset $\Phi(\mathcal N_\delta)$ defined in Appendix~\ref{app::proofs::cont} (Theorem~\ref{thrm::holder_main}). Continuity fails on all of $\Phi(\mathcal N)$. See Remark~\ref{rmq::separation_necessary}.
\end{enumerate}

\subsection{\texorpdfstring{$\Phi$}{Phi} is Lipschitz}
\label{app::proofs::lip}

\begin{lemma}[Affine $L^1$ Bound]\label{lemma::affine}
    Let $a<b$ and let $h\colon[a,b]\to\R$ be an affine function. Then
    \begin{equation}
        \int_a^b \abs{h(t)}\mes{t}
        \le
        \frac{b-a}{2}\bigl(\abs{h(a)}+\abs{h(b)}\bigr) \,.
    \end{equation}
\end{lemma}

\begin{proof}
    The composition of a convex function with an affine function is convex. Since $t\mapsto\abs{h(t)}$ is the composition of the convex function $\abs{\cdot}$ with the affine function $h$, it is convex on $[a,b]$. Every convex function on an interval is bounded above by the chord through its endpoint values:
    \begin{equation}
        \abs{h(t)}
        \le
        \frac{b-t}{b-a}\,\abs{h(a)}
        +
        \frac{t-a}{b-a}\,\abs{h(b)} \,,
        \qquad t\in[a,b] \,.
    \end{equation}
    Integrating over $[a,b]$ gives the result.
\end{proof}

We recall that a right-continuous function of bounded variation induces a Lebesgue-Stieltjes (total variation) measure $\abs{D y}$. See e.g.~\citet{folland1999real}, Sections~1.5 and~3.5.

\begin{lemma}[Bounded Variation Time-Warp Inequality]
    \label{lemma::bv_time_warp}
    Let $y:[0,\Tmax]\to\mathbb{R}$ be (continuous\footnote{Continuity is not required for the proof but makes the argument cleaner, and we do not require a more general statement for our proofs.} and) of bounded variation, with total variation
    $\mathrm{TV}(y):=\mathrm{TV}(y;\,[0,\Tmax])$.
    Let $\lambda:[0,\Tmax]\to[0,\Tmax]$ be measurable and set $\norm{\lambda-\mathrm{Id}}_{\infty}
    :=\sup_{t\in[0,\Tmax]}|\lambda(t)-t|.$
    Then
    \begin{equation}
        \label{eq::bound_tv_timerepam}
        \norm{y(\lambda(t))-y(t)}_{L^1}
        \le
        2\norm{\lambda-\mathrm{Id}}_{\infty}\,\mathrm{TV}(y) \,.
    \end{equation}
\end{lemma}

\begin{proof}
    Let $\abs{D y}$ denote the total variation measure of $y$, i.e. the Lebesgue-Stieltjes measure associated with the total variation function $\mathrm{TV}(y;\,[0,t])$.
    For any $a,b\in[0,\Tmax]$,
    \begin{equation}
        |y(b)-y(a)|
        \le
        \abs{D y}\big([a\wedge b,\; a\vee b]\big) \,.
    \end{equation}

    Set $\delta:=\norm{\lambda-\mathrm{Id}}_{\infty}$ (smaller than or equal to $\Tmax$). For each $t\in[0,\Tmax]$ we have $\lambda(t)\in[t-\delta,\,t+\delta]\cap[0,\Tmax]$, so that
    $[t \wedge \lambda(t),\; t \vee \lambda(t)] \subseteq [t-\delta,\,t+\delta]\cap[0,\Tmax]$.
    By monotonicity of measures,
    \begin{equation}
        |y(\lambda(t))-y(t)|
        \leq
        \abs{D y}\big([t \wedge \lambda(t),\; t \vee \lambda(t)] \big)
        \leq
        \abs{D y}\big([t-\delta,\,t+\delta]\cap[0,\Tmax]\big) \,.
    \end{equation}
    Integrating over $[0,\Tmax]$ and writing the right-hand side as a double integral,
    \begin{align}
        \int_0^{\Tmax} \abs{y(\lambda(t))-y(t)} \mes t
        & \le
        \int_0^{\Tmax} \int_{[0,\Tmax]} \11charac_{\{|u-t|\le\delta\}}\, \abs{D y}(\mes u) \mes t \,.
    \end{align}
    The integrand is nonnegative and jointly measurable
    (the set $\{(t,u):|u-t|\le\delta\}$ is closed in $\mathbb{R}^{2}$),
    so Tonelli's theorem gives the desired inequality
    \begin{align}
        \int_0^{\Tmax} \abs{y(\lambda(t))-y(t)}\, \mes t
        & \le
        \int_{[0,\Tmax]} \biggl(\int_0^{\Tmax} \11charac_{\{|u-t|\le\delta\}}\,\mes t\biggr) \abs{D y}(\mes u) \notag \\
        & \le
        \int_{[0,\Tmax]} 2\delta\, \abs{D y}(\mes u)
        =
        2\delta\,\mathrm{TV}(y) \,.
    \end{align}
\end{proof}

\begin{lemma}[Interarrival Telescoping Bound]\label{lemma::interarrival}
    Let $0 = t_0 < t_1 < \cdots < t_m < t_{m+1} = \Tmax$ and $0 = s_0 < s_1 < \cdots < s_n < s_{n+1} = \Tmax$ be two augmented grids (in the sense of Convention~\ref{conv::boundary}) with $m \le n$. Define the interarrival times
    \begin{equation}
        \tau_k := t_k - t_{k-1},\quad k=1,\dots,m{+}1,
        \qquad
        \varsigma_k := s_k - s_{k-1},\quad k=1,\dots,n{+}1\,,
    \end{equation}
    and set $\tau_0 := \varsigma_0 := 0$. Extend the shorter grid by padding: $t_i := \Tmax \,$ for $i = m{+}1,\dots,n$, so that both sequences $(t_i)_{i=1}^n$ and $(s_i)_{i=1}^n$ have the same length.

    Then the first $m{+}2$ interarrival differences are controlled by the event-time differences:
    \begin{equation}
        \sum_{k=0}^{m+1} \abs{\tau_k - \varsigma_k} \le 2\sum_{i=1}^{n}\abs{t_i - s_i} \,.
    \end{equation}
\end{lemma}

\begin{proof}
    For $k = 1, \dots, m{+}1$, the telescoping identity is
    \begin{equation}
        \tau_k - \varsigma_k
        = (t_k - t_{k-1}) - (s_k - s_{k-1})
        = (t_k - s_k) - (t_{k-1} - s_{k-1}) \, .
    \end{equation}
    The triangle inequality gives
    \begin{equation}
        \abs{\tau_k - \varsigma_k}
        \le
        \abs{t_k - s_k} + \abs{t_{k-1} - s_{k-1}} \,.
    \end{equation}
    Summing over $k = 1, \dots, m{+}1$ yields
    \begin{equation}
        \sum_{k=1}^{m+1} \abs{\tau_k - \varsigma_k}
        \le
        \sum_{k=1}^{m+1} \abs{t_k - s_k}
        +
        \sum_{k=0}^{m} \abs{t_k - s_k} \,.
    \end{equation}
    Since $\tau_0 = \varsigma_0 = 0$ and $t_0 = s_0 = 0$, the boundary terms vanish. Noting that all summands are non-negative, we obtain
    \begin{equation}
        \sum_{k=0}^{m+1} \abs{\tau_k - \varsigma_k}
        \le
        2 \sum_{k=1}^{m+1} \abs{t_k - s_k}
        \le
        2 \sum_{k=1}^{n} \abs{t_k - s_k} \,.
    \end{equation}
    When $m=n$, the sum includes the extra term $\abs{t_{n+1} - s_{n+1} } = \Tmax - \Tmax = 0$.
\end{proof}

\begin{lemma}[Grid Mismatch Controlled by a Time-Change]
    \label{lemma::time_change_bound}
    Let $\xi\in \mathcal N$ with $n$ events and augmented grid
    $0=s_0<s_1<\cdots<s_{n+1}=\Tmax$,
    and let $y=\Phi(\xi)\in \mathcal C$ be the piecewise-linear interpolant through
    $(s_k,\varsigma_k)_{k=0}^{n+1}$, where $\varsigma_0:=0$ and $\varsigma_k:=s_k-s_{k-1}$ for $k=1,\dots,n{+}1$ are the
    interarrival times (Convention~\ref{conv::boundary}).

    Let $0=t_0<t_1<\cdots<t_{m+1}=\Tmax$ be a second augmented grid with
    $m\le n$.
    Define:
    \begin{enumerate}[label=(\alph*)]
\item the piecewise-linear map $\lambda\colon[0,\Tmax]\to[0,s_{m+1}]$ by $\lambda(t_k)=s_k$ for $k=0,\dots,m{+}1$, affine on each $[t_k,t_{k+1}]$.
\item the re-gridded interpolant $\tilde y$ on the $t$-grid: $\tilde y(t_k):=\varsigma_k$ for $k=0,\dots,m{+}1$, affine between nodes.
    \end{enumerate}
    Extend the shorter grid by padding $t_i := \Tmax$ for $i = m+1,\dots,n$. Then:
    \begin{enumerate}[label=(\roman*)]
\item $\tilde y(t) = y(\lambda(t))$ for all $t\in[0,\Tmax]$.
\item $\norm{\tilde y - y}_{L^1}
\le 4\Tmax\,\norm{\lambda-\mathrm{Id}}_\infty$.
\item $\norm{\lambda - \mathrm{Id}}_\infty \le \sum_{k=1}^n \abs{t_k - s_k}$.
    \end{enumerate}
    Consequently,
    \begin{equation}\label{eq:grid:new}
        \norm{\tilde y - y}_{L^1}
        \le
        4\Tmax\sum_{k=1}^n \abs{t_k - s_k} \,.
    \end{equation}
\end{lemma}

\begin{proof}
    \noindent\textit{Part~(i): Identity $\tilde{y}=y\circ\lambda$.}\;
    Fix $k\in\{0,\dots,m\}$ and $t\in[t_k,t_{k+1}]$. Since the grids are strictly increasing, $t_{k+1}-t_k>0, k \leq m$ and $s_{k+1}-s_k>0, k \leq n$. The fact that the grid is strictly increasing also ensures that on the interval $[t_k,t_{k+1}]$, the map $\lambda$ is strictly increasing and affine with $\lambda(t_k)=s_k$ and $\lambda(t_{k+1})=s_{k+1}$, hence
    \begin{equation}
        \lambda\bigl([t_k,t_{k+1}]\bigr)=[s_k,s_{k+1}] \,,
    \end{equation}
    so in particular $\lambda(t)\in[s_k,s_{k+1}]$ and $y(\lambda(t))$ is given by the $k$-th linear piece of $y$.

    By definition of $\tilde{y}$,
    \begin{equation}
        \tilde{y}(t) = \varsigma_k + \frac{\varsigma_{k+1}-\varsigma_k}{t_{k+1}-t_k}(t-t_k) \,.
    \end{equation}
    Since $\lambda$ is linear on $[t_k,t_{k+1}]$,
    \begin{equation}
        \lambda(t)=s_k+\frac{s_{k+1}-s_k}{t_{k+1}-t_k}(t-t_k) \,,
        \qquad\text{equivalently}\qquad
        \frac{\lambda(t)-s_k}{s_{k+1}-s_k}=\frac{t-t_k}{t_{k+1}-t_k} \,.
    \end{equation}
    Substituting this into the expression for $y$ on $[s_k,s_{k+1}]$ yields
    \begin{equation}
        y(\lambda(t))
        =
        \varsigma_k + \frac{\varsigma_{k+1}-\varsigma_k}{s_{k+1}-s_k}\bigl(\lambda(t)-s_k\bigr)
        =
        \varsigma_k + \frac{\varsigma_{k+1}-\varsigma_k}{t_{k+1}-t_k}(t-t_k)
        =
        \tilde{y}(t) \,,
    \end{equation}
    which proves~(i).

    \medskip\noindent
    \textit{Part~(ii): $L^1$-bound via Lemma~\ref{lemma::bv_time_warp}.}\;
    The map $\lambda:[0,\Tmax]\to[0,\Tmax]$ is piecewise linear, hence measurable and $y$ is piecewise linear hence continuous and of bounded variation, so all hypotheses of Lemma~\ref{lemma::bv_time_warp} are fulfilled. Eq.~\ref{eq::bound_tv_timerepam} together with Part~(i) gives
    \begin{equation}
        \norm{\tilde y - y}_{L^1}
        = \norm{y\circ\lambda - y}_{L^1}
        \le 2\,\mathrm{TV}(y)\,\norm{\lambda - \mathrm{Id}}_\infty \,.
    \end{equation}

    Since $y$ is piecewise linear with nodes $\varsigma_0=0$, $\varsigma_k\ge 0$, and $\sum_{k=1}^{n+1}\varsigma_k=\Tmax$,
    \begin{equation}
        \mathrm{TV}(y)
        =
        \sum_{k=0}^{n}|\varsigma_{k+1}-\varsigma_k|
        \le
        \sum_{k=0}^{n}(\varsigma_{k+1}+\varsigma_k)
        \leq
        2\sum_{k=1}^{n+1}\varsigma_k
        =
        2\Tmax \,.
    \end{equation}
    This yields~(ii):
    \begin{equation}
        \norm{\tilde{y}-y}_{L^1} \le 4 \Tmax\,\norm{\lambda-\mathrm{Id}}_\infty \,.
    \end{equation}

    \medskip\noindent
    \textit{Part~(iii): Control of the time change.}\;
    Since $t\mapsto \lambda(t)-t$ is affine on each interval $[t_k,t_{k+1}]$, the function $|\lambda(t)-t|$ is convex on each such interval. Therefore, its maximum over $[0,\Tmax]$ is attained at a node. Because all terms are nonnegative, we may bound the maximum by the sum
    \begin{equation}
        \norm{\lambda-\mathrm{Id}}_\infty= \max_{0\le k\le m+1}|s_k-t_k| \leq \sum_{k=0}^{m+1}|s_k-t_k| \,.
    \end{equation}
    The term at $k=0$ vanishes since $s_0=t_0=0$. Finally, using the padding convention $t_i:=\Tmax$ for $i=m+1,\dots,n$, the sum $\sum_{i=1}^{n}|t_i-s_i|$ contains all nonzero terms $|s_k-t_k|$ for $k=1,\dots,m+1$. When $m=n$, the omitted endpoint term is $|t_{n+1}-s_{n+1}|=0$. It also contains additional nonnegative terms for $i=m+2,\dots,n$. Hence,
    \begin{equation}
        \norm{\lambda-\mathrm{Id}}_\infty \leq \sum_{k=1}^{m+1}|s_k-t_k|
        \le
        \sum_{i=1}^{n}|t_i-s_i| \,,
    \end{equation}
    which proves~(iii).

    \medskip\noindent
    \textit{Conclusion.}\;
    Combining~(ii) and~(iii) gives
    \begin{equation}
        \norm{\tilde{y}-y}_{L^1}
        \le
        4 \Tmax\sum_{i=1}^{n}|t_i-s_i| \,.
    \end{equation}
\end{proof}

\begin{remark}[On the Map $\lambda$]
When $m < n$, the map $\lambda$ satisfies $\lambda(\Tmax) = s_{m+1} < \Tmax$, so $\lambda$ is not surjective onto $[0,\Tmax]$. It aligns the grids only up to the $(m+1)$-th node and compresses the remaining time $[s_{m+1}, \Tmax]$ to the single point $\lambda(\Tmax)$. When $m = n$, $\lambda$ is a bijection of $[0,\Tmax]$.
\end{remark}

\begin{theorem}[Global Lipschitz continuity of $\Phi$]
    \label{thrm::Phi_lipschitz}
    Equip $\mathcal N$ with $\dtppsym$ and $\mathcal C$ with $d_1(x,y):=\int_0^{\Tmax}\abs{x(t)-y(t)}\mes{t}$. Then $\Phi\colon(\mathcal N,\dtppsym)\to(\mathcal C,d_1)$ is globally Lipschitz:
    \begin{equation}
        d_1\bigl(\Phi(\eta),\Phi(\xi)\bigr) \le 5\Tmax\,\dtpp{\eta}{\xi} \qquad\text{for all }\eta,\xi\in\mathcal N\,.
    \end{equation}
\end{theorem}

\begin{proof}
    Let $\eta \in \mathcal{N}$ have $m$ events at times $(t_k)_{k=1}^m$ and $\xi \in \mathcal{N}$ have $n$ events at times $(s_k)_{k=1}^n$, with $m\le n$ (without loss of generality, by symmetry of $\dtppsym$). Write $x:=\Phi(\eta)$, $y:=\Phi(\xi)$.

    Construct the re-gridded interpolant $\tilde y$ on the $t$-grid with nodal values $\tilde y(t_k):=\varsigma_k$ for $k=0,\dots,m{+}1$, as in Lemma~\ref{lemma::time_change_bound}. By the triangle inequality,
    \begin{equation}\label{eq:tri:new}
        \norm{x-y}_{L^1}
        \le
        \underbrace{\norm{x-\tilde y}_{L^1}}_{\text{(I): vertical}}
        +
        \underbrace{\norm{\tilde y-y}_{L^1}}_{\text{(II): horizontal}} \,.
    \end{equation}
    \noindent\textit{Bound~(I): Vertical mismatch (same grid, different nodal values).}\;
    On each $[t_k,t_{k+1}]$, $x-\tilde y$ is affine. By Lemma~\ref{lemma::affine},
    \begin{equation}
        \int_{t_k}^{t_{k+1}}\abs{x(t)-\tilde y(t)}\mes{t}
        \le
        \frac{\Delta_k}{2}
        \bigl(\abs{\tau_k-\varsigma_k}+\abs{\tau_{k+1}-\varsigma_{k+1}}\bigr) \,,
    \end{equation}
    where $\Delta_k:=t_{k+1}-t_k$. Summing over $k=0,\dots,m$, each $\abs{\tau_j-\varsigma_j}$ appears with coefficient at most $\Tmax/2$ (since $\Delta_{j-1}+\Delta_j\le\Tmax$). Hence
    \begin{equation}
        \norm{x-\tilde y}_{L^1}
        \le
        \frac{\Tmax}{2}\sum_{k=0}^{m+1}\abs{\tau_k-\varsigma_k}
        \le
        \Tmax\sum_{k=1}^n\abs{t_k-s_k}
        =
        \Tmax\;\dtpp{\eta}{\xi} \,,
    \end{equation}
    where the second inequality applies Lemma~\ref{lemma::interarrival}.

    \medskip\noindent
    \textit{Bound~(II): Horizontal mismatch (different grids).}\;
    Recall that $y=\Phi(\xi)$ is the piecewise-linear interpolant on the grid $0=s_0<s_1<\cdots<s_{n+1}=\Tmax$ with nodal values $y(s_k)=\varsigma_k$, and that $\tilde y$ is the piecewise-linear interpolant on $0=t_0<t_1<\cdots<t_{m+1}=\Tmax$ with the same nodal values $\tilde y(t_k)=\varsigma_k$ for $k=0,\dots,m{+}1$.

    Define the piecewise-linear increasing map $\lambda:[0,\Tmax]\to[0,\Tmax]$ by $\lambda(t_k)=s_k$ for $k=0,\dots,m{+}1$, linear on each $[t_k,t_{k+1}]$. $\lambda$ is the unique increasing piecewise-linear map sending until index $m$ each knot $t_k$ to $s_k$. Then Lemma~\ref{lemma::time_change_bound} gives $\tilde y(t)=y(\lambda(t))$ for all $t$ and
    \begin{equation}
        \norm{\tilde y - y}_{L^1}
        \le
        4\Tmax\sum_{k=1}^n\abs{t_k-s_k}
        =
        4\Tmax\;\dtpp{\eta}{\xi} \,.
    \end{equation}

    \medskip\noindent
    \textit{Conclusion.}\;
    Combining Eq.~\ref{eq:tri:new} with Bounds~(I) and~(II):
    \begin{equation}
        d_1\bigl(\Phi(\eta),\Phi(\xi)\bigr)
        \le
        (1+4)\,\Tmax\;\dtpp{\eta}{\xi}
        =
        5\Tmax\;\dtpp{\eta}{\xi} \,.
    \end{equation}
\end{proof}

\begin{remark}[Sharpness of the Lipschitz Constant]
The constant $5\Tmax$ in Theorem~\ref{thrm::Phi_lipschitz} is not sharp. Several intermediate bounds (notably $\max\le\sum$ in Lemma~\ref{lemma::time_change_bound}(iii)) are deliberately coarse. The structurally relevant feature is the linear scaling $\mathcal O(\Tmax)$, which is natural: the $L^1$ metrics on both spaces integrate over $[0,\Tmax]$.
\end{remark}

\subsection{\texorpdfstring{$\Phi$}{Phi} is a Bijection}
\label{app::proofs::bij}

\begin{definition}[Backshift Transform]
    \label{def::backshift_transform}
    For $x\in\mathcal C$, the backshift transform of $x$ is
    \begin{equation}
        f_x(t) := t - x(t) \,,\qquad t \in [0, \Tmax] \,.
    \end{equation}
\end{definition}

When $x = \Phi(\eta)$ for some $\eta\in\mathcal N$ with event times $0<t_1<\cdots<t_m<\Tmax$ and interarrival times $\tau_k = t_k - t_{k-1}$ (Convention~\ref{conv::boundary}), the key observation is that $f_x(t_k) = t_k - \tau_k = t_{k-1}$. That is, $f_x$ maps each event time to its predecessor. This is the basis for the level-set recovery of Lemma~\ref{lemma::level_set_recovery}.

Figure~\ref{fig::inverse} illustrates the level-set recursion for a single path. This construction detail supports the inverse proof and is included in the appendix rather than the main overview.

\begin{figure}
    \centering
    \includegraphics[width=0.5\linewidth]{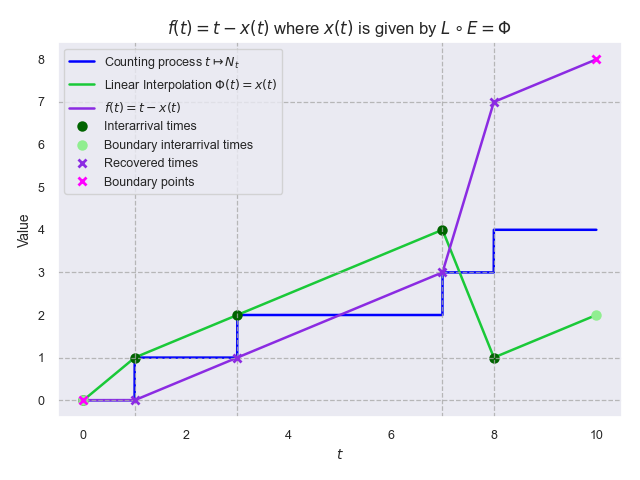}
    \caption{Recovery of event times with $f_x(t)=t-x(t)$ for $x=\Phi(\eta)$ by iterative level-hitting. Here, the TPP $\eta$ is constructed artificially with times $t_k = [1,3,7,8]$, interarrival times $\tau_k = [1,2,4,1]$, and $\Tmax = 10$. Starting from $t_0 = 0$, each event time is obtained as $t_k=\sup\{t\in(t_{k-1},\Tmax): f_x(t)=t_{k-1}\}$. Horizontal dashed lines denote thresholds $t_{k-1}$, and vertical dashed lines mark the corresponding hit times $t_k$. The figure illustrates the constructive recovery underlying the bijectivity of $\Phi$.}
    \label{fig::inverse}
\end{figure}

\begin{lemma}[Level-Set Recovery]
    \label{lemma::level_set_recovery}
    Let $\eta\in\mathcal N$ have event times $0 < t_1 < \cdots < t_m < \Tmax$, augmented by $t_0:=0$ and $t_{m+1}:=\Tmax$ (Convention~\ref{conv::boundary}), interarrival times $\tau_k := t_k - t_{k-1}$, and set $x := \Phi(\eta)$. If $m = 0$ then $f_x \equiv 0$ on $[0,\Tmax]$. Assume henceforth $m \ge 1$. Then the backshift transform $f_x(t) = t - x(t)$ satisfies the following properties.
    \begin{enumerate}[label=(\roman*)]
        \item $f_x$ is continuous and piecewise-linear on $[0,\Tmax]$.
        \item $f_x(0)=0$, $f_x(t_k) = t_{k-1}$ for $k=1,\dots,m$, and $f_x(\Tmax)=t_m$.
        \item $f_x\equiv 0$ on $[0,t_1]$.
        \item $f_x$ is strictly increasing on $[t_1, \Tmax]$.
        \item For $k = 1,\dots,m$, $t_k = \sup\{t \in (t_{k-1}, \Tmax) : f_x(t) = t_{k-1}\}$.
    \end{enumerate}
\end{lemma}

\begin{proof}
    \textit{(i)}\; Since $x \in \mathcal{C}$ is continuous and piecewise-linear on the grid $(t_k)_{k=0}^{m+1}$ and the map $t \mapsto t$ is linear, their difference $f_x$ inherits the same properties.

    \textit{(ii)}\; At each event time $t_k$,
    \begin{equation}
        x(t_k) = \tau_k = t_k - t_{k-1} \,,
    \end{equation}
    hence
    \begin{equation}
        f_x(t_k) = t_k - \tau_k = t_{k-1} \,.
    \end{equation}
    At the origin, $x(0) = \tau_0 = 0$, so $f_x(0) = 0$.

    \textit{(iii)}\;
    On $[0, t_1]$, $x$ has slope equal to $1$ by definition and is linear with $x(0)=0$. So $x(t) = t$ and $f_x \equiv 0$.

    \textit{(iv)}\; On each $[t_k, t_{k+1}]$ for $k = 1,\dots,m$, the function $f_x$ is affine with
    \begin{equation}
        f_x(t_{k+1}) - f_x(t_k) = t_k - t_{k-1} = \tau_k > 0 \,,
    \end{equation}
    where strict positivity follows from simplicity of $\eta$ (i.e.\ $\tau_k > 0$ for $k = 1,\dots,m$). Since $f_x$ is affine and strictly increasing on each subinterval and continuous across nodes, it is strictly increasing on all of $[t_1, \Tmax]$.

    \textit{(v)}\; For $k = 1$: by (iii), $f_x \equiv 0$ on $[0, t_1]$, and by (iv), $f_x > 0$ on $(t_1, \Tmax]$. Hence $\{t \in (0, \Tmax) : f_x(t) = 0\} = (0, t_1]$ and $\sup = t_1$. For $k \geq 2$: by (iv), $f_x$ is strictly increasing on $[t_{k-1}, \Tmax] \subset [t_1, \Tmax]$. Since $f_x(t_k) = t_{k-1}$ and $f_x$ is strictly increasing on this interval, $t_k$ is the unique solution of $f_x(t) = t_{k-1}$ in $[t_{k-1}, \Tmax]$, hence the supremum.
\end{proof}

\begin{lemma}[Inverse Recovery]
    \label{lemma::inversemap}
    For $x \in \Phi(\mathcal{N})$, define the sequence $(t_k)_{k \geq 0}$ by
    \begin{equation}
        \label{eq:inverse-recursion}
        t_0 := 0,
        \qquad
        t_k := \sup\bigl\{ t \in (t_{k-1}, \Tmax) : f_x(t) = t_{k-1} \bigr\},
        \quad k \ge 1 \,,
    \end{equation}
    and let $m\ge 0$ be the first index for which the level set $\{ t \in (t_m, \Tmax) : f_x(t) = t_m \}$ is empty or has supremum equal to $\Tmax$. Then:
    \begin{enumerate}[label=(\roman*)]
        \item The recovered sequence satisfies $0 < t_1 < t_2 < \cdots < t_m < \Tmax$ and $m < \infty$. Moreover, $m = 0$ if and only if $x(t)=t$ for all $t$, the embedding of the eventless path.

        \item The map $\Psi : \Phi(\mathcal{N}) \to \mathcal{N}$ defined by
        \begin{equation}\label{eq:Psi-def}
            \Psi(x)_t := \#\{k \in \{1,\dots,m\} : t_k \le t\} \,,
        \end{equation}
        is a well-defined element of $\mathcal N$.

       \item If $x = \Phi(\eta)$ for some $\eta\in\mathcal N$ with event times $(s_k)_{k=1}^n$, then $m = n$ and $t_k = s_k$ for $k=1,\dots,n$. Equivalently, $\Psi(\Phi(\eta)) = \eta$.
    \end{enumerate}
\end{lemma}

\begin{proof}
    Fix $\eta \in \mathcal{N}$ with event times $0 < s_1 < \cdots < s_n < \Tmax$, augmented by $s_0:=0$ and $s_{n+1}:=\Tmax$ (Convention~\ref{conv::boundary}), interarrival times $\varsigma_k := s_k - s_{k-1}$, and set $x := \Phi(\eta)$.
    If $n = 0$, then $x(t) = t$ and $f_x \equiv 0$, so the index-$0$ level set is $(0,\Tmax)$, whose supremum is $\Tmax$: the recursion stops immediately with $m = 0$ and $\Psi(x) \equiv 0 = \eta$. Assume henceforth $n \ge 1$.

    We show by induction that the recursion in Eq.~\ref{eq:inverse-recursion} recovers $t_k = s_k$ for every $k$. Well-definedness of $\Psi$ is then immediate.

    \medskip\noindent
    \textit{Part~(i): Strict monotonicity and finiteness.}\;
    \emph{Base case ($k=1$).}
    On $[0,s_1]$, the path $x$ interpolates linearly from $(0,0)$ to $(s_1,\varsigma_1)=(s_1,s_1)$, hence $x(t)=t$ and $f_x(t)=0$ for all $t \in [0,s_1]$. By Lemma~\ref{lemma::level_set_recovery}(iv), $f_x$ is strictly increasing on $[s_1,\Tmax]$, so $f_x(t)>0$ for all $t>s_1$.
    Therefore
    \begin{equation}
        \{ t \in (0,\Tmax) : f_x(t)=0 \} = (0,s_1] \,,
    \end{equation}
    and the supremum yields $t_1=s_1$.

    \emph{Inductive step.}
    Assume $t_j=s_j$ for all $j \le k-1$. By Lemma~\ref{lemma::level_set_recovery}, the equation $f_x(t)=s_{k-1}$ has a unique solution in $(s_{k-1},\Tmax)$, namely $s_k$. Hence $t_k = \sup\{ s_k \} = s_k,$ and in particular $t_k>t_{k-1}$. For each $k=1,\dots,n$ the level set has supremum $s_k<\Tmax$, so the stopping rule does not trigger before index $n$.

    \emph{Termination.}
    It remains to show that no time is recorded after $t_n=s_n$. By Lemma~\ref{lemma::level_set_recovery}(ii) to (iv), $f_x\le f_x(s_n)=s_{n-1}<s_n$ on $[0,s_n]$, and $f_x$ is strictly increasing on $[s_n,\Tmax]$ with $f_x(\Tmax)=s_n$. Hence $f_x(t)<s_n$ for every $t\in(s_n,\Tmax)$. The level set $\{ t \in (s_n,\Tmax) : f_x(t)=s_n \}$ is therefore empty, and the recursion terminates with $m=n<\infty$.

    \medskip\noindent
    \textit{Part~(ii): Well-definedness of $\Psi(x)$.}\;
    By (i), the sequence $(t_k)_{k=1}^m$ is finite and strictly increasing. The process defined in Eq.~\ref{eq:Psi-def} is therefore nondecreasing, \clag, has unit jumps only, satisfies $\Psi(x)_0=0$, and has $\Psi(x)_{\Tmax}=m<\infty$. Hence $\Psi(x) \in \mathcal{N}$. When $m=0$, $\Psi(x)\equiv 0$ is the eventless path, which lies in $\mathcal N$ trivially.
\end{proof}

\begin{theorem}[Bijective Embedding]
    \label{thrm::bijection}
    The following identities hold.
    \begin{enumerate}[label=(\roman*)]
        \item $\Psi \circ \Phi = \mathrm{Id}_{\mathcal{N}}$.
        \item $\Phi \circ \Psi = \mathrm{Id}_{\Phi(\mathcal{N})}$.
    \end{enumerate}
    In particular, $\Phi : \mathcal{N} \to \Phi(\mathcal{N})$ is a bijection with inverse $\Psi$.
\end{theorem}

\begin{proof}
    We continue with the notation of the previous lemmas. First we prove $\Psi \circ \Phi = \mathrm{Id}_{\mathcal{N}}$. Let $\eta \in \mathcal{N}$ have event times $(t_k)_{k=1}^m$. Lemma~\ref{lemma::inversemap} shows that applying the recursion to $x=\Phi(\eta)$ recovers all event times exactly, with no extra times. Therefore,
    \begin{equation}
        \Psi(\Phi(\eta))_t = \#\{k \in \{1,\dots,m\} : t_k \le t\} = \eta_t \,, \qquad t \in [0,\Tmax] \,.
    \end{equation}

    Next we prove $\Phi \circ \Psi = \mathrm{Id}_{\Phi(\mathcal{N})}$. Let $x \in \Phi(\mathcal{N})$, so $x=\Phi(\eta)$ for some $\eta \in \mathcal{N}$ with event times $(t_k)_{k=1}^m$. By the previous statement, $\Psi(x)$ recovers exactly these event times. Reapplying $\Phi$ yields the same piecewise-linear interpolation through
    $(t_k,\tau_k)$, hence $\Phi(\Psi(x))=x$.

    The maps $\Phi$ and $\Psi$ are therefore mutual inverses, and
    $\Phi : \mathcal{N} \to \Phi(\mathcal{N})$ is a bijection.
\end{proof}

\subsection{\texorpdfstring{$\Psi$}{Psi} is Continuous on Each Sieve \texorpdfstring{$\Phi(\mathcal N_\delta)$}{Phi(N\_delta)}}
\label{app::proofs::cont}
\textbf{\holder stability of the inverse.}\;
Fix $\delta > 0$ and let
\begin{equation}
    \mathcal{N}_\delta := \bigl\{\eta \in \mathcal{N} :
    \tau_k \ge \delta \text{ for all } k=1,\dots,m+1\bigr\} \,,
\end{equation}
where $m = \eta_{\Tmax}$ and $\tau_1,\dots,\tau_{m+1}$ are the interarrival times
(Convention~\ref{conv::boundary}).

The separation parameter $\delta$ excludes near-simultaneous events, which cause the embedding's slopes to become indistinguishable. In practice, temporal data recorded at finite precision satisfies $\eta\in\mathcal N_\delta$ for some resolution-dependent~$\delta$. The Lipschitz bound and bijectivity hold without any separation assumption.

\begin{remark}[Separation is necessary]
    \label{rmq::separation_necessary}
    The inverse $\Psi$ is not continuous on all of $\Phi(\mathcal N)$, so the sieve $\mathcal N_\delta$ cannot be dispensed with. Let $\eta\in\mathcal N$ have event times $0<t_1<\cdots<t_m<\Tmax$ (possibly $m=0$, with the convention $t_1:=\Tmax$), and for $0<h<t_1$ let $\eta^{(h)}\in\mathcal N$ have event times $\{h,t_1,\dots,t_m\}$. Then $\dtpp{\eta^{(h)}}{\eta}=\Tmax-h$, while $\norm{\Phi(\eta^{(h)})-\Phi(\eta)}_\infty\le h$, so $d_1\bigl(\Phi(\eta^{(h)}),\Phi(\eta)\bigr)\le\Tmax\,h\to0$ as $h\downarrow0$. Hence $\Phi(\eta^{(h)})\to\Phi(\eta)$ in $(\mathcal C,d_1)$ while $\dtpp{\eta^{(h)}}{\eta}\to\Tmax$, and $\Psi$ admits no modulus of continuity at any point of $\Phi(\mathcal N)$. On $\mathcal N_\delta$ this construction is unavailable: it forces the first interarrival time $\tau_1=h<\delta$.
\end{remark}

\begin{lemma}[Quantitative Bounds for $f_x$ on the Sieve]
    \label{lemma::props_fx_sieve}
    Let $\delta > 0$, $M_\delta := \Tmax/\delta$. Let $\eta \in \mathcal{N}_\delta$ have $m \ge 1$ events, and write $(t_k)_{k=0}^{m+1}$ for its augmented event-time grid and $(\tau_k)_{k=0}^{m+1}$ for the corresponding interarrival times (Convention~\ref{conv::boundary}), and set $x := \Phi(\eta)$, $f_x(t) := t - x(t)$. Then on each interval $[t_k, t_{k+1}]$, $k = 1,\dots,m$, the function $f_x$ is affine with slope
    \begin{equation}
        r_k \;:= \frac{\tau_k}{\tau_{k+1}}
        \in \bigl[M_\delta^{-1},\;M_\delta\bigr] \,.
    \end{equation}
    In particular, $f_x$ is $M_\delta$-Lipschitz on $[0, \Tmax]$.
\end{lemma}

\begin{proof}
    Fix $k \in \{1,\dots,m\}$.
    By construction, $x = \Phi(\eta)$ is affine on $[t_k, t_{k+1}]$ with values $x(t_k) = \tau_k$ and $x(t_{k+1}) = \tau_{k+1}$.
    Its slope is $(\tau_{k+1} - \tau_k)/\tau_{k+1}$,
    so $f_x = \mathrm{id} - x$ is affine on $[t_k, t_{k+1}]$ with slope
    \begin{equation}
        1 - \frac{\tau_{k+1} - \tau_k}{\tau_{k+1}}
        =
        \frac{\tau_k}{\tau_{k+1}} \,.
    \end{equation}
    Since $\eta \in \mathcal N_\delta$, every interarrival time lies in $[\delta, \Tmax]$, giving
    $r_k = \tau_k / \tau_{k+1} \in [\delta/\Tmax, \Tmax/\delta] = [1/M_\delta, M_\delta]$.

    Finally, $f_x$ is continuous and piecewise affine.
    On $[0, t_1]$ the slope is~$0$. On each subsequent piece it is $r_k \le M_\delta$.
    Hence $f_x$ is $M_\delta$-Lipschitz.
\end{proof}

\begin{proposition}[Restricted Invertibility of the Backshift Transform on the Sieve]
    \label{prop::Finverse}
    Let $\delta > 0$, $M_\delta := \Tmax/\delta$. Let $\eta \in \mathcal{N}_\delta$ have $m \ge 1$ events, and write $(t_k)_{k=0}^{m+1}$ for its augmented event-time grid (Convention~\ref{conv::boundary}), and set $x := \Phi(\eta)$, $f_x(t) := t - x(t)$. Then:
    \begin{enumerate}[label=(\roman*)]
        \item The restriction $f_x\big|_{[t_1,\Tmax]} : [t_1,\Tmax] \to [0, t_m]$ is a continuous, strictly increasing bijection.
        \item Its inverse $F_x := \bigl(f_x\big|_{[t_1,\Tmax]}\bigr)^{-1}$ satisfies, for all $u,v \in [0,t_m]$,
        \begin{equation}
            \abs{F_x(u) - F_x(v)} \le M_\delta\,\abs{u-v} \,.
        \end{equation}
    \end{enumerate}
\end{proposition}

\begin{proof}
    By Lemma~\ref{lemma::props_fx_sieve},  $f_x$ is affine on each $[t_k,t_{k+1}]$ ($k=1,\dots,m$) with strictly positive slope $r_k \ge M_{\delta}^{-1}$.
    In particular, $f_x$ is continuous and strictly increasing on $[t_1,\Tmax]$,
    hence a bijection onto $[f_x(t_1),\,f_x(\Tmax)]$.

    It remains to identify the endpoints. By Convention~\ref{conv::boundary}, $\tau_1=t_1-t_0=t_1$ and $\tau_{m+1}=t_{m+1}-t_m=\Tmax-t_m$. Since $x=\Phi(\eta)$ interpolates the nodes $(t_k,\tau_k)$, we have $x(t_1)=\tau_1=t_1$ and $x(\Tmax)=x(t_{m+1})=\tau_{m+1}=\Tmax-t_m$. Therefore,
    \begin{equation}
        f_x(t_1)=t_1-x(t_1)=0,
        \qquad
        f_x(\Tmax)=\Tmax-x(\Tmax)=t_m \,,
    \end{equation}
    so the range is exactly $[0,t_m]$.

    On each image interval $f_x([t_k,t_{k+1}])$, the inverse $F_x$ is affine with slope $1/r_k \le M_{\delta}$. Since $F_x$ is continuous and piecewise-affine on a compact interval with all slopes at most~$M_{\delta}$, it is globally $M_{\delta}$-Lipschitz.
\end{proof}

\begin{lemma}[$L^1$-$L^\infty$ Interpolation for Lipschitz functions]
    \label{lemma::interpolation}
    Let $h : [0, \Tmax] \to \R$ be $L$-Lipschitz with $h(0) = 0$ and $L>0$.
    Then
    \begin{equation}
        \norm{h}_\infty^2 \le 4\,L\;\norm{h}_1 \,,
    \end{equation}
    where $\norm{h}_1 := \int_0^{\Tmax} |h(t)| \mes t$ and $\norm{h}_\infty := \max_{t \in [0,\Tmax]} |h(t)|$.
\end{lemma}
\begin{proof}
    Let  $M := \norm{h}_\infty$ and choose $t^* \in [0,\Tmax]$ with $|h(t^*)| = M$. If $M = 0$ the claim is trivial, so assume $M > 0$. Since $h(0) = 0$ and $h$ is $L$-Lipschitz, $M = |h(t^*) - h(0)| \le L\,t^*$, so $t^* \ge M/L$. Consequently, the interval $[t^* - M/(2L),\; t^*]$ has length $M/(2L)$ and is contained in $[0, \Tmax]$. For $t$ in this interval, the Lipschitz bound gives
    \begin{equation}
        |h(t)|
        \ge |h(t^*)| - L\,|t - t^*|
        \ge M - L \cdot \frac{M}{2L}
        = \frac{M}{2} \,.
    \end{equation}
    Therefore
    \begin{equation}
        \norm{h}_1
        \ge
        \int_{t^* - M/(2L)}^{t^*} |h(t)| \mes t
        \ge
        \frac{M}{2} \cdot \frac{M}{2L}
        =
        \frac{M^2}{4L} \,.
    \end{equation}
\end{proof}

\begin{proposition}[Sup-norm Control]\label{prop::g_diff_bound}
    Let $\eta, \xi \in \mathcal N_\delta$ and set $x := \Phi(\eta)$, $y := \Phi(\xi)$, $g := x - y$.
    Assume $M_\delta := \Tmax/\delta \ge 2$. When $M_\delta < 2$, the sieve $\mathcal N_\delta$ contains at most the eventless path, so the only available pair is $x = y$ and all pairwise bounds are trivial.
    \begin{enumerate}[label=(\roman*)]
        \item Each of $x$ and $y$ is $(M_\delta - 1)$-Lipschitz on $[0, \Tmax]$, and $g$ satisfies $g(0) = 0$ and is $2(M_\delta - 1)$-Lipschitz.
        \item $\displaystyle \norm{g}_\infty \le 2\sqrt{2(M_\delta - 1)\,d_1(x, y)} \le 2 \sqrt{ 2 M_\delta \, d_1(x, y)}$.
    \end{enumerate}
\end{proposition}

\begin{proof}
    We treat three cases according to the event counts of $\eta$ and $\xi$.

    \medskip\noindent
    \textit{Case~$1$: Both $\eta$ and $\xi$ have zero events.}\;
    Then $x(t) = y(t) = t$ for all $t \in [0,\Tmax]$, so $g \equiv 0$ and both claims hold trivially.

    \medskip\noindent
    \textit{Case~$2$: Both $\eta$ and $\xi$ have at least one event.}\;
    Let $m \ge 1$ denote the event count of $\eta$, with augmented grid $0 = t_0 < t_1 < \cdots < t_m < t_{m+1} = \Tmax$ and interarrival times $\tau_k := t_k - t_{k-1}$.

    \medskip\noindent
    \textit{Part~(i): Lipschitz bound.}\;
    By Convention~\ref{conv::boundary}, $x(0) = \tau_0 = 0$ and $y(0) = \varsigma_0 = 0$, so $g(0) = 0$.

    We show $x = \Phi(\eta)$ is $(M_\delta - 1)$-Lipschitz.
    For $k \ge 1$, the slope of $x$ on $[t_k, t_{k+1}]$ is
    \begin{equation}
        \alpha_k
        \;:= \frac{\tau_{k+1} - \tau_k}{\tau_{k+1}}
        = 1 - \frac{\tau_k}{\tau_{k+1}} \,.
    \end{equation}
    Since $\eta \in \mathcal N_\delta$, every interarrival time lies in $[\delta, \Tmax]$, so $\tau_k/\tau_{k+1} \in [M_\delta^{-1}, M_\delta]$ and hence
    \begin{equation}
        |\alpha_k|
        \le \max\bigl(|1 - M_\delta^{-1}|,\; |1 - M_\delta|\bigr)
        = M_\delta - 1 \,.
    \end{equation}
    For $k = 0$, the slope is $\alpha_0 = \tau_1/\tau_1 = 1$.
    Since $m \ge 1$, we have $\Tmax \ge (m+1)\delta \ge 2\delta$, so $M_\delta \ge 2$ and $|\alpha_0| = 1 \le M_\delta - 1$.

    Thus $x$ is $(M_\delta - 1)$-Lipschitz on $[0, \Tmax]$.
    The same argument applies to $y$ (which also has $\ge 1$ event), so $y$ is $(M_\delta - 1)$-Lipschitz.
    Their difference $g$ is therefore $2(M_\delta - 1)$-Lipschitz.

    \medskip\noindent
    \textit{Case~$3$: Exactly one of $\eta$, $\xi$ has zero events.}\;
    Without loss of generality, suppose $\eta$ has $m = 0$ events and $\xi$ has $n \ge 1$ events.
    Then $x(t) = \Phi(\eta)(t) = t$ for all $t \in [0,\Tmax]$, so $x$ has constant slope~$1$.
    Since $\xi$ has $n \ge 1$ events, $\Tmax \ge (n+1)\delta \ge 2\delta$, so $M_\delta \ge 2$.
    Therefore $|{1}| = 1 \le M_\delta - 1$, and $x$ is $(M_\delta - 1)$-Lipschitz.

    The function $y = \Phi(\xi)$ has $n \ge 1$ events in $\mathcal N_\delta$, so by the slope analysis in Case~$2$ (applied to $\xi$), $y$ is also $(M_\delta - 1)$-Lipschitz.

    Since $g(0) = x(0) - y(0) = 0 - 0 = 0$ and $g$ is the difference of two $(M_\delta - 1)$-Lipschitz functions, $g$ is $2(M_\delta - 1)$-Lipschitz. This establishes Part~(i).

    \medskip\noindent
    \textit{Part~(ii): Sup-norm bound (all cases).}\;
    In all cases, Part~(i) gives that $g$ is $L$-Lipschitz with $L = 2(M_\delta - 1) > 0$ (recall $M_\delta \ge 2$) and $g(0) = 0$.
    By Lemma~\ref{lemma::interpolation},
    \begin{equation}
        \norm{g}_\infty^2
        \le 4L\,\norm{g}_1
        = 8(M_\delta - 1)\, d_1(x, y) \,.
    \end{equation}
    Taking square roots and using $M_\delta - 1 \le M_\delta$ gives both bounds.
\end{proof}

\begin{lemma}[Separation of Distinct-Count Images]
    \label{lemma::count_separation}
    Assume $M_\delta \ge 2$. Since the $m{+}1$ interarrival times of any $\eta \in \mathcal{N}_\delta$ with $m$ events are each at least $\delta$ and sum to $\Tmax$, we have $(m+1)\delta \le \Tmax$, i.e.\ $m \le M_\delta - 1$. As $m$ is an integer, $m \le \lfloor M_\delta\rfloor - 1$.
    For $n \in \{0,\dots,\lfloor M_\delta\rfloor - 1\}$, let $\mathcal{N}_\delta^{(n)} := \{\eta \in \mathcal{N}_\delta : \eta_{\Tmax} = n\}$ denote the subset of $\mathcal{N}_\delta$ with exactly $n$ events. Each such class is nonempty (the equispaced grid $t_k = k\delta$, $k=1,\dots,n$, is admissible because $(n+1)\delta \le \lfloor M_\delta\rfloor\,\delta \le \Tmax$), and these classes partition $\mathcal N_\delta$.
    For each pair $m \neq n$ in $\{0,\dots,\lfloor M_\delta\rfloor - 1\}$,
    \begin{equation}
        d_1\!\left(\Phi(\mathcal N_\delta^{(m)}),
                 \Phi(\mathcal N_\delta^{(n)})\right) > 0 \,.
    \end{equation}
    Consequently,
    \begin{equation}
        d_{\min}(\delta)
        :=
        \min_{\substack{m,n \in \{0,\dots,\lfloor M_\delta\rfloor - 1\}\\ m \neq n}}
        d_1\!\left(\Phi(\mathcal N_\delta^{(m)}),
                 \Phi(\mathcal N_\delta^{(n)})\right)
        > 0 \,.
    \end{equation}
\end{lemma}
\begin{proof}
    Fix $n \in \{0,\dots,\lfloor M_\delta\rfloor - 1\}$. Identify $\mathcal N_\delta^{(n)}$ with the subset of $\mathbb R^n$
    \begin{equation}
        \Big\{
        (t_1,\dots,t_n)\in[0,\Tmax]^n :
        t_k - t_{k-1} \ge \delta \text{ for } k=1,\dots,n,\;
        \Tmax - t_n \ge \delta
        \Big\},
        \quad t_0=0 \,,
    \end{equation}
    and equip $\mathcal N_\delta^{(n)}$ with the subspace topology inherited from $\mathbb R^n$. The defining constraints are finite intersections of closed half-spaces, hence the set is closed in the compact cube $[0,\Tmax]^n$. Therefore $\mathcal N_\delta^{(n)}$ is compact (the case $n=0$ being trivial). Furthermore, on $\mathcal N_\delta^{(n)}$, $\dtppsym$ coincides with the $\ell^1$ distance between vectors, and so do their topologies. $\Phi$ being Lipschitz with respect to $\dtppsym$ (Theorem~\ref{thrm::Phi_lipschitz}), the restriction on $\mathcal N_\delta^{(n)}$ is continuous. Thus, $\Phi(\mathcal N_\delta^{(n)})$ is compact in $(\mathcal C,d_1)$.

    Injectivity of $\Phi$ implies that $\Phi(\mathcal N_\delta^{(m)}) \cap \Phi(\mathcal N_\delta^{(n)}) = \varnothing$ for $m \neq n$.

    The continuous function $(x,y) \mapsto d_1(x,y)$ attains its minimum on the compact product of the two images $\Phi(\mathcal N_\delta^{(m)}) \times \Phi(\mathcal N_\delta^{(n)})$. Since the sets are disjoint, this minimum is strictly positive. Since the index set $\{0,\dots,\lfloor M_\delta\rfloor - 1\}$ is finite and, by $M_\delta \ge 2$, contains at least the two indices $0$ and $1$, the minimum in the definition of $d_{\min}(\delta)$ is taken over a finite, nonempty collection of strictly positive quantities. This yields $d_{\min}(\delta) > 0$.
\end{proof}

\begin{remark}
    The lemma shows that under the minimal spacing condition, the event count becomes a topologically stable feature of the embedding, as distinct counts are separated by a strictly positive $d_1$-distance.
\end{remark}

\begin{corollary}[Count stability]
    \label{cor::distance_same_count}
    Assume $M_\delta \ge 2$ and let $\eta,\xi \in \mathcal{N}_\delta$ satisfy $d_1\bigl(\Phi(\eta),\Phi(\xi)\bigr) < d_{\min}(\delta)$, the separation constant of Lemma~\ref{lemma::count_separation}. Then $\eta$ and $\xi$ have the same number of events.
\end{corollary}
\begin{proof}
    This is immediate from Lemma~\ref{lemma::count_separation}: if the event counts differed, then $\Phi(\eta)$ and $\Phi(\xi)$ would lie in distinct image sets whose $d_1$-distance is at least $d_{\min}(\delta)$.
\end{proof}

\begin{lemma}[Base-Case Perturbation]
    \label{lemma::base_case}
    Let $\eta,\xi \in \mathcal{N}_\delta$ share event count $m \ge 1$, with first event times $t_1$ and $s_1$, respectively.
    Set $x := \Phi(\eta)$, $y := \Phi(\xi)$, $\epsilon := \norm{x - y}_\infty$, and $M_\delta := \Tmax/\delta$. Then
    \begin{equation}
        |s_1 - t_1| \le M_\delta\epsilon \,.
    \end{equation}
\end{lemma}
\begin{proof}
    Assume $s_1 \ge t_1$ (the reverse case follows by symmetry).
    By level-set recovery, $f_x(t_1) = 0$ and $f_y(s_1) = 0$.
    Since $f_x$ has slope $\ge 1/M_\delta$ on $[t_1,\Tmax]$ (Lemma~\ref{lemma::props_fx_sieve}),
    \begin{equation}
        f_x(s_1) \ge f_x(t_1) + \tfrac{1}{M_\delta}(s_1 - t_1) = \tfrac{1}{M_\delta}|s_1 - t_1| \,.
    \end{equation}
    Furthermore, $\abs{f_x(s_1) - f_y(s_1)} = \abs{f_x(s_1) - 0} = f_x(s_1) \le \norm{f_x - f_y}_\infty = \norm{x-y}_\infty = \epsilon$.
\end{proof}

\begin{lemma}[Recursive Error Propagation]
    \label{lemma::recursion}
    Let $\eta,\xi \in \mathcal N_\delta$ share event count $m \ge 1$, with augmented event-time grids $(t_k)_{k=0}^{m+1}$ and $(s_k)_{k=0}^{m+1}$, respectively. Set $x := \Phi(\eta)$, $y := \Phi(\xi)$, and $\epsilon := \norm{x - y}_\infty$. Suppose $\epsilon < \delta$, and $M_\delta := \Tmax/\delta$. Then for every $k = 2,\dots,m$,
    \begin{equation}
        |s_k - t_k| \le M_\delta |s_{k-1} - t_{k-1}| + M_\delta\epsilon \,.
    \end{equation}
\end{lemma}

\begin{proof}
    By Proposition~\ref{prop::Finverse}, the restricted inverses $F_x := (f_x|_{[t_1,\Tmax]})^{-1}$ and $F_y := (f_y|_{[s_1,\Tmax]})^{-1}$ are $M_\delta$-Lipschitz, and the event times satisfy $t_k = F_x(t_{k-1})$, $s_k = F_y(s_{k-1})$ for $k \ge 2$.

    \medskip\noindent
    \textit{Step~$1$: $t_k \ge s_1$.}\;
    Suppose for contradiction that $t_k < s_1$. Then $f_y(t_k) = 0$ (since $f_y \equiv 0$ on $[0, s_1]$), while by Lemma~\ref{lemma::level_set_recovery} $f_x(t_k) = t_{k-1} \ge t_1 \ge \delta$. But $\norm{f_x - f_y}_\infty = \epsilon < \delta$, a contradiction.

    \medskip\noindent
    \textit{Step~$2$: $t_{k-1} \in [0, s_m]$.}\;
    Since $f_x(\Tmax) = t_m$ and $f_y(\Tmax) = s_m$, we have $|s_m - t_m| = |(f_x - f_y)(\Tmax)| \le \epsilon$. The $\delta$-separation of $\eta$ gives $t_{k-1} \le t_m - \delta$. Combining:
    \begin{equation}
        t_{k-1} \le t_m - \delta < t_m - \epsilon \le s_m \,,
    \end{equation}
    where the second inequality uses $\epsilon < \delta$ and the third uses $s_m \ge t_m - \epsilon$. In other words $t_{k-1} \in [0, s_m]$ as required to evaluate $F_y(t_{k-1})$.

    \medskip\noindent
    \textit{Step~$3$: Bounding $|F_y(t_{k-1}) - F_x(t_{k-1})|$.}\;
    By Step~$1$, $F_x(t_{k-1}) = t_k \ge s_1$. Since $t_{k-1} \in [0, s_m]$ (Step~$2$), $F_y(t_{k-1})$ is well-defined and lies in $[s_1, \Tmax]$ (the range of $F_y$). Thus both $F_y(t_{k-1})$ and $F_x(t_{k-1})$ lie in $[s_1, \Tmax]$, where $f_y$ has slope at least $ M_\delta^{-1}$, so
    \begin{equation}
        M_\delta^{-1}\abs{F_y(t_{k-1}) - F_x(t_{k-1})} \le |f_y(F_y(t_{k-1})) - f_y(F_x(t_{k-1}))| \,.
    \end{equation}
    Since $f_y(F_y(t_{k-1})) = t_{k-1} = f_x(F_x(t_{k-1}))$ and $|f_x(F_x(t_{k-1})) - f_y(F_x(t_{k-1}))| \le \norm{x - y}_\infty = \epsilon$, we obtain $|F_y(t_{k-1}) - F_x(t_{k-1})| \le M_\delta\epsilon$.

    \medskip\noindent
    \textit{Step~$4$: Triangle inequality.}\;
    \begin{align}
        |s_k - t_k|
        &= |F_y(s_{k-1}) - F_x(t_{k-1})| \notag \\
        &\le \underbrace{|F_y(s_{k-1}) - F_y(t_{k-1})|}_{\le\, M_\delta |s_{k-1} - t_{k-1}|}
        + \underbrace{|F_y(t_{k-1}) - F_x(t_{k-1})|}_{\le\, M_\delta\epsilon} \notag \\
        &\le M_\delta |s_{k-1} - t_{k-1}| + M_\delta\epsilon \,.
    \end{align}
\end{proof}

\begin{lemma}[Large-Perturbation Regime]
    \label{lemma::large_eps}
    Let $\eta,\xi \in \mathcal{N}_\delta$ share event count $m \ge 1$, and set $x := \Phi(\eta)$, $y := \Phi(\xi)$. Set $\epsilon:=\norm{x-y}_\infty$ and suppose $\epsilon\geq\delta$. Then
    \begin{equation}
        \dtpp{\eta}{\xi} \le m\sqrt{8} M_\delta^{3/2} d_1(x,y)^{1/2} \,.
    \end{equation}
\end{lemma}

\begin{proof}
    Since both measures have $m$ events in $[0,\Tmax)$, the discrete representation gives the trivial bound $\dtpp{\eta}{\xi} \le m\,\Tmax$. By Proposition~\ref{prop::g_diff_bound}, $d_1(x,y) \ge \epsilon^2 / (8 M_\delta) \ge \delta^2 / (8M_\delta)$. Hence
    \begin{equation}
        \dtpp{\eta}{\xi}
        \le m\,\Tmax \cdot \frac{\sqrt{d_1(x,y)}}{\sqrt{d_1(x,y)}}
        \le m\,\Tmax \cdot \frac{\sqrt{8 M_{\delta}}\,d_1(x,y)^{1/2}}{\delta}
        = m\sqrt{8} M_{\delta}^{3/2} d_1(x,y)^{1/2} \,.
    \end{equation}
\end{proof}

\begin{proposition}[Same-Count \holder Bound]\label{prop::same_count_holder}

    Let $\eta,\xi \in \mathcal{N}_\delta$ have the same event count $m \ge 0$.
    Set $x := \Phi(\eta)$, $y := \Phi(\xi)$, $\epsilon := \norm{x - y}_\infty$, and $M_\delta := \Tmax/\delta$.
    Then
    \begin{equation}\label{eq::_holder_intermediate_bound_first}
        \dtpp{\eta}{\xi} \le C_m(M_\delta)\epsilon \,,
        \quad
        C_m(M_\delta) := \frac{M_\delta^2(M_\delta^m - 1)}{(M_\delta-1)^2} \,,
    \end{equation}
    and consequently
    \begin{equation}\label{eq::holder_intermediate_bound}
        \dtpp{\eta}{\xi} \le 2\sqrt{2} C_m(M_\delta) M_\delta^{1/2} d_1(x, y)^{1/2} \le \frac{4M_\delta^{m+3}}{(M_\delta-1)^2} d_1(x,y)^{1/2} \,.
    \end{equation}
\end{proposition}

\begin{proof}
    The case $m = 0$ is trivial.
    Assume $m \ge 1$, then $M_\delta \ge m + 1 \ge 2$.
    Write $(t_k)_{k=0}^{m+1}$ and $(s_k)_{k=0}^{m+1}$ for the augmented event-time grids of $\eta$ and $\xi$ (Convention~\ref{conv::boundary}).

    \medskip\noindent
    \textit{Step~$1$: $\epsilon < \delta$.}
    Set $a_k := |s_k - t_k|$. The boundary convention gives $a_0 = 0$, and Lemmas~\ref{lemma::base_case} and~\ref{lemma::recursion} yield
    \begin{equation}
        a_k \le M_\delta a_{k-1} + M_\delta\epsilon \,,\quad k = 1,\dots,m \,.
    \end{equation}
    Unrolling from $a_0 = 0$,
    \begin{equation}
        a_k \le M_\delta\epsilon\sum_{j=0}^{k-1}M_\delta^j
        = \frac{M_\delta(M_\delta^k - 1)}{M_\delta - 1}\epsilon \,,\quad k = 1,\dots,m \,.
    \end{equation}
    By Eq.~\ref{eq::discrete_dtpp}, because both processes share the same event count, $\dtpp{\eta}{\xi} = \sum_{k=1}^m a_k$. Summing:
    \begin{equation}
        \sum_{k=1}^m a_k
        \le \frac{M_\delta\epsilon}{M_\delta-1}\sum_{k=1}^m (M_\delta^k - 1)
        = \frac{M_\delta\epsilon}{M_\delta-1}
        \biggl(\frac{M_\delta(M_\delta^m-1)}{M_\delta-1} - m\biggr)
        \le \frac{M_\delta^2(M_\delta^m - 1)}{(M_\delta-1)^2}\epsilon \,,
    \end{equation}
    where the last inequality drops the term $-m<0$.
    Substituting $\epsilon \le 2\sqrt{2M_\delta\,d_1 (x,y)} \le 4 M_\delta\sqrt{d_1(x,y)}$ (Proposition~\ref{prop::g_diff_bound}) into Eq.~\ref{eq::_holder_intermediate_bound_first} yields Eq.~\ref{eq::holder_intermediate_bound}. The final constant relaxes $M_\delta^m-1$ to~$M_\delta^m$ for a cleaner closed form.

    \medskip\noindent
    \textit{Step~$2$: $\epsilon \geq \delta$.}
    We first verify that Eq.~\ref{eq::_holder_intermediate_bound_first} continues to hold in this regime. Since both processes have $m$ events, Eq.~\ref{eq::discrete_dtpp} gives the trivial bound $\dtpp{\eta}{\xi} \le m\Tmax$, while
    \begin{equation}
        C_m(M_\delta)\,\epsilon
        \ge C_m(M_\delta)\,\delta
        = \frac{M_\delta(M_\delta^m - 1)}{(M_\delta - 1)^2}\,\Tmax
        \ge m\,\Tmax \,,
    \end{equation}
    where the last inequality uses $\frac{M_\delta^m - 1}{M_\delta - 1} = 1 + M_\delta + \cdots + M_\delta^{m-1} \ge m$ together with $\frac{M_\delta}{M_\delta - 1} \ge 1$. Hence Eq.~\ref{eq::_holder_intermediate_bound_first} holds unconditionally.

    We now establish Eq.~\ref{eq::holder_intermediate_bound} in this regime. For $m \ge 1$ and $\epsilon \ge \delta$, by Lemma~\ref{lemma::large_eps}, $\dtpp{\eta}{\xi} \le m\sqrt{8}\, M_\delta^{3/2}\,d_1(x,y)^{1/2}$.
    Since $(M_\delta - 1)^2 \le M_\delta^2$, it suffices to show $m\sqrt{8} \le 4\, M_\delta^{m-1/2}$, i.e. $m \le \sqrt{2}\, M_\delta^{m-1/2}$. We verify this by induction on $m$:
    \begin{itemize}
        \item \emph{Base case $m = 1$}: need $1 \le \sqrt{2}\, M_\delta^{1/2}$, which holds since $ M_\delta \ge 2$.
        \item \emph{Inductive step}: if $m \le \sqrt{2}\, M_\delta^{m-1/2}$, then $(m{+}1) \le 2m \le 2\sqrt{2}\, M_\delta^{m-1/2} \le \sqrt{2}\, M_\delta^{m+1/2}$ for $ M_\delta \ge 2$, where the last inequality uses $2 M_\delta^{-1} \le 1$.
    \end{itemize}
    Thus, the constant holds for all $m \ge 1$ and when $ M_\delta \ge 2$. Hence Eqs.~\ref{eq::_holder_intermediate_bound_first} and~\ref{eq::holder_intermediate_bound} hold in both cases, completing the proof.
\end{proof}

\begin{theorem}[\holder stability of $\Phi^{-1}$]
    \label{thrm::holder_main}
    Let $\delta>0$ and $M_\delta:=\Tmax/\delta$, and assume $M_\delta \ge 2$. When $M_\delta < 2$, the sieve $\mathcal N_\delta$ contains at most the eventless path and the statement holds trivially for any $\kappa>0$. There exists a constant $\kappa=\kappa(\delta, \Tmax) > 0$ such that for all $\eta,\xi\in\mathcal{N}_\delta$,
    \begin{equation}
        \dtpp{\eta}{\xi} \le \kappa d_1\!\bigl(\Phi(\eta),\Phi(\xi)\bigr)^{1/2} \,.
    \end{equation}
    Explicitly, one may take
    \begin{equation}\label{eq:C_explicit}
        \kappa = \max\!\biggl(
        \frac{4M_\delta^{\lfloor M_\delta\rfloor + 2}}{(M_\delta-1)^2},
        \frac{\lfloor M_\delta\rfloor\,\Tmax}{\sqrt{d_{\min}(\delta)}}
        \biggr) \,,
    \end{equation}
    where $d_{\min}(\delta) > 0$ is the separation constant from Lemma~\ref{lemma::count_separation}. Since $d_{\min}(\delta)$ is obtained by a compactness argument and is not quantified, we make no claim on the growth rate of $\kappa$ as $\delta \downarrow 0$.
\end{theorem}

\begin{proof}
    Let $m$ and $n$ denote the event counts of $\eta$ and $\xi$, and set $x := \Phi(\eta)$, $y := \Phi(\xi)$.

    \medskip\noindent
    \textit{Case~$1$: $m=n$.}\;
    Proposition~\ref{prop::same_count_holder} gives $\dtpp{\eta}{\xi}\le C_{\mathrm{rec}}\,d_1(x,y)^{1/2}$ with $C_{\mathrm{rec}}:=4\,M_\delta^{\lfloor M_\delta\rfloor+2}/(M_\delta-1)^2$, using $m = n \le \lfloor M_\delta\rfloor - 1$ (Lemma~\ref{lemma::count_separation}) together with $M_\delta \ge 2$ to bound $M_\delta^{m+3} \le M_\delta^{\lfloor M_\delta\rfloor+2}$. When $m = 0$, $\dtpp{\eta}{\xi} = 0$ and the bound is trivial.

    \medskip\noindent
    \textit{Case~$2$: $m\neq n$.}\;
    By Lemma~\ref{lemma::count_separation}, $d_1(x,y)\ge d_{\min}(\delta)>0$. Moreover, both event counts are at most $\lfloor M_\delta\rfloor - 1$ (Lemma~\ref{lemma::count_separation}), so Proposition~\ref{prop::discrete_dtpp} gives $\dtpp{\eta}{\xi}\le \max(m,n)\,\Tmax\le \lfloor M_\delta\rfloor\,\Tmax$. Hence
    \begin{equation}
        \dtpp{\eta}{\xi}
        \le \lfloor M_\delta\rfloor\,\Tmax
        = \frac{\lfloor M_\delta\rfloor\,\Tmax}{\sqrt{d_{\min}(\delta)}}\,
        \sqrt{d_{\min}(\delta)}
        \le
        \frac{\lfloor M_\delta\rfloor\,\Tmax}{\sqrt{d_{\min}(\delta)}}\,
        d_1(x,y)^{1/2} \,.
    \end{equation}
    Taking $\kappa := \max\!\bigl(C_{\mathrm{rec}},\; \lfloor M_\delta\rfloor\Tmax/\sqrt{d_{\min}(\delta)}\bigr)$ completes the proof.
\end{proof}

\begin{remark}[Interpretation and sharpening of the inverse constant]
    \label{rmq::kink_comparison}
    The same-count term in Eq.~\ref{eq:C_explicit} satisfies $C_{\mathrm{rec}}=\mathcal O(M_\delta^{\lfloor M_\delta\rfloor})$. This is a conservative proof constant and does not establish instability of the inverse on empirical paths. Its exponential dependence comes from unrolling the worst-case recursive error propagation $a_k\le M_\delta\, a_{k-1} +M_\delta\epsilon$ in Proposition~\ref{prop::same_count_holder}. Finite recording precision and the strictly increasing event-time convention ensure a positive minimum interarrival $\delta_{\mathrm{data}}>0$ on each finite dataset used in this work, so the separated-sieve guarantee applies to the observed paths. If one additionally requires that consecutive slopes of~$f_x$ differ by at least some fixed $\gamma>0$ (a slope-gap condition), a local, non-recursive argument gives the bound
    \begin{equation}
        \frac{\gamma}{4\lfloor M_\delta\rfloor}\dtpp{\eta}{\xi}^2 \le d_1\bigl(\Phi(\eta),\Phi(\xi)\bigr) \,.
    \end{equation}
    This yields $\kappa=\mathcal{O}(\sqrt{\lfloor M_\delta\rfloor/\gamma})$. The argument goes along the following lines. When two breakpoint grids are shifted by~$a_k$, interlacing ensures that on a neighbourhood of each~$t_k$, one embedding has a kink while the other is linear. The slope-gap condition~$\gamma$ guarantees this kink has opening $\ge\gamma$, producing an $L^1$ contribution greater or equal to $\gamma/4 \,a_k^2$. Since the neighbourhoods are disjoint (spacing larger than $\delta$, shifts smaller or equal to $\delta/4$), the contributions sum independently and Cauchy-Schwarz converts $\sum a_k^2 \ge \dtpp{\eta}{\xi}^2/\lfloor M_\delta\rfloor$. No error propagation occurs, hence the polynomial dependence.
\end{remark}

\section{Sampling and Training Procedures for \sigTPP}
\label{app::algorithms}

The following algorithms describe the free-running temporal sampling and training procedures used by \sigTPP~(Section~\ref{subsect::sigtpp}). 
Algorithm~\ref{alg:sigtpp-sampling-one-indexed} defines how to generate each next interarrival from the previously generated history and fresh noise, subject to the observation horizon and rollout cap. 
Algorithm~\ref{alg:sigtpp-training} calls this routine at each optimisation step, embeds the generated counting paths, and computes their truncated signatures for backpropagation.

  \begin{algorithm}[htb]
  \caption{Free-running sampling procedure for \sigTPP. Indices start at $1$ to match the augmented time grid of Definition~\ref{def::phi_embedding}, where $t_0:=0$ denotes the start of the horizon.}
  \label{alg:sigtpp-sampling-one-indexed}
  \begin{algorithmic}[1]
  \Require Generator $G_\theta$, horizon $\Tmax$, batch size $B$,
           maximum rollout length $L_{\max}$
  \Ensure A batch of generated event sequences on $[0,\Tmax)$,
           each containing at most $L_{\max}$ events

  \Procedure{SampleFreeRunning}
      {$G_\theta,\Tmax,B,L_{\max}$}
  \For{$b=1,\ldots,B$}
          \State Set $t_{b,0}\gets0$ and $i\gets1$
          \While{$t_{b,i-1}<\Tmax\ \land\ i\le L_{\max}$}
              \State Draw $z_{b,i}\sim\operatorname{Exp}(1)$ and set
                     $\varepsilon_{b,i}\gets\log z_{b,i}$
              \State $\ell_{b,i}\gets
                     G_\theta(\tau_{b,1:i-1},t_{b,i-1},\varepsilon_{b,i})$
              \State $\tau_{b,i}\gets\exp(\ell_{b,i})$ and
                     $t_{b,i}\gets t_{b,i-1}+\tau_{b,i}$
              \State $i\gets i+1$
          \EndWhile
          \State $n_b\gets
                 \max\!\left(\{j\in\{1,\ldots,i-1\}:t_{b,j}<\Tmax\}
                 \cup\{0\}\right)$
  \EndFor
      \State \Return $\{\tau_{b,1:n_b}\}_{b=1}^{B}$
  \EndProcedure
  \end{algorithmic}
  \end{algorithm}
  
  \begin{algorithm}[htb]
  \caption{Core free-running temporal optimisation for \sigTPP.}
  \label{alg:sigtpp-training}
  \begin{algorithmic}[1]
  \Require Training set $\mathcal D$, horizon $\Tmax$, signature degree $M$,
           optimisation steps $N_{\mathrm{steps}}$, batch size $B$,
           maximum rollout length $L_{\max}$
  \Ensure Trained generator parameters $\theta$

  \State Set $N\gets|\mathcal D|$ and index $\mathcal D$ as
         $(\eta_i)_{i=1}^{N}$
  \State $\displaystyle
         (\mathbf{s}_i)_{i=1}^{N}\gets
         \left(S^M(\widetilde\Phi(\eta_i))\right)_{i=1}^{N}$
  \State Let $\mu_{\mathcal D}$ and $\sigma_{\mathcal D}$ be the coordinate-wise
         mean and standard deviation of $(\mathbf{s}_i)_{i=1}^{N}$
  \State Initialise the generator $G_\theta$ and Adam optimiser

  \For{$k=1,\ldots,N_{\mathrm{steps}}$}
          \State $\{\tau_{b,1:n_b}\}_{b=1}^{B}
              \gets \Call{SampleFreeRunning}
              {G_\theta,\Tmax,B,L_{\max}}$
          \State $\widehat X_b\gets\widetilde\Phi(\widehat\eta_b)$, where
                 $\widehat\eta_b$ is induced by $\tau_{b,1:n_b}$
          \State $\widehat\mu_\theta\gets
              B^{-1}\sum_{b=1}^{B}
              S^M(\widehat X_b)$
          \State $\displaystyle
              \mathcal L_{\sigW}\gets
              \left\|
                  \frac{\widehat\mu_\theta-\mu_{\mathcal D}}
                       {\sigma_{\mathcal D}}
              \right\|_2^2$

          \State Update $\theta$ with Adam using $\mathcal L_{\sigW}$
  \EndFor
  \end{algorithmic}
  \end{algorithm}

\section{Evaluation Metrics}
\label{sec::metrics}
We evaluate generative TPP models along several complementary axes. 
Each metric targets a distinct structural property of the process, and together they provide a comprehensive picture of model quality. For all metrics a lower value is better.

\subsection{Theory-Motivated Metrics}

We define and theoretically justify $3$ theory-motivated metrics in Section~\ref{subsect::evaluation}. We recall them briefly here and describe their empirical estimators. Throughout this section, we assume access to $N_P$ real samples $\eta^{(1)},\ldots,\eta^{(N_P)} \sim P$ and $N_Q$ generated samples $\widehat\xi^{(1)},\ldots,\widehat\xi^{(N_Q)} \sim Q$.

\textbf{Signature Wasserstein-1 ($\sigW$).}
$\sigW$ measures discrepancy between path distributions through their expected truncated signatures, capturing higher-order temporal-dependence structure. It serves simultaneously as our training objective for \sigTPP~and as an evaluation metric. Replacing the population expectations in Eq.~\ref{eq::sig_metric} by sample means yields
\begin{equation}
    \widehat{\sigW}^{(M)}(P, Q)
    = \Bigg\|
        \frac{1}{N_P}\sum_{i=1}^{N_P} S^M\!\bigl(\widetilde{\Phi}(\eta^{(i)})\bigr)
        - \frac{1}{N_Q}\sum_{i=1}^{N_Q} S^M\!\bigl(\widetilde{\Phi}(\widehat\xi^{(i)})\bigr)
      \Bigg\|_2 \, .
\end{equation}
We report results at signature degree $3$ with unscaled signature components. During training, signature features are standardised (common setting in linear regressions).

\textbf{Energy Distance ($\mathcal{E}$).}
The energy distance is defined as
\begin{equation}
\label{eq::energy}
    \mathcal{E}   := 2 \cdot \E\left[\dtpp{N}{\widetilde N}\right]
              - \E\left[\dtppsym\Bigl(N,N'\Bigr)\right]
              - \E\left[\dtpp{\widetilde N}{\widetilde N'}\right] \,, \quad P, Q \in \mathcal{P}_1(\mathcal{N}) \, ,
\end{equation}
with $N, N' \sim P$ and $\widetilde N, \widetilde N' \sim Q$ independent~\cite{gneiting_strictly_2007, szekely_energy_2013}.
The energy distance is a distribution-level discrepancy on $\mathcal{P}_1(\mathcal{N}, \dtppsym)$ that compares the average inter-sample distance between $P$ and $Q$ against the within-distribution distances. Its empirical estimator is
\begin{equation}
\begin{aligned}
    \widehat{\mathcal{E}}(P,Q)
    &= \frac{2}{N_P N_Q}\sum_{i,j}\dtpp{\eta^{(i)}}{\widehat\xi^{(j)}}
    - \frac{1}{N_P(N_P-1)}\sum_{i \neq i'}\dtpp{\eta^{(i)}}{\eta^{(i')}} \\
    &\quad
    - \frac{1}{N_Q(N_Q-1)}\sum_{j \neq j'}\dtpp{\widehat\xi^{(j)}}{\widehat\xi^{(j')}} \, .
\end{aligned}
\end{equation}
Although $\mathcal{E} \ge 0$ at the population level (since $\dtppsym$ is of negative type), the empirical estimator $\widehat{\mathcal{E}}(P,Q)$ can take small negative values in finite samples, especially when $P$ and $Q$ are close. We compute it on batches of $2{,}000$ sequences so that all pairwise distances fit in memory. Remark that $\mathcal{E}$ reduces to the CRPS~\cite{pic_proper_2025} in the univariate case, providing a natural score for unconditional generation that handles variable-length sequences.

\textbf{Wasserstein-1 ($\WAS$).}
Since $(\mathcal{N}, \dtppsym)$ is a metric space, the Wasserstein-1 distance~\cite{ramdas2017wasserstein, villani_optimal_2009}
\begin{equation}
    \WAS\left(P, Q\right)  := \inf_{\pi \in \Pi(P, Q)} \int_{\mathcal{N} \times \mathcal{N}}
      \dtpp{N}{N'} \, \mes\pi(N, N') \, ,
\end{equation}
where $\Pi(P,Q)$ denotes the set of couplings of $P$ and $Q$, is a well-defined metric on $\mathcal{P}_1(\mathcal{N})$. 
$\WAS$ lifts the path-level metric $\dtppsym$ to a metric between distributions. It is the minimum cost of transporting generated mass onto real mass, with per-pair cost $\dtppsym$. With empirical measures $\mu = \frac{1}{N_P}\sum_{i=1}^{N_P} \delta_{\eta^{(i)}}$ and $\nu = \frac{1}{N_Q}\sum_{j=1}^{N_Q} \delta_{\widehat\xi^{(j)}}$, it is estimated by solving
\begin{equation}
    \widehat{W}_1(\mu,\nu)
    = \min_{\pi \in \Pi(\mu,\nu)}
      \sum_{i=1}^{N_P}\sum_{j=1}^{N_Q}
      \pi_{ij}\,\dtpp{\eta^{(i)}}{\widehat\xi^{(j)}} \,,
\end{equation}
where $\Pi(\mu,\nu)$ denotes the set of couplings between $\mu$ and $\nu$. As for the energy distance, we solve this linear program exactly on batches of $2{,}000$ sequences, chosen so the full cost matrix fits in memory.

\textbf{Continuous Ranked Probability Score (CRPS).}
CRPS evaluates the calibration and sharpness of 1-step-ahead predictive distributions over interarrival times. For each true event at position $l$ in sequence $n$, the model defines a predictive distribution $F_{n,l}$ conditional on the true past. We approximate it by the empirical cumulative distribution function $\widehat{F}_{n,l}$ of the $S$ Monte Carlo samples $\widehat{\mathcal{T}}_{:,n,l}$. We report the per-event average
\begin{equation}
    \overline{\text{CRPS}}
    = \frac{1}{\sum_{n=1}^{N_{\mathrm{seq}}} \ell_{n}}
      \sum_{n=1}^{N_{\mathrm{seq}}}\sum_{l=1}^{\ell_n}
      \text{CRPS}\!\bigl(\widehat{F}_{n,l},\,\mathcal{T}_{n,l}\bigr) \,,
\end{equation}
where the standard energy-form definition for a sample $y$
\begin{equation}
    \text{CRPS}(\widehat{F}, y)
    = \mathbb{E}_{X\sim\widehat{F}}\!\bigl[\abs{X-y}\bigr]
    - \tfrac{1}{2}\,
      \mathbb{E}_{X,X'\sim\widehat{F}}\!\bigl[\abs{X-X'}\bigr] \,,
\end{equation}
where $X, X'$ are i.i.d.\ draws from $\widehat{F}$.

\textbf{Scaling.} For comparability across datasets with different time horizons, we report $\widehat{\mathcal{E}}/\Tmax^2$ (resp.\ $\widehat{W}_1/\Tmax^2$) in all tables. The remaining metrics operate at the per-event level on interarrival times. Let $N_{\mathrm{seq}}$ denote the number of evaluation sequences and $\ell_n$ the (unpadded) length of sequence $n$, and write $\ell_{\max} = \max_n \ell_n$. We collect interarrivals in $\mathcal{T} \in \mathbb{R}^{N_{\mathrm{seq}} \times \ell_{\max}}$ and $S$ Monte Carlo samples of generated interarrivals in $\widehat{\mathcal{T}} \in \mathbb{R}^{S \times N_{\mathrm{seq}} \times \ell_{\max}}$, where entries beyond each sequence's true length are set to \texttt{NaN}. All estimators below use pairwise deletion of \texttt{NaN} entries (per-lag for autocorrelations). We write $\mathcal{T}_{n,l}$ for the $l$-th interarrival of sequence $n$ and $\widehat{\mathcal{T}}_{:,n,l} \in \mathbb{R}^{S}$ for the vector of $S$ Monte Carlo predictions at position $(n,l)$.

\subsection{Heuristic Metrics}
The following metrics target specific structural properties, second-order temporal dependence and marginal distributional shape, without formal guarantees. They are adapted from~\citet{NiEtAl2020} to the TPP setting. 

\textbf{Pairwise Correlation Discrepancy (PCD).}
PCD assesses whether the model reproduces second-order dependence across event positions: it compares the Pearson correlation matrix of interarrival times at different positions in real vs.\ generated samples, catching models whose marginals are correct but whose cross-position structure is not. Writing $\mathcal{T}_{:,p}$ for the vector of position-$p$ interarrivals across the $N_{\mathrm{seq}}$ sequences (and likewise, we slice $1$ Monte Carlo sample $\widehat{\mathcal{T}}_{0,:,p}$), let
\begin{equation}
        \rho_{p,q} = \mathrm{Corr}\!\bigl(\mathcal{T}_{:,p},\,\mathcal{T}_{:,q}\bigr) \, ,
    \qquad
    \widehat{\rho}_{p,q} = \mathrm{Corr}\!\bigl(\widehat{\mathcal{T}}_{0,:,p},\,\widehat{\mathcal{T}}_{0,:,q}\bigr) \,,
\end{equation}
denote the Pearson correlations, computed by pairwise deletion of \texttt{NaN} entries. Then
\begin{equation}
    \mathrm{PCD}
    = \frac{1}{|\mathcal{V}|}
      \sum_{(p,q) \in \mathcal{V}}
      \abs{ \rho_{p,q} - \widehat{\rho}_{p,q} } \, ,
\end{equation}
where $\mathcal{V}$ collects position pairs $(p,q)$ with $1 \le p < q \le \ell^\star$ for which both correlations are well-defined. To stabilise estimation, we restrict $\ell^\star$ to the largest prefix length such that every position $p \le \ell^\star$ retains at least $50$ valid (non-padded) samples in both $\mathcal{T}$ and $\widehat{\mathcal{T}}$.

\textbf{Autocorrelation Discrepancy (ACD).}
ACD probes the temporal memory of the process, how strongly an interarrival depends on those a few events earlier. For lag $\tau \ge 1$, let $\mathrm{AC}_\tau(\mathcal{T})$ denote the lag-$\tau$ sample autocorrelation of interarrivals, computed by pooling all valid pairs $(\mathcal{T}_{n,l},\,\mathcal{T}_{n,l+\tau})$ across sequences (with \texttt{NaN}-padded positions excluded), and define $\mathrm{AC}_\tau(\widehat{\mathcal{T}})$ analogously by slicing one Monte Carlo sample. We restrict to short lags, where temporal dependence is most informative and estimation is most reliable: we set $\tau_{\max} = 5$ in general, and $\tau_{\max} = 10$ on long-sequence datasets ($\ell_{\max} \ge 100$).
We report
\begin{equation}
    \mathrm{ACD}
    = \frac{1}{\tau_{\max}}
      \sum_{\tau \leq \tau_{\max}}
      \abs{ \mathrm{AC}_\tau(\mathcal{T}) - \mathrm{AC}_\tau(\widehat{\mathcal{T}}) } \, .
\end{equation}

\textbf{Histogram Distances ($\histITloss$, $\histintloss$).}
These metrics compare marginal distributional shape at two complementary levels: $\histITloss$ targets the interarrival times density (how long between events?) and $\histintloss$ the event-count density on $[0,\Tmax]$ (when do events happen?). Both are $L^1$ distances between empirical histograms computed on $B$ shared bins. $B$ in our setting follows a cube-root heuristic $2\, n^{1/3}$. For $\histITloss$, let $\widehat{p}_{\mathrm{true}}$ and $\widehat{p}_{\mathrm{gen}}$ be histogram estimators of the marginal density of $u = \log\tau$, where $\tau$ is an interarrival time, on bins covering the empirical range of $\log\tau$ in $\mathcal{T}$:
\begin{equation}
    \histITloss
    = \int \abs{\widehat{p}_{\mathrm{true}}(u) - \widehat{p}_{\mathrm{gen}}(u)}\,\mes u \,,
    \qquad u = \log\tau \,.
\end{equation}
For $\histintloss$, let $\widehat{\lambda}^{\mathrm{emp}}_{\mathrm{true}}$ and $\widehat{\lambda}^{\mathrm{emp}}_{\mathrm{gen}}$ be histogram estimators of the empirical event rate per unit time on $[0,\Tmax]$ (distinct from the conditional intensity used elsewhere in the paper):
\begin{equation}
    \histintloss
    = \int_0^{\Tmax}
      \abs{\widehat{\lambda}^{\mathrm{emp}}_{\mathrm{true}}(s) - \widehat{\lambda}^{\mathrm{emp}}_{\mathrm{gen}}(s)}\,\mes s \,.
\end{equation}

\subsection{\texorpdfstring{$1$-Step-Ahead Point Estimates (MAE)}{1-Step-Ahead Point Estimates (MAE)}}

MAE and MSE evaluate the accuracy of $1$-step-ahead point predictions of interarrival times. For each true event at position $l$ in sequence $n$, the model defines a predictive distribution $F_{n,l}$ conditional on the true past, from which we draw the $S$ Monte Carlo samples $\widehat{\mathcal{T}}_{:,n,l} \in \mathbb{R}^{S}$ already used for CRPS. The natural point predictors are the median of $F_{n,l}$ for MAE and its mean for MSE, which we estimate by their sample counterparts
\begin{equation}
    \widehat{x}^{\mathrm{MAE}}_{n,l}
    = \operatorname{median}\!\bigl(\widehat{\mathcal{T}}_{:,n,l}\bigr) \,,
    \qquad
    \widehat{x}^{\mathrm{MSE}}_{n,l}
    = \frac{1}{S}\sum_{s=1}^{S} \widehat{\mathcal{T}}_{s,n,l} \,.
\end{equation}
We report the per-event averages
\begin{equation}
    \mathrm{MAE}_\mathrm{med}
    = \frac{1}{\sum_{n=1}^{N_{\mathrm{seq}}} \ell_n}
      \sum_{n=1}^{N_{\mathrm{seq}}}\sum_{l=1}^{\ell_n}
      \abs{\widehat{x}^{\mathrm{MAE}}_{n,l} - \mathcal{T}_{n,l}} \,,
    \qquad
    \mathrm{MSE}_\mathrm{mean}
    = \frac{1}{\sum_{n=1}^{N_{\mathrm{seq}}} \ell_n}
      \sum_{n=1}^{N_{\mathrm{seq}}}\sum_{l=1}^{\ell_n}
      \bigl(\widehat{x}^{\mathrm{MSE}}_{n,l} - \mathcal{T}_{n,l}\bigr)^2 \,.
\end{equation}
The sample median used in $\widehat{x}^{\mathrm{MAE}}_{n,l}$ is a consistent but finite-$S$ biased plug-in for the population median of $F_{n,l}$. We take $S$ large enough that this bias is negligible relative to the sequence-level Monte Carlo error.

We also report the 1-step-ahead forecasting error:
\begin{equation}
\widehat{\mathrm{MAE}}
=
\frac{1}{\sum_{n=1}^{N_{\mathrm{seq}}} \ell_n}
\sum_{n=1}^{N_{\mathrm{seq}}}\sum_{l=1}^{\ell_n}
\left(
\frac{1}{S}\sum_{s=1}^{S}
\abs{\widehat{\mathcal{T}}_{s,n,l}-\mathcal{T}_{n,l}}
\right) \, .
\end{equation}

\subsection{Relative Score}
\label{subsec::relative_score}
Let $s_A^{(i)}$ denote the score of model $A$ on dataset $i \in \{1,\dots,D\}$, and $s_{\mathrm{ref}}^{(i)}$ the score of a fixed reference model. The normalised score of model $A$ is the geometric mean of per-dataset ratios:
\begin{equation}
    \bar{s}_A
    = \left(\prod_{i=1}^{D} \frac{s_A^{(i)}}{s_{\mathrm{ref}}^{(i)}}\right)^{\!\frac{1}{D}}
    = \exp\!\left(\frac{1}{D}\sum_{i=1}^{D}
        \log \frac{s_A^{(i)}}{s_{\mathrm{ref}}^{(i)}}
    \right) \,.
\end{equation}
A value $\bar{s}_A < 1$ indicates that $A$ outperforms the reference on average. Crucially, comparing two models $A$ and $B$ via their normalised score ratio is independent of the choice of reference, i.e.
\begin{equation}
    \frac{\bar{s}_A}{\bar{s}_B}
    = \frac{
        \bigl(\prod_i s_A^{(i)}/s_{\mathrm{ref}}^{(i)}\bigr)^{1/D}
    }{
        \bigl(\prod_i s_B^{(i)}/s_{\mathrm{ref}}^{(i)}\bigr)^{1/D}
    }
    = \left(\prod_{i=1}^{D} \frac{s_A^{(i)}}{s_B^{(i)}}\right)^{\!\frac{1}{D}} \,,
    \label{eq:ratio_independence}
\end{equation}
where the $s_{\mathrm{ref}}^{(i)}$ cancel~\cite{fleming1986how}.

\section{Additional Results}
\label{sec::additional_results}

\subsection{Performance Comparison}

Table~\ref{table::additional_datasets} completes the $12$-dataset benchmark with $3$ datasets: a self-correcting process (SC), a stationary renewal process (SR), and Yelp-Airport (YLA). The main text considers them when reporting results. We enumerate them here for space reasons. Appendix~\ref{sect::datasets} describes them, Table~\ref{table::datasets_stats} gives their statistics, and Figures~\ref{fig::self_correcting_diagnostics}, \ref{fig::stationary_renewal_diagnostics}, and~\ref{fig::yelp_airport_diagnostics} show their diagnostics.

\begin{sidewaystable}[p]
\centering
\begin{adjustbox}{max totalsize={\textheight}{\textwidth}}
\begin{minipage}{\textheight}
\centering
\resultcaption{Performance of \sigTPP~and baselines across all $12$ datasets (including the $3$ extra datasets not shown in the main text).}
\label{table::additional_datasets}
\smallskip
\scriptsize
\renewcommand{\arraystretch}{0.78}
\setlength{\tabcolsep}{3pt}
\setlength{\aboverulesep}{1pt}
\setlength{\belowrulesep}{1pt}
\begin{adjustbox}{max width=\linewidth}
\begin{tabular}[t]{c@{\hspace{3pt}}l*{6}{r}c@{}*{6}{r}}
  \cmidrule[\heavyrulewidth](lr){1-8}\cmidrule[\heavyrulewidth](lr){10-15}
  Metric & Model & \multicolumn{1}{c}{PS} & \multicolumn{1}{c}{IP} & \multicolumn{1}{c}{H1} & \multicolumn{1}{c}{H3} & \multicolumn{1}{c}{SC} & \multicolumn{1}{c}{SR} & \cellcolor{white}\makebox[0.8em]{} & \multicolumn{1}{c}{EQ} & \multicolumn{1}{c}{SO} & \multicolumn{1}{c}{TB} & \multicolumn{1}{c}{TX} & \multicolumn{1}{c}{YLP} & \multicolumn{1}{c}{YLA} \\
  \cmidrule(lr){1-8}\cmidrule(lr){10-15}
  \multirow{6}{*}{\rotatebox[origin=c]{90}{\shortstack[c]{$\mathcal{E}$\\[-1pt]{\tiny $(\times 10^{\scalebox{1}[1]{-}3})$}}}}
  & \sigTPP & \textbf{0.53} \footnotesize{(0.93)} & \textbf{0.42} \footnotesize{(0.32)} & \textbf{1.16} \footnotesize{(0.74)} & \textbf{1.32} \footnotesize{(1.60)} & \underline{1.01} \footnotesize{(0.25)} & \textbf{2.18} \footnotesize{(2.09)} & \cellcolor{white}\makebox[0.8em]{} & \underline{32.4} \footnotesize{(0.65)} & \underline{12.0} \footnotesize{(1.71)} & 512 \footnotesize{(25.2)} & \textbf{76.3} \footnotesize{(2.95)} & \underline{38.4} \footnotesize{(14.9)} & \textbf{7.77} \footnotesize{(4.47)} \\
  & \cellcolor{gray!15}VAE & \cellcolor{gray!15}8.64 \footnotesize{(3.47)} & \cellcolor{gray!15}\underline{2.24} \footnotesize{(0.64)} & \cellcolor{gray!15}16.8 \footnotesize{(3.85)} & \cellcolor{gray!15}20.0 \footnotesize{(6.11)} & \cellcolor{gray!15}1.51 \footnotesize{(0.37)} & \cellcolor{gray!15}73.19 \footnotesize{(14.4)} & \cellcolor{white}\makebox[0.8em]{} & \cellcolor{gray!15}34.5 \footnotesize{(0.64)} & \cellcolor{gray!15}\textbf{9.64} \footnotesize{(2.55)} & \cellcolor{gray!15}545 \footnotesize{(24.7)} & \cellcolor{gray!15}77.7 \footnotesize{(2.73)} & \cellcolor{gray!15}\textbf{33.6} \footnotesize{(17.6)} & \cellcolor{gray!15}27.89 \footnotesize{(13.9)} \\
  & DDPM & \underline{2.72} \footnotesize{(1.11)} & 2.56 \footnotesize{(0.63)} & \underline{10.7} \footnotesize{(2.73)} & \underline{1.61} \footnotesize{(1.35)} & \textbf{0.76} \footnotesize{(0.16)} & \underline{7.93} \footnotesize{(4.93)} & \cellcolor{white}\makebox[0.8em]{} & \textbf{26.0} \footnotesize{(0.66)} & 24.5 \footnotesize{(6.37)} & \textbf{455} \footnotesize{(21.5)} & \underline{76.3} \footnotesize{(2.77)} & 57.5 \footnotesize{(20.5)} & \underline{13.08} \footnotesize{(6.52)} \\
  & \cellcolor{gray!15}WGAN & \cellcolor{gray!15}4.99 \footnotesize{(1.83)} & \cellcolor{gray!15}3.20 \footnotesize{(1.03)} & \cellcolor{gray!15}16.9 \footnotesize{(3.93)} & \cellcolor{gray!15}11.4 \footnotesize{(4.00)} & \cellcolor{gray!15}17.66 \footnotesize{(1.58)} & \cellcolor{gray!15}68.33 \footnotesize{(14.9)} & \cellcolor{white}\makebox[0.8em]{} & \cellcolor{gray!15}32.6 \footnotesize{(0.71)} & \cellcolor{gray!15}26.5 \footnotesize{(3.34)} & \cellcolor{gray!15}\underline{502} \footnotesize{(22.4)} & \cellcolor{gray!15}81.9 \footnotesize{(2.80)} & \cellcolor{gray!15}49.2 \footnotesize{(24.3)} & \cellcolor{gray!15}22.22 \footnotesize{(12.4)} \\
  & \determodel & 330 \footnotesize{(12.8)} & 108 \footnotesize{(4.85)} & 510 \footnotesize{(15.7)} & 640 \footnotesize{(25.3)} & 22.62 \footnotesize{(1.15)} & 2824 \footnotesize{(44.0)} & \cellcolor{white}\makebox[0.8em]{} & 47.2 \footnotesize{(0.92)} & 478 \footnotesize{(22.6)} & 958 \footnotesize{(32.1)} & 161 \footnotesize{(5.46)} & 2402 \footnotesize{(90)} & 984.2 \footnotesize{(48.2)} \\
  & \cellcolor{gray!15}\gammamodel & \cellcolor{gray!15}4.38 \footnotesize{(2.58)} & \cellcolor{gray!15}20.3 \footnotesize{(3.22)} & \cellcolor{gray!15}17.3 \footnotesize{(3.87)} & \cellcolor{gray!15}25.5 \footnotesize{(5.92)} & \cellcolor{gray!15}24.93 \footnotesize{(2.13)} & \cellcolor{gray!15}13.54 \footnotesize{(3.40)} & \cellcolor{white}\makebox[0.8em]{} & \cellcolor{gray!15}35.8 \footnotesize{(0.66)} & \cellcolor{gray!15}31.4 \footnotesize{(3.78)} & \cellcolor{gray!15}539 \footnotesize{(23.5)} & \cellcolor{gray!15}85.6 \footnotesize{(2.99)} & \cellcolor{gray!15}250 \footnotesize{(59.0)} & \cellcolor{gray!15}45.93 \footnotesize{(15.7)} \\
  \cmidrule(lr){2-8}\cmidrule(lr){10-15}
  \multirow{6}{*}{\rotatebox[origin=c]{90}{\shortstack[c]{$\WAS$\\[-1pt]{\tiny $(\times 10^{\scalebox{1}[1]{-}2})$}}}}
  & \sigTPP & \textbf{5.26} \footnotesize{(0.20)} & \textbf{6.55} \footnotesize{(0.14)} & \textbf{3.75} \footnotesize{(0.31)} & \textbf{10.1} \footnotesize{(0.42)} & 2.55 \footnotesize{(0.05)} & \textbf{10.25} \footnotesize{(0.60)} & \cellcolor{white}\makebox[0.8em]{} & \textbf{2.91} \footnotesize{(0.04)} & \underline{7.44} \footnotesize{(0.42)} & 58.2 \footnotesize{(2.56)} & \underline{8.02} \footnotesize{(0.22)} & \textbf{15.0} \footnotesize{(1.86)} & \textbf{8.48} \footnotesize{(0.73)} \\
  & \cellcolor{gray!15}VAE & \cellcolor{gray!15}7.12 \footnotesize{(0.64)} & \cellcolor{gray!15}\underline{6.76} \footnotesize{(0.18)} & \cellcolor{gray!15}6.84 \footnotesize{(0.57)} & \cellcolor{gray!15}13.8 \footnotesize{(1.11)} & \cellcolor{gray!15}\textbf{2.42} \footnotesize{(0.04)} & \cellcolor{gray!15}22.31 \footnotesize{(1.99)} & \cellcolor{white}\makebox[0.8em]{} & \cellcolor{gray!15}\underline{2.91} \footnotesize{(0.05)} & \cellcolor{gray!15}\textbf{7.08} \footnotesize{(0.76)} & \cellcolor{gray!15}59.8 \footnotesize{(2.40)} & \cellcolor{gray!15}8.13 \footnotesize{(0.23)} & \cellcolor{gray!15}\underline{16.5} \footnotesize{(2.90)} & \cellcolor{gray!15}10.83 \footnotesize{(1.45)} \\
  & DDPM & \underline{6.15} \footnotesize{(0.24)} & 7.27 \footnotesize{(0.15)} & \underline{5.93} \footnotesize{(0.49)} & \underline{10.5} \footnotesize{(0.50)} & \underline{2.48} \footnotesize{(0.03)} & \underline{11.80} \footnotesize{(0.95)} & \cellcolor{white}\makebox[0.8em]{} & 3.02 \footnotesize{(0.05)} & 10.9 \footnotesize{(1.22)} & \textbf{52.8} \footnotesize{(2.31)} & \textbf{7.74} \footnotesize{(0.18)} & 20.0 \footnotesize{(2.06)} & \underline{10.27} \footnotesize{(0.76)} \\
  & \cellcolor{gray!15}WGAN & \cellcolor{gray!15}6.43 \footnotesize{(0.40)} & \cellcolor{gray!15}7.68 \footnotesize{(0.29)} & \cellcolor{gray!15}6.97 \footnotesize{(0.61)} & \cellcolor{gray!15}12.8 \footnotesize{(0.83)} & \cellcolor{gray!15}5.12 \footnotesize{(0.18)} & \cellcolor{gray!15}23.36 \footnotesize{(2.44)} & \cellcolor{white}\makebox[0.8em]{} & \cellcolor{gray!15}2.99 \footnotesize{(0.04)} & \cellcolor{gray!15}10.3 \footnotesize{(0.64)} & \cellcolor{gray!15}\underline{56.1} \footnotesize{(2.07)} & \cellcolor{gray!15}8.38 \footnotesize{(0.27)} & \cellcolor{gray!15}18.8 \footnotesize{(2.78)} & \cellcolor{gray!15}10.76 \footnotesize{(1.48)} \\
  & \determodel & 28.3 \footnotesize{(0.65)} & 18.2 \footnotesize{(0.34)} & 38.1 \footnotesize{(0.62)} & 54.6 \footnotesize{(1.10)} & 3.66 \footnotesize{(0.08)} & 166.21 \footnotesize{(1.85)} & \cellcolor{white}\makebox[0.8em]{} & 3.74 \footnotesize{(0.07)} & 36.2 \footnotesize{(0.80)} & 86.0 \footnotesize{(2.48)} & 14.1 \footnotesize{(0.40)} & 143 \footnotesize{(3.55)} & 61.20 \footnotesize{(1.98)} \\
  & \cellcolor{gray!15}\gammamodel & \cellcolor{gray!15}6.33 \footnotesize{(0.54)} & \cellcolor{gray!15}9.50 \footnotesize{(0.49)} & \cellcolor{gray!15}7.05 \footnotesize{(0.59)} & \cellcolor{gray!15}15.1 \footnotesize{(1.07)} & \cellcolor{gray!15}7.27 \footnotesize{(0.23)} & \cellcolor{gray!15}13.24 \footnotesize{(0.90)} & \cellcolor{white}\makebox[0.8em]{} & \cellcolor{gray!15}3.11 \footnotesize{(0.05)} & \cellcolor{gray!15}11.0 \footnotesize{(0.66)} & \cellcolor{gray!15}60.3 \footnotesize{(2.15)} & \cellcolor{gray!15}8.79 \footnotesize{(0.28)} & \cellcolor{gray!15}35.2 \footnotesize{(3.68)} & \cellcolor{gray!15}12.64 \footnotesize{(1.56)} \\
  \cmidrule(lr){2-8}\cmidrule(lr){10-15}
  \multirow{6}{*}{\rotatebox[origin=c]{90}{\shortstack[c]{$\sigW$\\[-1pt]{\tiny $(\times 10^{\scalebox{1}[1]{-}5})$}}}}
  & \sigTPP & \textbf{1.42} \footnotesize{(1.52)} & 1.20 \footnotesize{(1.05)} & \textbf{44.0} \footnotesize{(29.2)} & 1.02 \footnotesize{(0.75)} & 2.81 \footnotesize{(2.29)} & \textbf{2.32} \footnotesize{(1.91)} & \cellcolor{white}\makebox[0.8em]{} & 409 \footnotesize{(36.3)} & \textbf{0.09} \footnotesize{(0.10)} & \underline{2.27} \footnotesize{(0.94)} & \underline{4.74} \footnotesize{(1.22)} & \underline{0.77} \footnotesize{(0.86)} & 4.64 \footnotesize{(4.30)} \\
  & \cellcolor{gray!15}VAE & \cellcolor{gray!15}13.4 \footnotesize{(6.42)} & \cellcolor{gray!15}\textbf{0.83} \footnotesize{(0.89)} & \cellcolor{gray!15}155 \footnotesize{(28.1)} & \cellcolor{gray!15}\textbf{0.25} \footnotesize{(0.26)} & \cellcolor{gray!15}\textbf{0.57} \footnotesize{(0.62)} & \cellcolor{gray!15}25.37 \footnotesize{(4.20)} & \cellcolor{white}\makebox[0.8em]{} & \cellcolor{gray!15}\underline{397} \footnotesize{(38.1)} & \cellcolor{gray!15}3.70 \footnotesize{(0.52)} & \cellcolor{gray!15}\textbf{0.45} \footnotesize{(0.30)} & \cellcolor{gray!15}9.63 \footnotesize{(1.61)} & \cellcolor{gray!15}1.08 \footnotesize{(1.10)} & \cellcolor{gray!15}\textbf{2.01} \footnotesize{(2.33)} \\
  & DDPM & 129 \footnotesize{(24.3)} & 64.0 \footnotesize{(9.18)} & 302 \footnotesize{(65.9)} & 10.9 \footnotesize{(2.70)} & 19.35 \footnotesize{(5.75)} & \underline{6.37} \footnotesize{(2.06)} & \cellcolor{white}\makebox[0.8em]{} & 1144 \footnotesize{(56)} & 2.58 \footnotesize{(0.44)} & 9.78 \footnotesize{(1.99)} & 5.52 \footnotesize{(1.21)} & 1.70 \footnotesize{(1.08)} & 8.39 \footnotesize{(5.68)} \\
  & \cellcolor{gray!15}WGAN & \cellcolor{gray!15}17.4 \footnotesize{(6.92)} & \cellcolor{gray!15}\underline{1.01} \footnotesize{(0.70)} & \cellcolor{gray!15}\underline{141} \footnotesize{(30.9)} & \cellcolor{gray!15}0.80 \footnotesize{(0.70)} & \cellcolor{gray!15}\underline{1.59} \footnotesize{(1.19)} & \cellcolor{gray!15}17.56 \footnotesize{(4.25)} & \cellcolor{white}\makebox[0.8em]{} & \cellcolor{gray!15}603 \footnotesize{(45.9)} & \cellcolor{gray!15}3.95 \footnotesize{(0.64)} & \cellcolor{gray!15}7.81 \footnotesize{(1.44)} & \cellcolor{gray!15}6.80 \footnotesize{(1.51)} & \cellcolor{gray!15}\textbf{0.42} \footnotesize{(0.47)} & \cellcolor{gray!15}\underline{2.11} \footnotesize{(2.45)} \\
  & \determodel & 62.7 \footnotesize{(12.7)} & 8.01 \footnotesize{(1.76)} & 862 \footnotesize{(112)} & 1.09 \footnotesize{(0.39)} & 2.84 \footnotesize{(2.07)} & 11.69 \footnotesize{(3.14)} & \cellcolor{white}\makebox[0.8em]{} & \textbf{146} \footnotesize{(14.5)} & \underline{0.19} \footnotesize{(0.16)} & 49.7 \footnotesize{(4.45)} & 24.5 \footnotesize{(2.82)} & 7.78 \footnotesize{(0.61)} & 16.56 \footnotesize{(7.03)} \\
  & \cellcolor{gray!15}\gammamodel & \cellcolor{gray!15}\underline{5.76} \footnotesize{(4.25)} & \cellcolor{gray!15}23.8 \footnotesize{(4.84)} & \cellcolor{gray!15}284 \footnotesize{(51.3)} & \cellcolor{gray!15}\underline{0.63} \footnotesize{(0.29)} & \cellcolor{gray!15}4.94 \footnotesize{(1.64)} & \cellcolor{gray!15}43.94 \footnotesize{(9.84)} & \cellcolor{white}\makebox[0.8em]{} & \cellcolor{gray!15}572 \footnotesize{(38.1)} & \cellcolor{gray!15}6.66 \footnotesize{(0.81)} & \cellcolor{gray!15}12.3 \footnotesize{(3.47)} & \cellcolor{gray!15}\textbf{3.07} \footnotesize{(1.16)} & \cellcolor{gray!15}10.6 \footnotesize{(1.40)} & \cellcolor{gray!15}3.11 \footnotesize{(1.96)} \\
  \cmidrule(lr){2-8}\cmidrule(lr){10-15}
  \multirow{6}{*}{\rotatebox[origin=c]{90}{\shortstack[c]{CRPS\\[-1pt]{\tiny $(\times 10^{\scalebox{1}[1]{-}1})$}}}}
  & \sigTPP & 4.64 \footnotesize{(0.08)} & \underline{3.04} \footnotesize{(0.03)} & 10.2 \footnotesize{(0.27)} & 2.40 \footnotesize{(0.03)} & \underline{3.69} \footnotesize{(0.03)} & 4.42 \footnotesize{(0.10)} & \cellcolor{white}\makebox[0.8em]{} & 5.61 \footnotesize{(0.11)} & 5.43 \footnotesize{(0.10)} & 0.37 \footnotesize{(0.01)} & \underline{1.13} \footnotesize{(0.02)} & 3.22 \footnotesize{(0.13)} & 4.78 \footnotesize{(0.14)} \\
  & \cellcolor{gray!15}VAE & \cellcolor{gray!15}\textbf{4.56} \footnotesize{(0.08)} & \cellcolor{gray!15}3.07 \footnotesize{(0.03)} & \cellcolor{gray!15}\textbf{9.53} \footnotesize{(0.26)} & \cellcolor{gray!15}\underline{2.33} \footnotesize{(0.03)} & \cellcolor{gray!15}\textbf{3.37} \footnotesize{(0.02)} & \cellcolor{gray!15}\textbf{4.30} \footnotesize{(0.10)} & \cellcolor{white}\makebox[0.8em]{} & \cellcolor{gray!15}\underline{5.13} \footnotesize{(0.12)} & \cellcolor{gray!15}\underline{4.64} \footnotesize{(0.09)} & \cellcolor{gray!15}\underline{0.36} \footnotesize{(0.01)} & \cellcolor{gray!15}\textbf{1.06} \footnotesize{(0.02)} & \cellcolor{gray!15}\textbf{2.74} \footnotesize{(0.12)} & \cellcolor{gray!15}\underline{4.67} \footnotesize{(0.14)} \\
  & DDPM & 4.81 \footnotesize{(0.09)} & 3.27 \footnotesize{(0.04)} & 10.1 \footnotesize{(0.30)} & \textbf{2.23} \footnotesize{(0.03)} & 5.50 \footnotesize{(0.05)} & 4.37 \footnotesize{(0.10)} & \cellcolor{white}\makebox[0.8em]{} & \textbf{4.57} \footnotesize{(0.11)} & \textbf{4.63} \footnotesize{(0.10)} & \textbf{0.34} \footnotesize{(0.01)} & 1.20 \footnotesize{(0.02)} & \underline{2.76} \footnotesize{(0.11)} & \textbf{4.62} \footnotesize{(0.14)} \\
  & \cellcolor{gray!15}WGAN & \cellcolor{gray!15}\underline{4.58} \footnotesize{(0.08)} & \cellcolor{gray!15}\textbf{3.03} \footnotesize{(0.03)} & \cellcolor{gray!15}\underline{9.59} \footnotesize{(0.28)} & \cellcolor{gray!15}2.36 \footnotesize{(0.03)} & \cellcolor{gray!15}3.94 \footnotesize{(0.03)} & \cellcolor{gray!15}\underline{4.32} \footnotesize{(0.10)} & \cellcolor{white}\makebox[0.8em]{} & \cellcolor{gray!15}5.85 \footnotesize{(0.10)} & \cellcolor{gray!15}4.93 \footnotesize{(0.09)} & \cellcolor{gray!15}0.37 \footnotesize{(0.01)} & \cellcolor{gray!15}1.15 \footnotesize{(0.02)} & \cellcolor{gray!15}2.91 \footnotesize{(0.11)} & \cellcolor{gray!15}4.81 \footnotesize{(0.14)} \\
  & \determodel & 6.29 \footnotesize{(0.11)} & 4.34 \footnotesize{(0.05)} & 12.7 \footnotesize{(0.37)} & 3.23 \footnotesize{(0.04)} & 4.94 \footnotesize{(0.03)} & 5.22 \footnotesize{(0.12)} & \cellcolor{white}\makebox[0.8em]{} & 7.05 \footnotesize{(0.16)} & 6.51 \footnotesize{(0.12)} & 0.42 \footnotesize{(0.01)} & 1.49 \footnotesize{(0.02)} & 3.66 \footnotesize{(0.15)} & 6.15 \footnotesize{(0.19)} \\
  & \cellcolor{gray!15}\gammamodel & \cellcolor{gray!15}4.60 \footnotesize{(0.08)} & \cellcolor{gray!15}3.21 \footnotesize{(0.03)} & \cellcolor{gray!15}9.95 \footnotesize{(0.26)} & \cellcolor{gray!15}2.39 \footnotesize{(0.03)} & \cellcolor{gray!15}3.96 \footnotesize{(0.03)} & \cellcolor{gray!15}4.45 \footnotesize{(0.10)} & \cellcolor{white}\makebox[0.8em]{} & \cellcolor{gray!15}6.16 \footnotesize{(0.09)} & \cellcolor{gray!15}4.91 \footnotesize{(0.08)} & \cellcolor{gray!15}0.37 \footnotesize{(0.01)} & \cellcolor{gray!15}1.15 \footnotesize{(0.02)} & \cellcolor{gray!15}3.19 \footnotesize{(0.11)} & \cellcolor{gray!15}4.83 \footnotesize{(0.14)} \\
  \cmidrule(lr){2-8}\cmidrule(lr){10-15}
  \multirow{6}{*}{\rotatebox[origin=c]{90}{\shortstack[c]{$\histintloss$\\[-1pt]{\tiny $(\times 10^{\scalebox{1}[1]{-}2})$}}}}
  & \sigTPP & \underline{3.35} \footnotesize{(0.46)} & \textbf{3.09} \footnotesize{(0.35)} & \textbf{6.09} \footnotesize{(1.04)} & \underline{2.68} \footnotesize{(0.40)} & \textbf{2.44} \footnotesize{(0.39)} & \textbf{4.01} \footnotesize{(0.72)} & \cellcolor{white}\makebox[0.8em]{} & 19.3 \footnotesize{(0.99)} & 9.00 \footnotesize{(0.68)} & \textbf{3.77} \footnotesize{(0.90)} & 6.31 \footnotesize{(1.05)} & 13.3 \footnotesize{(1.14)} & \underline{13.55} \footnotesize{(1.20)} \\
  & \cellcolor{gray!15}VAE & \cellcolor{gray!15}4.46 \footnotesize{(0.57)} & \cellcolor{gray!15}4.90 \footnotesize{(0.43)} & \cellcolor{gray!15}13.1 \footnotesize{(1.22)} & \cellcolor{gray!15}\textbf{2.25} \footnotesize{(0.40)} & \cellcolor{gray!15}\underline{2.66} \footnotesize{(0.36)} & \cellcolor{gray!15}9.61 \footnotesize{(0.76)} & \cellcolor{white}\makebox[0.8em]{} & \cellcolor{gray!15}\textbf{17.9} \footnotesize{(1.12)} & \cellcolor{gray!15}8.94 \footnotesize{(0.69)} & \cellcolor{gray!15}9.57 \footnotesize{(1.06)} & \cellcolor{gray!15}6.99 \footnotesize{(1.04)} & \cellcolor{gray!15}\textbf{10.8} \footnotesize{(1.14)} & \cellcolor{gray!15}\textbf{9.34} \footnotesize{(0.94)} \\
  & DDPM & 7.93 \footnotesize{(0.64)} & 7.18 \footnotesize{(0.41)} & 15.2 \footnotesize{(1.05)} & 5.10 \footnotesize{(0.50)} & 3.77 \footnotesize{(0.34)} & 8.71 \footnotesize{(0.77)} & \cellcolor{white}\makebox[0.8em]{} & \underline{19.0} \footnotesize{(0.62)} & 9.95 \footnotesize{(0.63)} & \underline{7.01} \footnotesize{(1.08)} & \textbf{4.57} \footnotesize{(1.04)} & 16.9 \footnotesize{(1.11)} & 15.09 \footnotesize{(0.95)} \\
  & \cellcolor{gray!15}WGAN & \cellcolor{gray!15}7.31 \footnotesize{(0.67)} & \cellcolor{gray!15}\underline{4.47} \footnotesize{(0.37)} & \cellcolor{gray!15}13.3 \footnotesize{(1.31)} & \cellcolor{gray!15}3.37 \footnotesize{(0.49)} & \cellcolor{gray!15}4.07 \footnotesize{(0.36)} & \cellcolor{gray!15}8.96 \footnotesize{(0.81)} & \cellcolor{white}\makebox[0.8em]{} & \cellcolor{gray!15}24.1 \footnotesize{(1.13)} & \cellcolor{gray!15}\textbf{6.25} \footnotesize{(0.77)} & \cellcolor{gray!15}14.0 \footnotesize{(1.24)} & \cellcolor{gray!15}\underline{5.06} \footnotesize{(0.74)} & \cellcolor{gray!15}\underline{13.3} \footnotesize{(1.03)} & \cellcolor{gray!15}14.34 \footnotesize{(0.77)} \\
  & \determodel & 12.6 \footnotesize{(0.71)} & 12.8 \footnotesize{(0.43)} & 20.2 \footnotesize{(1.25)} & 4.17 \footnotesize{(0.52)} & 5.78 \footnotesize{(0.41)} & 25.51 \footnotesize{(1.09)} & \cellcolor{white}\makebox[0.8em]{} & 20.3 \footnotesize{(2.84)} & \underline{7.38} \footnotesize{(0.37)} & 41.6 \footnotesize{(3.32)} & 18.0 \footnotesize{(1.33)} & 25.3 \footnotesize{(0.93)} & 18.17 \footnotesize{(1.25)} \\
  & \cellcolor{gray!15}\gammamodel & \cellcolor{gray!15}\textbf{3.17} \footnotesize{(0.45)} & \cellcolor{gray!15}11.0 \footnotesize{(0.47)} & \cellcolor{gray!15}\underline{8.11} \footnotesize{(1.08)} & \cellcolor{gray!15}3.41 \footnotesize{(0.43)} & \cellcolor{gray!15}4.32 \footnotesize{(0.32)} & \cellcolor{gray!15}\underline{5.92} \footnotesize{(0.67)} & \cellcolor{white}\makebox[0.8em]{} & \cellcolor{gray!15}27.0 \footnotesize{(1.12)} & \cellcolor{gray!15}11.5 \footnotesize{(0.60)} & \cellcolor{gray!15}16.2 \footnotesize{(1.08)} & \cellcolor{gray!15}7.05 \footnotesize{(0.73)} & \cellcolor{gray!15}31.7 \footnotesize{(1.20)} & \cellcolor{gray!15}18.58 \footnotesize{(0.80)} \\
  \cmidrule(lr){2-8}\cmidrule(lr){10-15}
  \multirow{6}{*}{\rotatebox[origin=c]{90}{\shortstack[c]{$\histITloss$\\[-1pt]{\tiny $(\times 10^{\scalebox{1}[1]{-}2})$}}}}
  & \sigTPP & \underline{2.95} \footnotesize{(0.60)} & \textbf{1.97} \footnotesize{(0.37)} & \textbf{3.84} \footnotesize{(0.84)} & 7.64 \footnotesize{(0.56)} & \underline{3.75} \footnotesize{(0.61)} & \underline{2.31} \footnotesize{(0.38)} & \cellcolor{white}\makebox[0.8em]{} & 12.9 \footnotesize{(0.84)} & 10.5 \footnotesize{(0.41)} & \underline{7.33} \footnotesize{(0.39)} & \underline{4.35} \footnotesize{(0.44)} & \underline{4.73} \footnotesize{(0.99)} & 6.34 \footnotesize{(1.12)} \\
  & \cellcolor{gray!15}VAE & \cellcolor{gray!15}4.16 \footnotesize{(0.58)} & \cellcolor{gray!15}5.47 \footnotesize{(0.46)} & \cellcolor{gray!15}\underline{4.07} \footnotesize{(0.67)} & \cellcolor{gray!15}4.01 \footnotesize{(0.48)} & \cellcolor{gray!15}6.86 \footnotesize{(0.64)} & \cellcolor{gray!15}3.39 \footnotesize{(0.44)} & \cellcolor{white}\makebox[0.8em]{} & \cellcolor{gray!15}14.8 \footnotesize{(0.97)} & \cellcolor{gray!15}10.6 \footnotesize{(0.56)} & \cellcolor{gray!15}\textbf{6.29} \footnotesize{(0.53)} & \cellcolor{gray!15}4.36 \footnotesize{(0.40)} & \cellcolor{gray!15}\textbf{4.02} \footnotesize{(0.68)} & \cellcolor{gray!15}\textbf{4.70} \footnotesize{(0.97)} \\
  & DDPM & 6.08 \footnotesize{(0.72)} & 4.10 \footnotesize{(0.45)} & 5.45 \footnotesize{(0.89)} & \underline{2.32} \footnotesize{(0.47)} & \textbf{2.57} \footnotesize{(0.57)} & \textbf{2.07} \footnotesize{(0.35)} & \cellcolor{white}\makebox[0.8em]{} & \textbf{9.33} \footnotesize{(0.82)} & \underline{5.22} \footnotesize{(0.63)} & 10.6 \footnotesize{(0.72)} & \textbf{2.02} \footnotesize{(0.40)} & 6.47 \footnotesize{(0.80)} & 6.82 \footnotesize{(1.25)} \\
  & \cellcolor{gray!15}WGAN & \cellcolor{gray!15}3.74 \footnotesize{(0.55)} & \cellcolor{gray!15}4.70 \footnotesize{(0.46)} & \cellcolor{gray!15}5.33 \footnotesize{(0.74)} & \cellcolor{gray!15}4.70 \footnotesize{(0.56)} & \cellcolor{gray!15}15.78 \footnotesize{(1.96)} & \cellcolor{gray!15}4.28 \footnotesize{(0.35)} & \cellcolor{white}\makebox[0.8em]{} & \cellcolor{gray!15}\underline{9.58} \footnotesize{(0.58)} & \cellcolor{gray!15}\textbf{2.72} \footnotesize{(0.62)} & \cellcolor{gray!15}7.86 \footnotesize{(0.70)} & \cellcolor{gray!15}6.72 \footnotesize{(0.55)} & \cellcolor{gray!15}5.69 \footnotesize{(0.93)} & \cellcolor{gray!15}11.09 \footnotesize{(1.42)} \\
  & \determodel & 60.7 \footnotesize{(5.26)} & 48.4 \footnotesize{(1.46)} & 53.5 \footnotesize{(2.86)} & 66.6 \footnotesize{(6.25)} & 42.19 \footnotesize{(8.80)} & 81.21 \footnotesize{(3.83)} & \cellcolor{white}\makebox[0.8em]{} & 48.5 \footnotesize{(3.15)} & 54.9 \footnotesize{(7.16)} & 81.0 \footnotesize{(2.02)} & 56.9 \footnotesize{(6.67)} & 45.5 \footnotesize{(1.62)} & 38.25 \footnotesize{(1.16)} \\
  & \cellcolor{gray!15}\gammamodel & \cellcolor{gray!15}\textbf{2.86} \footnotesize{(0.64)} & \cellcolor{gray!15}\underline{3.21} \footnotesize{(0.48)} & \cellcolor{gray!15}6.74 \footnotesize{(1.00)} & \cellcolor{gray!15}\textbf{1.71} \footnotesize{(0.59)} & \cellcolor{gray!15}9.02 \footnotesize{(0.52)} & \cellcolor{gray!15}12.32 \footnotesize{(0.47)} & \cellcolor{white}\makebox[0.8em]{} & \cellcolor{gray!15}9.93 \footnotesize{(0.83)} & \cellcolor{gray!15}5.70 \footnotesize{(0.75)} & \cellcolor{gray!15}15.1 \footnotesize{(1.26)} & \cellcolor{gray!15}4.84 \footnotesize{(0.41)} & \cellcolor{gray!15}13.2 \footnotesize{(0.85)} & \cellcolor{gray!15}\underline{5.07} \footnotesize{(0.90)} \\
  \cmidrule(lr){2-8}\cmidrule(lr){10-15}
  \multirow{6}{*}{\rotatebox[origin=c]{90}{\shortstack[c]{ACD\\[-1pt]{\tiny $(\times 10^{\scalebox{1}[1]{-}2})$}}}}
  & \sigTPP & \underline{1.35} \footnotesize{(0.46)} & \textbf{0.70} \footnotesize{(0.20)} & \underline{1.70} \footnotesize{(0.45)} & \textbf{1.54} \footnotesize{(0.39)} & \textbf{1.44} \footnotesize{(0.32)} & 0.98 \footnotesize{(0.22)} & \cellcolor{white}\makebox[0.8em]{} & \textbf{3.95} \footnotesize{(0.68)} & 8.93 \footnotesize{(0.51)} & \textbf{0.94} \footnotesize{(0.20)} & 2.97 \footnotesize{(0.43)} & \textbf{2.11} \footnotesize{(0.40)} & \underline{2.45} \footnotesize{(0.56)} \\
  & \cellcolor{gray!15}VAE & \cellcolor{gray!15}1.53 \footnotesize{(0.42)} & \cellcolor{gray!15}\underline{0.71} \footnotesize{(0.20)} & \cellcolor{gray!15}2.73 \footnotesize{(0.62)} & \cellcolor{gray!15}2.44 \footnotesize{(0.45)} & \cellcolor{gray!15}2.37 \footnotesize{(0.33)} & \cellcolor{gray!15}0.72 \footnotesize{(0.20)} & \cellcolor{white}\makebox[0.8em]{} & \cellcolor{gray!15}7.17 \footnotesize{(0.76)} & \cellcolor{gray!15}\textbf{2.79} \footnotesize{(0.51)} & \cellcolor{gray!15}2.48 \footnotesize{(0.40)} & \cellcolor{gray!15}\underline{2.26} \footnotesize{(0.42)} & \cellcolor{gray!15}\underline{2.61} \footnotesize{(0.53)} & \cellcolor{gray!15}2.61 \footnotesize{(0.58)} \\
  & DDPM & 1.86 \footnotesize{(0.52)} & 1.64 \footnotesize{(0.25)} & 2.54 \footnotesize{(0.47)} & \underline{1.76} \footnotesize{(0.44)} & \underline{1.82} \footnotesize{(0.32)} & \underline{0.65} \footnotesize{(0.20)} & \cellcolor{white}\makebox[0.8em]{} & 6.95 \footnotesize{(0.76)} & \underline{5.17} \footnotesize{(0.51)} & \underline{1.41} \footnotesize{(0.35)} & \textbf{1.87} \footnotesize{(0.45)} & 3.35 \footnotesize{(0.60)} & \textbf{2.27} \footnotesize{(0.56)} \\
  & \cellcolor{gray!15}WGAN & \cellcolor{gray!15}2.39 \footnotesize{(0.49)} & \cellcolor{gray!15}1.94 \footnotesize{(0.28)} & \cellcolor{gray!15}\textbf{1.60} \footnotesize{(0.39)} & \cellcolor{gray!15}5.03 \footnotesize{(0.45)} & \cellcolor{gray!15}6.83 \footnotesize{(0.33)} & \cellcolor{gray!15}0.76 \footnotesize{(0.21)} & \cellcolor{white}\makebox[0.8em]{} & \cellcolor{gray!15}\underline{6.95} \footnotesize{(0.76)} & \cellcolor{gray!15}6.57 \footnotesize{(0.51)} & \cellcolor{gray!15}3.20 \footnotesize{(0.41)} & \cellcolor{gray!15}3.48 \footnotesize{(0.47)} & \cellcolor{gray!15}5.51 \footnotesize{(0.46)} & \cellcolor{gray!15}2.93 \footnotesize{(0.62)} \\
  & \determodel & 49.3 \footnotesize{(0.53)} & 35.3 \footnotesize{(0.28)} & 52.5 \footnotesize{(0.71)} & 63.4 \footnotesize{(0.45)} & 29.58 \footnotesize{(0.35)} & 75.24 \footnotesize{(0.25)} & \cellcolor{white}\makebox[0.8em]{} & 61.9 \footnotesize{(0.76)} & 72.7 \footnotesize{(0.51)} & 74.2 \footnotesize{(0.41)} & 60.7 \footnotesize{(0.50)} & 64.4 \footnotesize{(0.62)} & 74.03 \footnotesize{(0.62)} \\
  & \cellcolor{gray!15}\gammamodel & \cellcolor{gray!15}\textbf{1.32} \footnotesize{(0.44)} & \cellcolor{gray!15}1.28 \footnotesize{(0.26)} & \cellcolor{gray!15}4.08 \footnotesize{(0.71)} & \cellcolor{gray!15}6.12 \footnotesize{(0.45)} & \cellcolor{gray!15}7.95 \footnotesize{(0.23)} & \cellcolor{gray!15}\textbf{0.62} \footnotesize{(0.19)} & \cellcolor{white}\makebox[0.8em]{} & \cellcolor{gray!15}16.8 \footnotesize{(0.76)} & \cellcolor{gray!15}8.03 \footnotesize{(0.51)} & \cellcolor{gray!15}2.92 \footnotesize{(0.41)} & \cellcolor{gray!15}3.51 \footnotesize{(0.47)} & \cellcolor{gray!15}4.61 \footnotesize{(0.61)} & \cellcolor{gray!15}2.86 \footnotesize{(0.59)} \\
  \cmidrule(lr){2-8}\cmidrule(lr){10-15}
  \multirow{6}{*}{\rotatebox[origin=c]{90}{\shortstack[c]{PCD\\[-1pt]{\tiny $(\times 10^{\scalebox{1}[1]{-}2})$}}}}
  & \sigTPP & \textbf{7.36} \footnotesize{(0.53)} & \textbf{6.01} \footnotesize{(0.42)} & 9.07 \footnotesize{(0.85)} & 8.56 \footnotesize{(0.40)} & \underline{5.47} \footnotesize{(0.30)} & \textbf{7.72} \footnotesize{(0.44)} & \cellcolor{white}\makebox[0.8em]{} & \textbf{7.51} \footnotesize{(0.53)} & 12.4 \footnotesize{(0.33)} & 5.94 \footnotesize{(0.26)} & 6.11 \footnotesize{(0.23)} & 12.9 \footnotesize{(0.69)} & \underline{12.35} \footnotesize{(0.73)} \\
  & \cellcolor{gray!15}VAE & \cellcolor{gray!15}7.43 \footnotesize{(0.53)} & \cellcolor{gray!15}6.40 \footnotesize{(0.43)} & \cellcolor{gray!15}\underline{8.82} \footnotesize{(0.86)} & \cellcolor{gray!15}\underline{8.39} \footnotesize{(0.40)} & \cellcolor{gray!15}5.51 \footnotesize{(0.30)} & \cellcolor{gray!15}7.76 \footnotesize{(0.43)} & \cellcolor{white}\makebox[0.8em]{} & \cellcolor{gray!15}9.92 \footnotesize{(0.74)} & \cellcolor{gray!15}\textbf{10.2} \footnotesize{(0.28)} & \cellcolor{gray!15}\textbf{5.83} \footnotesize{(0.27)} & \cellcolor{gray!15}\underline{5.85} \footnotesize{(0.22)} & \cellcolor{gray!15}\textbf{12.4} \footnotesize{(0.71)} & \cellcolor{gray!15}\textbf{12.15} \footnotesize{(0.77)} \\
  & DDPM & 7.54 \footnotesize{(0.51)} & 6.27 \footnotesize{(0.43)} & \textbf{8.25} \footnotesize{(0.97)} & \textbf{8.33} \footnotesize{(0.42)} & \textbf{5.28} \footnotesize{(0.31)} & \underline{7.73} \footnotesize{(0.44)} & \cellcolor{white}\makebox[0.8em]{} & 8.44 \footnotesize{(0.69)} & 11.1 \footnotesize{(0.32)} & \underline{5.84} \footnotesize{(0.25)} & \textbf{5.61} \footnotesize{(0.22)} & \underline{12.8} \footnotesize{(0.68)} & 12.35 \footnotesize{(0.76)} \\
  & \cellcolor{gray!15}WGAN & \cellcolor{gray!15}7.49 \footnotesize{(0.53)} & \cellcolor{gray!15}\underline{6.15} \footnotesize{(0.43)} & \cellcolor{gray!15}8.87 \footnotesize{(0.84)} & \cellcolor{gray!15}8.82 \footnotesize{(0.41)} & \cellcolor{gray!15}7.63 \footnotesize{(0.28)} & \cellcolor{gray!15}7.92 \footnotesize{(0.43)} & \cellcolor{white}\makebox[0.8em]{} & \cellcolor{gray!15}\underline{7.55} \footnotesize{(0.54)} & \cellcolor{gray!15}\underline{11.0} \footnotesize{(0.31)} & \cellcolor{gray!15}5.92 \footnotesize{(0.28)} & \cellcolor{gray!15}6.25 \footnotesize{(0.22)} & \cellcolor{gray!15}13.1 \footnotesize{(0.69)} & \cellcolor{gray!15}12.77 \footnotesize{(0.76)} \\
  & \determodel & 96.5 \footnotesize{(0.43)} & 82.7 \footnotesize{(0.28)} & 59.7 \footnotesize{(0.83)} & 97.9 \footnotesize{(0.22)} & 55.48 \footnotesize{(0.22)} & 70.57 \footnotesize{(1.06)} & \cellcolor{white}\makebox[0.8em]{} & 73.4 \footnotesize{(0.81)} & 79.3 \footnotesize{(0.27)} & 94.8 \footnotesize{(0.34)} & 43.4 \footnotesize{(0.31)} & 72.4 \footnotesize{(0.80)} & 76.31 \footnotesize{(1.13)} \\
  & \cellcolor{gray!15}\gammamodel & \cellcolor{gray!15}\underline{7.39} \footnotesize{(0.53)} & \cellcolor{gray!15}6.39 \footnotesize{(0.42)} & \cellcolor{gray!15}9.03 \footnotesize{(0.85)} & \cellcolor{gray!15}8.74 \footnotesize{(0.41)} & \cellcolor{gray!15}8.34 \footnotesize{(0.27)} & \cellcolor{gray!15}7.74 \footnotesize{(0.43)} & \cellcolor{white}\makebox[0.8em]{} & \cellcolor{gray!15}18.8 \footnotesize{(0.81)} & \cellcolor{gray!15}12.0 \footnotesize{(0.35)} & \cellcolor{gray!15}5.86 \footnotesize{(0.27)} & \cellcolor{gray!15}6.24 \footnotesize{(0.22)} & \cellcolor{gray!15}13.3 \footnotesize{(0.69)} & \cellcolor{gray!15}12.75 \footnotesize{(0.76)} \\
  \cmidrule[\heavyrulewidth](lr){1-8}\cmidrule[\heavyrulewidth](lr){10-15}
\end{tabular}
\end{adjustbox}
\end{minipage}
\end{adjustbox}
\end{sidewaystable}

\begin{table}[h]
\renewcommand{\arraystretch}{1.05}%
\setlength{\tabcolsep}{4pt}%
  \centering
  \resultcaption{Pointwise errors (MAE, \MAEmed{}, and \MSEmean) for \sigTPP~and baselines on PS, IP, H1, H3, EQ, SO, TB, TX, and YLP.}
\label{table::pointwise_results_combined}
  \vspace{0.4em}
  \small
\setlength{\tabcolsep}{0.005\linewidth}
\setlength{\aboverulesep}{0.5pt}
\setlength{\belowrulesep}{0.5pt}
  \begin{adjustbox}{width=\linewidth}
    \begin{tabular}[t]{l*{4}{r}@{\hspace{2pt}}c@{}*{5}{r}}
      \toprule
      \multirow{2}{*}{Model} & \multicolumn{1}{c}{PS} & \multicolumn{1}{c}{IP} & \multicolumn{1}{c}{H1} & \multicolumn{1}{c}{H3} & \cellcolor{white}\makebox[0.8em]{} & \multicolumn{1}{c}{EQ} & \multicolumn{1}{c}{SO} & \multicolumn{1}{c}{TB} & \multicolumn{1}{c}{TX} & \multicolumn{1}{c}{YLP} \\
      \cmidrule(lr){2-5}\cmidrule(lr){7-11}
      & \multicolumn{4}{c}{MAE ($\times 10^{\scalebox{1}[1]{-}1}$)} & \cellcolor{white}\makebox[0.8em]{} & \multicolumn{5}{c}{MAE ($\times 10^{\scalebox{1}[1]{-}1}$)} \\
      \cmidrule{2-5}\cmidrule{7-11}
      \sigTPP & 9.51 \footnotesize{(0.08)} & 6.09 \footnotesize{(0.04)} & 22.39 \footnotesize{(0.41)} & 4.79 \footnotesize{(0.04)} & \cellcolor{white}\makebox[0.8em]{} & 10.75 \footnotesize{(0.20)} & 9.88 \footnotesize{(0.22)} & 0.69 \footnotesize{(0.01)} & 2.21 \footnotesize{(0.02)} & 6.12 \footnotesize{(0.23)} \\
      \rowcolor{gray!10}%
      VAE & \underline{8.57} \footnotesize{(0.08)} & 6.06 \footnotesize{(0.03)} & \underline{18.16} \footnotesize{(0.30)} & \underline{4.42} \footnotesize{(0.04)} & \cellcolor{white}\makebox[0.8em]{} & \underline{10.33} \footnotesize{(0.20)} & \underline{8.67} \footnotesize{(0.14)} & \underline{0.67} \footnotesize{(0.01)} & \underline{2.05} \footnotesize{(0.02)} & \underline{4.91} \footnotesize{(0.15)} \\
      DDPM & 9.53 \footnotesize{(0.09)} & 6.08 \footnotesize{(0.04)} & 18.76 \footnotesize{(0.32)} & 4.51 \footnotesize{(0.04)} & \cellcolor{white}\makebox[0.8em]{} & 12.55 \footnotesize{(0.26)} & 11.08 \footnotesize{(0.17)} & 0.71 \footnotesize{(0.01)} & 2.20 \footnotesize{(0.02)} & 6.00 \footnotesize{(0.16)} \\
      \rowcolor{gray!10}%
      WGAN & 8.68 \footnotesize{(0.09)} & \underline{5.98} \footnotesize{(0.04)} & 18.18 \footnotesize{(0.35)} & 4.66 \footnotesize{(0.04)} & \cellcolor{white}\makebox[0.8em]{} & 13.98 \footnotesize{(0.14)} & 10.41 \footnotesize{(0.12)} & 0.73 \footnotesize{(0.01)} & 2.21 \footnotesize{(0.02)} & 5.49 \footnotesize{(0.16)} \\
      \determodel & \textbf{6.29} \footnotesize{(0.11)} & \textbf{4.34} \footnotesize{(0.05)} & \textbf{12.66} \footnotesize{(0.37)} & \textbf{3.23} \footnotesize{(0.04)} & \cellcolor{white}\makebox[0.8em]{} & \textbf{7.05} \footnotesize{(0.16)} & \textbf{6.51} \footnotesize{(0.12)} & \textbf{0.42} \footnotesize{(0.01)} & \textbf{1.49} \footnotesize{(0.02)} & \textbf{3.66} \footnotesize{(0.15)} \\
      \rowcolor{gray!10}%
      \gammamodel & 9.17 \footnotesize{(0.08)} & 6.53 \footnotesize{(0.03)} & 19.72 \footnotesize{(0.26)} & 4.68 \footnotesize{(0.03)} & \cellcolor{white}\makebox[0.8em]{} & 14.22 \footnotesize{(0.09)} & 10.09 \footnotesize{(0.08)} & 0.71 \footnotesize{(0.01)} & 2.16 \footnotesize{(0.02)} & 6.52 \footnotesize{(0.11)} \\
      \cmidrule{1-5}\cmidrule{7-11}
      & \multicolumn{4}{c}{\MAEmed{} ($\times 10^{\scalebox{1}[1]{-}1}$)} & \cellcolor{white}\makebox[0.8em]{} & \multicolumn{5}{c}{\MAEmed{} ($\times 10^{\scalebox{1}[1]{-}1}$)} \\
      \cmidrule{2-5}\cmidrule{7-11}
      \sigTPP & 6.41 \footnotesize{(0.10)} & \underline{4.23} \footnotesize{(0.04)} & 13.36 \footnotesize{(0.36)} & 3.26 \footnotesize{(0.04)} & \cellcolor{white}\makebox[0.8em]{} & 7.71 \footnotesize{(0.15)} & 7.25 \footnotesize{(0.14)} & 0.43 \footnotesize{(0.01)} & \underline{1.46} \footnotesize{(0.02)} & 4.30 \footnotesize{(0.17)} \\
      \rowcolor{gray!10}%
      VAE & \underline{6.31} \footnotesize{(0.11)} & 4.28 \footnotesize{(0.04)} & 12.94 \footnotesize{(0.35)} & \underline{3.18} \footnotesize{(0.04)} & \cellcolor{white}\makebox[0.8em]{} & \underline{6.95} \footnotesize{(0.16)} & \underline{6.37} \footnotesize{(0.13)} & 0.42 \footnotesize{(0.01)} & \textbf{1.39} \footnotesize{(0.02)} & 3.68 \footnotesize{(0.14)} \\
            DDPM & 6.56 \footnotesize{(0.11)} & 4.43 \footnotesize{(0.05)} & \underline{12.81} \footnotesize{(0.35)} & \textbf{3.07} \footnotesize{(0.04)} & \cellcolor{white}\makebox[0.8em]{} & \textbf{6.02} \footnotesize{(0.15)} & \textbf{6.28} \footnotesize{(0.13)} & \textbf{0.40} \footnotesize{(0.01)} & 1.55 \footnotesize{(0.02)} & \textbf{3.60} \footnotesize{(0.15)} \\
      \rowcolor{gray!10}%
      WGAN & 6.35 \footnotesize{(0.10)} & \textbf{4.22} \footnotesize{(0.04)} & 12.86 \footnotesize{(0.37)} & 3.21 \footnotesize{(0.04)} & \cellcolor{white}\makebox[0.8em]{} & 7.58 \footnotesize{(0.16)} & 6.61 \footnotesize{(0.13)} & 0.43 \footnotesize{(0.01)} & 1.49 \footnotesize{(0.02)} & 3.75 \footnotesize{(0.14)} \\
      \determodel & \textbf{6.29} \footnotesize{(0.11)} & 4.34 \footnotesize{(0.05)} & \textbf{12.66} \footnotesize{(0.37)} & 3.23 \footnotesize{(0.04)} & \cellcolor{white}\makebox[0.8em]{} & 7.05 \footnotesize{(0.16)} & 6.51 \footnotesize{(0.12)} & \underline{0.42} \footnotesize{(0.01)} & 1.49 \footnotesize{(0.02)} & \underline{3.66} \footnotesize{(0.15)} \\
      \rowcolor{gray!10}%
      \gammamodel & 6.34 \footnotesize{(0.11)} & 4.41 \footnotesize{(0.04)} & 13.32 \footnotesize{(0.34)} & 3.25 \footnotesize{(0.04)} & \cellcolor{white}\makebox[0.8em]{} & 8.18 \footnotesize{(0.13)} & 6.63 \footnotesize{(0.11)} & 0.43 \footnotesize{(0.01)} & 1.50 \footnotesize{(0.02)} & 3.91 \footnotesize{(0.16)} \\
      \cmidrule{1-5}\cmidrule{7-11}
      & \multicolumn{4}{c}{\MSEmean{} ($\times 10^{\scalebox{1}[1]{-}1}$)} & \cellcolor{white}\makebox[0.8em]{} & \multicolumn{5}{c}{\MSEmean{} ($\times 10^{\scalebox{1}[1]{-}1}$)} \\
      \cmidrule{2-5}\cmidrule{7-11}
      \sigTPP & 8.42 \footnotesize{(0.32)} & \underline{3.87} \footnotesize{(0.10)} & 51.20 \footnotesize{(2.68)} & 2.51 \footnotesize{(0.08)} & \cellcolor{white}\makebox[0.8em]{} & 16.39 \footnotesize{(0.72)} & 14.52 \footnotesize{(0.66)} & 0.16 \footnotesize{(0.01)} & 0.90 \footnotesize{(0.05)} & 11.63 \footnotesize{(1.00)} \\
      \rowcolor{gray!10}%
      VAE & \textbf{8.13} \footnotesize{(0.35)} & 3.89 \footnotesize{(0.10)} & \textbf{43.84} \footnotesize{(2.63)} & \underline{2.43} \footnotesize{(0.08)} & \cellcolor{white}\makebox[0.8em]{} & \underline{14.71} \footnotesize{(0.70)} & \textbf{9.82} \footnotesize{(0.37)} & \underline{0.16} \footnotesize{(0.01)} & \textbf{0.82} \footnotesize{(0.05)} & \textbf{7.77} \footnotesize{(0.73)} \\
      DDPM & 9.03 \footnotesize{(0.37)} & 4.49 \footnotesize{(0.12)} & 49.96 \footnotesize{(3.04)} & \textbf{2.26} \footnotesize{(0.08)} & \cellcolor{white}\makebox[0.8em]{} & \textbf{13.06} \footnotesize{(0.60)} & \underline{10.20} \footnotesize{(0.36)} & \textbf{0.15} \footnotesize{(0.01)} & \underline{0.89} \footnotesize{(0.05)} & \underline{8.16} \footnotesize{(0.76)} \\
      \rowcolor{gray!10}%
      WGAN & \underline{8.17} \footnotesize{(0.34)} & \textbf{3.83} \footnotesize{(0.10)} & \underline{44.21} \footnotesize{(2.71)} & 2.48 \footnotesize{(0.08)} & \cellcolor{white}\makebox[0.8em]{} & 17.99 \footnotesize{(0.73)} & 11.47 \footnotesize{(0.37)} & 0.16 \footnotesize{(0.01)} & 0.90 \footnotesize{(0.06)} & 8.55 \footnotesize{(0.72)} \\
      \determodel & 8.84 \footnotesize{(0.40)} & 3.89 \footnotesize{(0.10)} & 50.52 \footnotesize{(3.26)} & 2.65 \footnotesize{(0.09)} & \cellcolor{white}\makebox[0.8em]{} & 17.13 \footnotesize{(0.84)} & 11.52 \footnotesize{(0.47)} & 0.17 \footnotesize{(0.01)} & 0.93 \footnotesize{(0.06)} & 9.32 \footnotesize{(0.81)} \\
      \rowcolor{gray!10}%
      \gammamodel & 8.29 \footnotesize{(0.35)} & 4.29 \footnotesize{(0.11)} & 46.80 \footnotesize{(2.68)} & 2.53 \footnotesize{(0.09)} & \cellcolor{white}\makebox[0.8em]{} & 18.99 \footnotesize{(0.66)} & 10.86 \footnotesize{(0.34)} & 0.16 \footnotesize{(0.01)} & 0.90 \footnotesize{(0.06)} & 9.65 \footnotesize{(0.77)} \\
      \bottomrule
    \end{tabular}%
  \end{adjustbox}
\end{table}

\begin{table}[h]
\renewcommand{\arraystretch}{1.05}%
  \centering
  \resultcaption{Pointwise errors (MAE, \MAEmed{}, and \MSEmean) for \sigTPP~and baselines on SC, SR, and YLA, completing Table~\ref{table::pointwise_results_combined}.}
\label{table::pointwise_results_additional}
  \vspace{0.4em}
  \small
\setlength{\tabcolsep}{0.005\linewidth}
\setlength{\aboverulesep}{0.5pt}
\setlength{\belowrulesep}{0.5pt}
  \begin{adjustbox}{width=\linewidth}
    \begin{tabular}[t]{l*{3}{r}@{\hspace{2pt}}c@{}*{3}{r}@{\hspace{2pt}}c@{}*{3}{r}}
      \toprule
      \multirow{2}{*}{Model} & \multicolumn{1}{c}{SC} & \multicolumn{1}{c}{SR} & \multicolumn{1}{c}{YLA} & \cellcolor{white}\makebox[0.8em]{} & \multicolumn{1}{c}{SC} & \multicolumn{1}{c}{SR} & \multicolumn{1}{c}{YLA} & \cellcolor{white}\makebox[0.8em]{} & \multicolumn{1}{c}{SC} & \multicolumn{1}{c}{SR} & \multicolumn{1}{c}{YLA} \\
      \cmidrule(lr){2-4}\cmidrule(lr){6-8}\cmidrule(lr){10-12}
      & \multicolumn{3}{c}{MAE ($\times 10^{\scalebox{1}[1]{-}1}$)} & \cellcolor{white}\makebox[0.8em]{} & \multicolumn{3}{c}{\MAEmed{} ($\times 10^{\scalebox{1}[1]{-}1}$)} & \cellcolor{white}\makebox[0.8em]{} & \multicolumn{3}{c}{\MSEmean{} ($\times 10^{\scalebox{1}[1]{-}1}$)} \\
      \cmidrule{2-4}\cmidrule{6-8}\cmidrule{10-12}
      \sigTPP & 6.79 \footnotesize{(0.04)} & 11.77 \footnotesize{(0.08)} & 9.47 \footnotesize{(0.15)} & \cellcolor{white}\makebox[0.8em]{} & 5.31 \footnotesize{(0.05)} & 5.24 \footnotesize{(0.12)} & 6.25 \footnotesize{(0.18)} & \cellcolor{white}\makebox[0.8em]{} & 4.25 \footnotesize{(0.08)} & 19.51 \footnotesize{(0.94)} & 14.44 \footnotesize{(0.93)} \\
      \rowcolor{gray!10}%
      VAE & \underline{6.26} \footnotesize{(0.03)} & \underline{8.13} \footnotesize{(0.10)} & \underline{7.97} \footnotesize{(0.16)} & \cellcolor{white}\makebox[0.8em]{} & \textbf{4.86} \footnotesize{(0.04)} & \underline{5.22} \footnotesize{(0.12)} & 6.17 \footnotesize{(0.17)} & \cellcolor{white}\makebox[0.8em]{} & \textbf{3.69} \footnotesize{(0.06)} & \textbf{17.22} \footnotesize{(1.01)} & \underline{13.82} \footnotesize{(0.98)} \\
      DDPM & 8.76 \footnotesize{(0.05)} & 9.65 \footnotesize{(0.09)} & 9.76 \footnotesize{(0.16)} & \cellcolor{white}\makebox[0.8em]{} & 7.62 \footnotesize{(0.06)} & 5.24 \footnotesize{(0.12)} & \textbf{6.07} \footnotesize{(0.18)} & \cellcolor{white}\makebox[0.8em]{} & 8.75 \footnotesize{(0.14)} & 18.11 \footnotesize{(1.01)} & \textbf{13.75} \footnotesize{(0.98)} \\
      \rowcolor{gray!10}%
      WGAN & 7.45 \footnotesize{(0.03)} & 8.59 \footnotesize{(0.12)} & 9.06 \footnotesize{(0.16)} & \cellcolor{white}\makebox[0.8em]{} & 5.65 \footnotesize{(0.05)} & 5.23 \footnotesize{(0.12)} & 6.22 \footnotesize{(0.18)} & \cellcolor{white}\makebox[0.8em]{} & 4.81 \footnotesize{(0.08)} & \underline{17.37} \footnotesize{(1.01)} & 14.50 \footnotesize{(1.00)} \\
      \determodel & \textbf{4.94} \footnotesize{(0.03)} & \textbf{5.22} \footnotesize{(0.12)} & \textbf{6.15} \footnotesize{(0.19)} & \cellcolor{white}\makebox[0.8em]{} & \underline{4.94} \footnotesize{(0.03)} & \textbf{5.22} \footnotesize{(0.12)} & \underline{6.15} \footnotesize{(0.19)} & \cellcolor{white}\makebox[0.8em]{} & \underline{3.98} \footnotesize{(0.06)} & 18.76 \footnotesize{(1.10)} & 15.87 \footnotesize{(1.10)} \\
      \rowcolor{gray!10}%
      \gammamodel & 7.74 \footnotesize{(0.03)} & 9.90 \footnotesize{(0.10)} & 9.00 \footnotesize{(0.14)} & \cellcolor{white}\makebox[0.8em]{} & 5.69 \footnotesize{(0.04)} & 5.27 \footnotesize{(0.12)} & 6.21 \footnotesize{(0.17)} & \cellcolor{white}\makebox[0.8em]{} & 5.00 \footnotesize{(0.09)} & 17.71 \footnotesize{(0.99)} & 14.75 \footnotesize{(1.00)} \\
      \bottomrule
    \end{tabular}%
  \end{adjustbox}
\end{table}

\begin{figure}[!h]
  \centering
      \centering
      \includegraphics[width=\textwidth]{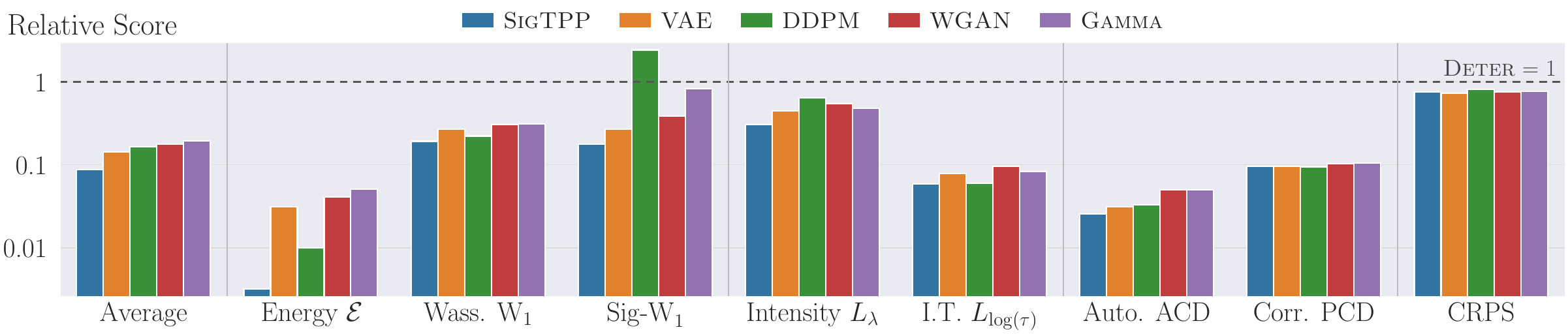}
\ratiocaption{the $6$ synthetic datasets}{The exact values are enumerated in Table~\ref{table::geo_improvement_vs_deter_grouped}.}
\label{fig::normalised_score_results_synth}
\end{figure}
\begin{figure}[!h]
  \centering
      \centering
      \includegraphics[width=\textwidth]{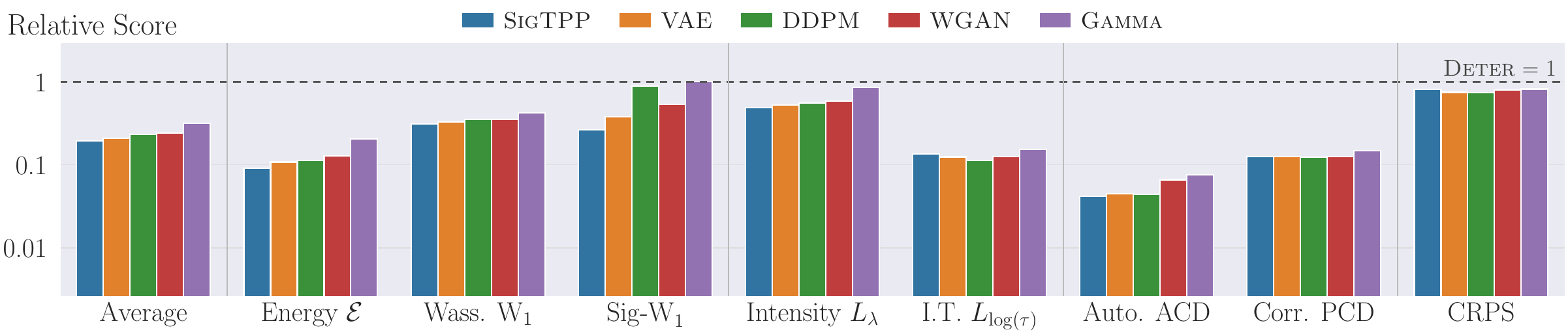}
\ratiocaption{the $6$ real-world datasets}{The exact values are enumerated in Table~\ref{table::geo_improvement_vs_deter_grouped}.}
\label{fig::normalised_score_results_real}
\end{figure}

\begin{table}[h]
\centering
\caption{First inactive signature degree $M^{\mathrm{all}}$ per dataset, defined as the first degree $m$ for which all $D^m$ coordinates have empirical standard deviation below $10^{-8}$ across the training set.}
\label{table::sig_degree_datasets}
\footnotesize
\setlength{\tabcolsep}{5pt}
\begin{tabular}{l | cccccccccccc}
\toprule
Dataset & PS & IP & H1 & H3 & SC & SR & EQ & SO & TB & TX & YLP & YLA \\
\midrule
$M^{\mathrm{all}}$ & 8 & 8 & $>$8 & 6 & 7 & 7 & $>$8 & 6 & 6 & 7 & 6 & 6 \\
\bottomrule
\end{tabular}
\end{table}

\begin{table}[h]
\renewcommand{\arraystretch}{1.05}%
\setlength{\tabcolsep}{4pt}%
  \centering
  \rankcaption{Ranks for \sigTPP~and baselines by metric, for all datasets.}
\label{table::avg_rank_by_metric_grouped}
  \vspace{0.4em}
  \small
\setlength{\tabcolsep}{0.005\linewidth}
\setlength{\aboverulesep}{0.5pt}
\setlength{\belowrulesep}{0.5pt}
  \begin{adjustbox}{width=\linewidth}
    \begin{tabular}[t]{l*{9}{r}}
      \toprule
      \multirow{2}{*}{Metric} & \multicolumn{3}{c}{\sigTPP} & \multicolumn{3}{c}{VAE} & \multicolumn{3}{c}{DDPM} \\
      \cmidrule(lr){2-4}\cmidrule(lr){5-7}\cmidrule(lr){8-10}
      & Synth. & Real & All & Synth. & Real & All & Synth. & Real & All \\
      \midrule
      $\mathcal{E}$ & \textbf{1.17} \footnotesize{(0.41)} & \textbf{1.83} \footnotesize{(0.75)} & \textbf{1.50} \footnotesize{(0.67)} & 3.67 \footnotesize{(1.21)} & 3.00 \footnotesize{(1.67)} & 3.33 \footnotesize{(1.44)} & \underline{2.00} \footnotesize{(0.63)} & \underline{2.17} \footnotesize{(1.17)} & \underline{2.08} \footnotesize{(0.90)} \\
      \rowcolor{gray!10}%
      $\WAS$ & \textbf{1.33} \footnotesize{(0.82)} & \textbf{1.67} \footnotesize{(0.82)} & \textbf{1.50} \footnotesize{(0.80)} & 3.17 \footnotesize{(1.47)} & \underline{2.67} \footnotesize{(1.21)} & 2.92 \footnotesize{(1.31)} & \underline{2.17} \footnotesize{(0.41)} & \underline{2.67} \footnotesize{(1.51)} & \underline{2.42} \footnotesize{(1.08)} \\
      $\sigW$ & \textbf{2.17} \footnotesize{(1.33)} & \textbf{2.33} \footnotesize{(1.03)} & \textbf{2.25} \footnotesize{(1.14)} & \underline{2.33} \footnotesize{(1.63)} & \underline{2.67} \footnotesize{(1.63)} & \underline{2.50} \footnotesize{(1.57)} & 5.17 \footnotesize{(1.60)} & 4.17 \footnotesize{(1.17)} & 4.67 \footnotesize{(1.44)} \\
      \rowcolor{gray!10}%
      CRPS & 3.67 \footnotesize{(1.37)} & 3.67 \footnotesize{(1.21)} & 3.67 \footnotesize{(1.23)} & \textbf{1.50} \footnotesize{(0.84)} & \textbf{1.67} \footnotesize{(0.52)} & \textbf{1.58} \footnotesize{(0.67)} & 4.00 \footnotesize{(1.79)} & \underline{1.83} \footnotesize{(1.60)} & 2.92 \footnotesize{(1.98)} \\
      $\histintloss$ & \textbf{1.33} \footnotesize{(0.52)} & \underline{2.67} \footnotesize{(1.03)} & \textbf{2.00} \footnotesize{(1.04)} & \underline{2.83} \footnotesize{(1.33)} & \textbf{2.17} \footnotesize{(1.33)} & \underline{2.50} \footnotesize{(1.31)} & 4.33 \footnotesize{(1.21)} & 3.00 \footnotesize{(1.55)} & 3.67 \footnotesize{(1.50)} \\
      \rowcolor{gray!10}%
      $\histITloss$ & \textbf{2.17} \footnotesize{(1.47)} & \underline{2.83} \footnotesize{(0.98)} & \textbf{2.50} \footnotesize{(1.24)} & 3.33 \footnotesize{(1.03)} & \textbf{2.67} \footnotesize{(1.97)} & 3.00 \footnotesize{(1.54)} & \underline{2.67} \footnotesize{(1.63)} & \textbf{2.67} \footnotesize{(1.51)} & \underline{2.67} \footnotesize{(1.50)} \\
      ACD & \textbf{2.00} \footnotesize{(1.55)} & \underline{2.17} \footnotesize{(1.60)} & \textbf{2.08} \footnotesize{(1.51)} & 3.00 \footnotesize{(0.63)} & 2.50 \footnotesize{(1.05)} & 2.75 \footnotesize{(0.87)} & \underline{2.83} \footnotesize{(0.98)} & \textbf{2.00} \footnotesize{(0.89)} & \underline{2.42} \footnotesize{(1.00)} \\
      \rowcolor{gray!10}%
      PCD & \textbf{2.17} \footnotesize{(1.60)} & 3.17 \footnotesize{(1.60)} & 2.67 \footnotesize{(1.61)} & \underline{3.17} \footnotesize{(1.17)} & \textbf{1.67} \footnotesize{(1.21)} & \underline{2.42} \footnotesize{(1.38)} & \textbf{2.17} \footnotesize{(1.60)} & \underline{2.33} \footnotesize{(0.82)} & \textbf{2.25} \footnotesize{(1.22)} \\
      \midrule
      Avg. & \textbf{2.00} \footnotesize{(1.35)} & \underline{2.54} \footnotesize{(1.25)} & \textbf{2.27} \footnotesize{(1.33)} & \underline{2.88} \footnotesize{(1.28)} & \textbf{2.38} \footnotesize{(1.36)} & \underline{2.62} \footnotesize{(1.34)} & 3.17 \footnotesize{(1.65)} & 2.60 \footnotesize{(1.40)} & 2.89 \footnotesize{(1.55)} \\
      \midrule
      \multirow{2}{*}{Metric} & \multicolumn{3}{c}{WGAN} & \multicolumn{3}{c}{\gammamodel} & \multicolumn{3}{c}{\determodel} \\
      \cmidrule(lr){2-4}\cmidrule(lr){5-7}\cmidrule(lr){8-10}
      & Synth. & Real & All & Synth. & Real & All & Synth. & Real & All \\
      \midrule
      $\mathcal{E}$ & 3.83 \footnotesize{(0.41)} & 3.17 \footnotesize{(0.75)} & 3.50 \footnotesize{(0.67)} & 4.50 \footnotesize{(1.22)} & 4.83 \footnotesize{(0.41)} & 4.67 \footnotesize{(0.89)} & 5.83 \footnotesize{(0.41)} & 6.00 \footnotesize{(0.00)} & 5.92 \footnotesize{(0.29)} \\
      \rowcolor{gray!10}%
      $\WAS$ & 4.17 \footnotesize{(0.75)} & 3.00 \footnotesize{(0.63)} & 3.58 \footnotesize{(0.90)} & 4.50 \footnotesize{(1.22)} & 5.00 \footnotesize{(0.00)} & 4.75 \footnotesize{(0.87)} & 5.67 \footnotesize{(0.82)} & 6.00 \footnotesize{(0.00)} & 5.83 \footnotesize{(0.58)} \\
      $\sigW$ & 2.83 \footnotesize{(0.98)} & 3.33 \footnotesize{(1.63)} & 3.08 \footnotesize{(1.31)} & 4.00 \footnotesize{(1.67)} & 4.17 \footnotesize{(1.94)} & 4.08 \footnotesize{(1.73)} & 4.50 \footnotesize{(1.05)} & 4.33 \footnotesize{(2.25)} & 4.42 \footnotesize{(1.68)} \\
      \rowcolor{gray!10}%
      CRPS & \underline{2.17} \footnotesize{(0.75)} & 3.50 \footnotesize{(0.55)} & \underline{2.83} \footnotesize{(0.94)} & 3.83 \footnotesize{(0.75)} & 4.33 \footnotesize{(0.82)} & 4.08 \footnotesize{(0.79)} & 5.83 \footnotesize{(0.41)} & 6.00 \footnotesize{(0.00)} & 5.92 \footnotesize{(0.29)} \\
      $\histintloss$ & 3.50 \footnotesize{(0.84)} & 2.83 \footnotesize{(1.47)} & 3.17 \footnotesize{(1.19)} & 3.17 \footnotesize{(1.72)} & 5.67 \footnotesize{(0.52)} & 4.42 \footnotesize{(1.78)} & 5.83 \footnotesize{(0.41)} & 4.67 \footnotesize{(1.51)} & 5.25 \footnotesize{(1.22)} \\
      \rowcolor{gray!10}%
      $\histITloss$ & 3.83 \footnotesize{(0.75)} & 3.17 \footnotesize{(1.60)} & 3.50 \footnotesize{(1.24)} & 3.00 \footnotesize{(1.90)} & 3.67 \footnotesize{(1.21)} & 3.33 \footnotesize{(1.56)} & 6.00 \footnotesize{(0.00)} & 6.00 \footnotesize{(0.00)} & 6.00 \footnotesize{(0.00)} \\
      ACD & 3.83 \footnotesize{(1.47)} & 4.00 \footnotesize{(1.26)} & 3.92 \footnotesize{(1.31)} & 3.33 \footnotesize{(1.97)} & 4.33 \footnotesize{(0.52)} & 3.83 \footnotesize{(1.47)} & 6.00 \footnotesize{(0.00)} & 6.00 \footnotesize{(0.00)} & 6.00 \footnotesize{(0.00)} \\
      \rowcolor{gray!10}%
      PCD & 3.83 \footnotesize{(1.17)} & 3.67 \footnotesize{(1.37)} & 3.75 \footnotesize{(1.22)} & 3.67 \footnotesize{(1.03)} & 4.17 \footnotesize{(0.75)} & 3.92 \footnotesize{(0.90)} & 6.00 \footnotesize{(0.00)} & 6.00 \footnotesize{(0.00)} & 6.00 \footnotesize{(0.00)} \\
      \midrule
      Avg. & 3.50 \footnotesize{(1.07)} & 3.33 \footnotesize{(1.19)} & 3.42 \footnotesize{(1.13)} & 3.75 \footnotesize{(1.48)} & 4.52 \footnotesize{(1.05)} & 4.14 \footnotesize{(1.33)} & 5.71 \footnotesize{(0.68)} & 5.62 \footnotesize{(1.10)} & 5.67 \footnotesize{(0.91)} \\
      \bottomrule
    \end{tabular}%
  \end{adjustbox}
\end{table}

\begin{table}[h]
\renewcommand{\arraystretch}{1.05}%
\setlength{\tabcolsep}{4pt}%
  \centering
  \caption{Geometric means of per-dataset score ratios to \determodel, by metric, for all datasets. \resultlegend}
\label{table::geo_improvement_vs_deter_grouped}
  \vspace{0.4em}
  \small
\setlength{\tabcolsep}{0.003\linewidth}
\setlength{\aboverulesep}{0.5pt}
\setlength{\belowrulesep}{0.5pt}
  \begin{adjustbox}{width=\linewidth}
    \begin{tabular}[t]{l*{15}{c}}
      \toprule
      \multirow{2}{*}{Metric} & \multicolumn{3}{c}{\sigTPP} & \multicolumn{3}{c}{VAE} & \multicolumn{3}{c}{DDPM} & \multicolumn{3}{c}{WGAN} & \multicolumn{3}{c}{\gammamodel} \\
      \cmidrule(lr){2-4}\cmidrule(lr){5-7}\cmidrule(lr){8-10}\cmidrule(lr){11-13}\cmidrule(lr){14-16}
      & Synth. & Real & All & Synth. & Real & All & Synth. & Real & All & Synth. & Real & All & Synth. & Real & All \\
      \midrule
      $\mathcal{E}$ & \textbf{0.003} & \textbf{0.091} & \textbf{0.017} & 0.031 & \underline{0.108} & 0.058 & \underline{0.010} & 0.112 & \underline{0.033} & 0.041 & 0.130 & 0.073 & 0.051 & 0.204 & 0.102 \\
      \rowcolor{gray!10}%
      $\WAS$ & \textbf{0.193} & \textbf{0.310} & \textbf{0.245} & 0.269 & \underline{0.328} & 0.297 & \underline{0.223} & 0.353 & \underline{0.281} & 0.305 & 0.357 & 0.330 & 0.314 & 0.422 & 0.364 \\
      $\sigW$ & \textbf{0.178} & \textbf{0.263} & \textbf{0.217} & \underline{0.271} & \underline{0.383} & \underline{0.322} & 2.442 & 0.896 & 1.479 & 0.390 & 0.542 & 0.460 & 0.835 & 1.014 & 0.920 \\
      \rowcolor{gray!10}%
      CRPS & 0.763 & 0.822 & 0.792 & \textbf{0.734} & \underline{0.752} & \textbf{0.743} & 0.817 & \textbf{0.745} & 0.780 & \underline{0.755} & 0.803 & \underline{0.779} & 0.774 & 0.823 & 0.798 \\
      $\histintloss$ & \textbf{0.306} & \textbf{0.494} & \textbf{0.389} & \underline{0.449} & \underline{0.526} & \underline{0.486} & 0.645 & 0.557 & 0.599 & 0.546 & 0.584 & 0.565 & 0.480 & 0.860 & 0.643 \\
      \rowcolor{gray!10}%
      $\histITloss$ & \textbf{0.059} & 0.135 & \underline{0.089} & 0.079 & \underline{0.125} & 0.099 & \underline{0.060} & \textbf{0.114} & \textbf{0.083} & 0.097 & 0.126 & 0.111 & 0.083 & 0.153 & 0.113 \\
      ACD & \textbf{0.025} & \textbf{0.042} & \textbf{0.033} & \underline{0.031} & 0.045 & \underline{0.037} & 0.033 & \underline{0.044} & 0.038 & 0.050 & 0.066 & 0.057 & 0.050 & 0.076 & 0.062 \\
      \rowcolor{gray!10}%
      PCD & \underline{0.096} & 0.126 & \underline{0.110} & 0.097 & 0.126 & 0.110 & \textbf{0.095} & \textbf{0.124} & \textbf{0.108} & 0.103 & \underline{0.125} & 0.114 & 0.105 & 0.148 & 0.125 \\
      \midrule
      Avg. & \textbf{0.088} & \textbf{0.196} & \textbf{0.131} & \underline{0.144} & \underline{0.210} & \underline{0.174} & 0.164 & 0.235 & 0.196 & 0.178 & 0.243 & 0.208 & 0.196 & 0.319 & 0.250 \\
      \bottomrule
    \end{tabular}%
  \end{adjustbox}
\end{table}

\begin{figure}[!htbp]
    \centering
    \begin{subfigure}[b]{\textwidth}
        \centering
        \includegraphics[width=\textwidth]{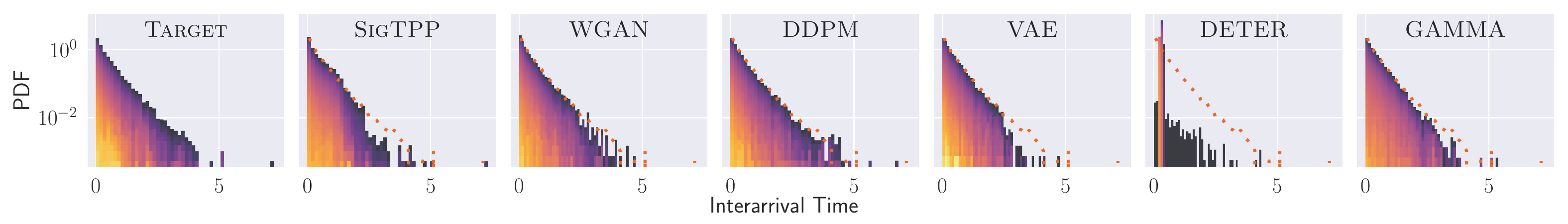}
        \par\vspace{-4pt}
        \makebox[\textwidth][c]{\small(a)}
    \end{subfigure}
    \par\vspace{4pt}
    \begin{subfigure}[b]{\textwidth}
        \centering
        \includegraphics[width=0.49\textwidth]{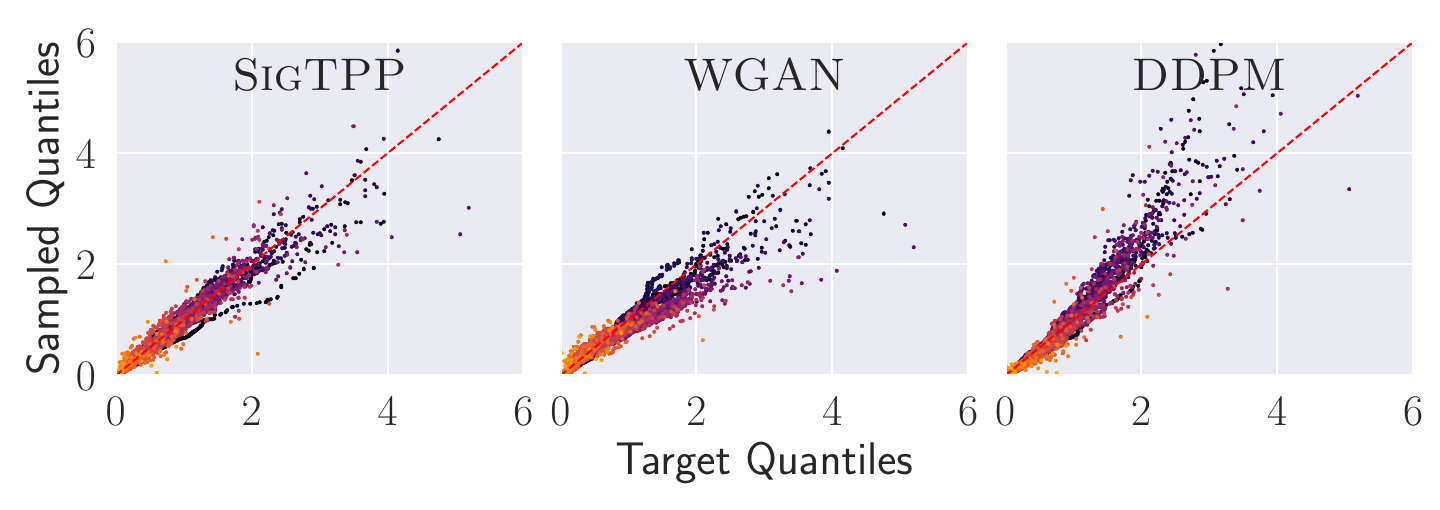}
        \hfill
        \includegraphics[width=0.49\textwidth]{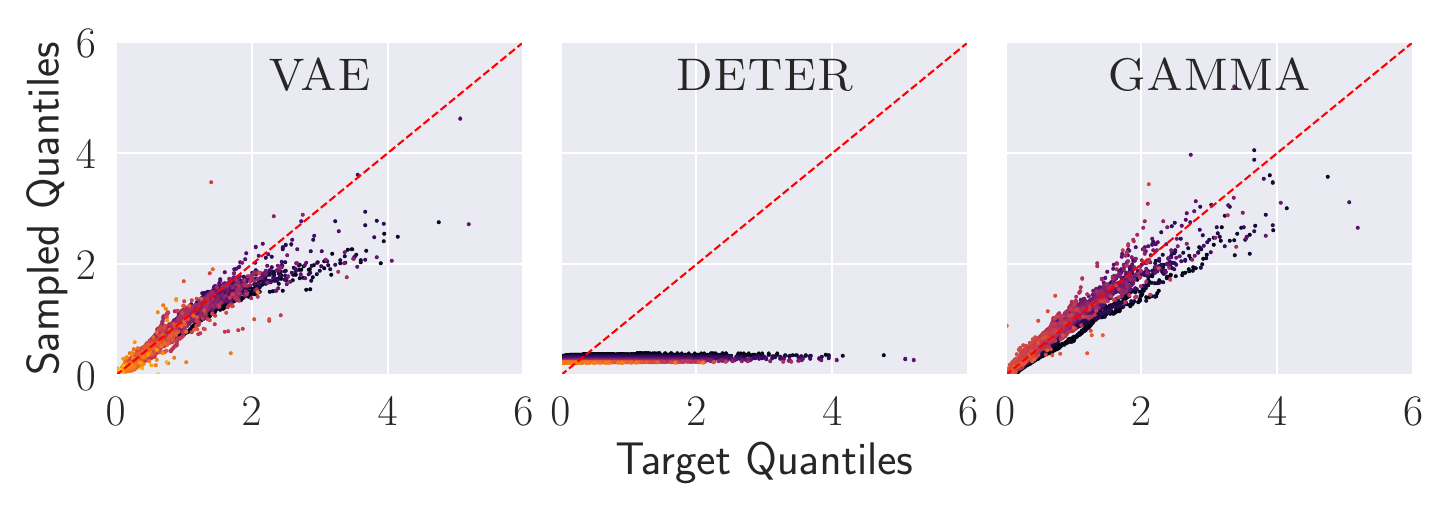}
        \par\vspace{-4pt}
        \makebox[\textwidth][c]{\small(b)}
    \end{subfigure}
    \par\vspace{4pt}
    \begin{subfigure}[b]{\textwidth}
        \centering
        \includegraphics[width=\textwidth]{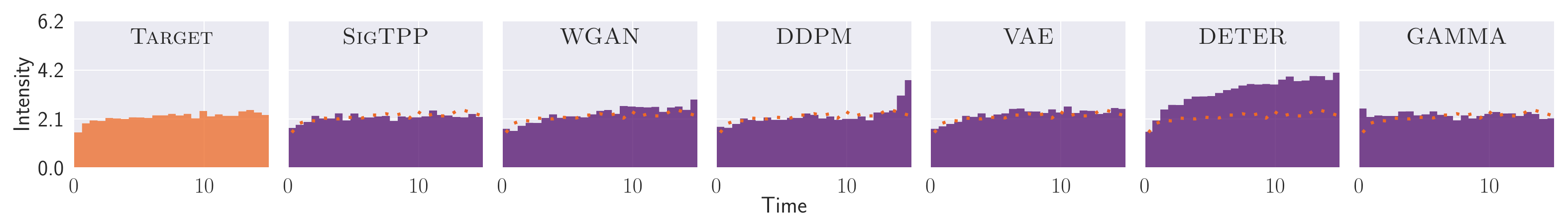}
        \par\vspace{-4pt}
        \makebox[\textwidth][c]{\small(c)}
    \end{subfigure}
    \par\vspace{4pt}
    \begin{subfigure}[t]{0.75\textwidth}
        \centering
        \includegraphics[width=\textwidth]{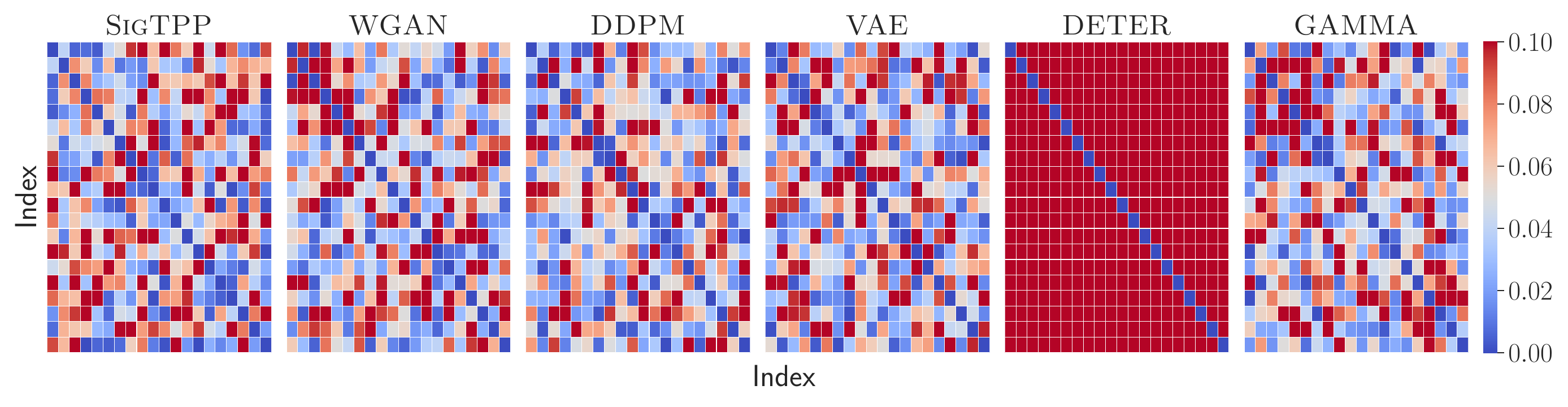}
        \par\vspace{-4pt}
        \makebox[\textwidth][c]{\small(d)}
    \end{subfigure}
    \hfill
    \begin{subfigure}[t]{0.23\textwidth}
        \centering
        \includegraphics[width=\textwidth]{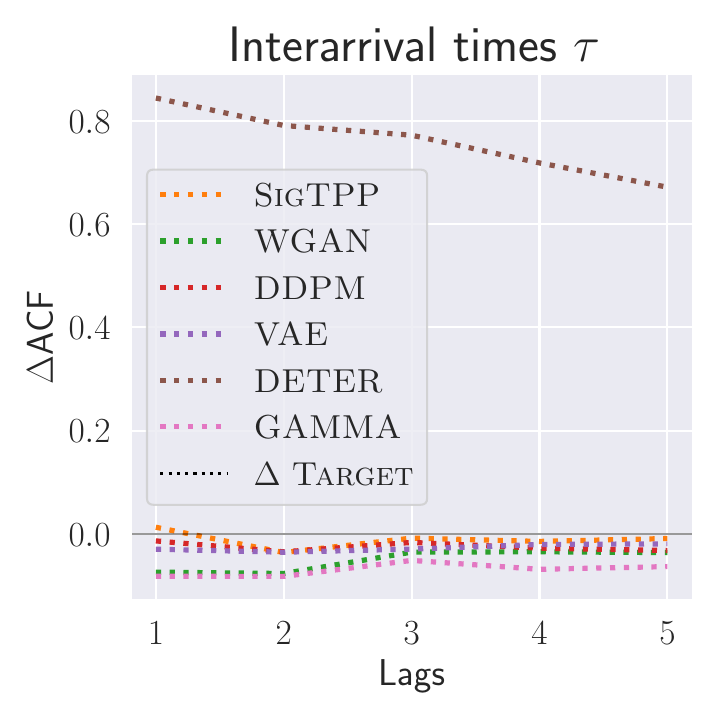}
        \par\vspace{-4pt}
        \makebox[\textwidth][c]{\small(e)}
    \end{subfigure}
    \caption{Qualitative evaluation on H3. Model panels are ordered \sigTPP, WGAN, DDPM, VAE, \determodel, and \gammamodel.
    (a) Histogram estimates of the interarrival times density $p(x)$ on a logarithmic scale. The leftmost panel shows the target distribution, the others the distributions of samples generated by each model.
    (b) Quantile-quantile (QQ) plots. Each subplot shows empirical quantiles of the reference data ($x$-axis) against generated sample quantiles ($y$-axis). The red dashed line is the identity $y=x$.
    (c) Empirical intensity functions $\hat{\lambda}(t)$ estimated from the same sequences.
    (d) Absolute Pearson correlation-matrix differences for up to the first $20$ event positions. Each heatmap entry shows the absolute difference between the empirical correlation of two sequence components in the generated samples and in the reference data.
    (e) Signed autocorrelation error for $5$ lags. Each curve shows the difference between the ACF of the indicated model and the empirical ACF of the reference process, so that the target corresponds to the horizontal line at zero.}
    \label{fig::hawkes_qualitative_summary}
\end{figure}

\begin{figure}[!htbp]
    \centering
    \begin{subfigure}[b]{\textwidth}
        \centering
        \includegraphics[width=\textwidth]{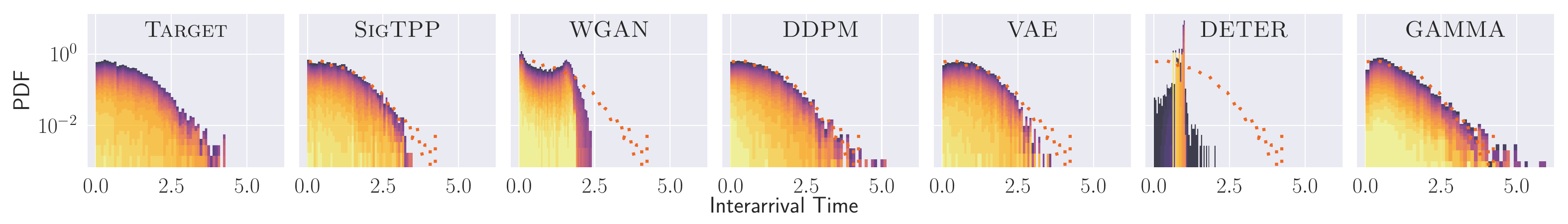}
        \par\vspace{-4pt}
        \makebox[\textwidth][c]{\small(a)}
    \end{subfigure}
    \par\vspace{4pt}
    \begin{subfigure}[b]{\textwidth}
        \centering
        \includegraphics[width=\textwidth]{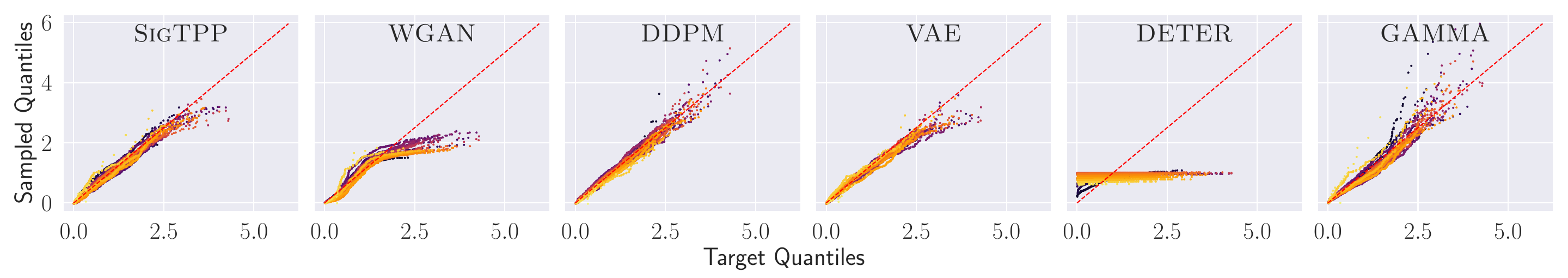}
        \par\vspace{-4pt}
        \makebox[\textwidth][c]{\small(b)}
    \end{subfigure}
    \par\vspace{4pt}
    \begin{subfigure}[b]{\textwidth}
        \centering
        \includegraphics[width=\textwidth]{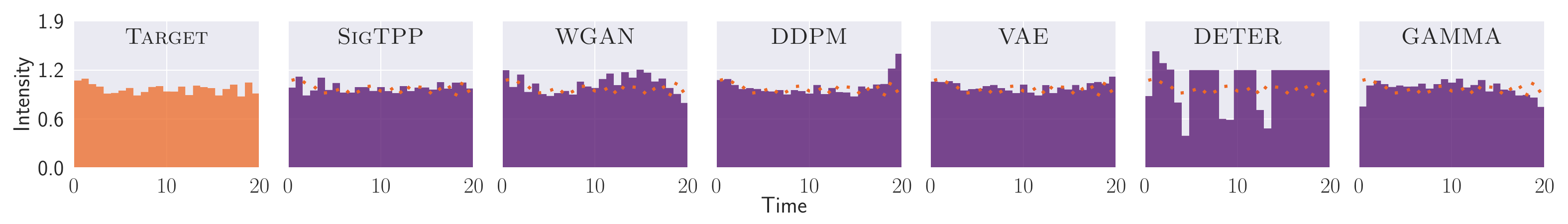}
        \par\vspace{-4pt}
        \makebox[\textwidth][c]{\small(c)}
    \end{subfigure}
    \par\vspace{4pt}
    \begin{subfigure}[t]{0.75\textwidth}
        \centering
        \includegraphics[width=\textwidth]{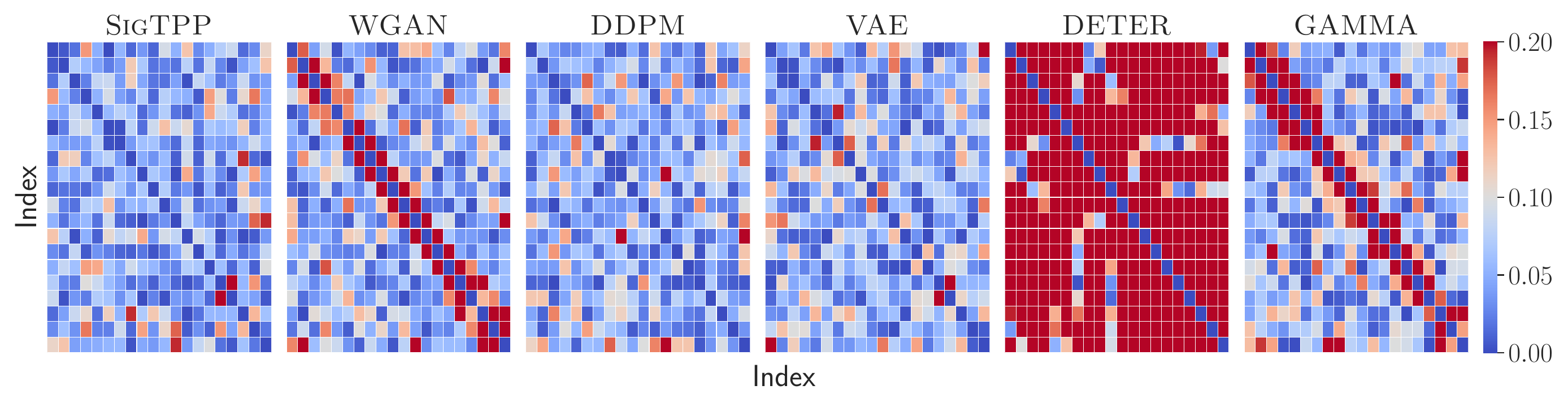}
        \par\vspace{-4pt}
        \makebox[\textwidth][c]{\small(d)}
    \end{subfigure}
    \hfill
    \begin{subfigure}[t]{0.23\textwidth}
        \centering
        \includegraphics[width=\textwidth]{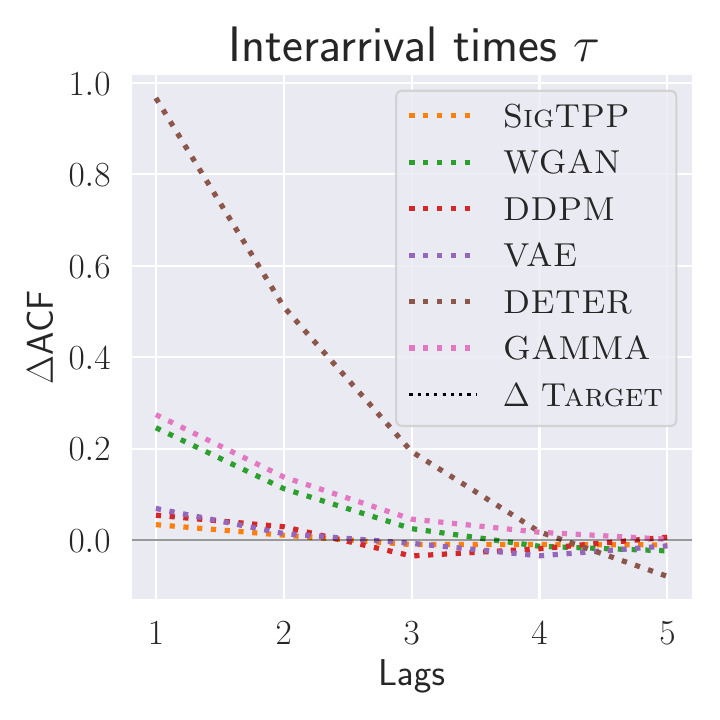}
        \par\vspace{-4pt}
        \makebox[\textwidth][c]{\small(e)}
    \end{subfigure}
    \caption{Qualitative evaluation on SC. Model panels are ordered \sigTPP, WGAN, DDPM, VAE, \determodel, and \gammamodel.
    (a) Histogram estimates of the interarrival times density $p(x)$ on a logarithmic scale. The leftmost panel shows the target distribution, the others the distributions of samples generated by each model.
    (b) Quantile-quantile (QQ) plots. Each subplot shows empirical quantiles of the reference data ($x$-axis) against generated sample quantiles ($y$-axis). The red dashed line is the identity $y=x$.
    (c) Empirical intensity functions $\hat{\lambda}(t)$ estimated from the same sequences.
    (d) Absolute Pearson correlation-matrix differences for up to the first $20$ event positions. Each heatmap entry shows the absolute difference between the empirical correlation of two sequence components in the generated samples and in the reference data.
    (e) Signed autocorrelation error for $5$ lags. Each curve shows the difference between the ACF of the indicated model and the empirical ACF of the reference process, so that the target corresponds to the horizontal line at zero.}
    \label{fig::sc_qualitative_summary}
\end{figure}

\begin{figure}[!htbp]
    \centering
    \begin{subfigure}[b]{\textwidth}
        \centering
        \includegraphics[width=\textwidth]{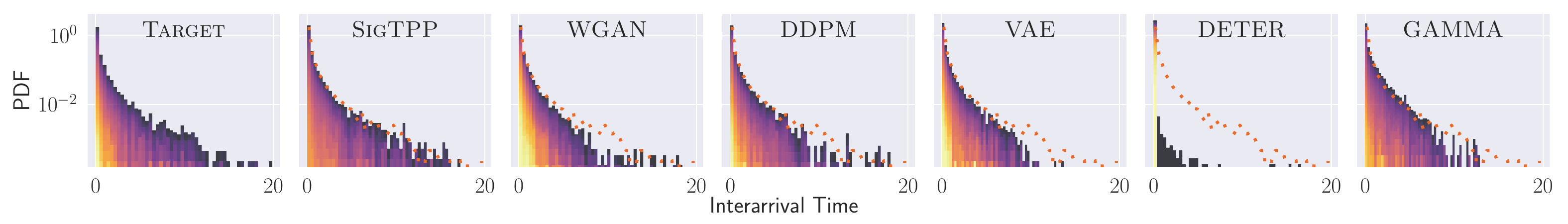}
        \par\vspace{-4pt}
        \makebox[\textwidth][c]{\small(a)}
    \end{subfigure}
    \par\vspace{4pt}
    \begin{subfigure}[b]{\textwidth}
        \centering
        \includegraphics[width=\textwidth]{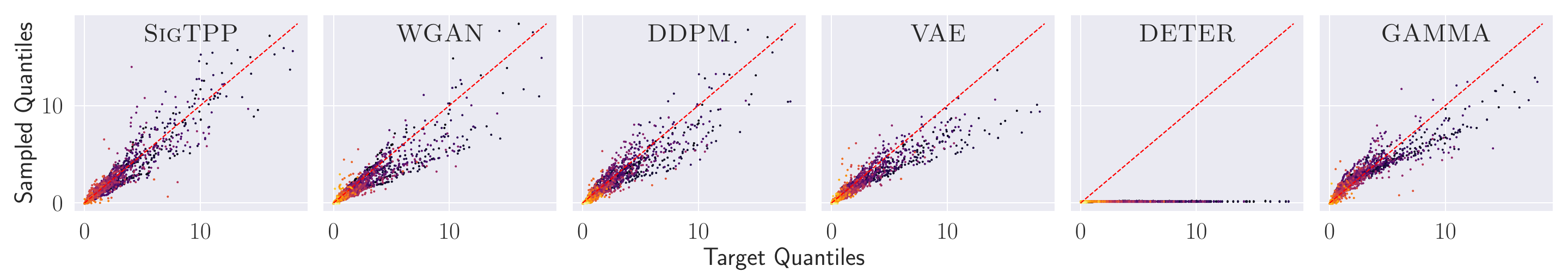}
        \par\vspace{-4pt}
        \makebox[\textwidth][c]{\small(b)}
    \end{subfigure}
    \par\vspace{4pt}
    \begin{subfigure}[b]{\textwidth}
        \centering
        \includegraphics[width=\textwidth]{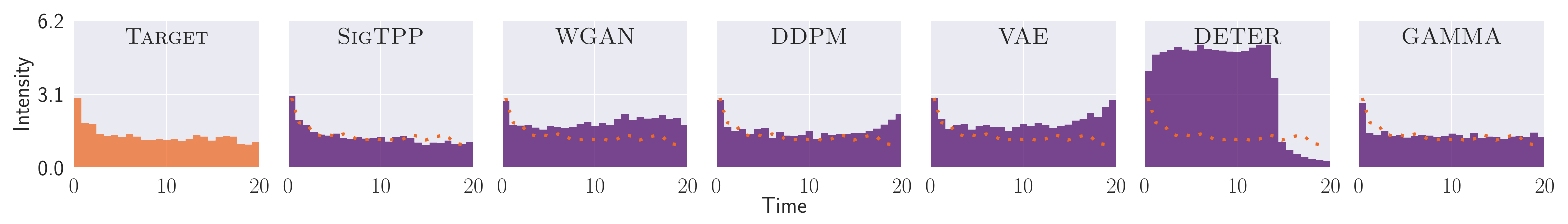}
        \par\vspace{-4pt}
        \makebox[\textwidth][c]{\small(c)}
    \end{subfigure}
    \par\vspace{4pt}
    \begin{subfigure}[t]{0.75\textwidth}
        \centering
        \includegraphics[width=\textwidth]{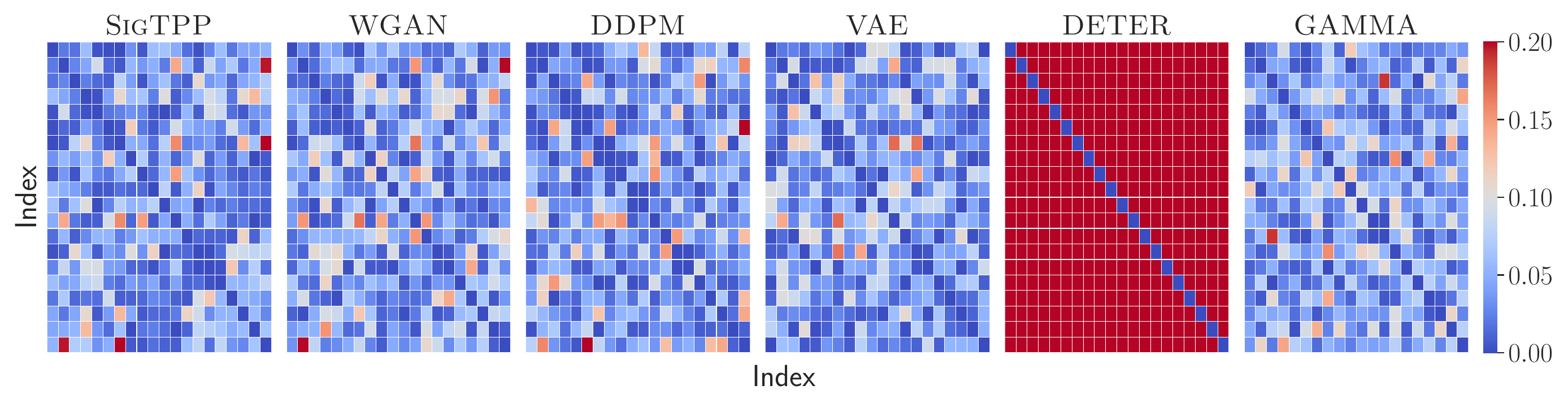}
        \par\vspace{-4pt}
        \makebox[\textwidth][c]{\small(d)}
    \end{subfigure}
    \hfill
    \begin{subfigure}[t]{0.23\textwidth}
        \centering
        \includegraphics[width=\textwidth]{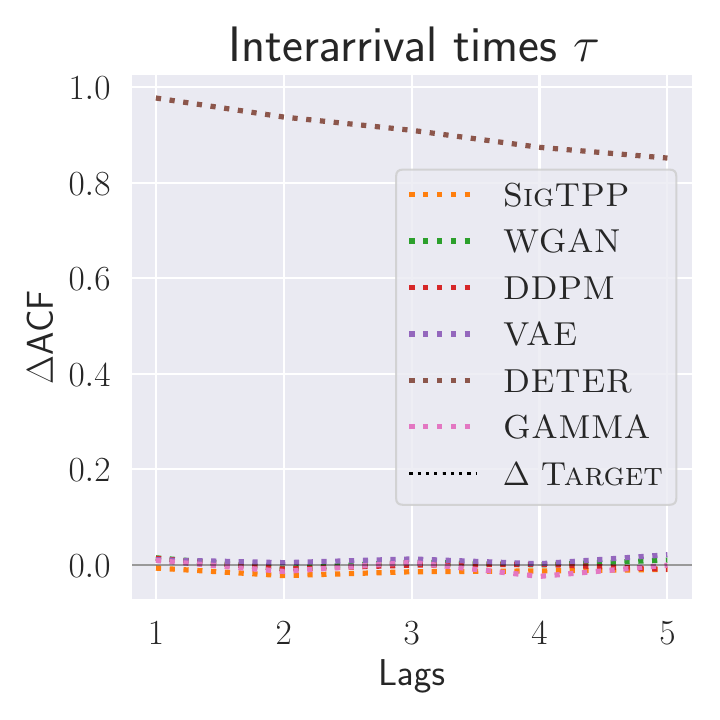}
        \par\vspace{-4pt}
        \makebox[\textwidth][c]{\small(e)}
    \end{subfigure}
    \caption{Qualitative evaluation on SR. Model panels are ordered \sigTPP, WGAN, DDPM, VAE, \determodel, and \gammamodel.
    (a) Histogram estimates of the interarrival times density $p(x)$ on a logarithmic scale. The leftmost panel shows the target distribution, the others the distributions of samples generated by each model.
    (b) Quantile-quantile (QQ) plots. Each subplot shows empirical quantiles of the reference data ($x$-axis) against generated sample quantiles ($y$-axis). The red dashed line is the identity $y=x$.
    (c) Empirical intensity functions $\hat{\lambda}(t)$ estimated from the same sequences.
    (d) Absolute Pearson correlation-matrix differences for up to the first $20$ event positions. Each heatmap entry shows the absolute difference between the empirical correlation of two sequence components in the generated samples and in the reference data.
    (e) Signed autocorrelation error for $5$ lags. Each curve shows the difference between the ACF of the indicated model and the empirical ACF of the reference process, so that the target corresponds to the horizontal line at zero.}
    \label{fig::sr_qualitative_summary}
\end{figure}

\begin{figure}[!htbp]
    \centering
    \begin{subfigure}[b]{\textwidth}
        \centering
        \includegraphics[width=\textwidth]{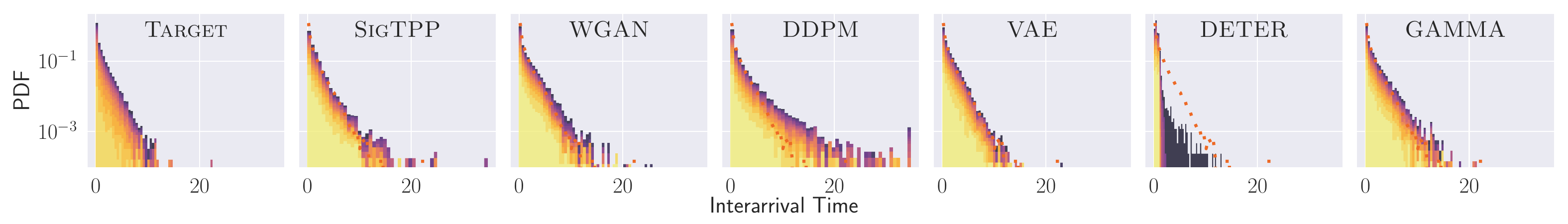}
        \par\vspace{-4pt}
        \makebox[\textwidth][c]{\small(a)}
    \end{subfigure}
    \par\vspace{4pt}
    \begin{subfigure}[b]{\textwidth}
        \centering
        \includegraphics[width=\textwidth]{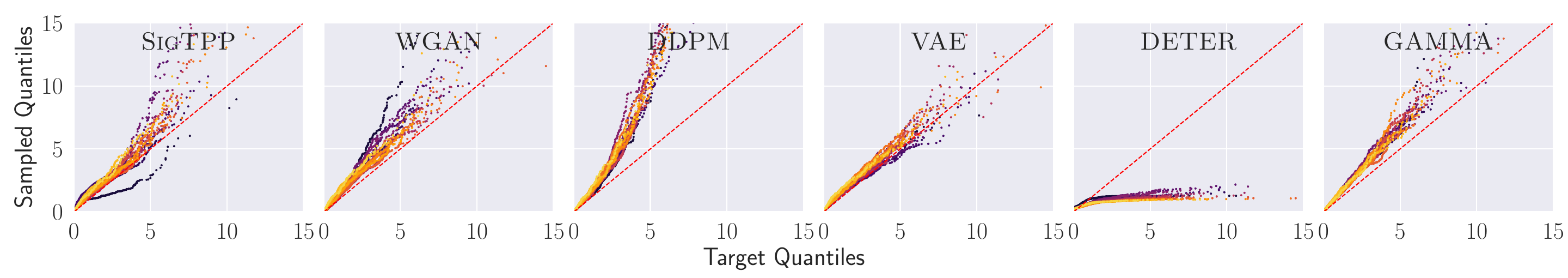}
        \par\vspace{-4pt}
        \makebox[\textwidth][c]{\small(b)}
    \end{subfigure}
    \par\vspace{4pt}
    \begin{subfigure}[b]{\textwidth}
        \centering
        \includegraphics[width=\textwidth]{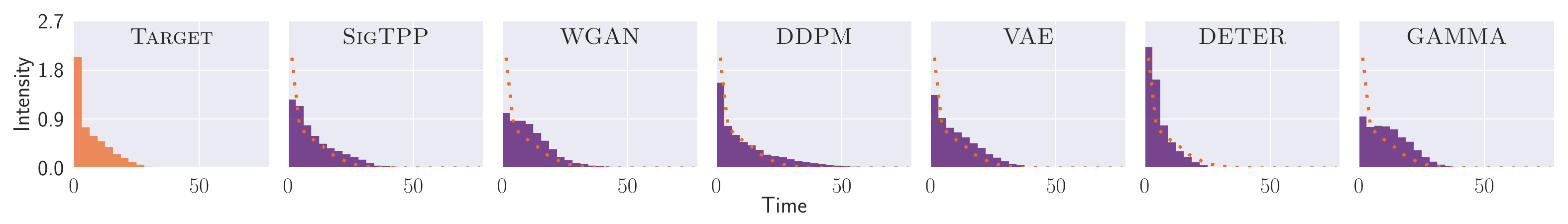}
        \par\vspace{-4pt}
        \makebox[\textwidth][c]{\small(c)}
    \end{subfigure}
    \par\vspace{4pt}
    \begin{subfigure}[t]{0.75\textwidth}
        \centering
        \includegraphics[width=\textwidth]{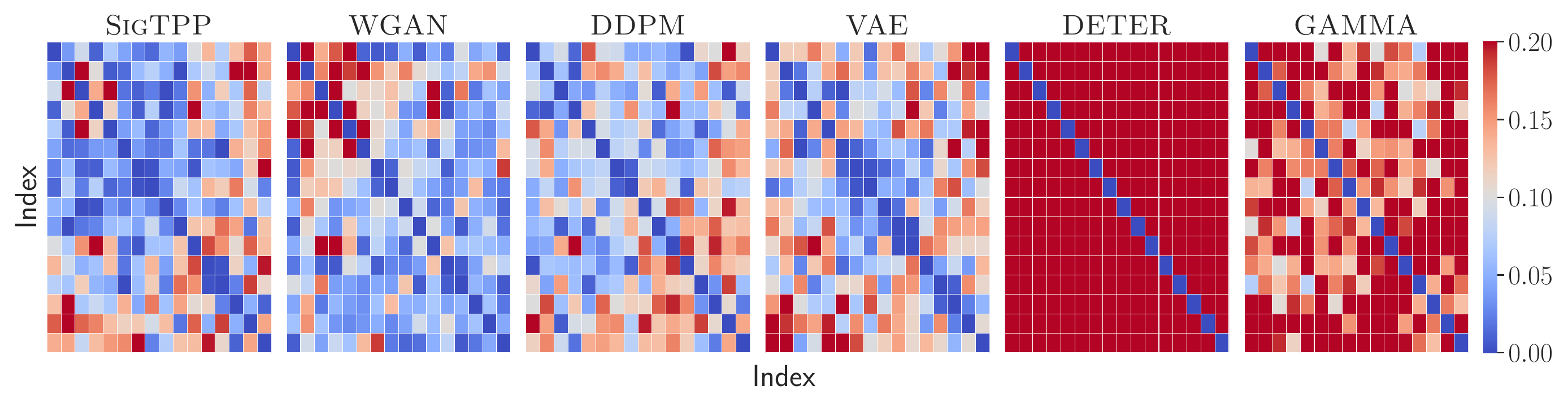}
        \par\vspace{-4pt}
        \makebox[\textwidth][c]{\small(d)}
    \end{subfigure}
    \hfill
    \begin{subfigure}[t]{0.23\textwidth}
        \centering
        \includegraphics[width=\textwidth]{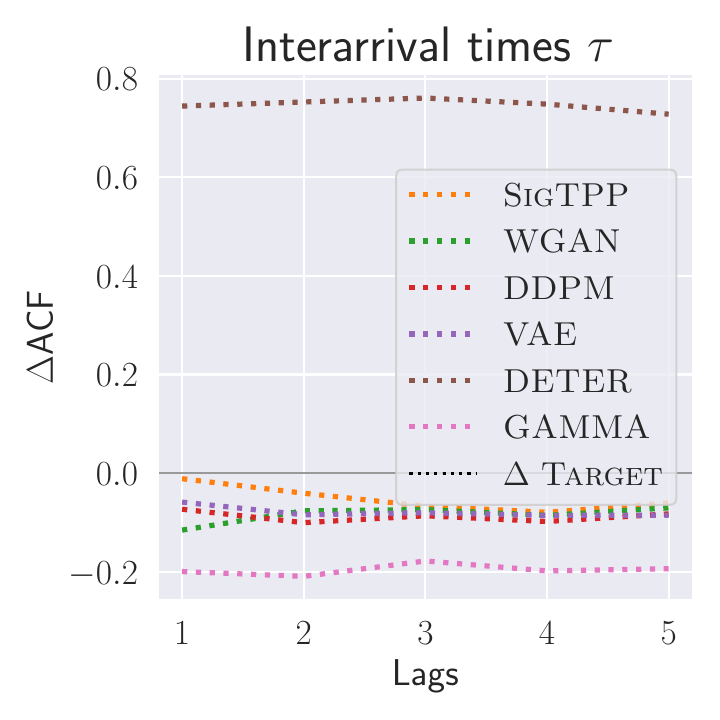}
        \par\vspace{-4pt}
        \makebox[\textwidth][c]{\small(e)}
    \end{subfigure}
    \caption{Qualitative evaluation on EQ. Model panels are ordered \sigTPP, WGAN, DDPM, VAE, \determodel, and \gammamodel.
    (a) Histogram estimates of the interarrival times density $p(x)$ on a logarithmic scale. The leftmost panel shows the target distribution, the others the distributions of samples generated by each model.
    (b) Quantile-quantile (QQ) plots. Each subplot shows empirical quantiles of the reference data ($x$-axis) against generated sample quantiles ($y$-axis). The red dashed line is the identity $y=x$.
    (c) Empirical intensity functions $\hat{\lambda}(t)$ estimated from the same sequences.
    (d) Absolute Pearson correlation-matrix differences for up to the first $20$ event positions. Each heatmap entry shows the absolute difference between the empirical correlation of two sequence components in the generated samples and in the reference data.
    (e) Signed autocorrelation error for $5$ lags. Each curve shows the difference between the ACF of the indicated model and the empirical ACF of the reference process, so that the target corresponds to the horizontal line at zero.}
    \label{fig::eq_qualitative_summary}
\end{figure}

\begin{figure}[!htbp]
    \centering
    \begin{subfigure}[b]{\textwidth}
        \centering
        \includegraphics[width=\textwidth]{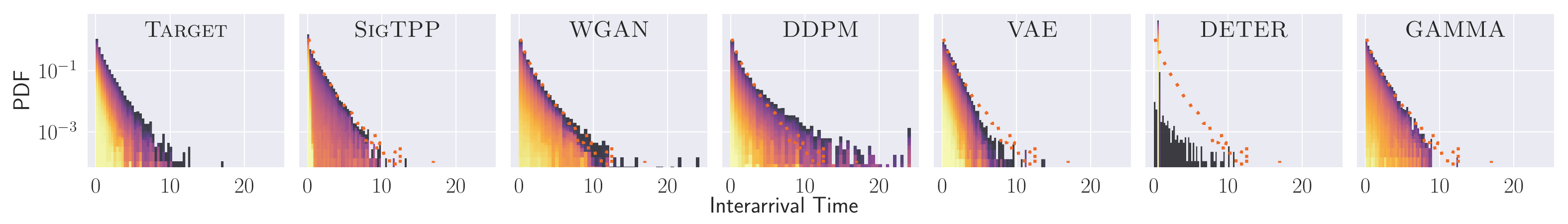}
        \par\vspace{-4pt}
        \makebox[\textwidth][c]{\small(a)}
    \end{subfigure}
    \par\vspace{4pt}
    \begin{subfigure}[b]{\textwidth}
        \centering
        \includegraphics[width=\textwidth]{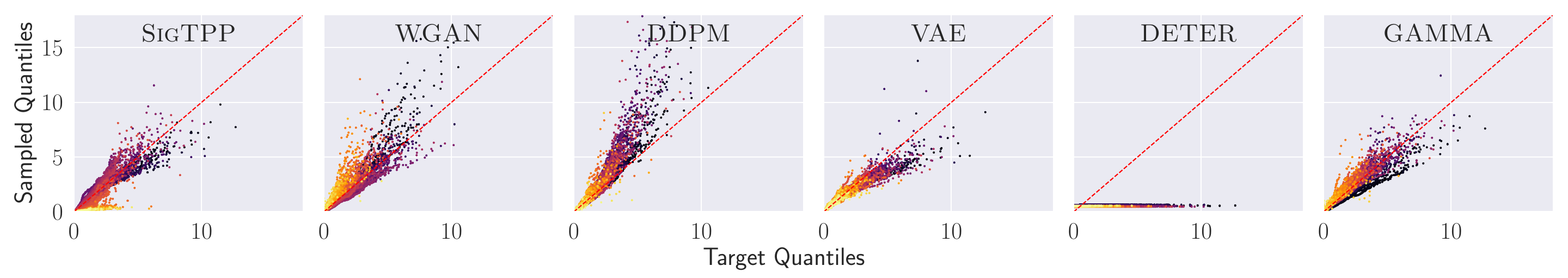}
        \par\vspace{-4pt}
        \makebox[\textwidth][c]{\small(b)}
    \end{subfigure}
    \par\vspace{4pt}
    \begin{subfigure}[b]{\textwidth}
        \centering
        \includegraphics[width=\textwidth]{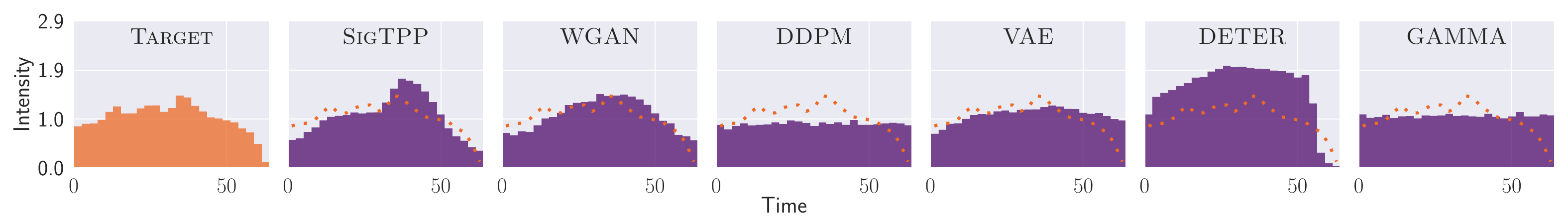}
        \par\vspace{-4pt}
        \makebox[\textwidth][c]{\small(c)}
    \end{subfigure}
    \par\vspace{4pt}
    \begin{subfigure}[t]{0.75\textwidth}
        \centering
        \includegraphics[width=\textwidth]{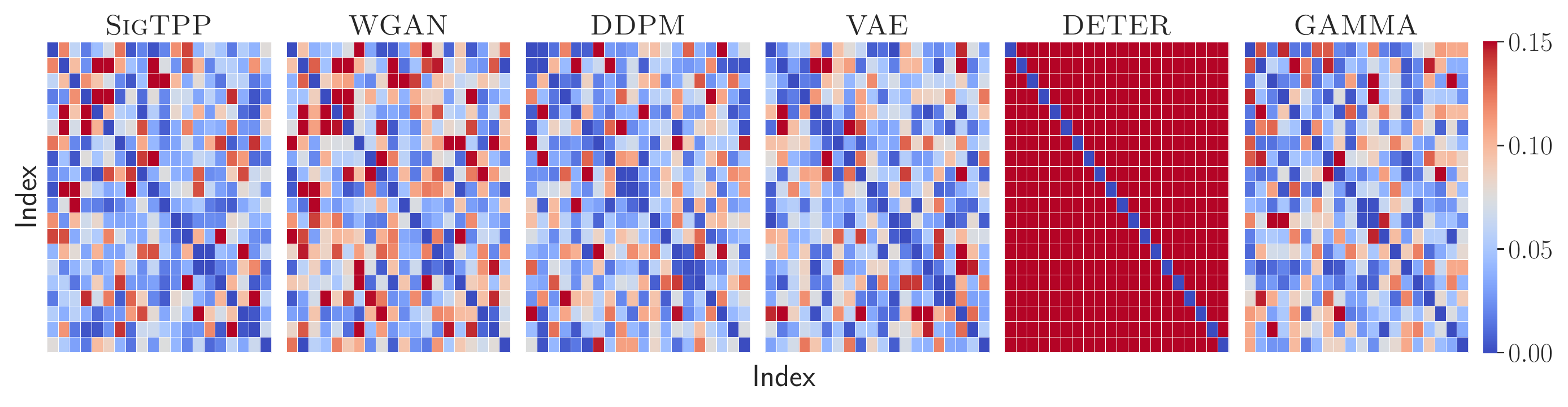}
        \par\vspace{-4pt}
        \makebox[\textwidth][c]{\small(d)}
    \end{subfigure}
    \hfill
    \begin{subfigure}[t]{0.23\textwidth}
        \centering
        \includegraphics[width=\textwidth]{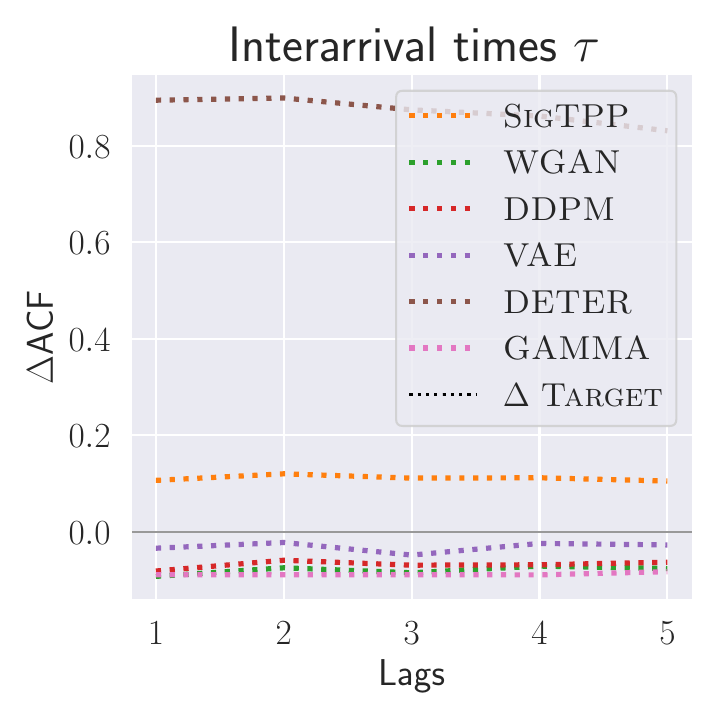}
        \par\vspace{-4pt}
        \makebox[\textwidth][c]{\small(e)}
    \end{subfigure}
    \caption{Qualitative evaluation on SO. Model panels are ordered \sigTPP, WGAN, DDPM, VAE, \determodel, and \gammamodel.
    (a) Histogram estimates of the interarrival times density $p(x)$ on a logarithmic scale. The leftmost panel shows the target distribution, the others the distributions of samples generated by each model.
    (b) Quantile-quantile (QQ) plots. Each subplot shows empirical quantiles of the reference data ($x$-axis) against generated sample quantiles ($y$-axis). The red dashed line is the identity $y=x$.
    (c) Empirical intensity functions $\hat{\lambda}(t)$ estimated from the same sequences.
    (d) Absolute Pearson correlation-matrix differences for up to the first $20$ event positions. Each heatmap entry shows the absolute difference between the empirical correlation of two sequence components in the generated samples and in the reference data.
    (e) Signed autocorrelation error for $5$ lags. Each curve shows the difference between the ACF of the indicated model and the empirical ACF of the reference process, so that the target corresponds to the horizontal line at zero.}
    \label{fig::so_qualitative_summary}
\end{figure}

\begin{figure}[!htbp]
    \centering
    \begin{subfigure}[b]{\textwidth}
        \centering
        \includegraphics[width=\textwidth]{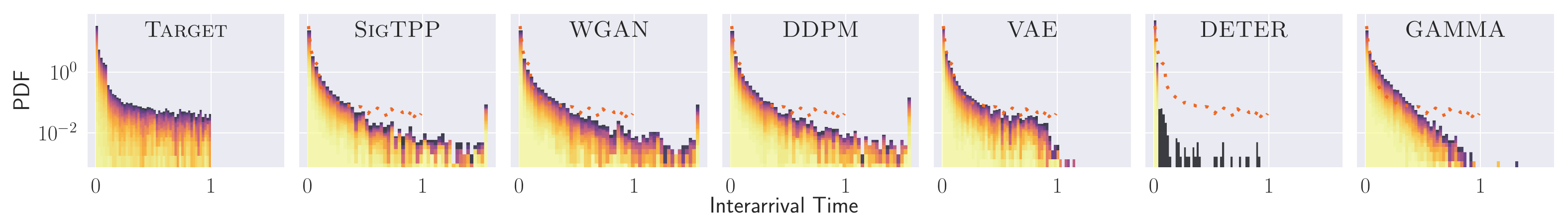}
        \par\vspace{-4pt}
        \makebox[\textwidth][c]{\small(a)}
    \end{subfigure}
    \par\vspace{4pt}
    \begin{subfigure}[b]{\textwidth}
        \centering
        \includegraphics[width=\textwidth]{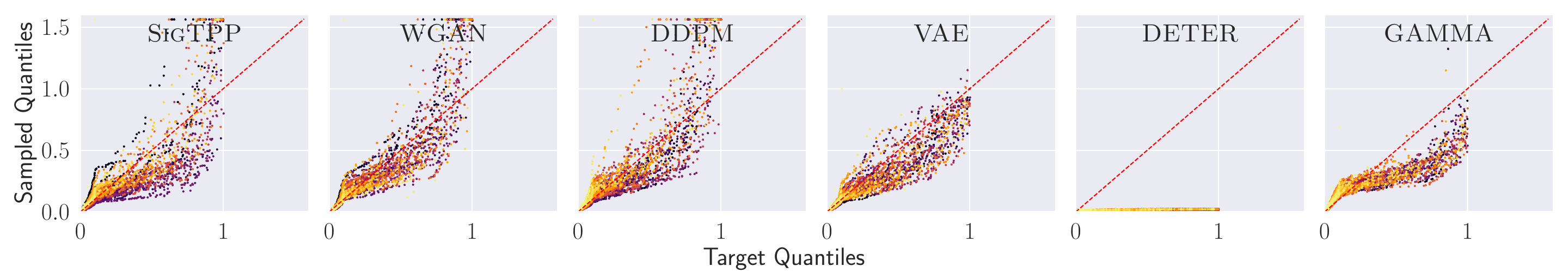}
        \par\vspace{-4pt}
        \makebox[\textwidth][c]{\small(b)}
    \end{subfigure}
    \par\vspace{4pt}
    \begin{subfigure}[b]{\textwidth}
        \centering
        \includegraphics[width=\textwidth]{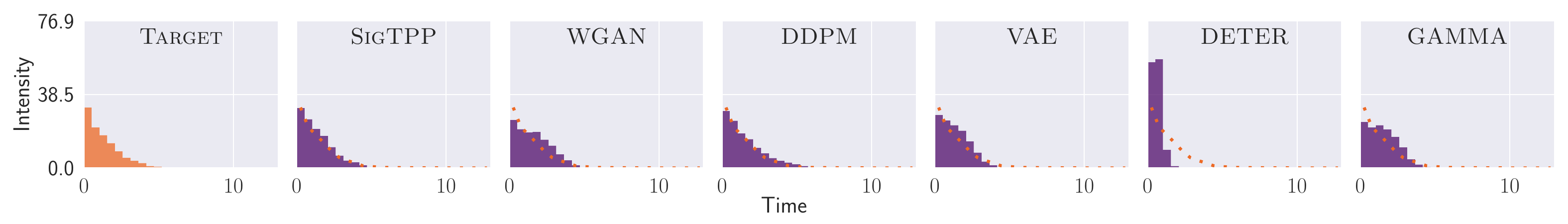}
        \par\vspace{-4pt}
        \makebox[\textwidth][c]{\small(c)}
    \end{subfigure}
    \par\vspace{4pt}
    \begin{subfigure}[t]{0.75\textwidth}
        \centering
        \includegraphics[width=\textwidth]{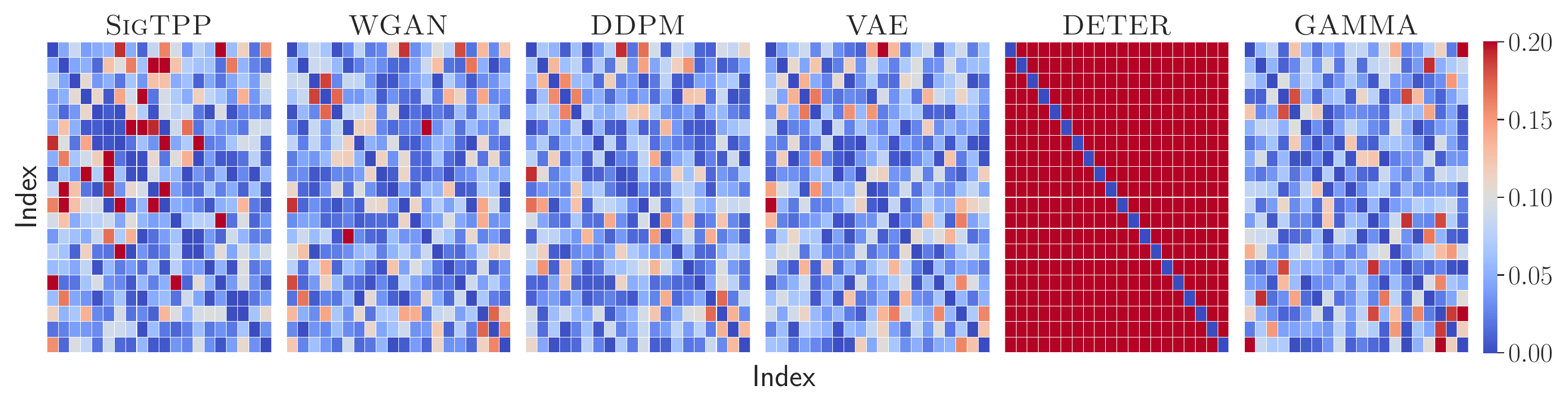}
        \par\vspace{-4pt}
        \makebox[\textwidth][c]{\small(d)}
    \end{subfigure}
    \hfill
    \begin{subfigure}[t]{0.23\textwidth}
        \centering
        \includegraphics[width=\textwidth]{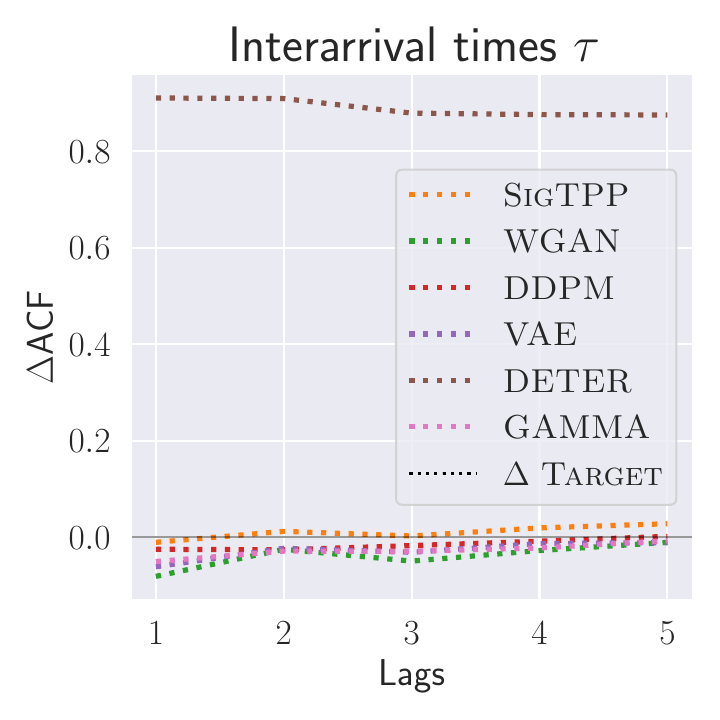}
        \par\vspace{-4pt}
        \makebox[\textwidth][c]{\small(e)}
    \end{subfigure}
    \caption{Qualitative evaluation on TB. Model panels are ordered \sigTPP, WGAN, DDPM, VAE, \determodel, and \gammamodel.
    (a) Histogram estimates of the interarrival times density $p(x)$ on a logarithmic scale. The leftmost panel shows the target distribution, the others the distributions of samples generated by each model.
    (b) Quantile-quantile (QQ) plots. Each subplot shows empirical quantiles of the reference data ($x$-axis) against generated sample quantiles ($y$-axis). The red dashed line is the identity $y=x$.
    (c) Empirical intensity functions $\hat{\lambda}(t)$ estimated from the same sequences.
    (d) Absolute Pearson correlation-matrix differences for up to the first $20$ event positions. Each heatmap entry shows the absolute difference between the empirical correlation of two sequence components in the generated samples and in the reference data.
    (e) Signed autocorrelation error for $5$ lags. Each curve shows the difference between the ACF of the indicated model and the empirical ACF of the reference process, so that the target corresponds to the horizontal line at zero.}
    \label{fig::taobao_qualitative_summary}
\end{figure}

\begin{figure}[!htbp]
    \centering
    \begin{subfigure}[b]{\textwidth}
        \centering
        \includegraphics[width=\textwidth]{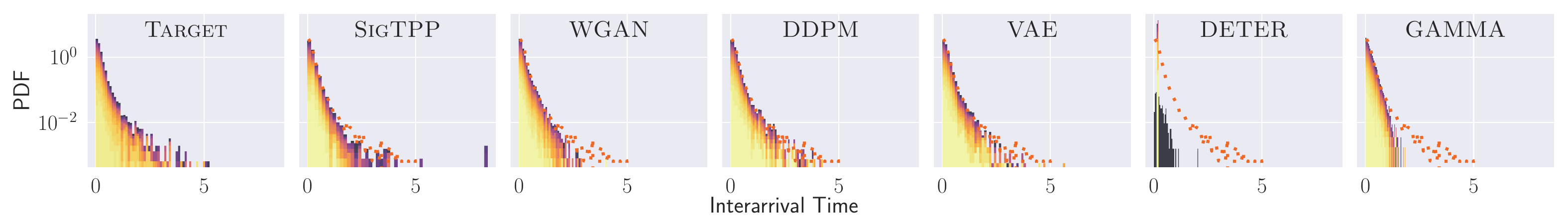}
        \par\vspace{-4pt}
        \makebox[\textwidth][c]{\small(a)}
    \end{subfigure}
    \par\vspace{4pt}
    \begin{subfigure}[b]{\textwidth}
        \centering
        \includegraphics[width=\textwidth]{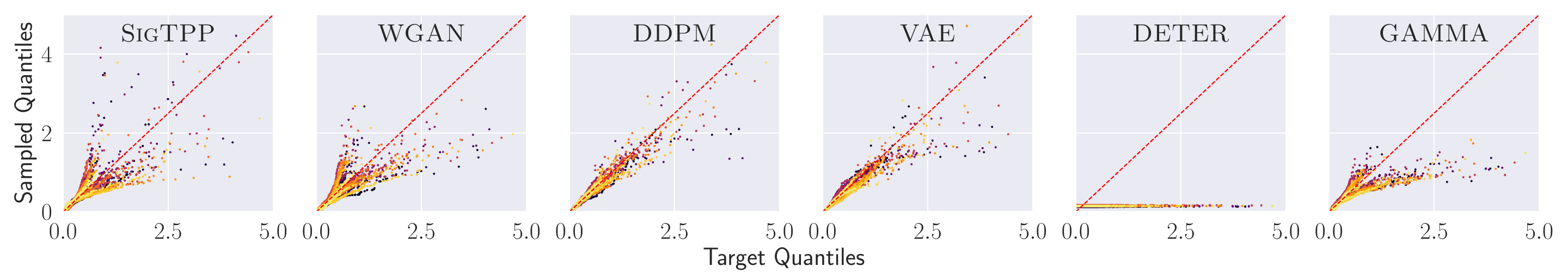}
        \par\vspace{-4pt}
        \makebox[\textwidth][c]{\small(b)}
    \end{subfigure}
    \par\vspace{4pt}
    \begin{subfigure}[b]{\textwidth}
        \centering
        \includegraphics[width=\textwidth]{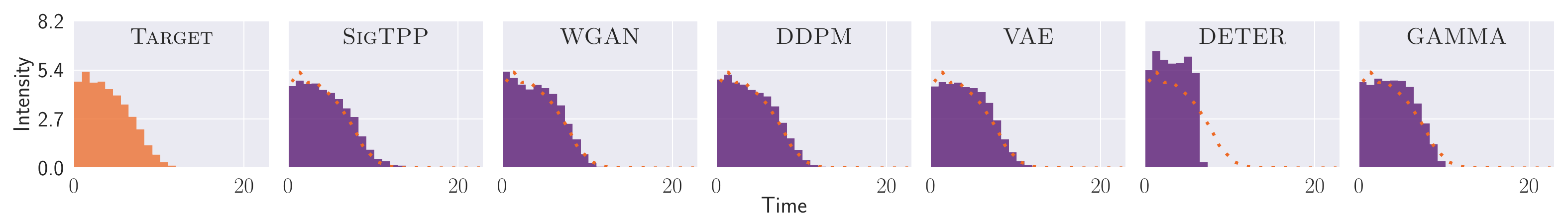}
        \par\vspace{-4pt}
        \makebox[\textwidth][c]{\small(c)}
    \end{subfigure}
    \par\vspace{4pt}
    \begin{subfigure}[t]{0.75\textwidth}
        \centering
        \includegraphics[width=\textwidth]{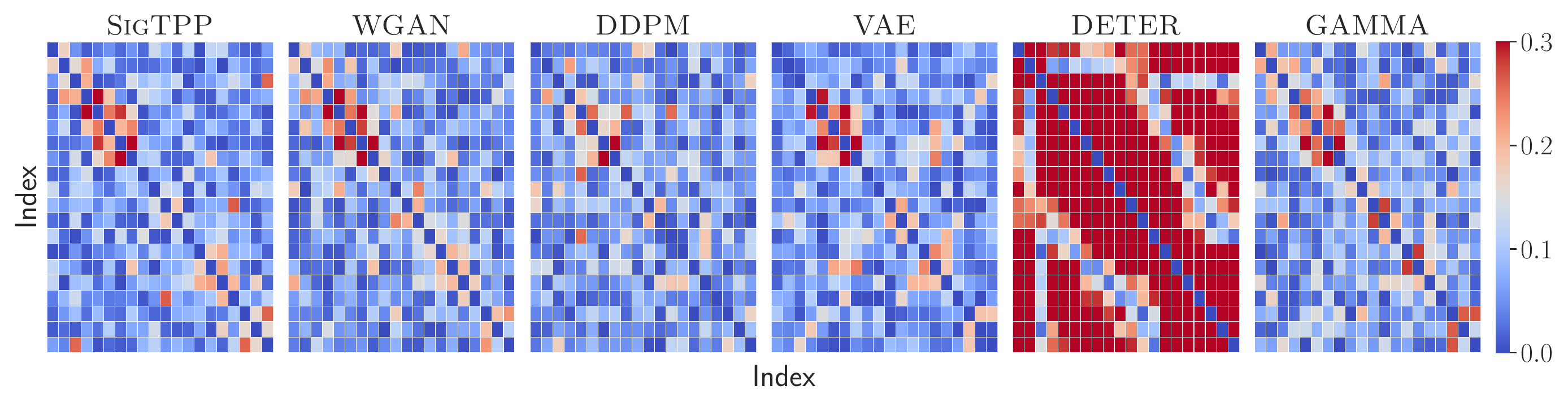}
        \par\vspace{-4pt}
        \makebox[\textwidth][c]{\small(d)}
    \end{subfigure}
    \hfill
    \begin{subfigure}[t]{0.23\textwidth}
        \centering
        \includegraphics[width=\textwidth]{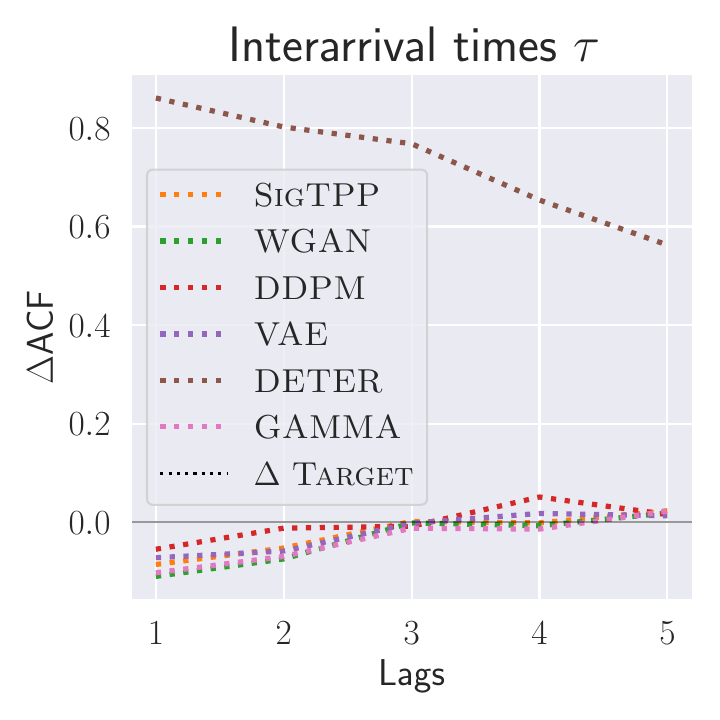}
        \par\vspace{-4pt}
        \makebox[\textwidth][c]{\small(e)}
    \end{subfigure}
    \caption{Qualitative evaluation on TX. Model panels are ordered \sigTPP, WGAN, DDPM, VAE, \determodel, and \gammamodel.
    (a) Histogram estimates of the interarrival times density $p(x)$ on a logarithmic scale. The leftmost panel shows the target distribution, the others the distributions of samples generated by each model.
    (b) Quantile-quantile (QQ) plots. Each subplot shows empirical quantiles of the reference data ($x$-axis) against generated sample quantiles ($y$-axis). The red dashed line is the identity $y=x$.
    (c) Empirical intensity functions $\hat{\lambda}(t)$ estimated from the same sequences.
    (d) Absolute Pearson correlation-matrix differences for up to the first $20$ event positions. Each heatmap entry shows the absolute difference between the empirical correlation of two sequence components in the generated samples and in the reference data.
    (e) Signed autocorrelation error for $5$ lags. Each curve shows the difference between the ACF of the indicated model and the empirical ACF of the reference process, so that the target corresponds to the horizontal line at zero.}
    \label{fig::taxi_qualitative_summary}
\end{figure}

\begin{figure}[!htbp]
    \centering
    \begin{subfigure}[b]{\textwidth}
        \centering
        \includegraphics[width=\textwidth]{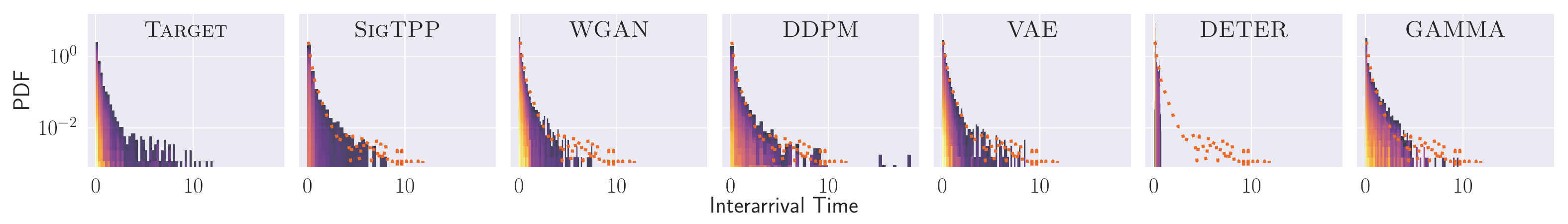}
        \par\vspace{-4pt}
        \makebox[\textwidth][c]{\small(a)}
    \end{subfigure}
    \par\vspace{4pt}
    \begin{subfigure}[b]{\textwidth}
        \centering
        \includegraphics[width=\textwidth]{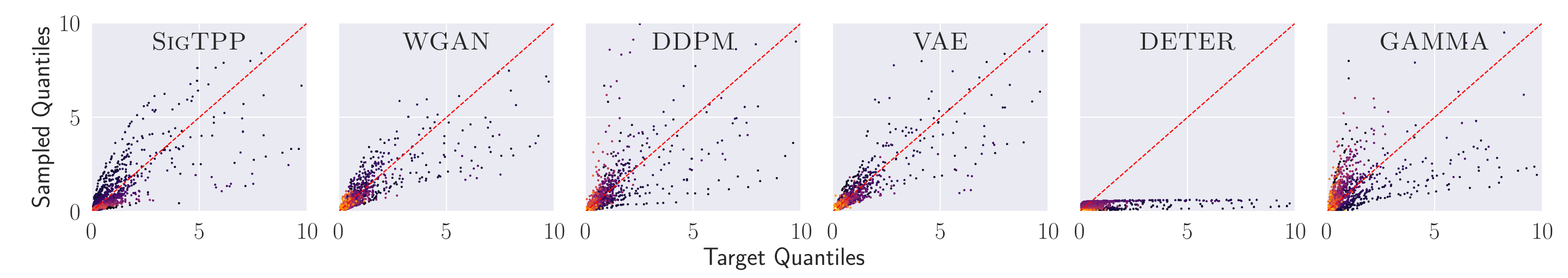}
        \par\vspace{-4pt}
        \makebox[\textwidth][c]{\small(b)}
    \end{subfigure}
    \par\vspace{4pt}
    \begin{subfigure}[b]{\textwidth}
        \centering
        \includegraphics[width=\textwidth]{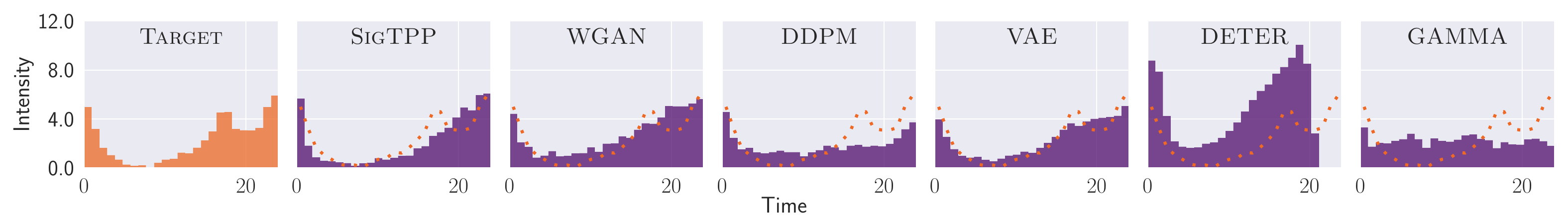}
        \par\vspace{-4pt}
        \makebox[\textwidth][c]{\small(c)}
    \end{subfigure}
    \par\vspace{4pt}
    \begin{subfigure}[t]{0.75\textwidth}
        \centering
        \includegraphics[width=\textwidth]{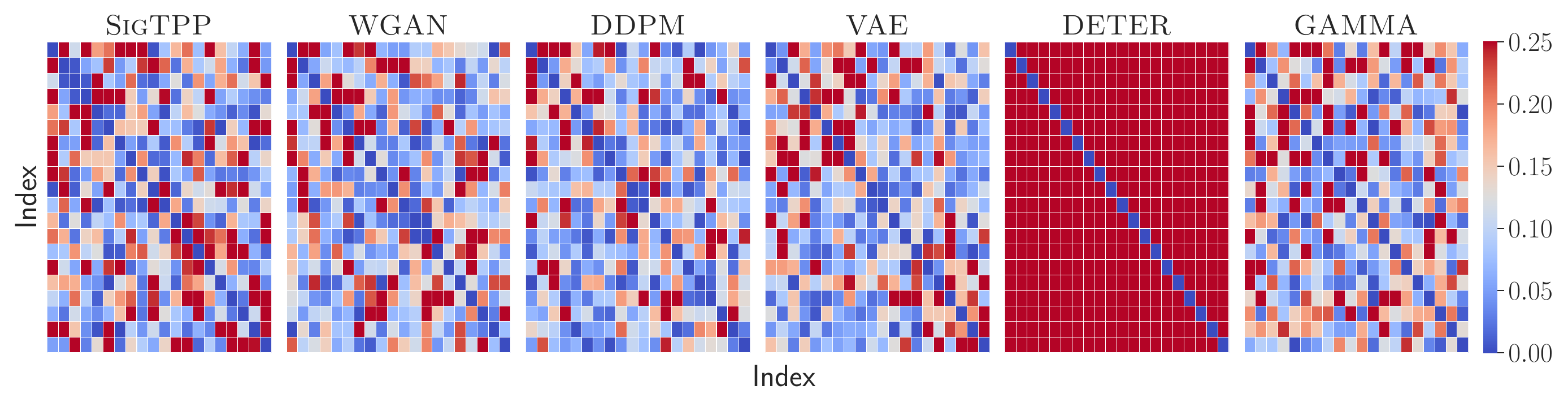}
        \par\vspace{-4pt}
        \makebox[\textwidth][c]{\small(d)}
    \end{subfigure}
    \hfill
    \begin{subfigure}[t]{0.23\textwidth}
        \centering
        \includegraphics[width=\textwidth]{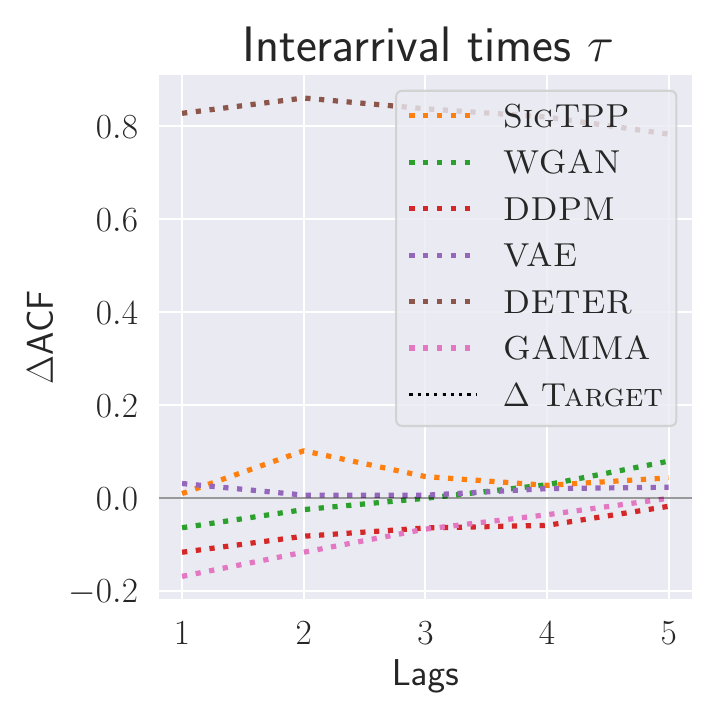}
        \par\vspace{-4pt}
        \makebox[\textwidth][c]{\small(e)}
    \end{subfigure}
    \caption{Qualitative evaluation on YLP. Model panels are ordered \sigTPP, WGAN, DDPM, VAE, \determodel, and \gammamodel.
    (a) Histogram estimates of the interarrival times density $p(x)$ on a logarithmic scale. The leftmost panel shows the target distribution, the others the distributions of samples generated by each model.
    (b) Quantile-quantile (QQ) plots. Each subplot shows empirical quantiles of the reference data ($x$-axis) against generated sample quantiles ($y$-axis). The red dashed line is the identity $y=x$.
    (c) Empirical intensity functions $\hat{\lambda}(t)$ estimated from the same sequences.
    (d) Absolute Pearson correlation-matrix differences for up to the first $20$ event positions. Each heatmap entry shows the absolute difference between the empirical correlation of two sequence components in the generated samples and in the reference data.
    (e) Signed autocorrelation error for $5$ lags. Each curve shows the difference between the ACF of the indicated model and the empirical ACF of the reference process, so that the target corresponds to the horizontal line at zero.}
    \label{fig::yelp_qualitative_summary}
\end{figure}

\subsection{Comparison to EasyTPP}
We also compare \sigTPP~with $5$ methods provided by the EasyTPP benchmark~\cite{xue2024easytpp}: the recurrent Neural Hawkes Process (NHP)~\cite{mei2016neural} and Recurrent Marked Temporal Point Process (RMTPP)~\cite{original_2016}, the attention-based Transformer Hawkes Process (THP)~\cite{zuo_transformer_2020} and Self-Attentive Hawkes Process (SAHP)~\cite{zhang2020self}, and the continuous-time State-Space Point Process (S2P2)~\cite{s2p2}. These models retain their likelihood training objectives, but we select the best model on the same criterion as the other method for fairness.

Computational constraints limit this comparison to half of the $12$ benchmark datasets: homogeneous Poisson (PS), inhomogeneous Poisson (IP), univariate Hawkes (H1), Earthquake (EQ), Yelp-Airport (YLA), and Taobao (TB). Appendix~\ref{sect::datasets} describes the datasets. We exclude the S2P2 run on TB because thinning exceeded the GPU time budget. Table~\ref{table::additional_easytpp} reports mean test scores with standard errors in parentheses. These results do not enter the main benchmark aggregates. Nonetheless, the findings are in line with the main results. 

\sigTPP~achieves the lowest geometric mean relative score on $6$ of the $8$ metrics (Figure~\ref{fig::easytpp_gmrae}), and the best overall average rank of $2.40$, followed by NHP at $2.94$, see Table~\ref{table::easytpp_avg_rank_by_metric}. The EasyTPP baselines obtain lower aggregate CRPS, consistent with the conditional predictive advantage observed in the main comparison.

\begin{table}[p]
\renewcommand{\arraystretch}{0.9}%
\centering
\resultcaption{Performance of \sigTPP~and $5$ EasyTPP baselines on PS, IP, H1, EQ, YLA, and TB. Dashes denote the excluded S2P2 run on TB, which exceeded the GPU time budget.}
\label{table::additional_easytpp}
\smallskip
\setlength{\tabcolsep}{4pt}
\setlength{\aboverulesep}{1pt}
\setlength{\belowrulesep}{1pt}
\begin{adjustbox}{width=\linewidth}
\begin{tabular}[t]{c l*{3}{r}@{\hspace{6pt}}*{3}{r}}
  \toprule
   & Model & \multicolumn{1}{c}{PS} & \multicolumn{1}{c}{IP} & \multicolumn{1}{c}{H1} & \multicolumn{1}{c}{EQ} & \multicolumn{1}{c}{YLA} & \multicolumn{1}{c}{TB} \\
  \cmidrule(lr){2-5}\cmidrule(lr){6-8}
  \multirow{6}{*}{\rotatebox[origin=c]{90}{\shortstack[c]{$\mathcal{E}$ {\scriptsize $(\times 10^{\scalebox{1}[1]{-}3})$}}}}
   & \sigTPP & \textbf{0.53} \footnotesize{(0.93)} & \textbf{0.42} \footnotesize{(0.32)} & \textbf{1.16} \footnotesize{(0.74)} & 32.40 \footnotesize{(0.65)} & \textbf{7.77} \footnotesize{(4.47)} & \underline{511.8} \footnotesize{(25.25)} \\
   & \cellcolor{gray!15}NHP & \cellcolor{gray!15}4.02 \footnotesize{(2.39)} & \cellcolor{gray!15}\underline{2.07} \footnotesize{(0.62)} & \cellcolor{gray!15}10.8 \footnotesize{(3.61)} & \cellcolor{gray!15}\underline{31.65} \footnotesize{(0.62)} & \cellcolor{gray!15}56.57 \footnotesize{(19.67)} & \cellcolor{gray!15}\textbf{455.9} \footnotesize{(20.96)} \\
   & RMTPP & \underline{2.13} \footnotesize{(1.28)} & 3.37 \footnotesize{(0.87)} & 25.6 \footnotesize{(5.03)} & 31.77 \footnotesize{(0.62)} & \underline{10.43} \footnotesize{(7.81)} & 579.3 \footnotesize{(33.24)} \\
   & \cellcolor{gray!15}SAHP & \cellcolor{gray!15}5.91 \footnotesize{(2.47)} & \cellcolor{gray!15}2.75 \footnotesize{(1.13)} & \cellcolor{gray!15}\underline{8.97} \footnotesize{(3.09)} & \cellcolor{gray!15}33.47 \footnotesize{(0.58)} & \cellcolor{gray!15}27.34 \footnotesize{(17.65)} & \cellcolor{gray!15}541.5 \footnotesize{(24.80)} \\
   & THP & 2.86 \footnotesize{(1.28)} & 3.75 \footnotesize{(1.45)} & 26.3 \footnotesize{(5.29)} & \textbf{29.74} \footnotesize{(0.80)} & 12.99 \footnotesize{(7.39)} & 585.2 \footnotesize{(25.41)} \\
   & \cellcolor{gray!15}S2P2 & \cellcolor{gray!15}21.31 \footnotesize{(5.00)} & \cellcolor{gray!15}2.89 \footnotesize{(0.67)} & \cellcolor{gray!15}10.1 \footnotesize{(3.73)} & \cellcolor{gray!15}32.29 \footnotesize{(0.63)} & \cellcolor{gray!15}50.03 \footnotesize{(18.22)} & \multicolumn{1}{c}{\cellcolor{gray!15}--} \\
  \cmidrule(lr){2-5}\cmidrule(lr){6-8}
  \multirow{6}{*}{\rotatebox[origin=c]{90}{\shortstack[c]{$\WAS$ {\scriptsize $(\times 10^{\scalebox{1}[1]{-}2})$}}}}
   & \sigTPP & \textbf{5.26} \footnotesize{(0.20)} & \textbf{6.55} \footnotesize{(0.14)} & \textbf{3.75} \footnotesize{(0.31)} & \underline{2.91} \footnotesize{(0.04)} & \textbf{8.48} \footnotesize{(0.73)} & \underline{58.17} \footnotesize{(2.56)} \\
   & \cellcolor{gray!15}NHP & \cellcolor{gray!15}6.26 \footnotesize{(0.48)} & \cellcolor{gray!15}\underline{7.12} \footnotesize{(0.18)} & \cellcolor{gray!15}5.95 \footnotesize{(0.55)} & \cellcolor{gray!15}2.95 \footnotesize{(0.05)} & \cellcolor{gray!15}14.28 \footnotesize{(1.91)} & \cellcolor{gray!15}\textbf{53.02} \footnotesize{(2.04)} \\
   & RMTPP & 5.99 \footnotesize{(0.34)} & 7.47 \footnotesize{(0.24)} & 8.11 \footnotesize{(0.76)} & 2.94 \footnotesize{(0.05)} & 9.06 \footnotesize{(0.87)} & 66.74 \footnotesize{(2.61)} \\
   & \cellcolor{gray!15}SAHP & \cellcolor{gray!15}6.76 \footnotesize{(0.52)} & \cellcolor{gray!15}7.38 \footnotesize{(0.28)} & \cellcolor{gray!15}\underline{5.74} \footnotesize{(0.58)} & \cellcolor{gray!15}2.96 \footnotesize{(0.05)} & \cellcolor{gray!15}11.48 \footnotesize{(1.81)} & \cellcolor{gray!15}59.83 \footnotesize{(2.45)} \\
   & THP & \underline{5.89} \footnotesize{(0.29)} & 7.75 \footnotesize{(0.29)} & 8.34 \footnotesize{(0.60)} & \textbf{2.81} \footnotesize{(0.05)} & \underline{8.60} \footnotesize{(0.92)} & 64.28 \footnotesize{(2.22)} \\
   & \cellcolor{gray!15}S2P2 & \cellcolor{gray!15}9.23 \footnotesize{(0.93)} & \cellcolor{gray!15}7.48 \footnotesize{(0.23)} & \cellcolor{gray!15}5.92 \footnotesize{(0.64)} & \cellcolor{gray!15}3.04 \footnotesize{(0.05)} & \cellcolor{gray!15}14.73 \footnotesize{(2.32)} & \multicolumn{1}{c}{\cellcolor{gray!15}--} \\
  \cmidrule(lr){2-5}\cmidrule(lr){6-8}
  \multirow{6}{*}{\rotatebox[origin=c]{90}{\shortstack[c]{$\sigW$ {\scriptsize $(\times 10^{\scalebox{1}[1]{-}5})$}}}}
   & \sigTPP & \textbf{1.42} \footnotesize{(1.52)} & \textbf{1.20} \footnotesize{(1.05)} & \underline{44.0} \footnotesize{(29.2)} & \underline{409.2} \footnotesize{(36.30)} & 4.64 \footnotesize{(4.30)} & \underline{2.27} \footnotesize{(0.94)} \\
   & \cellcolor{gray!15}NHP & \cellcolor{gray!15}10.33 \footnotesize{(6.19)} & \cellcolor{gray!15}13.0 \footnotesize{(4.04)} & \cellcolor{gray!15}98.1 \footnotesize{(25.0)} & \cellcolor{gray!15}513.4 \footnotesize{(41.31)} & \cellcolor{gray!15}\textbf{1.52} \footnotesize{(1.83)} & \cellcolor{gray!15}18.28 \footnotesize{(2.76)} \\
   & RMTPP & 13.08 \footnotesize{(7.49)} & 23.1 \footnotesize{(5.39)} & 304 \footnotesize{(66.1)} & 979.5 \footnotesize{(59.51)} & 5.93 \footnotesize{(4.85)} & 17.41 \footnotesize{(2.62)} \\
   & \cellcolor{gray!15}SAHP & \cellcolor{gray!15}35.39 \footnotesize{(12.05)} & \cellcolor{gray!15}\underline{3.92} \footnotesize{(2.18)} & \cellcolor{gray!15}98.7 \footnotesize{(26.9)} & \cellcolor{gray!15}\textbf{406.4} \footnotesize{(37.45)} & \cellcolor{gray!15}3.55 \footnotesize{(3.62)} & \cellcolor{gray!15}\textbf{0.53} \footnotesize{(0.36)} \\
   & THP & 21.23 \footnotesize{(8.72)} & 17.2 \footnotesize{(4.72)} & 614 \footnotesize{(103)} & 570.6 \footnotesize{(46.18)} & 27.20 \footnotesize{(11.58)} & 7.01 \footnotesize{(1.35)} \\
   & \cellcolor{gray!15}S2P2 & \cellcolor{gray!15}\underline{10.01} \footnotesize{(5.94)} & \cellcolor{gray!15}9.94 \footnotesize{(3.50)} & \cellcolor{gray!15}\textbf{40.7} \footnotesize{(12.6)} & \cellcolor{gray!15}626.1 \footnotesize{(43.81)} & \cellcolor{gray!15}\underline{2.38} \footnotesize{(2.51)} & \multicolumn{1}{c}{\cellcolor{gray!15}--} \\
  \cmidrule(lr){2-5}\cmidrule(lr){6-8}
  \multirow{6}{*}{\rotatebox[origin=c]{90}{\shortstack[c]{CRPS {\scriptsize $(\times 10^{\scalebox{1}[1]{-}1})$}}}}
   & \sigTPP & 4.64 \footnotesize{(0.08)} & 3.04 \footnotesize{(0.03)} & 10.2 \footnotesize{(0.27)} & 5.61 \footnotesize{(0.11)} & 4.78 \footnotesize{(0.14)} & \underline{0.37} \footnotesize{(0.01)} \\
   & \cellcolor{gray!15}NHP & \cellcolor{gray!15}\underline{4.50} \footnotesize{(0.08)} & \cellcolor{gray!15}2.99 \footnotesize{(0.03)} & \cellcolor{gray!15}\textbf{9.33} \footnotesize{(0.27)} & \cellcolor{gray!15}\textbf{5.22} \footnotesize{(0.11)} & \cellcolor{gray!15}4.74 \footnotesize{(0.14)} & \cellcolor{gray!15}\textbf{0.37} \footnotesize{(0.01)} \\
   & RMTPP & 4.51 \footnotesize{(0.08)} & \underline{2.98} \footnotesize{(0.03)} & 9.50 \footnotesize{(0.27)} & 5.33 \footnotesize{(0.12)} & \underline{4.58} \footnotesize{(0.14)} & 0.37 \footnotesize{(0.01)} \\
   & \cellcolor{gray!15}SAHP & \cellcolor{gray!15}4.50 \footnotesize{(0.08)} & \cellcolor{gray!15}2.99 \footnotesize{(0.03)} & \cellcolor{gray!15}9.42 \footnotesize{(0.28)} & \cellcolor{gray!15}\underline{5.23} \footnotesize{(0.12)} & \cellcolor{gray!15}4.65 \footnotesize{(0.14)} & \cellcolor{gray!15}0.37 \footnotesize{(0.01)} \\
   & THP & \textbf{4.49} \footnotesize{(0.08)} & \textbf{2.98} \footnotesize{(0.03)} & 9.58 \footnotesize{(0.27)} & 5.29 \footnotesize{(0.12)} & \textbf{4.49} \footnotesize{(0.14)} & 0.40 \footnotesize{(0.01)} \\
   & \cellcolor{gray!15}S2P2 & \cellcolor{gray!15}4.55 \footnotesize{(0.09)} & \cellcolor{gray!15}3.00 \footnotesize{(0.03)} & \cellcolor{gray!15}\underline{9.36} \footnotesize{(0.27)} & \cellcolor{gray!15}5.26 \footnotesize{(0.12)} & \cellcolor{gray!15}4.71 \footnotesize{(0.15)} & \multicolumn{1}{c}{\cellcolor{gray!15}--} \\
  \cmidrule(lr){2-5}\cmidrule(lr){6-8}
  \multirow{6}{*}{\rotatebox[origin=c]{90}{\shortstack[c]{{\footnotesize $\histintloss$} {\scriptsize $(\times 10^{\scalebox{1}[1]{-}2})$}}}}
   & \sigTPP & \textbf{3.35} \footnotesize{(0.46)} & \textbf{3.09} \footnotesize{(0.35)} & \textbf{6.09} \footnotesize{(1.04)} & 19.34 \footnotesize{(0.99)} & 13.55 \footnotesize{(1.20)} & \textbf{3.77} \footnotesize{(0.90)} \\
   & \cellcolor{gray!15}NHP & \cellcolor{gray!15}4.31 \footnotesize{(0.54)} & \cellcolor{gray!15}5.01 \footnotesize{(0.41)} & \cellcolor{gray!15}13.1 \footnotesize{(1.13)} & \cellcolor{gray!15}\underline{16.73} \footnotesize{(0.92)} & \cellcolor{gray!15}18.68 \footnotesize{(0.82)} & \cellcolor{gray!15}10.37 \footnotesize{(1.21)} \\
   & RMTPP & 4.34 \footnotesize{(0.56)} & 5.99 \footnotesize{(0.46)} & 16.2 \footnotesize{(1.15)} & 21.30 \footnotesize{(0.91)} & 11.37 \footnotesize{(1.13)} & 21.93 \footnotesize{(1.54)} \\
   & \cellcolor{gray!15}SAHP & \cellcolor{gray!15}5.64 \footnotesize{(0.59)} & \cellcolor{gray!15}\underline{4.09} \footnotesize{(0.37)} & \cellcolor{gray!15}13.6 \footnotesize{(1.17)} & \cellcolor{gray!15}18.22 \footnotesize{(1.08)} & \cellcolor{gray!15}\underline{10.51} \footnotesize{(0.99)} & \cellcolor{gray!15}\underline{8.37} \footnotesize{(0.99)} \\
   & THP & 5.37 \footnotesize{(0.61)} & 4.32 \footnotesize{(0.39)} & 18.2 \footnotesize{(1.12)} & \textbf{16.09} \footnotesize{(1.12)} & \textbf{10.27} \footnotesize{(1.07)} & 16.73 \footnotesize{(1.00)} \\
   & \cellcolor{gray!15}S2P2 & \cellcolor{gray!15}\underline{3.36} \footnotesize{(0.54)} & \cellcolor{gray!15}5.11 \footnotesize{(0.50)} & \cellcolor{gray!15}\underline{6.32} \footnotesize{(1.05)} & \cellcolor{gray!15}22.35 \footnotesize{(0.81)} & \cellcolor{gray!15}18.61 \footnotesize{(1.04)} & \multicolumn{1}{c}{\cellcolor{gray!15}--} \\
  \cmidrule(lr){2-5}\cmidrule(lr){6-8}
  \multirow{6}{*}{\rotatebox[origin=c]{90}{\shortstack[c]{{\footnotesize $\histITloss$} {\scriptsize $(\times 10^{\scalebox{1}[1]{-}2})$}}}}
   & \sigTPP & \textbf{2.95} \footnotesize{(0.60)} & \underline{1.97} \footnotesize{(0.37)} & \textbf{3.84} \footnotesize{(0.84)} & \underline{12.86} \footnotesize{(0.84)} & 6.34 \footnotesize{(1.12)} & \textbf{7.33} \footnotesize{(0.39)} \\
   & \cellcolor{gray!15}NHP & \cellcolor{gray!15}4.12 \footnotesize{(0.53)} & \cellcolor{gray!15}2.55 \footnotesize{(0.39)} & \cellcolor{gray!15}6.34 \footnotesize{(0.85)} & \cellcolor{gray!15}13.29 \footnotesize{(1.01)} & \cellcolor{gray!15}3.92 \footnotesize{(0.91)} & \cellcolor{gray!15}20.76 \footnotesize{(0.72)} \\
   & RMTPP & 3.66 \footnotesize{(0.64)} & 3.60 \footnotesize{(0.45)} & 7.64 \footnotesize{(0.82)} & 21.02 \footnotesize{(0.93)} & \textbf{3.67} \footnotesize{(0.90)} & 24.60 \footnotesize{(0.91)} \\
   & \cellcolor{gray!15}SAHP & \cellcolor{gray!15}4.52 \footnotesize{(0.59)} & \cellcolor{gray!15}\textbf{1.72} \footnotesize{(0.32)} & \cellcolor{gray!15}5.47 \footnotesize{(0.76)} & \cellcolor{gray!15}14.39 \footnotesize{(1.02)} & \cellcolor{gray!15}\underline{3.67} \footnotesize{(0.98)} & \cellcolor{gray!15}\underline{20.15} \footnotesize{(0.75)} \\
   & THP & \underline{3.44} \footnotesize{(0.54)} & 4.96 \footnotesize{(0.46)} & 7.64 \footnotesize{(0.89)} & 17.62 \footnotesize{(1.02)} & 4.09 \footnotesize{(0.90)} & 36.72 \footnotesize{(0.64)} \\
   & \cellcolor{gray!15}S2P2 & \cellcolor{gray!15}4.36 \footnotesize{(0.60)} & \cellcolor{gray!15}2.92 \footnotesize{(0.41)} & \cellcolor{gray!15}\underline{4.48} \footnotesize{(0.68)} & \cellcolor{gray!15}\textbf{11.79} \footnotesize{(0.77)} & \cellcolor{gray!15}3.70 \footnotesize{(0.92)} & \multicolumn{1}{c}{\cellcolor{gray!15}--} \\
  \cmidrule(lr){2-5}\cmidrule(lr){6-8}
  \multirow{6}{*}{\rotatebox[origin=c]{90}{\shortstack[c]{ACD {\scriptsize $(\times 10^{\scalebox{1}[1]{-}2})$}}}}
   & \sigTPP & \underline{1.35} \footnotesize{(0.46)} & \textbf{0.70} \footnotesize{(0.20)} & \textbf{1.70} \footnotesize{(0.45)} & \underline{3.95} \footnotesize{(0.68)} & \textbf{2.45} \footnotesize{(0.56)} & 0.94 \footnotesize{(0.20)} \\
   & \cellcolor{gray!15}NHP & \cellcolor{gray!15}1.49 \footnotesize{(0.42)} & \cellcolor{gray!15}1.71 \footnotesize{(0.26)} & \cellcolor{gray!15}2.04 \footnotesize{(0.42)} & \cellcolor{gray!15}\textbf{1.78} \footnotesize{(0.63)} & \cellcolor{gray!15}\underline{2.46} \footnotesize{(0.58)} & \cellcolor{gray!15}\textbf{0.83} \footnotesize{(0.25)} \\
   & RMTPP & \textbf{1.30} \footnotesize{(0.41)} & 2.38 \footnotesize{(0.27)} & 2.67 \footnotesize{(0.41)} & 5.14 \footnotesize{(0.76)} & 2.62 \footnotesize{(0.58)} & 1.53 \footnotesize{(0.36)} \\
   & \cellcolor{gray!15}SAHP & \cellcolor{gray!15}1.70 \footnotesize{(0.43)} & \cellcolor{gray!15}1.07 \footnotesize{(0.23)} & \cellcolor{gray!15}\underline{1.83} \footnotesize{(0.44)} & \cellcolor{gray!15}4.01 \footnotesize{(0.76)} & \cellcolor{gray!15}3.13 \footnotesize{(0.60)} & \cellcolor{gray!15}\underline{0.86} \footnotesize{(0.24)} \\
   & THP & 1.69 \footnotesize{(0.47)} & \underline{0.91} \footnotesize{(0.22)} & 2.38 \footnotesize{(0.40)} & 9.25 \footnotesize{(0.76)} & 5.96 \footnotesize{(0.62)} & 1.20 \footnotesize{(0.28)} \\
   & \cellcolor{gray!15}S2P2 & \cellcolor{gray!15}1.78 \footnotesize{(0.49)} & \cellcolor{gray!15}1.80 \footnotesize{(0.28)} & \cellcolor{gray!15}2.16 \footnotesize{(0.57)} & \cellcolor{gray!15}6.20 \footnotesize{(0.76)} & \cellcolor{gray!15}2.53 \footnotesize{(0.59)} & \multicolumn{1}{c}{\cellcolor{gray!15}--} \\
  \cmidrule(lr){2-5}\cmidrule(lr){6-8}
  \multirow{6}{*}{\rotatebox[origin=c]{90}{\shortstack[c]{PCD {\scriptsize $(\times 10^{\scalebox{1}[1]{-}2})$}}}}
   & \sigTPP & \textbf{7.36} \footnotesize{(0.53)} & \textbf{6.01} \footnotesize{(0.42)} & 9.07 \footnotesize{(0.85)} & 7.51 \footnotesize{(0.53)} & \underline{12.35} \footnotesize{(0.73)} & 5.94 \footnotesize{(0.26)} \\
   & \cellcolor{gray!15}NHP & \cellcolor{gray!15}\underline{7.39} \footnotesize{(0.52)} & \cellcolor{gray!15}6.04 \footnotesize{(0.44)} & \cellcolor{gray!15}\textbf{8.19} \footnotesize{(0.92)} & \cellcolor{gray!15}\textbf{5.06} \footnotesize{(0.41)} & \cellcolor{gray!15}12.78 \footnotesize{(0.79)} & \cellcolor{gray!15}\textbf{5.82} \footnotesize{(0.26)} \\
   & RMTPP & 7.42 \footnotesize{(0.52)} & 6.06 \footnotesize{(0.43)} & \underline{8.25} \footnotesize{(0.93)} & 6.39 \footnotesize{(0.52)} & \textbf{12.20} \footnotesize{(0.74)} & \underline{5.82} \footnotesize{(0.26)} \\
   & \cellcolor{gray!15}SAHP & \cellcolor{gray!15}7.42 \footnotesize{(0.51)} & \cellcolor{gray!15}6.10 \footnotesize{(0.43)} & \cellcolor{gray!15}8.47 \footnotesize{(0.88)} & \cellcolor{gray!15}\underline{5.88} \footnotesize{(0.52)} & \cellcolor{gray!15}12.60 \footnotesize{(0.76)} & \cellcolor{gray!15}5.96 \footnotesize{(0.24)} \\
   & THP & 7.40 \footnotesize{(0.52)} & 6.05 \footnotesize{(0.43)} & 8.35 \footnotesize{(0.92)} & 10.11 \footnotesize{(0.76)} & 12.45 \footnotesize{(0.73)} & 6.00 \footnotesize{(0.24)} \\
   & \cellcolor{gray!15}S2P2 & \cellcolor{gray!15}7.45 \footnotesize{(0.51)} & \cellcolor{gray!15}\underline{6.04} \footnotesize{(0.43)} & \cellcolor{gray!15}8.47 \footnotesize{(0.87)} & \cellcolor{gray!15}9.65 \footnotesize{(0.71)} & \cellcolor{gray!15}12.70 \footnotesize{(0.79)} & \multicolumn{1}{c}{\cellcolor{gray!15}--} \\
  \bottomrule
\end{tabular}
\end{adjustbox}
\end{table}

\begin{table}[t]
\renewcommand{\arraystretch}{1.05}%
\centering
\rankcaption{Ranks for \sigTPP~and $5$ EasyTPP baselines by metric, on PS, IP, H1, EQ, YLA, and TB. S2P2 uses $5$ datasets ($40$ pairs), while the other models use $6$ ($48$ pairs).}
\label{table::easytpp_avg_rank_by_metric}
\smallskip
\setlength{\tabcolsep}{4pt}
\setlength{\aboverulesep}{0.5pt}
\setlength{\belowrulesep}{0.5pt}
\begin{tabular}[t]{l*{6}{r}}
  \toprule
  Metric & \multicolumn{1}{c}{\sigTPP} & \multicolumn{1}{c}{NHP} & \multicolumn{1}{c}{RMTPP} & \multicolumn{1}{c}{SAHP} & \multicolumn{1}{c}{THP} & \multicolumn{1}{c}{S2P2} \\
  \midrule
  $\mathcal{E}$ & \textbf{1.83} \footnotesize{(1.60)} & \underline{3.17} \footnotesize{(1.83)} & 3.50 \footnotesize{(1.38)} & 3.83 \footnotesize{(1.47)} & 4.00 \footnotesize{(2.00)} & 4.40 \footnotesize{(1.14)} \\
  \rowcolor{gray!10}
  $\WAS$ & \textbf{1.33} \footnotesize{(0.52)} & \underline{3.33} \footnotesize{(1.51)} & 3.83 \footnotesize{(0.98)} & 3.67 \footnotesize{(1.21)} & 3.50 \footnotesize{(2.17)} & 5.20 \footnotesize{(1.30)} \\
  $\sigW$ & \textbf{2.00} \footnotesize{(1.10)} & 3.17 \footnotesize{(1.33)} & 5.00 \footnotesize{(0.89)} & 2.83 \footnotesize{(1.94)} & 4.83 \footnotesize{(1.17)} & \underline{2.60} \footnotesize{(1.52)} \\
  \rowcolor{gray!10}
  CRPS & 5.33 \footnotesize{(1.63)} & \textbf{2.17} \footnotesize{(1.60)} & 3.50 \footnotesize{(1.22)} & 3.00 \footnotesize{(0.63)} & \underline{2.83} \footnotesize{(2.04)} & 3.80 \footnotesize{(1.30)} \\
  $\histintloss$ & \textbf{2.00} \footnotesize{(1.55)} & 3.50 \footnotesize{(1.38)} & 4.67 \footnotesize{(1.03)} & \underline{3.17} \footnotesize{(1.60)} & 3.33 \footnotesize{(2.07)} & 4.00 \footnotesize{(1.87)} \\
  \rowcolor{gray!10}
  $\histITloss$ & \textbf{2.17} \footnotesize{(1.94)} & 3.50 \footnotesize{(0.55)} & 4.17 \footnotesize{(1.94)} & \underline{3.00} \footnotesize{(1.79)} & 4.67 \footnotesize{(1.37)} & \underline{3.00} \footnotesize{(1.58)} \\
  ACD & \textbf{1.67} \footnotesize{(0.82)} & \underline{2.33} \footnotesize{(1.21)} & 4.33 \footnotesize{(1.86)} & 3.33 \footnotesize{(1.37)} & 4.50 \footnotesize{(1.52)} & 4.60 \footnotesize{(1.14)} \\
  \rowcolor{gray!10}
  PCD & \underline{2.83} \footnotesize{(1.94)} & \textbf{2.33} \footnotesize{(1.97)} & 3.00 \footnotesize{(1.67)} & 4.17 \footnotesize{(1.33)} & 4.00 \footnotesize{(1.26)} & 4.40 \footnotesize{(1.52)} \\
  \midrule
  Avg. & \textbf{2.40} \footnotesize{(1.81)} & \underline{2.94} \footnotesize{(1.46)} & 4.00 \footnotesize{(1.46)} & 3.38 \footnotesize{(1.42)} & 3.96 \footnotesize{(1.74)} & 4.00 \footnotesize{(1.54)} \\
  \bottomrule
\end{tabular}
\end{table}

\begin{figure}[!b]
  \centering
  \includegraphics[width=\textwidth]{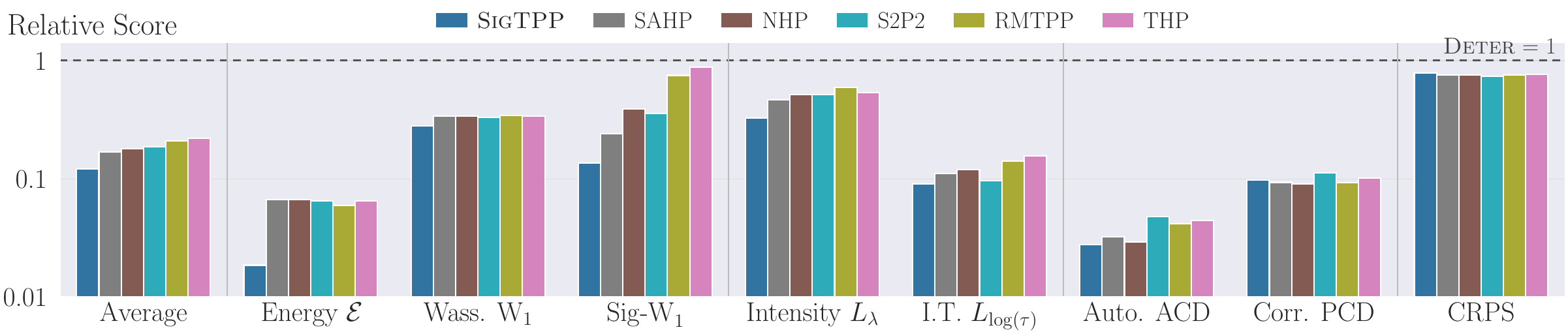}
    \ratiocaption{the $6$ datasets PS, IP, H1, EQ, YLA, and TB}{}
  \label{fig::easytpp_gmrae}
\end{figure}

\subsection{Effect of Auxiliary Conditional Regression}
\label{subsect::conditional_supervision}
Tables~\ref{table::synth_results_dist}, \ref{table::real_results_dist}, and~\ref{table::additional_datasets} show a CRPS gap between \sigTPP~and the conditional baselines. The signature objective matches complete-trajectory distributions but does not directly supervise the one-step conditional distributions assessed by CRPS. We test whether adding conditional supervision narrows this gap while retaining pathwise fidelity. Starting from \sigTPP, we keep the same generator and add a teacher-forced conditional regression term,
\begin{equation}
    \mathcal{L}_{\mathrm{hybrid}}
    = \mathcal{L}_{\sigW} + \lambda\mathcal{L}_{\mathrm{cond}} \, ,
\end{equation}
where $\lambda\geq0$ weights the auxiliary loss, $\mathcal{L}_{\sigW}$ is the signature loss in Algorithm~\ref{alg:sigtpp-training}, and $\mathcal{L}_{\mathrm{cond}}$ is one-step error on the next transformed log-interarrival given the observed history. We denote this hybrid model by \sigREG. Appendix~\ref{subsect::trainingdetails} describes the loss and the search grid.

We train \sigREG~on all datasets with $\lambda\in\{0,0.1,1\}$. Setting $\lambda=0$ recovers \sigTPP. These runs use a different search grid from the main benchmark.
Tables~\ref{table::sigreg_synth} and~\ref{table::sigreg_real} report the test scores on the $12$ datasets.

We summarise each metric by its geometric mean relative improvement over \sigTPP~($\lambda=0$) across the $12$ datasets in Tables~\ref{table::sigreg_synth} and~\ref{table::sigreg_real}. With $\lambda=1$, \sigREG~lowers CRPS on $10$ datasets. The $2$ increases, on PS and SR, are below $0.2\%$. The mean CRPS improvement is $14\%$. The largest gains are on SO ($61\%$) and YLP ($41\%$), where \sigTPP~($\lambda=0$) trails the best conditional baseline the most. These gains therefore partly reflect weak \sigTPP~runs on these datasets.
The conditional term also improves path-level scores on average. Across datasets, $\mathcal{E}$, $\WAS$, and $\sigW$ improve by $26\%$, $12\%$, and $30\%$, respectively. All $8$ metrics improve, with a relative improvement of $18\%$ across the $96$ dataset-metric pairs. These aggregates do not imply uniform gains. In particular, the relative improvement in $\sigW$ is $-14\%$ on synthetic tasks and $58\%$ on real-world tasks. On SO, energy distance and Wasserstein distance improve, but $\histintloss$ and $\sigW$ worsen. These results favour $\lambda=1$ in this ablation, but do not establish a monotone effect of the regression weight.

With the best of VAE, DDPM, and WGAN on each dataset as reference, the CRPS gap falls from $23\%$ at $\lambda=0$ to $6\%$ at $\lambda=1$. \sigREG~still trails that baseline on all $12$ datasets. Therefore, this combined loss improves one-step predictive scores while retaining or improving path-level fidelity on average, with dataset-specific trade-offs.
Further experiments are needed to understand the mechanism and to improve the signature loss.

\begin{table}[t]
\renewcommand{\arraystretch}{0.9}
\setlength{\tabcolsep}{4pt}
\centering
\resultcaption{Performance of \sigREG~for auxiliary-loss weights $\lambda\in\{0,0.1,1\}$ on PS, IP, H1, H3, SC, and SR. Setting $\lambda=0$ recovers \sigTPP~without regression.}
\label{table::sigreg_synth}
\vspace{0.4em}
\small
\setlength{\aboverulesep}{0.5pt}
\setlength{\belowrulesep}{0.5pt}
\begin{adjustbox}{width=0.8 \linewidth}
\begin{tabular}{l c *{6}{r}}
  \toprule
  Metric & $\lambda$ & \multicolumn{1}{c}{PS} & \multicolumn{1}{c}{IP} & \multicolumn{1}{c}{H1} & \multicolumn{1}{c}{H3} & \multicolumn{1}{c}{SC} & \multicolumn{1}{c}{SR} \\
  \midrule
  \multirow{3}{*}{\shortstack[l]{$\mathcal{E}$\\[-1pt]{\scriptsize $(\times 10^{\scalebox{1}[1]{-}3})$}}} & $0$ & 1.52 \footnotesize{(1.42)} & \underline{1.37} \footnotesize{(0.37)} & 8.23 \footnotesize{(2.33)} & \underline{3.64} \footnotesize{(2.04)} & 1.10 \footnotesize{(0.26)} & \underline{1.88} \footnotesize{(2.51)} \\
   & \cellcolor{gray!15}$0.1$ & \cellcolor{gray!15}\textbf{0.81} \footnotesize{(1.04)} & \cellcolor{gray!15}1.69 \footnotesize{(1.07)} & \cellcolor{gray!15}\underline{5.63} \footnotesize{(1.79)} & \cellcolor{gray!15}\textbf{3.33} \footnotesize{(1.95)} & \cellcolor{gray!15}\textbf{0.92} \footnotesize{(0.34)} & \cellcolor{gray!15}\textbf{0.75} \footnotesize{(1.50)} \\
   & $1$ & \underline{1.08} \footnotesize{(0.73)} & \textbf{1.27} \footnotesize{(0.40)} & \textbf{1.28} \footnotesize{(0.90)} & 3.85 \footnotesize{(2.37)} & \underline{0.99} \footnotesize{(0.37)} & 3.67 \footnotesize{(3.30)} \\
  \cmidrule(lr){2-8}
  \multirow{3}{*}{\shortstack[l]{$\WAS$\\[-1pt]{\scriptsize $(\times 10^{\scalebox{1}[1]{-}2})$}}} & \cellcolor{gray!15}$0$ & \cellcolor{gray!15}\underline{5.59} \footnotesize{(0.34)} & \cellcolor{gray!15}\underline{6.92} \footnotesize{(0.14)} & \cellcolor{gray!15}8.44 \footnotesize{(0.85)} & \cellcolor{gray!15}\underline{11.0} \footnotesize{(0.61)} & \cellcolor{gray!15}\underline{2.53} \footnotesize{(0.04)} & \cellcolor{gray!15}11.1 \footnotesize{(0.84)} \\
   & $0.1$ & \textbf{5.37} \footnotesize{(0.25)} & \textbf{6.73} \footnotesize{(0.23)} & \underline{7.72} \footnotesize{(0.92)} & \textbf{10.9} \footnotesize{(0.60)} & \textbf{2.52} \footnotesize{(0.04)} & \textbf{9.79} \footnotesize{(0.49)} \\
   & \cellcolor{gray!15}$1$ & \cellcolor{gray!15}5.81 \footnotesize{(0.27)} & \cellcolor{gray!15}6.97 \footnotesize{(0.13)} & \cellcolor{gray!15}\textbf{3.96} \footnotesize{(0.31)} & \cellcolor{gray!15}11.3 \footnotesize{(0.85)} & \cellcolor{gray!15}2.54 \footnotesize{(0.04)} & \cellcolor{gray!15}\underline{10.2} \footnotesize{(0.89)} \\
  \cmidrule(lr){2-8}
  \multirow{3}{*}{\shortstack[l]{$\sigW$\\[-1pt]{\scriptsize $(\times 10^{\scalebox{1}[1]{-}5})$}}} & $0$ & \underline{4.59} \footnotesize{(3.93)} & \underline{2.75} \footnotesize{(1.76)} & \underline{22.4} \footnotesize{(19.2)} & \underline{1.01} \footnotesize{(0.75)} & \underline{0.51} \footnotesize{(0.63)} & \underline{3.40} \footnotesize{(2.54)} \\
   & \cellcolor{gray!15}$0.1$ & \cellcolor{gray!15}\textbf{1.39} \footnotesize{(1.51)} & \cellcolor{gray!15}\textbf{0.51} \footnotesize{(0.56)} & \cellcolor{gray!15}\textbf{17.1} \footnotesize{(14.6)} & \cellcolor{gray!15}1.77 \footnotesize{(1.03)} & \cellcolor{gray!15}\textbf{0.48} \footnotesize{(0.68)} & \cellcolor{gray!15}4.60 \footnotesize{(2.95)} \\
   & $1$ & 4.60 \footnotesize{(4.28)} & 5.68 \footnotesize{(2.66)} & 39.7 \footnotesize{(26.3)} & \textbf{0.78} \footnotesize{(0.69)} & 0.53 \footnotesize{(0.78)} & \textbf{2.51} \footnotesize{(1.78)} \\
  \cmidrule(lr){2-8}
  \multirow{3}{*}{\shortstack[l]{CRPS\\[-1pt]{\scriptsize $(\times 10^{\scalebox{1}[1]{-}1})$}}} & \cellcolor{gray!15}$0$ & \cellcolor{gray!15}\textbf{4.65} \footnotesize{(0.08)} & \cellcolor{gray!15}3.06 \footnotesize{(0.03)} & \cellcolor{gray!15}10.3 \footnotesize{(0.24)} & \cellcolor{gray!15}2.39 \footnotesize{(0.03)} & \cellcolor{gray!15}3.93 \footnotesize{(0.05)} & \cellcolor{gray!15}\textbf{4.42} \footnotesize{(0.10)} \\
   & $0.1$ & 4.68 \footnotesize{(0.08)} & \underline{3.05} \footnotesize{(0.03)} & \underline{10.1} \footnotesize{(0.24)} & \underline{2.38} \footnotesize{(0.03)} & \underline{3.56} \footnotesize{(0.03)} & 4.43 \footnotesize{(0.10)} \\
   & \cellcolor{gray!15}$1$ & \cellcolor{gray!15}\underline{4.65} \footnotesize{(0.08)} & \cellcolor{gray!15}\textbf{3.04} \footnotesize{(0.03)} & \cellcolor{gray!15}\textbf{10.0} \footnotesize{(0.25)} & \cellcolor{gray!15}\textbf{2.34} \footnotesize{(0.03)} & \cellcolor{gray!15}\textbf{3.50} \footnotesize{(0.03)} & \cellcolor{gray!15}\underline{4.43} \footnotesize{(0.10)} \\
  \cmidrule(lr){2-8}
  \multirow{3}{*}{\shortstack[l]{$\histintloss$\\[-1pt]{\scriptsize $(\times 10^{\scalebox{1}[1]{-}2})$}}} & $0$ & 3.30 \footnotesize{(0.45)} & 3.27 \footnotesize{(0.38)} & \underline{5.44} \footnotesize{(1.14)} & \underline{3.04} \footnotesize{(0.47)} & 4.09 \footnotesize{(0.42)} & \underline{3.89} \footnotesize{(0.83)} \\
   & \cellcolor{gray!15}$0.1$ & \cellcolor{gray!15}\textbf{3.13} \footnotesize{(0.50)} & \cellcolor{gray!15}\textbf{3.00} \footnotesize{(0.35)} & \cellcolor{gray!15}\textbf{5.08} \footnotesize{(1.07)} & \cellcolor{gray!15}3.21 \footnotesize{(0.47)} & \cellcolor{gray!15}\underline{3.13} \footnotesize{(0.41)} & \cellcolor{gray!15}4.22 \footnotesize{(0.86)} \\
   & $1$ & \underline{3.18} \footnotesize{(0.48)} & \underline{3.23} \footnotesize{(0.38)} & 6.04 \footnotesize{(1.06)} & \textbf{2.71} \footnotesize{(0.44)} & \textbf{2.57} \footnotesize{(0.40)} & \textbf{3.73} \footnotesize{(0.68)} \\
  \cmidrule(lr){2-8}
  \multirow{3}{*}{\shortstack[l]{$\histITloss$\\[-1pt]{\scriptsize $(\times 10^{\scalebox{1}[1]{-}2})$}}} & \cellcolor{gray!15}$0$ & \cellcolor{gray!15}2.90 \footnotesize{(0.60)} & \cellcolor{gray!15}\underline{2.79} \footnotesize{(0.45)} & \cellcolor{gray!15}16.4 \footnotesize{(0.59)} & \cellcolor{gray!15}7.99 \footnotesize{(0.45)} & \cellcolor{gray!15}5.50 \footnotesize{(0.69)} & \cellcolor{gray!15}2.15 \footnotesize{(0.36)} \\
   & $0.1$ & \underline{2.54} \footnotesize{(0.51)} & 2.99 \footnotesize{(0.46)} & \underline{14.6} \footnotesize{(0.60)} & \underline{7.17} \footnotesize{(0.44)} & \textbf{3.40} \footnotesize{(0.79)} & \textbf{1.86} \footnotesize{(0.41)} \\
   & \cellcolor{gray!15}$1$ & \cellcolor{gray!15}\textbf{2.47} \footnotesize{(0.64)} & \cellcolor{gray!15}\textbf{2.19} \footnotesize{(0.40)} & \cellcolor{gray!15}\textbf{3.22} \footnotesize{(0.74)} & \cellcolor{gray!15}\textbf{2.11} \footnotesize{(0.54)} & \cellcolor{gray!15}\underline{4.28} \footnotesize{(1.05)} & \cellcolor{gray!15}\underline{2.00} \footnotesize{(0.36)} \\
  \cmidrule(lr){2-8}
  \multirow{3}{*}{\shortstack[l]{ACD\\[-1pt]{\scriptsize $(\times 10^{\scalebox{1}[1]{-}2})$}}} & $0$ & 1.61 \footnotesize{(0.47)} & \underline{0.94} \footnotesize{(0.22)} & 6.82 \footnotesize{(0.71)} & 4.05 \footnotesize{(0.45)} & 2.04 \footnotesize{(0.38)} & \textbf{0.66} \footnotesize{(0.20)} \\
   & \cellcolor{gray!15}$0.1$ & \cellcolor{gray!15}\underline{1.46} \footnotesize{(0.45)} & \cellcolor{gray!15}\textbf{0.73} \footnotesize{(0.21)} & \cellcolor{gray!15}\underline{6.42} \footnotesize{(0.71)} & \cellcolor{gray!15}\textbf{3.81} \footnotesize{(0.45)} & \cellcolor{gray!15}\textbf{1.55} \footnotesize{(0.37)} & \cellcolor{gray!15}\underline{0.66} \footnotesize{(0.19)} \\
   & $1$ & \textbf{1.35} \footnotesize{(0.33)} & 1.05 \footnotesize{(0.24)} & \textbf{1.52} \footnotesize{(0.39)} & \underline{4.02} \footnotesize{(0.45)} & \underline{1.56} \footnotesize{(0.37)} & 0.75 \footnotesize{(0.20)} \\
  \cmidrule(lr){2-8}
  \multirow{3}{*}{\shortstack[l]{PCD\\[-1pt]{\scriptsize $(\times 10^{\scalebox{1}[1]{-}2})$}}} & \cellcolor{gray!15}$0$ & \cellcolor{gray!15}\underline{7.34} \footnotesize{(0.51)} & \cellcolor{gray!15}\underline{6.08} \footnotesize{(0.43)} & \cellcolor{gray!15}10.3 \footnotesize{(1.00)} & \cellcolor{gray!15}8.58 \footnotesize{(0.41)} & \cellcolor{gray!15}5.87 \footnotesize{(0.30)} & \cellcolor{gray!15}7.85 \footnotesize{(0.45)} \\
   & $0.1$ & \textbf{7.33} \footnotesize{(0.52)} & \textbf{6.03} \footnotesize{(0.43)} & \underline{9.97} \footnotesize{(0.96)} & \underline{8.55} \footnotesize{(0.41)} & \underline{5.62} \footnotesize{(0.30)} & \underline{7.78} \footnotesize{(0.44)} \\
   & \cellcolor{gray!15}$1$ & \cellcolor{gray!15}7.38 \footnotesize{(0.52)} & \cellcolor{gray!15}6.09 \footnotesize{(0.44)} & \cellcolor{gray!15}\textbf{9.14} \footnotesize{(0.86)} & \cellcolor{gray!15}\textbf{8.47} \footnotesize{(0.41)} & \cellcolor{gray!15}\textbf{5.52} \footnotesize{(0.30)} & \cellcolor{gray!15}\textbf{7.76} \footnotesize{(0.44)} \\
  \bottomrule
\end{tabular}
\end{adjustbox}
\end{table}

\begin{table}[t]
\renewcommand{\arraystretch}{1.05}
\setlength{\tabcolsep}{4pt}
\centering
\resultcaption{Performance of \sigREG~for auxiliary-loss weights $\lambda\in\{0,0.1,1\}$ on EQ, SO, TB, TX, YLP, and YLA. Setting $\lambda=0$ recovers \sigTPP~without regression.}
\label{table::sigreg_real}
\vspace{0.4em}
\small
\setlength{\aboverulesep}{0.5pt}
\setlength{\belowrulesep}{0.5pt}
\begin{adjustbox}{width=0.8 \linewidth}
\begin{tabular}{l c *{6}{r}}
  \toprule
  Metric & $\lambda$ & \multicolumn{1}{c}{EQ} & \multicolumn{1}{c}{SO} & \multicolumn{1}{c}{TB} & \multicolumn{1}{c}{TX} & \multicolumn{1}{c}{YLP} & \multicolumn{1}{c}{YLA} \\
  \midrule
  \multirow{3}{*}{\shortstack[l]{$\mathcal{E}$\\[-1pt]{\scriptsize $(\times 10^{\scalebox{1}[1]{-}3})$}}} & $0$ & \textbf{31.3} \footnotesize{(0.58)} & 57.8 \footnotesize{(10.3)} & \textbf{512} \footnotesize{(23.6)} & \textbf{77.3} \footnotesize{(2.94)} & \underline{54.3} \footnotesize{(10.4)} & \underline{8.87} \footnotesize{(5.68)} \\
   & \cellcolor{gray!15}$0.1$ & \cellcolor{gray!15}33.4 \footnotesize{(0.63)} & \cellcolor{gray!15}\underline{14.4} \footnotesize{(4.78)} & \cellcolor{gray!15}534 \footnotesize{(25.6)} & \cellcolor{gray!15}79.9 \footnotesize{(2.98)} & \cellcolor{gray!15}89.9 \footnotesize{(19.2)} & \cellcolor{gray!15}14.8 \footnotesize{(8.28)} \\
   & $1$ & \underline{31.9} \footnotesize{(0.61)} & \textbf{14.1} \footnotesize{(3.25)} & \underline{529} \footnotesize{(26.0)} & \underline{79.1} \footnotesize{(2.61)} & \textbf{29.1} \footnotesize{(13.3)} & \textbf{8.79} \footnotesize{(5.22)} \\
  \cmidrule(lr){2-8}
  \multirow{3}{*}{\shortstack[l]{$\WAS$\\[-1pt]{\scriptsize $(\times 10^{\scalebox{1}[1]{-}2})$}}} & \cellcolor{gray!15}$0$ & \cellcolor{gray!15}\textbf{2.95} \footnotesize{(0.05)} & \cellcolor{gray!15}14.3 \footnotesize{(1.16)} & \cellcolor{gray!15}\textbf{57.3} \footnotesize{(2.29)} & \cellcolor{gray!15}\textbf{7.98} \footnotesize{(0.19)} & \cellcolor{gray!15}\underline{18.0} \footnotesize{(1.43)} & \cellcolor{gray!15}\textbf{8.46} \footnotesize{(0.90)} \\
   & $0.1$ & 3.00 \footnotesize{(0.05)} & \underline{8.39} \footnotesize{(0.99)} & \underline{59.1} \footnotesize{(2.45)} & 8.29 \footnotesize{(0.22)} & 20.6 \footnotesize{(1.64)} & 9.56 \footnotesize{(0.95)} \\
   & \cellcolor{gray!15}$1$ & \cellcolor{gray!15}\underline{2.95} \footnotesize{(0.05)} & \cellcolor{gray!15}\textbf{7.79} \footnotesize{(0.73)} & \cellcolor{gray!15}59.2 \footnotesize{(2.65)} & \cellcolor{gray!15}\underline{8.05} \footnotesize{(0.18)} & \cellcolor{gray!15}\textbf{14.7} \footnotesize{(1.83)} & \cellcolor{gray!15}\underline{8.71} \footnotesize{(0.69)} \\
  \cmidrule(lr){2-8}
  \multirow{3}{*}{\shortstack[l]{$\sigW$\\[-1pt]{\scriptsize $(\times 10^{\scalebox{1}[1]{-}5})$}}} & $0$ & 790 \footnotesize{(50.2)} & \textbf{0.19} \footnotesize{(0.17)} & 2.16 \footnotesize{(0.89)} & 3.13 \footnotesize{(0.96)} & 2.34 \footnotesize{(1.72)} & \textbf{5.79} \footnotesize{(4.92)} \\
   & \cellcolor{gray!15}$0.1$ & \cellcolor{gray!15}\textbf{457} \footnotesize{(38.4)} & \cellcolor{gray!15}0.51 \footnotesize{(0.29)} & \cellcolor{gray!15}\textbf{0.17} \footnotesize{(0.19)} & \cellcolor{gray!15}\underline{3.10} \footnotesize{(1.00)} & \cellcolor{gray!15}\underline{1.99} \footnotesize{(1.68)} & \cellcolor{gray!15}24.6 \footnotesize{(11.0)} \\
   & $1$ & \underline{566} \footnotesize{(43.3)} & \underline{0.22} \footnotesize{(0.18)} & \underline{0.21} \footnotesize{(0.20)} & \textbf{0.56} \footnotesize{(0.38)} & \textbf{0.69} \footnotesize{(0.85)} & \underline{8.20} \footnotesize{(6.05)} \\
  \cmidrule(lr){2-8}
  \multirow{3}{*}{\shortstack[l]{CRPS\\[-1pt]{\scriptsize $(\times 10^{\scalebox{1}[1]{-}1})$}}} & \cellcolor{gray!15}$0$ & \cellcolor{gray!15}5.85 \footnotesize{(0.11)} & \cellcolor{gray!15}13.3 \footnotesize{(0.54)} & \cellcolor{gray!15}0.37 \footnotesize{(0.01)} & \cellcolor{gray!15}\underline{1.14} \footnotesize{(0.02)} & \cellcolor{gray!15}4.80 \footnotesize{(0.43)} & \cellcolor{gray!15}4.82 \footnotesize{(0.13)} \\
   & $0.1$ & \underline{5.58} \footnotesize{(0.10)} & \underline{5.28} \footnotesize{(0.09)} & \underline{0.37} \footnotesize{(0.01)} & 1.15 \footnotesize{(0.02)} & \underline{3.27} \footnotesize{(0.12)} & \underline{4.79} \footnotesize{(0.14)} \\
   & \cellcolor{gray!15}$1$ & \cellcolor{gray!15}\textbf{5.39} \footnotesize{(0.11)} & \cellcolor{gray!15}\textbf{5.13} \footnotesize{(0.09)} & \cellcolor{gray!15}\textbf{0.36} \footnotesize{(0.01)} & \cellcolor{gray!15}\textbf{1.11} \footnotesize{(0.02)} & \cellcolor{gray!15}\textbf{2.83} \footnotesize{(0.12)} & \cellcolor{gray!15}\textbf{4.76} \footnotesize{(0.14)} \\
  \cmidrule(lr){2-8}
  \multirow{3}{*}{\shortstack[l]{$\histintloss$\\[-1pt]{\scriptsize $(\times 10^{\scalebox{1}[1]{-}2})$}}} & $0$ & \underline{22.7} \footnotesize{(0.89)} & \textbf{5.98} \footnotesize{(0.41)} & 9.94 \footnotesize{(1.15)} & \underline{6.32} \footnotesize{(1.03)} & \underline{14.0} \footnotesize{(1.15)} & \textbf{13.7} \footnotesize{(1.20)} \\
   & \cellcolor{gray!15}$0.1$ & \cellcolor{gray!15}23.9 \footnotesize{(1.14)} & \cellcolor{gray!15}\underline{7.16} \footnotesize{(0.64)} & \cellcolor{gray!15}\underline{7.94} \footnotesize{(0.96)} & \cellcolor{gray!15}8.15 \footnotesize{(0.96)} & \cellcolor{gray!15}16.2 \footnotesize{(0.91)} & \cellcolor{gray!15}16.7 \footnotesize{(1.17)} \\
   & $1$ & \textbf{20.4} \footnotesize{(0.99)} & 7.42 \footnotesize{(0.53)} & \textbf{4.29} \footnotesize{(0.83)} & \textbf{4.10} \footnotesize{(0.85)} & \textbf{9.40} \footnotesize{(0.94)} & \underline{14.6} \footnotesize{(1.22)} \\
  \cmidrule(lr){2-8}
  \multirow{3}{*}{\shortstack[l]{$\histITloss$\\[-1pt]{\scriptsize $(\times 10^{\scalebox{1}[1]{-}2})$}}} & \cellcolor{gray!15}$0$ & \cellcolor{gray!15}16.7 \footnotesize{(0.84)} & \cellcolor{gray!15}8.17 \footnotesize{(0.61)} & \cellcolor{gray!15}\underline{13.4} \footnotesize{(0.68)} & \cellcolor{gray!15}\underline{4.18} \footnotesize{(0.42)} & \cellcolor{gray!15}\textbf{8.54} \footnotesize{(1.73)} & \cellcolor{gray!15}\underline{7.96} \footnotesize{(1.09)} \\
   & $0.1$ & \underline{14.4} \footnotesize{(0.86)} & \textbf{6.62} \footnotesize{(0.81)} & \textbf{10.3} \footnotesize{(0.44)} & \textbf{4.00} \footnotesize{(0.47)} & \underline{18.4} \footnotesize{(1.41)} & 9.68 \footnotesize{(1.20)} \\
   & \cellcolor{gray!15}$1$ & \cellcolor{gray!15}\textbf{11.3} \footnotesize{(0.90)} & \cellcolor{gray!15}\underline{7.13} \footnotesize{(0.86)} & \cellcolor{gray!15}13.8 \footnotesize{(0.48)} & \cellcolor{gray!15}7.27 \footnotesize{(0.58)} & \cellcolor{gray!15}18.9 \footnotesize{(0.70)} & \cellcolor{gray!15}\textbf{7.28} \footnotesize{(1.34)} \\
  \cmidrule(lr){2-8}
  \multirow{3}{*}{\shortstack[l]{ACD\\[-1pt]{\scriptsize $(\times 10^{\scalebox{1}[1]{-}2})$}}} & $0$ & \textbf{5.42} \footnotesize{(0.76)} & 13.0 \footnotesize{(0.51)} & 2.13 \footnotesize{(0.39)} & 3.72 \footnotesize{(0.47)} & \textbf{1.63} \footnotesize{(0.33)} & \textbf{2.23} \footnotesize{(0.58)} \\
   & \cellcolor{gray!15}$0.1$ & \cellcolor{gray!15}\underline{6.27} \footnotesize{(0.76)} & \cellcolor{gray!15}\textbf{6.47} \footnotesize{(0.51)} & \cellcolor{gray!15}\underline{1.10} \footnotesize{(0.27)} & \cellcolor{gray!15}\underline{2.73} \footnotesize{(0.41)} & \cellcolor{gray!15}3.80 \footnotesize{(0.62)} & \cellcolor{gray!15}2.83 \footnotesize{(0.62)} \\
   & $1$ & 7.86 \footnotesize{(0.76)} & \underline{8.77} \footnotesize{(0.51)} & \textbf{0.97} \footnotesize{(0.25)} & \textbf{1.48} \footnotesize{(0.36)} & \underline{2.95} \footnotesize{(0.60)} & \underline{2.37} \footnotesize{(0.56)} \\
  \cmidrule(lr){2-8}
  \multirow{3}{*}{\shortstack[l]{PCD\\[-1pt]{\scriptsize $(\times 10^{\scalebox{1}[1]{-}2})$}}} & \cellcolor{gray!15}$0$ & \cellcolor{gray!15}\underline{8.42} \footnotesize{(0.62)} & \cellcolor{gray!15}\underline{12.6} \footnotesize{(0.28)} & \cellcolor{gray!15}5.85 \footnotesize{(0.27)} & \cellcolor{gray!15}6.25 \footnotesize{(0.23)} & \cellcolor{gray!15}13.7 \footnotesize{(0.70)} & \cellcolor{gray!15}12.4 \footnotesize{(0.72)} \\
   & $0.1$ & \textbf{7.84} \footnotesize{(0.64)} & 13.8 \footnotesize{(0.45)} & \textbf{5.84} \footnotesize{(0.26)} & \underline{6.16} \footnotesize{(0.24)} & \underline{13.2} \footnotesize{(0.69)} & \underline{12.4} \footnotesize{(0.73)} \\
   & \cellcolor{gray!15}$1$ & \cellcolor{gray!15}10.1 \footnotesize{(0.70)} & \cellcolor{gray!15}\textbf{10.9} \footnotesize{(0.32)} & \cellcolor{gray!15}\underline{5.85} \footnotesize{(0.25)} & \cellcolor{gray!15}\textbf{5.80} \footnotesize{(0.23)} & \cellcolor{gray!15}\textbf{12.6} \footnotesize{(0.69)} & \cellcolor{gray!15}\textbf{12.3} \footnotesize{(0.73)} \\
  \bottomrule
\end{tabular}
\end{adjustbox}
\end{table}

\clearpage
\section{Information about Datasets}
\label{sect::datasets}
\begin{table}[t]
\centering
\caption{Dataset statistics. $K$ denotes the number of event types (marks).}
\label{table::datasets_stats}
\footnotesize
\setlength{\tabcolsep}{4pt}
\begin{tabular}{l r r r r r r r r r r}
\toprule
Dataset & $K$
  & \multicolumn{3}{c}{\# Events}
  & \multicolumn{3}{c}{\# Sequences}
  & \multicolumn{3}{c}{Seq.\ Length} \\
\cmidrule(lr){3-5}\cmidrule(lr){6-8}\cmidrule(lr){9-11}
 &
  & Train & Val & Test
  & Train & Val & Test
  & Min & Mean & Max \\
\midrule
PS                      & 3  & 14293  & 4796  & 4837  & 1200 & 400  & 400  & 2  & 12.0  & 24  \\
\rowcolor{gray!10}%
IP                     & 3  & 45232  & 14916 & 14944 & 3000 & 1000 & 1000 & 2  & 15.0  & 31  \\
H1                      & 1  & 11720  & 3638  & 3833  & 1200 & 400  & 400  & 2  & 9.6   & 36  \\
\rowcolor{gray!10}%
H3                      & 3  & 39357  & 12915 & 13061 & 1200 & 400  & 400  & 11 & 32.7  & 70  \\
SC                      & 1  & 24369  & 8138  & 8100  & 1200 & 400  & 400  & 18 & 20.3  & 23  \\
\rowcolor{gray!10}%
SR                      & 1  & 31368  & 10982 & 11081 & 1200 & 400  & 400  & 1  & 26.7  & 75  \\
\midrule
\rowcolor{gray!10}%
EQ                      & 7  & 46358  & 6212  & 13845 & 3000 & 400  & 900  & 1  & 15.4  & 17  \\
TB                      & 17 & 72183  & 11272 & 27955 & 1300 & 200  & 500  & 27 & 55.7  & 63  \\
\rowcolor{gray!10}%
TX                      & 10 & 50454  & 7204  & 14420 & 1400 & 200  & 400  & 35 & 36.0  & 37  \\
SO                      & 22 & 88766  & 25266 & 26015 & 1401 & 401  & 401  & 38 & 63.6  & 100 \\
\rowcolor{gray!10}%
YLA                     & 1  & 5803   & 1991  & 1923  & 191  & 63   & 65   & 9  & 30.5  & 55  \\
YLP                     & 1  & 10542  & 3550  & 3531  & 191  & 63   & 65   & 3  & 55.2  & 107 \\
\bottomrule
\end{tabular}
\end{table}

The benchmark contains $6$ synthetic and $6$ real-world datasets. Together they span memoryless, time-varying, self-exciting, renewal, and self-correcting generative regimes. All synthetic datasets are split $60\%$/$20\%$/$20\%$ into training, validation, and test sets. Real-world datasets use the splits described below.

\DatasetDiagnosticsFigure{Poisson}{PS}{hp_three_marks}
\DatasetDiagnosticsFigure{inhomogeneous Poisson}{IP}{ihp_three_marks}
\DatasetDiagnosticsFigure{Hawkes}{H1}{hawkesdatamodule}
\DatasetDiagnosticsFigure{$3$-dimensional Hawkes}{H3}{hawkes3x3datamodule}
\DatasetDiagnosticsFigure{self-correcting}{SC}{self_correcting}
\DatasetDiagnosticsFigure{stationary renewal}{SR}{stationary_renewal}
\DatasetDiagnosticsFigure{earthquake}{EQ}{earthquake}
\DatasetDiagnosticsFigure{Stack Overflow}{SO}{stackoverflow}
\DatasetDiagnosticsFigure{Taobao}{TB}{taobao}
\DatasetDiagnosticsFigure{Taxi}{TX}{taxi}
\DatasetDiagnosticsFigure{Yelp-Airport}{YLA}{yelp_airport}
\DatasetDiagnosticsFigure{Yelp}{YLP}{yelp_mississauga}

\textbf{Homogeneous Poisson (\textbf{PS}).} We generate $2{,}000$ event sequences from a homogeneous Poisson process on $[0,\Tmax]$ with horizon $\Tmax=12$ and constant rate $\lambda=1$. Equivalently, interarrival times are sampled i.i.d.\ from $\mathrm{Exp}(1)$.

\textbf{Inhomogeneous Poisson (\textbf{IP}).} We generate $5{,}000$ sequences from an inhomogeneous Poisson process on $[0,\Tmax]$ with $\Tmax=10$ and piecewise-constant intensity
\begin{equation}
\lambda(t) = 1 + \11charac_{t \geq \Tmax/2} \,.
\end{equation}
Thus, the rate is $1$ on the first half of the window and $2$ on the second half. Sequences are simulated using the Lewis-Ogata thinning algorithm with upper bound $\bar{\lambda}=2$.

\textbf{Univariate Hawkes (\textbf{H1}).}
We generate $2{,}000$ sequences from a univariate Hawkes process with exponential kernel,
\begin{equation}
\lambda(t \mid \mathcal{H}_t)
= \mu + \alpha \sum_{t_i < t} \beta e^{-\beta(t-t_i)} \,,
\end{equation}
using $\mu=0.3$, $\alpha=0.4$, $\beta=1.0$, and horizon $\Tmax=20$. Simulation is performed with the \texttt{tick} library. Since $\alpha<1$, the process is stationary. This setting introduces temporal clustering through self-excitation and tests whether a model can capture history-dependent intensities.

\textbf{Multivariate Hawkes (\textbf{H3}).}
We generate $2{,}000$ sequences from a $3$-dimensional Hawkes process with exponential kernels using \texttt{tick}. The baseline intensities are $\mu_k=0.5$ for all $k$, the decay parameters are $\beta_{k\ell}=1$ for all pairs $(k,\ell)$, and the excitation matrix is
\begin{equation}
A =
\begin{pmatrix}
0.5 & 0.1 & 0.0 \\
0.1 & 0.0 & 0.0 \\
0.0 & 0.0 & 0.1
\end{pmatrix} \,,
\end{equation}
with horizon $\Tmax=15$. The 3 component streams are merged into a single time-ordered event sequence, and the source dimension is used as a categorical mark ($K=3$). This dataset evaluates a model's ability to capture both cross-excitation across dimensions and mark-dependent intensity structure.

\textbf{Self-correcting (\textbf{SC}).}
We generate $2{,}000$ sequences on $[0,\Tmax)$ with $\Tmax=20$ from a self-correcting process with conditional intensity
\begin{equation}
\lambda(t \mid \mathcal{H}_t) = \exp\bigl(\mu t - \alpha N_t\bigr) \,,
\end{equation}
using $\mu=1$ and $\alpha=1$. The intensity grows between events, and each event scales it by $e^{-\alpha}$. Sequences are therefore more regular than under a Poisson process: lengths range from 18 to 23 events, and consecutive interarrival times are negatively correlated.

\textbf{Stationary renewal (\textbf{SR}).}
We generate $2{,}000$ sequences from a renewal process on $[0,\Tmax)$ with $\Tmax=20$. Interarrival times $\tau_k$ are i.i.d.\ lognormal with mean $1$ and standard deviation $6$, that is $\log \tau_k$ is Gaussian with mean $-s^2/2$ and variance $s^2=\log 37$. The heavy tail produces bursts of short gaps separated by long silences, and sequence lengths range from 1 to 75 events. Since each sequence starts with a fresh renewal at $t=0$, the empirical intensity is highest near $t=0$ before it settles. Interarrival times are independent by construction, so their autocorrelation vanishes at every nonzero lag.

SR and SC follow the synthetic processes of~\citet{Omi_paper}. We simulate them with the samplers of~\citet{shchur2020fast}, which also generated the SR and SC data of the benchmark of~\citet{addthin}. 

\textbf{Taxi (TX).}
Time-stamped taxi pick-up and drop-off events in New York City~\cite{xue2024easytpp}, partitioned by borough and action type into $K=10$ event categories ($5$ boroughs $\times$ $2$ actions).
The subset comprises $2{,}000$ randomly sampled drivers with a mean sequence length of $36$ events.
The standard EasyTPP train/validation/test split is used.

\textbf{Stack Overflow.}
A dataset of user-interaction events on Stack Overflow, where event types correspond to distinct awarded-badge categories~\cite{xue2024easytpp}.
The standard EasyTPP train/validation/test split is used.

\textbf{Taobao.}
User click sequences on the Taobao shopping platform, collected between 25 November and 3 December 2017~\cite{xue2024easytpp}.
Item categories are ranked by frequency. The most frequent are retained and all remaining categories are merged into a single \emph{other} class, yielding $K=17$ event types.
The subset comprises the $2{,}000$ most active users with a mean sequence length of $56$ events.
The standard EasyTPP train/validation/test split is used.

\textbf{Earthquake.}
Time-stamped seismic events recorded across the Conterminous United States from 1996 to 2023, derived from the USGS catalogue~\cite{xue2024easytpp}.
The standard EasyTPP train/validation/test split is used.

\textbf{Yelp-Mississauga (\textbf{YLP}).}
This real-world dataset, taken from the benchmark of~\citet{addthin}, contains user check-in event sequences for businesses in the city of Mississauga during 2018. As in that benchmark, it is unmarked ($K=1$) and each sequence contains only ordered check-in times. 

\textbf{Yelp-Airport (\textbf{YLA}).}
This dataset comes from the same benchmark and Yelp source as YLP. It contains check-in times at McCarran International Airport during 2018, with $319$ sequences on a $24$-hour window. It is unmarked ($K=1$) and split as YLP.

\section{Dataset Training Details}
\label{subsect::trainingdetails}
\textbf{Categorical event types (marks).} For datasets with $K>1$ event types, event $i$ carries a categorical mark $m_i\in\{1,\ldots,K\}$, where Table~\ref{table::datasets_stats} reports $K$. We concatenate a learnable embedding of $m_i$ with the log-interarrival encoding before passing the result to the LSTM (with hidden state $h_i \in \R^H, H \in \mathbb{N}$), so $h_i$ encodes both times and marks. 
A categorical head on $h_i$ predicts the next mark and is trained by cross entropy against the observed $m_{i+1}$ using observed histories (teacher forcing). During free-running generation, the head samples the next mark, whose embedding enters the next recurrent update. For unmarked datasets ($K=1$), we omit the mark embedding, head, and loss. \\
\textbf{Optimisation}. All experiments were optimised using the Adam optimiser~\cite{2015-kingma}, with learning rates tailored to each specific task. The search implementation selects checkpoints by the mean of the position-wise validation interarrival and event-time histogram losses (\texttt{val\_epdf}). It ranks configurations by the sum of minimum ranks on $8$ validation diagnostics: $\sigW$, the $2$ position-wise histogram losses, $\mathcal{E}$, $\WAS$, ACD, PCD, and CRPS. Pathwise diagnostics use free-running generation, and CRPS uses predictions conditional on observed histories. \\
\textbf{Source Code Availability}. The complete source code to reproduce the experiments of the paper is available at \href{https://github.com/Code-Cornelius/sigtpp}{\texttt{github.com/Code-Cornelius/sigtpp}}. \\
\textbf{Software Environment}. All numerical experiments were conducted on a system running Ubuntu 24.04.4 LTS with Python 3.8 and PyTorch 1.12.1+cu116. CUDA was available through PyTorch, with cuDNN version 8302. \\
\textbf{Hardware Infrastructure.} The machine is equipped with $3$ NVIDIA L40S GPUs, each with $48$ GB of VRAM, using NVIDIA driver version 580.126.09. The system is powered by dual AMD EPYC 9354 $32$-core processors, providing $128$ CPU threads in total, with a maximum CPU frequency of $3.80$ GHz. It has $768$ GB of RAM and NVMe SSD storage.

\begin{table}[t]
\centering
\caption{Hyperparameter search grid for all models. $^\dagger$Fixed value, not searched.}
\label{table::hyperparam_grid}
\begin{tabular}{llc}
\toprule
Model & Hyperparameter & Values \\
\midrule
\multirow{4}{*}{WGAN}
  & Generator LR $\eta_G$     & $10^{-3},\ 10^{-4},\ 10^{-5}$ \\
  & Discriminator LR $\eta_D$ & $10^{-2},\ 10^{-3},\ 10^{-4}$ \\
  & Hidden size               & $16,\ 32$ \\
  & Lipschitz reg.\ $\lambda$ & $10^{-2},\ 10^{-3},\ 10^{-4}$ \\
\midrule
\multirow{2}{*}{\determodel}
  & Learning rate $\eta$ & $10^{-2},\ 10^{-4}$ \\
  & Hidden size          & $8,\ 16,\ 32$ \\
\midrule
\gammamodel
  & Learning rate $\eta$ & $1,\ 10^{-2},\ 10^{-4}$ \\
\midrule
\multirow{6}{*}{VAE}
  & Learning rate $\eta$         & $10^{-3},\ 10^{-4},\ 10^{-5}$ \\
  & Hidden size                   & $16,\ 32$ \\
  & Latent dimension              & $8,\ 16$ \\
  & Reconstruction weight         & $0.01,\ 0.1,\ 1,\ 10$ \\
  & KL annealing epochs$^\dagger$ & $50$ \\
  & Free bits$^\dagger$           & $0.1$ \\
\midrule
\multirow{4}{*}{DDPM}
  & Learning rate $\eta$  & $10^{-3},\ 10^{-4}$ \\
  & Hidden size            & $16,\ 32$ \\
  & Diffusion steps $T$    & $10,\ 50,\ 100$ \\
  & Batch size             & $512,\ 2048,\ 4096,\ 30000$ \\
\midrule
\multirow{7}{*}{\sigTPP}
  & Generator LR $\eta_G$            & $10^{-3},\ 5\cdot10^{-4},\ 10^{-4}$ \\
  & Hidden size                       & $16,\ 32$ \\
  & Log-interarrival embedding dimension$^\dagger$ & $64$ \\
  & Signature degree$^\dagger$        & $8$ \\
  & Teacher forcing                    & True, False \\
  & Terminal anchor                   & free, residual \\
  & Detach cumulative channel         & True, False \\
  \midrule
\multirow{4}{*}{\sigREG}
  & Generator LR $\eta_G$             & $10^{-3},\ 10^{-4}$ \\
  & Regression loss                    & $L^1,\ L^2$ \\
  & Noise mode                         & existing, fixed \\
  & Conditional-loss weight $\lambda$  & $0,\ 0.1,\ 1$ \\
\bottomrule
\end{tabular}
\end{table}

\textbf{Notes on \sigTPP~hyperparameters.}
Table~\ref{table::hyperparam_grid} lists the hyperparameters of~\sigTPP. The terminal anchor determines how each path's endpoint is handled before computing its signature. The theoretical analysis uses the $\Phi$-embedded representation on a fixed compact domain, which is the setting covered by the stability results.
In the implementation, we evaluate two path representations during validation. The ``residual'' setting uses the $\Phi$-based terminal anchor, while the ``free'' setting computes signatures directly over the observed variable-length interval ending at the last cumulative time.
The ``detach cumulative channel'' flag changes only the backward pass. Detachment leaves the forward path, signature, and numerical value of $\sigW$ unchanged, but drops the derivative through the cumulative-time coordinate.

Theorem~\ref{thrm::stability} concerns the fixed-horizon embedding $\Phi$ of Definition~\ref{def::phi_embedding}. The residual option implements this embedding by appending the terminal knot $(\Tmax,\Tmax-t_m)$ in time-augmented coordinates, which records the final censored interval. The ``free'' option omits this knot and ends at the last event. The stated stability bounds do not cover this variant.

Lastly, training uses standardised degree-$8$ signature coordinates, whereas evaluation uses unscaled degree-$3$ coordinates.

\textbf{Degree ablation.} Table~\ref{table::ablation_M} uses a separate \sigTPP~grid. It crosses the signature degree $M\in\{2,3,4,5,6,8\}$ with generator learning rate $\eta_G\in\{10^{-2},10^{-3},10^{-4}\}$, hidden size $\{16,32\}$, and teacher forcing in the signature branch (True, False), giving $3\times6\times2\times2=72$ configurations per dataset. It fixes the free terminal anchor, and no cumulative-channel detachment.

\textbf{\sigREG~ablation.}
For the ablation in Appendix~\ref{subsect::conditional_supervision}, the generator predicts each next event from the observed prefix in a separate teacher-forced pass and averages absolute ($L^1$) or squared ($L^2$) errors over all valid next-event positions in the batch. With noise mode equal to existing, the error acts on individual noisy outputs. With noise mode equal to fixed, it acts on the output with decoder noise fixed at its mean. The latter is a deterministic prediction, not an average over generated outputs. All runs fix hidden size $32$, signature degree $8$, the free terminal anchor and no cumulative-channel detachment. 

\section{Limitations}
\label{sect::limitations}
A primary limitation of our approach is the computational cost of the truncated signature transform, whose feature dimension scales exponentially with the signature degree $M$ and polynomially with the path dimension $d$. This may hinder direct extension to high-dimensional or marked TPPs. 
Second, our current embedding relies on a linear interpolation lift. While effective in practice, this lift introduces an artificial continuous transition at jump times and is therefore not fully faithful to the jump structure of the process. 
Third, our theoretical analysis relies on a determinacy property of the pushforward law of the temporal point process under the embedding, but such determinacy conditions are already difficult to verify in practice even for continuous processes, and we do not investigate them further here.
Finally, our theoretical bounds are preliminary and likely conservative: the constants are loose, and a sharper characterisation of the embedding remains open. A better understanding of how the required truncation degree depends on the nature of the temporal point process, together with the relationship between moment conditions and the law of the process, is promising directions for future work.
\end{document}